\documentclass[letterpaper,10pt]{article}
\usepackage{parskip}
\usepackage[english]{babel}
\usepackage[nottoc]{tocbibind}

\usepackage[square]{natbib}
\usepackage{comment}
\usepackage{booktabs}
\usepackage{tabularx}
\usepackage{adjustbox}
\usepackage{array}
\usepackage{diagbox}
\usepackage{makecell}
\usepackage{tikz}
\usepackage{amsmath}
\usepackage{amsfonts}
\usepackage{amssymb}
\usepackage{amsthm}
\usepackage{url,ifthen}
\usepackage{enumitem}
\usepackage{srcltx}
\usepackage{multirow}
\usepackage{boxedminipage}
\usepackage[margin=1.1in]{geometry}
\usepackage{nicefrac}
\usepackage{xspace}
\usepackage{graphicx}
\usepackage{color}
\usepackage{colortbl}
\usepackage{setspace}
\usepackage{soul}

\usepackage{hyperref}
\usepackage{cleveref}
\usepackage{mathptmx}
\usepackage{algpseudocode}

\definecolor{DarkGreen}{rgb}{0.1,0.5,0.1}
\definecolor{DarkRed}{rgb}{0.5,0.1,0.1}
\definecolor{DarkBlue}{rgb}{0.1,0.1,0.5}
\definecolor{Gray}{rgb}{0.2,0.2,0.2}

\usepackage{listings}
\lstdefinestyle{mystyle}{
    commentstyle=\color{DarkBlue},
    keywordstyle=\color{DarkRed},
    numberstyle=\tiny\color{Gray},
    stringstyle=\color{DarkGreen},
    basicstyle=\footnotesize,
    breakatwhitespace=false,         
    breaklines=true,                 
    captionpos=b,                    
    keepspaces=true,                 
    numbers=left,                    
    numbersep=5pt,                  
    showspaces=false,                
    showstringspaces=false,
    showtabs=false,                  
    tabsize=2
}
\usepackage[small]{caption}
\usepackage{hyperref}
\hypersetup{
    unicode=false,          
    pdftoolbar=true,        
    pdfmenubar=true,        
    pdffitwindow=false,      
    pdfnewwindow=true,      
    colorlinks=true,       
    linkcolor=DarkGreen,          
    citecolor=DarkGreen,        
    filecolor=DarkRed,      
    urlcolor=DarkBlue,          
    pdftitle={},
    pdfauthor={},
}

\def\draft{1}

\def\submit{0}

\ifnum\draft=1 
    \def\ShowAuthNotes{1}
\else
    \def\ShowAuthNotes{0}
\fi

\ifnum\submit=1
\newcommand{\forsubmit}[1]{#1}
\newcommand{\forreals}[1]{}
\else
\newcommand{\forreals}[1]{#1}
\newcommand{\forsubmit}[1]{}
\fi

\ifnum\ShowAuthNotes=1
\newcommand{\authnote}[2]{{ \footnotesize \bf{\color{DarkRed}[#1's Note:
{\color{DarkBlue}#2}]}}}
\else
\newcommand{\authnote}[2]{}
\fi

\newtheorem{theorem}{Theorem}[section]
\newtheorem{remark}[theorem]{Remark}
\newtheorem{lemma}[theorem]{Lemma}
\newtheorem{corollary}[theorem]{Corollary}

\theoremstyle{definition}

\newtheoremstyle{example_contd}
{\topsep} {\topsep}%
{}
{}
{\bfseries}
{.}
{1em}
{\thmname{#1} \thmnumber{ #2}\thmnote{#3} (continued)}

\theoremstyle{example_contd}

\newcommand{\chapterref}[1]{\hyperref[ch:#1]{Chapter~\ref{ch:#1}}}

\newcommand{\claimref}[1]{\hyperref[claim:#1]{Claim~\ref{claim:#1}}}

\newcommand{\corollaryref}[1]{\hyperref[cor:#1]{Corollary~\ref{cor:#1}}}

\newcommand{\definitionref}[1]{\hyperref[def:#1]{Definition~\ref{def:#1}}}

\newcommand{\equationref}[1]{\hyperref[eq:#1]{Equation~\ref{eq:#1}}}

\newcommand{\factref}[1]{\hyperref[fact:#1]{Fact~\ref{fact:#1}}}

\newcommand{\figureref}[1]{\hyperref[fig:#1]{Figure~\ref{fig:#1}}}

\newcommand{\tableref}[1]{\hyperref[tab:#1]{Table~\ref{tab:#1}}}

\newcommand{\itemref}[1]{\hyperref[item:#1]{Item~(\ref{item:#1})}}

\newcommand{\lemmaref}[1]{\hyperref[lem:#1]{Lemma~\ref{lem:#1}}}

\newcommand{\propref}[1]{\hyperref[prop:#1]{Proposition~\ref{prop:#1}}}

\newcommand{\propositionref}[1]{\hyperref[prop:#1]{Proposition~\ref{prop:#1}}}

\newcommand{\remarkref}[1]{\hyperref[rem:#1]{Remark~\ref{rem:#1}}}

\newcommand{\sectionref}[1]{\hyperref[sec:#1]{Section~\ref{sec:#1}}}

\newcommand{\theoremref}[1]{\hyperref[thm:#1]{Theorem~\ref{thm:#1}}}

\usepackage[utf8]{inputenc}
\usepackage[T1]{fontenc}
\usepackage{kpfonts}
\usepackage{microtype}

\newcommand{\Psymb}{\mathbb{P}}

\DeclareMathOperator*{\ProbOp}{\Psymb r}
\renewcommand{\Pr}{\ProbOp}

\renewcommand{\le}{\leqslant}

\renewcommand{\ge}{\geqslant}

\usepackage{bm}

\newcommand{\ignore}[1]{}

\newcommand{\KL}{\mathrm{KL}}

\renewcommand{\epsilon}{\varepsilon}

\newcommand{\remove}[1]{}

\usepackage{color-edits}
\addauthor{hui}{cyan}

\usepackage{bm}
\usepackage{amsfonts}
\usepackage[utf8]{inputenc} 
\usepackage[T1]{fontenc}    
\usepackage{hyperref}       
\usepackage{url}            
\usepackage{booktabs}       
\usepackage{amsfonts}       
\usepackage{nicefrac}       
\usepackage{microtype}      
\usepackage{xcolor}         

\usepackage{algorithm}
\usepackage{algpseudocode}
\usepackage{diagbox}
\usepackage{makecell}
\usepackage{booktabs}
\usepackage{tikz}
\usepackage{amsmath}
\allowdisplaybreaks
\usepackage{amssymb}
\usepackage{amsthm}
\usepackage{url,ifthen}
\usepackage{enumitem}
\usepackage{srcltx}
\usepackage{multirow}
\usepackage{boxedminipage}
\usepackage{nicefrac}
\usepackage{xspace}
\usepackage{graphicx}
\usepackage{color}
\usepackage{colortbl}
\usepackage{setspace}
\usepackage{soul}
\usepackage{graphicx}
\usepackage{wrapfig}
\usepackage{subcaption}
\usepackage{enumitem}
\usepackage{cleveref}
\definecolor{DarkGreen}{rgb}{0.1,0.5,0.1}
\definecolor{DarkRed}{rgb}{0.5,0.1,0.1}
\definecolor{DarkBlue}{rgb}{0.1,0.1,0.5}
\definecolor{Gray}{rgb}{0.2,0.2,0.2}

\title{Ask to Learn: Query-Based Collaborative Reinforcement Learning}

\usepackage{authblk}
\usepackage{placeins}
\usepackage{capt-of}
\usepackage{longtable}
\usepackage{needspace}

\author[1]{Kaicheng Zhang}
\affil[1]{University of North Carolina at Chapel Hill}
\author[1]{Hongqian Niu}
\author[2]{Laixi Shi}
\affil[2]{Johns Hopkins University}
\author[1]{Didong Li}
\author[3]{Guannan Qu}
\affil[3]{Carnegie Mellon University}
\author[1]{Zhun~Deng}
\date{}
\begin{document}

\maketitle

 \begin{abstract}
Offline reinforcement learning is fundamentally limited by the coverage of the available data, while collecting additional trajectories through online exploration may be costly, unsafe, or infeasible. We study whether this limitation can be overcome without directly interacting with the underlying environment. Motivated by practical settings such as collaborator-mediated simulation, where a collaborator can provide information complementary to the offline data, we introduce \emph{query-based collaborative reinforcement learning} to study how to acquire this information efficiently. The learner starts with a fixed offline dataset and adaptively queries state-time pairs; the collaborator samples an action according to a fixed, possibly suboptimal policy, and the simulator returns the corresponding reward and next state. For finite-horizon tabular problems, we derive an instance-dependent information-theoretic lower bound on the additional queries required to identify an $\epsilon$-optimal policy and develop a local quadratic approximation. Guided by this characterization, we develop an adaptive query-allocation algorithm whose query complexity nearly matches the local benchmark under suitable conditions. Experiments on multiple synthetic and semi-synthetic environments demonstrate the empirical efficiency of the proposed method.
\end{abstract}

\newpage
\setcounter{tocdepth}{2}
\tableofcontents
\newpage

\section{Introduction}

Collecting new trajectories for reinforcement learning (RL) can be costly,
slow, or unsafe. In healthcare, for example, it is generally infeasible to
deploy exploratory treatment policies on patients solely for data collection.
Offline RL avoids such interaction by learning from historical records, but
its guarantees depend critically on the coverage of the available data
\citep{jin2021pessimism,gen2024settling,xie2021bellman,yin2021towards}. The absence of feasible online
exploration, however, does not preclude access to additional information
through simulation. We study how a learner can acquire this information
efficiently through collaboration.

We introduce \emph{query-based collaborative reinforcement learning}, in
which a learner augments a fixed offline dataset with observations from a
collaborator-mediated simulator. At each query, the learner selects a
state-time pair; the collaborator samples an action according to a fixed,
possibly suboptimal policy, and the simulator returns the corresponding
reward and next state. The learner thus controls where information is
collected, but not which action generates it. Our central question is:

\emph{Given the available offline data, when is it possible to identify
a near-optimal policy through collaboration, and how should queries
be allocated to do so efficiently?}

This setting is motivated by systems that couple decision models with
simulators. In head-and-neck cancer treatment,
\citet{wentzel2025ditto} combine a model of physician treatment decisions
with a patient simulator that predicts treatment responses and clinical
outcomes. In our abstraction of such a system, a query specifies the
patient state and treatment stage, while the physician model selects the
treatment and the simulator supplies its consequences. Related structures
arise in power-system dispatch, where a rule-based decision strategy
interacts with a simulated grid \citep{xu2024real}, and in process control,
where a human digital twin emulates an operator's corrective actions in a
simulated plant \citep{balaji2023act}.

The collaborator need not be uniformly informative. For example, a
hospital's records may contain treatments that an externally developed
clinical digital twin rarely or never selects, while the twin may provide
information about patient states poorly represented in those records.
The two sources should therefore be viewed as complementary rather than
interchangeable. Unlike imitation learning, our objective is not to recover
the collaborator's policy, but to use both information sources to identify
a near-optimal policy for the underlying decision process.

This division of control also distinguishes our setting from standard
trajectory-based hybrid RL, which supplements offline data by executing
exploration policies \citep{ball2023efficient,huang2025augmenting,song2022hybrid}. Here, the learner can request
observations at a chosen state-time pair without first reaching that state
through a trajectory. Conversely, unlike unrestricted generative-model
access \citep{agarwal2020model,azar2012sample,sidford2018near,zanette2019almost}, it cannot directly select the sampled
action. The query strategy must therefore account for both the information
already available offline and the collaborator's sampling restrictions.

Uniformly querying all state-time pairs is generally inefficient.
The offline data may already estimate some parts of the decision process
accurately, while leaving others uncertain. More importantly, not every
uncertainty is equally relevant: identifying a good policy requires resolving
comparisons between plausible alternatives, rather than estimating every
reward and transition equally accurately. Poor coverage alone therefore
need not make a state the best query location. An efficient strategy
should direct queries toward the policy comparisons that remain
statistically difficult after accounting for the offline data, subject to
which observations the collaborator can supply.

The sampling restriction also means that some uncertainty may persist
indefinitely. If the collaborator never selects an action at a given
state-time pair, no amount of querying can provide additional observations
of that action there. Its estimation remains limited by the fixed offline
data, potentially imposing an irreducible \emph{certification limit}.
A learning procedure must therefore account not only for how to allocate
queries, but also for when additional queries cannot sufficiently reduce
the uncertainty needed to certify the target accuracy.
\begin{figure}[!t]
    \centering
     \includegraphics[width=\linewidth]{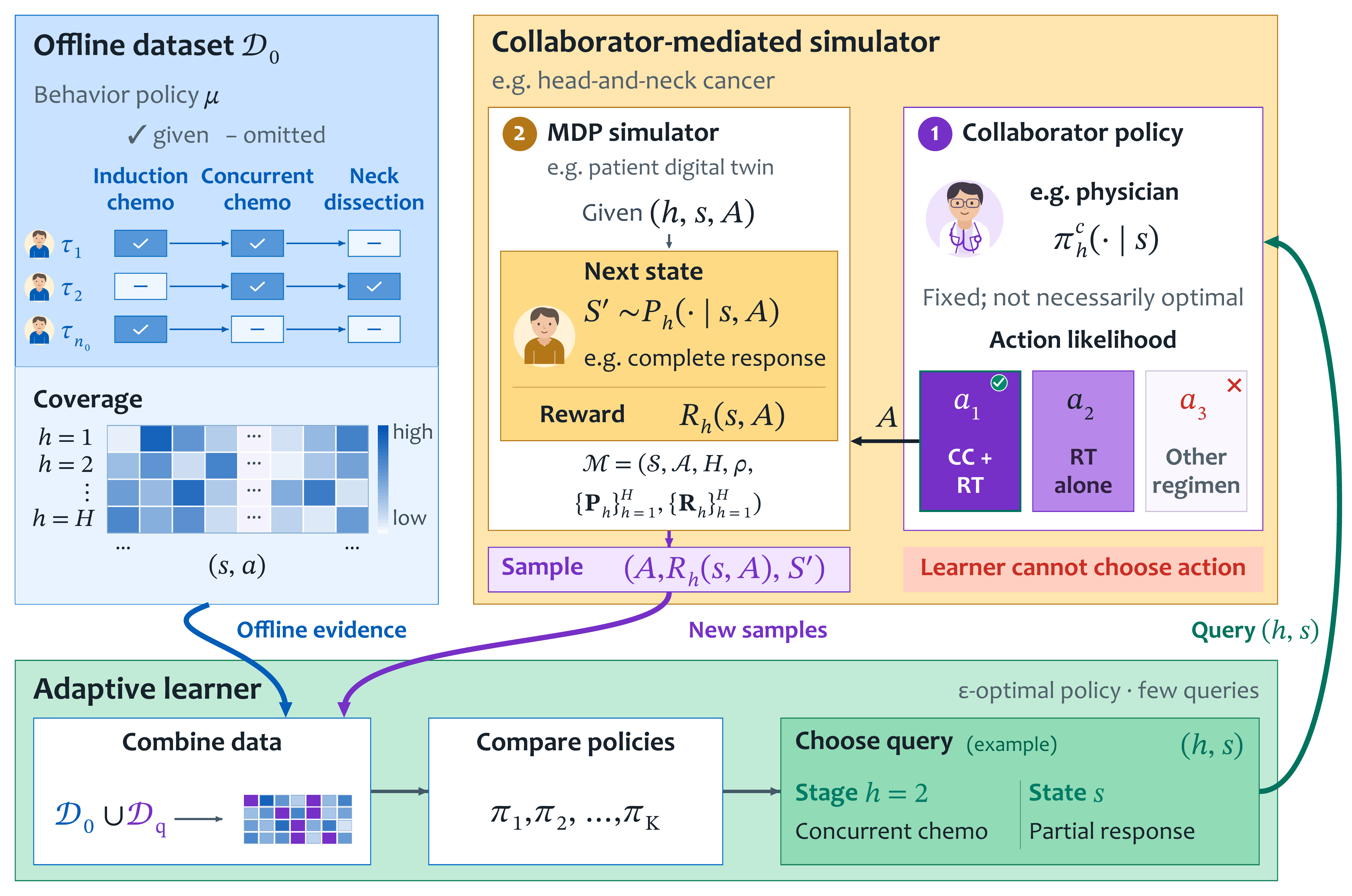}
    \caption{Query-based collaborative RL. The learner augments offline data with simulated experience, choosing state-time queries but not the actions sampled by the fixed-policy collaborator.}
    \label{fig:result}
\end{figure}

\noindent\textbf{Contributions.}
We initiate the framework of query-based collaborative reinforcement learning, where an instance-dependent analysis is natural because the value of additional queries depends strongly on the offline coverage, collaborator support and the structure of the underlying MDP, which can vary substantially across instances. Hence, a worst-case analysis obscures this heterogeneity and cannot capture how the offline data and collaborator jointly reduce the remaining uncertainty. 
For a target suboptimality $\epsilon>0$ and failure probability
$\delta\in(0,1)$, our main results are as follows.
\begin{itemize}[left=0pt]
   \item \textbf{Information-theoretic lower bound and local characterization.}
We derive an exact KL-based lower-bound benchmark for the additional
collaborator queries required to identify an $\epsilon$-optimal policy with
probability at least $1-\delta$. We then derive a tractable local quadratic
surrogate that makes explicit the roles of policy-comparison margins, local
reward-transition variability, offline coverage, and collaborator support.

    \item \textbf{Lower-bound-guided adaptive querying.}
    We show that the same local variance structure governs the lower-bound
    approximation and the leading finite-sample uncertainty of policy-value
    comparisons. Guided by this connection, we develop an epoch-based
    algorithm that progressively narrows the candidate policy class and
    allocates queries through a local minimax design problem. The design
    jointly selects a reference policy and state--time query locations,
    concentrating observations on comparisons that remain difficult after
    exploiting the offline data.

    \item \textbf{Finite-sample guarantees and near-matching complexity.}
    We establish finite-query termination and high-probability suboptimality
    guarantees for the returned policy, covering both near-optimality
    certification and stopping at the certification limit. In the certifiable 
    regime, under the stated assumptions, the leading query-complexity term is governed by the same local quadratic
characteristic, evaluated at a polylogarithmically reduced effective
offline sample size. We additionally establish robustness to row-wise simulator
misspecification measured in total variation.
\end{itemize}
Empirically, we evaluate the method on three synthetic environments (contextual bandit, multi-stage MDP, and an MDP with long-term consequences), and a semi-synthetic drug-response application. Compared with uniform and coverage-based allocation, our adaptive allocation certifies an $\epsilon$-optimal policy with $3$-$22$ times fewer collaborator queries in an efficient implementation and returns better policies at equal budgets.

\section{Related Work}\label{sec:related work}

\paragraph{Worst-case sample complexity in tabular RL.}
A large body of work studies the worst-case sample complexity of learning a
near-optimal policy in tabular MDPs.
In offline RL, where the learner must rely entirely on a fixed dataset,
 guarantees of related work necessarily depend on how well the behavior
policy covers the state--action pairs relevant to an optimal policy.
A sequence of works has established sample complexity guarantees
under various notions of coverage, including 
single-policy concentrability 
\citep{shi2022pessimistic,gen2024settling}.

At the other extreme, the generative-model setting allows the learner to
request samples from arbitrary state--action pairs.
Classical and more recent results characterize the minimax sample complexity
of uniformly querying under such access
\citep{gheshlaghi2013minimax,agarwal2020model}.

\paragraph{Instance-dependent policy identification.}
Our work is closely related to a growing literature on instance-dependent
reinforcement learning, which seeks to exploit the specific structure of the
underlying MDP to obtain substantially sharper sample-complexity guarantees
than worst-case bounds.  Rather than characterizing difficulty solely through
global quantities such as the size of the state--action space or worst-case
coverage coefficients, these results adapt to instance-specific features such
as visitation probabilities, local variances, and suboptimality gaps.  Hence the resulting guarantees can be much smaller than worst-case rates on
favorable instances.

In offline tabular RL, \citet{yin2021towards} derived variance-sensitive
instance-dependent guarantees while \citet{wang2022gap} showed that,
under suitable coverage conditions, the sample complexity can be improved
when the minimum suboptimality gap
 is bounded away from zero.

A related line of work studies instance-dependent policy identification in
online RL.
For online tabular RL, \citet{wagenmaker2022beyond} developed
instance-dependent PAC guarantees that adapt to local action gaps, thereby allowing substantially smaller
sample complexity on MDPs in which suboptimal actions can be identified
quickly.  In contrast, \citet{wagenmaker2022instance} studied
instance-dependent policy identification from a policy-level perspective,
with complexity controlled by the suboptimal policy gap.
More recently, \citet{narang2024sample} showed that directly estimating
value differences between candidate policies can lead to substantially
sharper identification guarantees than estimating each policy value
separately.

Instance dependence has also been studied under generative-model access.
For discounted tabular MDPs, \citet{zanette2019almost} derived
structure-aware sample-complexity guarantees that
adapt to local reward and transition variances as well as suboptimality action
gaps, and can avoid an explicit polynomial dependence on the effective
horizon in favorable instances.

\paragraph{Imitation learning and learning with expert queries.}
Several ideas underlying our framework also appear in the imitation-learning
literature, particularly the selective use of expert supervision and the
combination of expert information with other sources of data.
\citet{judah14a} studied active imitation learning, where the learner
adaptively selects informative states at which to query the expert for an
action, with the goal of reducing the number of expert labels.
\citet{liu2022learning} studied adversarial policy imitation with linear
function approximation and, in their offline setting, combined expert
demonstrations with an additional offline dataset collected by a prior
policy.  They showed that this supplementary dataset can improve the
statistical efficiency of imitation learning under suitable coverage
conditions.
More closely related to our use of the collaborator,
\citet{amortila2022few} considered efficient learning with a small number
of expert-action queries.  In their setting, an expert query provides only
the expert's action and the learner implements it and interacts with the environment, and the goal is to learn a policy that performs at least nearly as well
as the expert rather than to identify a near-optimal policy.

\paragraph{Adaptive querying and complementary information.}
Efficient information acquisition is also studied beyond RL.
\citet{ding2026whom} adaptively select questions and respondents to infer
latent group properties under budget constraints, while
\citet{liu2026pets} allocate reasoning samples across questions for
efficient test-time self-consistency. Related work also exploits
complementary information: weak-strong model collaboration combines
specialized knowledge with general reasoning
\citep{jiao2025synergistic}. Our setting instead concerns near-optimal
policy identification from offline data and adaptively queried simulator
observations. The collaborator's fixed action policy restricts which
uncertainties can be reduced, making both query efficiency and the
attainability of the target accuracy central to the analysis.

\section{Setup}
\label{sec:setup}
In this section, we formalize the reinforcement learning setting we study.

\noindent\textbf{Tabular MDPs.} Following \citet{gen2024settling}, we consider a finite-horizon,
time-inhomogeneous tabular MDP
$$\mathbf{\mathcal M}=(\mathcal S,\mathcal A,H,\rho,
\{\mathbf P_h\}_{h=1}^H,\{\mathbf R_h\}_{h=1}^H),$$
where $\mathcal S$ and $\mathcal A$ are finite state and action spaces
of sizes $|\mathcal S|$ and $|\mathcal A|$, $H$ is the horizon,
and $\rho$ is the initial-state distribution.
Write $[m]:=\{1,\ldots,m\}$ for integers $m\ge1$.
At stage $h\in[H]$, the transition kernel $\mathbf P_h$ assigns
probability $P_h(s'\mid s,a)$ to next state $s'$ from $(s,a)$.
The random reward vector $\mathbf R_h\in[0,1]^{|\mathcal S||\mathcal A|}$
has components $R_h(s,a)$ with means
$r_h(s,a):=\mathbb E[R_h(s,a)]$. For every $(h,s,a)$, we assume that, conditional on
$(S_h,A_h)=(s,a)$, the reward and next state are independent.

\noindent\textbf{Preliminaries for analysis.} A policy $\pi=\{\pi_h\}_{h=1}^H$ assigns action distributions
$\pi_h(\cdot\mid s)\in\Delta_{\mathcal A}$, where $\Delta_{\mathcal A}$
is the probability simplex over $\mathcal A$. Let $\Pi_d$ denote
the deterministic policy class, with $\pi_h(s)\in\mathcal A$
the selected action. Write $\mathbb E_{\mathbf{\mathcal M}}^\pi$
and $\mathbb P_{\mathbf{\mathcal M},\pi}$ for trajectory expectation
and probability under $\pi$ in $\mathbf{\mathcal M}$.
Define the value and action-value functions by
\begin{align*}
V_h^\pi(s)
&:=\mathbb E_{\mathbf{\mathcal M}}^\pi
\left[\sum_{t=h}^H R_t(S_t,A_t)\mid S_h=s\right],\\
Q_h^\pi(s,a)
&:=r_h(s,a)+
\mathbb E_{S'\sim P_h(\cdot\mid s,a)}[V_{h+1}^\pi(S')],
\end{align*}
with the terminal convention $V_{H+1}^{\pi}(s)=0$.
The initial value is
$V_1^\pi:=\mathbb E_{S_1\sim\rho}[V_1^\pi(S_1)]$.
The state and state-action occupancies are
\[
d_h^\pi(s):=\mathbb P_{\mathbf{\mathcal M},\pi}(S_h=s),
\quad
d_{hsa}^\pi:=d_h^\pi(s)\pi_h(a\mid s).
\]

We further define the optimal value and the set of optimal deterministic
policies by
\[
V_1^\star
:=
\max_{\pi\in\Pi_d}V_1^\pi,
\qquad
\Pi^\star(\mathbf{\mathcal M})
:=
\arg\max_{\pi\in\Pi_d}V_1^\pi.
\]
The policy gap
\[
\Delta_{\mathbf{\mathcal M}}(\pi)
:=
V_1^\star-V_1^\pi
\]
defines the $\epsilon$-optimal set
\[
\Pi_\epsilon(\mathbf{\mathcal M})
:=
\left\{
\pi\in\Pi_d:
\Delta_{\mathbf{\mathcal M}}(\pi)\le\epsilon
\right\}.
\]

For any
$\pi^\star\in\Pi^\star(\mathbf{\mathcal M})$, define
\begin{equation}
\Delta_{\min}
:=
\min_{\pi\in\Pi_d\setminus\{\pi^\star\}}
\Delta_{\mathbf{\mathcal M}}(\pi).
\label{eq:min-policy-gap}
\end{equation}
This definition does not depend on the choice of
$\pi^\star$. In particular, if the optimal deterministic policy is
unique, then $\Delta_{\min}$ is the smallest strictly positive
suboptimality gap, whereas if there are multiple optimal deterministic
policies, then
$
\Delta_{\min}=0.
$

\noindent\textit{Empirical estimation.} We next construct the empirical MDP from a given dataset $\mathcal D$ and define corresponding policy values and visitation probabilities. For each $(h,s,a)$, suppose the dataset $\mathcal D$ contains
$N_{hsa}$ reward-next-state observations $\{(R_{hsa}^{(i)},S_{hsa}^{\prime(i)})\}_{i=1}^{N_{hsa}},$ such that
\[
\mathcal D
=
\bigcup_{h,s,a}
\left\{
(R_{hsa}^{(i)},S_{hsa}^{\prime(i)}):
i\in[N_{hsa}]
\right\}.
\]
For $N_{hsa}>0$, define
\begin{align*}
\widehat r_h^\mathcal D(s,a)
&:=\frac{1}{N_{hsa}}\sum_{i=1}^{N_{hsa}}R_{hsa}^{(i)},\\
\widehat P_h^\mathcal D(s'\mid s,a)
&:=\frac{1}{N_{hsa}}\sum_{i=1}^{N_{hsa}}
\mathbf 1\{S_{hsa}^{\prime(i)}=s'\}.
\end{align*}
These estimates define the empirical MDP
$\widehat{\mathbf{\mathcal M}}^\mathcal D
=(\mathcal S,\mathcal A,H,\rho,\widehat P^\mathcal D,
\widehat r^\mathcal D)$.
For any policy $\pi$, let $\widehat V_h^{\pi,\mathcal D}$
and $\widehat Q_h^{\pi,\mathcal D}$ denote its value and
action-value functions under this MDP. Its state and state--action
occupancies are
\begin{align*}
\widehat d_h^{\pi,\mathcal D}(s)
&:=\mathbb P_{\widehat{\mathbf{\mathcal M}}^\mathcal D,\pi}(S_h=s),\qquad
\widehat d_{hsa}^{\pi,\mathcal D}
:=\widehat d_h^{\pi,\mathcal D}(s)\pi_h(a\mid s).
\end{align*}

\noindent\textbf{Beyond offline data: collaborator-mediated simulator.} 
The learner receives an offline dataset $\mathcal D_0=\{\tau_i\}_{i=1}^{n_0}$
of $n_0$ trajectories generated by an unknown behavior policy $\mu$
in $\mathbf{\mathcal M}$, with $N_{0,hsa}$ observations of
$(R_h,S_{h+1})$ collected at $(h,s,a)$. 
It also accesses a collaborator-mediated simulator: for any queried
state-time pair $(h,s)$, the collaborator draws
$A\sim\pi_h^c(\cdot\mid s)$ under a \textit{fixed}, possibly suboptimal
policy $\pi^c$, and the simulator returns reward $R_h(s,A)$
and next state $S'\sim P_h(\cdot\mid s,A)$.
Each request and its resulting sample $(A,R_h(s,A),S')$
constitute a \textit{collaborator query}\footnote{For simplicity, in the main development we assume that the simulator returns
$(R,S')$ according to the true reward-transition law conditional on
$(h,s,A)$. In Appendix~\ref{sec:shifted-simulator}, we extend the analysis to the case where
$(R,S')$ is drawn from a noisy approximation of this law, allowing both
the reward and transition distributions to be misspecified.}.
Thus, the learner chooses the state-time pair, but not the sampled action.

We define the minimum positive collaborator-action probability by
\[
p_{\min}^c
:=
\min_{(h,s,a):\,\pi_h^c(a\mid s)>0}
\pi_h^c(a\mid s).
\]

 This interaction protocol differs from both standard online RL and the generative-model setting. In those settings, the learner chooses the action to execute or sample, whereas here a fundamental new challenge is that the learner chooses only $(h,s)$, while the action is determined
by the collaborator policy $\pi_h^c(\cdot\mid s)$. Hence, even with
infinitely many queries, full state-action coverage may be impossible because
if $\pi_h^c(a\mid s)=0$, no additional samples can be collected at
$(h,s,a)$. Consequently, the learner will be unable to reduce the uncertainty associated with some state-action pairs beyond what is already provided by the offline data, even with unlimited collaborator queries. This creates a learning objective fundamentally different from those in standard online reinforcement learning and generative-model settings, since the learner must account for uncertainty that cannot be eliminated through additional interaction.

\textbf{Our Desiderata: Optimal Allocation of Collaborator Queries.}
Given a target accuracy $\epsilon>0$ and confidence level
$\delta\in(0,1)$, an algorithm is $(\epsilon,\delta)$-PAC
 if it 
returns a policy $\widehat\pi\in\Pi_d$ satisfying
\[
\mathbb P_{\mathcal M}
\left(
V_1^\star
-
V_1^{\widehat\pi}
\le\epsilon
\right)
\ge1-\delta.
\]
Here the probability includes the randomness of the offline
dataset, the collaborator observations, and the algorithm.
Our goal is to design an adaptive procedure that accounts for the statistical
limitations of both available information sources. If the target accuracy $\epsilon$ is attainable from the fixed offline
dataset together with additional collaborator-mediated observations, the
procedure should achieve the desired $(\epsilon,\delta)$-PAC guarantee.
If $\epsilon$ is not attainable, the procedure should instead return a
policy satisfying an $(\epsilon',\delta)$-PAC guarantee for some attainable
resolution $\epsilon'>\epsilon$. Here, $\epsilon'$ should reflect the
irreducible uncertainty induced by the finite offline dataset and by
state-action pairs whose uncertainty cannot be further reduced through
collaborator queries.

When the target $\epsilon$ is attainable, our further objective is to use as
few collaborator queries as possible. The key design question is therefore
how to determine the required query budget $B$ and allocate it across state-time pairs. We represent a
normalized query allocation by
\[
\boldsymbol\lambda
=
\{\lambda_{hs}\}_{h,s}
\in\Delta_{H|\mathcal S|}:=\{\boldsymbol \ell\in\mathbb R_+^{H\times |\mathcal S|}:\sum_{h,s}\ell_{hs}=1\},
\]
where $\lambda_{hs}$ denotes the fraction of collaborator queries assigned
to $(h,s)$. Since a query at $(h,s)$ generates an action according to
$\pi_h^c(\cdot\mid s)$, the corresponding sampling proportion at
$(h,s,a)$ is $\lambda_{hs}\pi_h^c(a\mid s)$. Consequently, the optimal
allocation depends on the particular MDP, the information already contained
in the offline data, and the collaborator policy. Our objective is therefore
\textit{instance-dependent}: whenever the target accuracy is attainable, we seek an
adaptive $(\epsilon,\delta)$-PAC procedure whose collaborator-query
complexity is governed by the intrinsic difficulty of the underlying
instance.

\section{A Key Quantity Guiding Algorithm Design }
\label{sec:key-quantity}

To guide query allocation, we derive an information-theoretic lower bound
on the additional collaborator queries needed for an $(\epsilon,\delta)$-PAC
guarantee and develop a local approximation that reveals an explicit local
variance structure. The same local variance structure governs finite-sample
uncertainty in policy-value comparisons. Reducing this uncertainty targets
the difficult comparisons identified by the lower bound, motivating our
adaptive query-design criterion. Auxiliary notation for this section is
collected in Table~\ref{tab:key-notation}.

\begin{table}[t]
\centering
\scriptsize
\caption{Auxiliary notation used in Section~\ref{sec:key-quantity}.}
\label{tab:key-notation}
\setlength{\tabcolsep}{3pt}
\renewcommand{\arraystretch}{1.08}
\begin{tabular}{@{}p{0.22\columnwidth}p{0.5\columnwidth}p{0.17\columnwidth}@{}}
\toprule
Symbol & Definition & Meaning \\
\midrule

$\mathcal L_{\mathbf{\mathcal M}}(X)$
&
Law of a random variable $X$ under $\mathbf{\mathcal M}$
&
Probability law
\\

$\nu_{hsa}^{\mathbf{\mathcal M}}$
&
$\mathcal L_{\mathbf{\mathcal M}}
(R_h(s,a),S_{h+1}\mid S_h=s,A_h=a)$
&
Observation law
\\

\shortstack[l]{$\mathrm{KL}_{hsa}
(\mathbf{\mathcal M},\widetilde{\mathbf{\mathcal M}})$}
&
$\mathrm{KL}(\nu_{hsa}^{\mathbf{\mathcal M}}
\Vert\nu_{hsa}^{\widetilde{\mathbf{\mathcal M}}})$
&
Local information
\\

$\Delta_{H|\mathcal S|}$
&
$\{\mathbf\lambda\in\mathbb R_+^{H\times |\mathcal S|}:
\sum_{h,s}\lambda_{hs}=1\}$
&
Query allocations
\\

$q_{hsa}^{\mathbf\lambda}$
&
$\lambda_{hs}\pi_h^c(a\mid s)$
&
Sampling proportion
\\

$J_\mu(\mathbf{\mathcal M},\widetilde{\mathbf{\mathcal M}})$
&
$\sum_{h,s,a}d_{hsa}^{\mu}
\mathrm{KL}_{hsa}(\mathbf{\mathcal M},\widetilde{\mathbf{\mathcal M}})$
&
Offline information
\\

$I_{\mathbf\lambda}(\mathbf{\mathcal M},\widetilde{\mathbf{\mathcal M}})$
&
$\sum_{h,s,a}q_{hsa}^{\mathbf\lambda}
\mathrm{KL}_{hsa}(\mathbf{\mathcal M},\widetilde{\mathbf{\mathcal M}})$
&
Query information
\\

$D(\pi,\bar\pi)$
&
$V_1^\pi-V_1^{\bar\pi}$
&
Value difference
\\

$\widehat D(\pi,\bar\pi)$
&
$\widehat V_1^\pi-\widehat V_1^{\bar\pi}$
&
Empirical value difference
\\

$b(\pi,\bar\pi)$
&
$\epsilon-D(\pi,\bar\pi)$
&
Certification margin
\\

$G^r_{hsa}(\pi,\bar\pi)$
&
$d_{hsa}^{\pi}-d_{hsa}^{\bar\pi}$
&
Reward sensitivity
\\

$G^P_{hsa}(\pi,\bar\pi)(s')$
&
$d_{hsa}^{\pi}V_{h+1}^{\pi}(s')
-d_{hsa}^{\bar\pi}V_{h+1}^{\bar\pi}(s')$
&
Transition sensitivity
\\

$\Psi_{hsa}^{\pi,\bar\pi}(r,s')$
&
$G^r_{hsa}(\pi,\bar\pi)r
+G^P_{hsa}(\pi,\bar\pi)(s')$
&
Local influence
\\

$L_{hsa}(\pi,\bar\pi)$
&
$|G^r_{hsa}(\pi,\bar\pi)|
+\operatorname{range}(G^P_{hsa}(\pi,\bar\pi))$
&
Range bound
\\

\bottomrule
\end{tabular}
\end{table}

\noindent\textbf{Information-theoretic lower bound.}
For all $\pi\in\Pi_\epsilon(\mathbf{\mathcal M})$, let
\[
\mathcal A_\epsilon(\pi)
:=
\big\{
\widetilde{\mathbf{\mathcal M}}:
V_1^\star(\widetilde {\mathbf{\mathcal M}})
-
V_1^{\pi}(\widetilde {\mathbf{\mathcal M}})
>
\epsilon
\big\}
\]
denote the set of alternative MDPs under which $\pi$ is no longer
$\epsilon$-optimal. Here all alternative MDPs share the same state and action spaces, horizon,
initial-state distribution, behavior policy, and collaborator policy as
$\mathcal M$. Only the reward and transition laws are allowed to vary.

We next characterize the asymptotic query complexity in the
high-confidence regime. Following the standard $\log(1/\delta)$ scaling in
fixed-confidence identification \citep{al2023towards, garivier2016optimal}, we consider a joint asymptotic regime in
which the offline sample size grows on the same scale, so that offline and
queried information contribute nontrivially at leading order. Specifically,
let $n_{0,\delta}$ denote the number of offline trajectories available at
confidence level $\delta$, and suppose that
\[
\frac{n_{0,\delta}}{\log(1/\delta)}
\longrightarrow
\kappa\in[0,\infty).
\]
Under this normalization, the offline information and the information
acquired through collaborator queries can be compared directly in the
asymptotic lower bound.

We define the characteristic threshold
\[
\begin{aligned}
T_{\epsilon,\kappa}^\star(\mathbf{\mathcal M})
&:=
\min_{\bar\pi\in\Pi_\epsilon(\mathbf{\mathcal M})}
\inf_{\boldsymbol \lambda\in\Delta_{H|\mathcal S|}}
\inf\Bigg\{
t\ge0:
\inf_{\widetilde{\mathbf{\mathcal M}}\in\mathcal A_\epsilon(\bar\pi)}
\Big[
\kappa J_\mu(\mathbf{\mathcal M},\widetilde {\mathbf{\mathcal M}})
+
tI_\mathbf \lambda(\mathbf{\mathcal M},\widetilde {\mathbf{\mathcal M}})
\Big]
\ge1
\Bigg\},
\end{aligned}
\]
which captures the smallest asymptotic query budget, under the best state-time allocation, required to rule out all alternative MDPs.

\begin{theorem}[Information-theoretic lower bound]
\label{thm:information-lower-bound}
For any family of $(\epsilon,\delta)$-PAC algorithms, let
$\tau_\delta$ denote the number of collaborator queries before
stopping. Then
\[
\liminf_{\delta\downarrow0}
\frac{\mathbb E_\mathbf{\mathcal M}[\tau_\delta]}
{\log(1/\delta)}
\ge
T_{\epsilon,\kappa}^\star(\mathbf{\mathcal M}).
\]
\end{theorem}

Although $T_{\epsilon,\kappa}^\star(\mathcal M)$ provides an exact
KL-based characterization of our information-theoretic lower-bound
benchmark, optimization over alternative MDPs leaves its dependence on
policy gaps, occupancies, reward--transition variability, offline coverage,
and collaborator support implicit.
We therefore linearize policy-value differences around $\mathbf{\mathcal M}$
and quadratically approximate the KL divergence. The resulting formulation
expresses difficulty through policy-comparison margins and local variances,
directly connecting the lower bound to the finite-sample uncertainty
controlled by our algorithm.

For a deterministic policy pair $(\bar\pi,\pi)$, based on the notation in
Table~\ref{tab:key-notation}, define the local variance as
\[
\Gamma_{hsa}(\pi,\bar\pi)
:=
\operatorname{Var}_\mathbf{\mathcal M}\!\left(
\Psi_{hsa}^{\pi,\bar\pi}(R_h,S_{h+1})
\,\middle|\,
S_h=s,A_h=a
\right).
\]

\begin{corollary}[Local quadratic characteristic threshold]
\label{cor:local-quadratic-lower-bound}

Consider the local alternative problem in $T_{\epsilon,\kappa}^\star$ restricted to admissible
$\widetilde{\mathbf{\mathcal M}}\in\mathcal A_\epsilon(\bar\pi)$ whose
row distributions satisfy
\begin{equation}
\nu_{hsa}^{\widetilde{\mathbf{\mathcal M}}}
\sim
\nu_{hsa}^{\mathbf{\mathcal M}},
\qquad
\left\|
\log
\frac{
d\nu_{hsa}^{\widetilde{\mathbf{\mathcal M}}}
}{
d\nu_{hsa}^{\mathbf{\mathcal M}}
}
\right\|_\infty
<\infty,
\qquad
\forall(h,s,a),
\label{eq:lq-admissible-local-alternative}
\end{equation} where $\sim $ denotes mutual absolute continuity.
Consider its local quadratic surrogate obtained, for each challenger
$\pi\neq\bar\pi$, by replacing the exact alternative constraint with
its first-order linearization at $\mathbf{\mathcal M}$ and each
row-wise KL divergence with its leading quadratic term.
Then the resulting local quadratic characteristic threshold is

\[
\begin{aligned}
T_{\mathrm{LQ},\kappa}(\mathbf{\mathcal M})
&:=
\min_{\bar\pi\in\Pi_\epsilon(\mathbf{\mathcal M})}
\inf_{\boldsymbol\lambda\in\Delta_{H|\mathcal S|}}
\inf\Bigg\{
t\ge0:
\max_{\pi\in\Pi_d\setminus\{\bar\pi\}}
\frac{2}{b(\pi,\bar\pi)^2}
\sum_{h,s,a}
\frac{
\Gamma_{hsa}(\pi,\bar\pi)
}{
\kappa d_{hsa}^{\mu}
+
tq_{hsa}^{\boldsymbol\lambda}
}
\le1
\Bigg\}.
\end{aligned}
\]
\end{corollary}

Here since the denominator in the local quadratic criterion may vanish, we
adopt the following extended-value convention throughout. For
$\Gamma\ge0$,
\[
\frac{\Gamma}{0}
:=
\begin{cases}
+\infty, & \Gamma>0,\\
0, & \Gamma=0.
\end{cases}
\]
Moreover, any policy-comparison term with
$b(\pi,\bar\pi)=0$ is defined to be $+\infty$.
The same convention applies to all subsequent quantities derived from
the local quadratic criterion.

Thus, $\Gamma_{hsa}(\pi,\bar\pi)$ governs the local statistical
uncertainty contributed by observations at $(h,s,a)$ to the comparison
between $\pi$ and $\bar\pi$.

\noindent\textbf{Local Bernstein bound for policy-value differences.}
Interestingly, the same local variance structure also appears in finite-sample
estimation. Suppose that, for each $(h,s,a)$, we observe $N_{hsa}$
independent reward--next-state samples from
$\nu_{hsa}^{\mathbf{\mathcal M}}$, and let
$\widehat D(\pi,\bar\pi)$ denote the corresponding plug-in estimate of
$D(\pi,\bar\pi)$.

\begin{lemma}[Local Bernstein bound for policy-value differences]
\label{lem:local-bernstein}
For any fixed $\pi,\bar\pi\in\Pi_d$, with probability at least
$1-\delta$,
\[
\begin{aligned}
\left|
\widehat D(\pi,\bar\pi)
-
D(\pi,\bar\pi)
\right|
\le\;&
C
\sqrt{
\log\frac{2}{\delta}
\sum_{h,s,a}
\frac{
\Gamma_{hsa}(\pi,\bar\pi)
}{
N_{hsa}
}
}
\\
&+
C
\log\frac{2}{\delta}
\max_{h,s,a}
\frac{
L_{hsa}(\pi,\bar\pi)
}{
N_{hsa}
}
+
|\operatorname{Rem}_{N}(\pi,\bar\pi)|,
\end{aligned}
\]
where $\operatorname{Rem}_{N}(\pi,\bar\pi)$ collects the higher-order
terms in the plug-in expansion of the policy-value difference and $C>0$ is a universal constant.
\end{lemma}

Lemma~\ref{lem:local-bernstein} shows that the leading finite-sample
uncertainty in comparing $\pi$ and $\bar\pi$ is governed by
\[
\sum_{h,s,a}
\frac{
\Gamma_{hsa}(\pi,\bar\pi)
}{
N_{hsa}
},
\]
which is precisely the same variance structure appearing in
$T_{\mathrm{LQ},\kappa}(\mathbf{\mathcal M})$. This motivates us to focus on this quantity, whose role in policy certification we now make explicit.

\paragraph{Policy-certification interpretation.}
Let $\pi^\star$ denote an optimal deterministic policy. Then
$\bar\pi$ is $\epsilon$-optimal if and only if
\[
\epsilon-D(\pi^\star,\bar\pi)\ge0.
\]
Since $\pi^\star$ is unknown to the learner, we express this inequality
equivalently as the uniform condition
\[
b(\pi,\bar\pi)
:=
\epsilon-D(\pi,\bar\pi)\ge0,
\qquad \forall\pi\in\Pi_d.
\]

Since the true margins are unknown, the learner must certify them using
their empirical counterparts
\[
\widehat b(\pi,\bar\pi)
=
\epsilon-\widehat D(\pi,\bar\pi).
\]
For a challenger $\pi$ with $b(\pi,\bar\pi)>0$, it is sufficient
to estimate $D(\pi,\bar\pi)$ to an accuracy smaller than this margin.
In particular, if
\[
\left|
\widehat D(\pi,\bar\pi)-D(\pi,\bar\pi)
\right|
\lesssim
b(\pi,\bar\pi),
\]
then the sign of the corresponding certification inequality can be
reliably determined.
By Lemma~\ref{lem:local-bernstein}, up to logarithmic and lower-order
terms, this corresponds to 
\begin{equation}
\label{eq:berstein condition}
\sum_{h,s,a}
\frac{
\Gamma_{hsa}(\pi,\bar\pi)
}{
N_{hsa}
}
\lesssim
b(\pi,\bar\pi)^2.
\end{equation}

Hence, requiring \eqref{eq:berstein condition} to hold uniformly over all
$\pi\neq\bar\pi$ gives rise to the maximization over $\pi\in \Pi_d\setminus \{\bar \pi\}$ in $T_{\mathrm{LQ},\kappa}$. Since it suffices to certify one
$\epsilon$-optimal policy, the minimization over
$\bar\pi\in\Pi_\epsilon(\mathbf{\mathcal M})$ selects the easiest valid candidate,
while the infimum over
$\boldsymbol \lambda\in\Delta_{H|\mathcal S|}$ selects the most informative
admissible query allocation. The remaining infimum over $t$ therefore gives
the smallest additional query budget for which such uniform certification
is possible.

\section{Algorithm Overview}
\label{sec:algorithm}

In this section, we develop an epoch-based adaptive algorithm that progressively narrows
the set of plausible policies. Our design is partly inspired by the adaptive elimination
and experimental-design principles in \citet{zanette2019almost} and \citet{narang2024sample}, while adapting
these ideas to our setting. At each epoch, the learner first uses fresh data to refine
the active policy class and, if the algorithm does not stop, solves a lower-bound-guided
query-allocation problem for the next epoch to reduce uncertainty in the remaining policy
comparisons. We outline the design below and defer the formal algorithm to
Appendix~\ref{app:full_algorithm}.

\noindent\textbf{Trimming and splitting offline data.}
We apply the two-fold trimming procedure of \citet{gen2024settling}
to remove within-trajectory dependence for concentration analysis,
then split the data into a historical sample $\mathcal D_0^\mathcal G$
for estimating policy sensitivities and an inference sample
$\mathcal D_0^\mathcal D$ for estimating policy differences.
We partition the latter into disjoint epoch-wise folds
$\{\mathcal D_{0,k}^\mathcal D\}_{k\ge1}$ with fractions
\[
w_k:=\frac{6}{\pi^2k^2}.
\]
Each fold's outcomes are revealed only at epoch $k$, ensuring independence
from the policy class and query allocation conditional on the pre-epoch
history and preserving the estimator's centering.

\noindent\textbf{Epoch-wise iteration.}
At epoch $k\ge1$, we set $\varepsilon_k:=H2^{1-k}$ and issue $m_k^{\mathrm{st}}$ collaborator queries
at every state--time pair and, for $k\ge2$, execute the allocation
$\widehat{\boldsymbol n}_{k-1}$ chosen in the previous epoch.
The stabilization batch ensures fresh coverage of collaborator-supported
rows, reducing non-leading errors to $ O(\epsilon_k)$.
Together with the offline fold $\mathcal D_{0,k}^\mathcal D$,
these samples form the fresh inference sample $\mathcal D_k^\mathcal D$.

Set
\[
\widehat\Pi_0=\widehat\Pi_1=\Pi_d.
\]
At $k=1$, no policy-class update is required. For $k\ge2$, we use only
$\mathcal D_k^\mathcal D$ to estimate policy differences relative
to the previous reference policy $\bar\pi_{k-1}$ and update
\begin{align*}
\widehat\Pi_k
:=\Bigl\{\pi\in\widehat\Pi_{k-1}:&
\max_{\pi'\in\widehat\Pi_{k-1}}
\bigl[
\widehat D_k(\pi',\bar\pi_{k-1})
-
\widehat D_k(\pi,\bar\pi_{k-1})
\bigr]
\le\tfrac52\epsilon_k
\Bigr\}.
\end{align*}
Thus, we retain policies whose empirical values are within
$5\epsilon_k/2$ of the best in $\widehat\Pi_{k-1}$.

For $k=1$, we choose an arbitrary
\[
\widehat\pi_1\in\widehat\Pi_1.
\]
For every $k\ge2$, after the active class has been updated, choose
\[
\widehat\pi_k
\in
\arg\max_{\pi\in\widehat\Pi_k}
\widehat V_1^{\pi,k,\mathcal D}.
\]

If the algorithm does not stop after the current localization step, we add
the fresh sample $\mathcal D_k^\mathcal D$ to the historical sample.
Using the resulting empirical MDP $\widehat M_k^\mathcal G$ and its
reward and transition sensitivities, we construct tight high-probability
upper bounds
$\overline\Gamma_{k,hsa}(\pi,\bar\pi)$ on
$\Gamma_{hsa}(\pi,\bar\pi)$ for
$\pi,\bar\pi\in\widehat\Pi_k$.

Define the regularized margin
\[
\widehat b_{k,\mathrm{reg}}(\pi,\bar\pi)
:=
\max\{
\epsilon-\widehat D_k(\pi,\bar\pi),
4\epsilon_k
\}.
\]
For a state--time allocation
$\boldsymbol n=\{n_{hs}\}_{h,s}$,
the predicted next-epoch row counts are
\[
N_{k+1,hsa}^{+}(\boldsymbol n)
:=
N_{0,k+1,hsa}^{\mathcal D}
+
\bigl(
m_{k+1}^{\mathrm{st}}+n_{hs}
\bigr)
\widehat\pi_h^c(a\mid s),
\]
where $N_{0,k+1,hsa}^{\mathcal D}$ is the next offline fold's row count,
$m_{k+1}^{\mathrm{st}}$ is the stabilization budget per state--time pair,
and $\widehat\pi_h^c(a\mid s)$ denotes the collaborator-policy proxy
described in Appendix~\ref{app:full_algorithm}; when $\pi^c$ is known,
we simply take $\widehat\pi^c=\pi^c$.

The leading prospective radius is
\[
\widehat{\operatorname{rad}}_k^{+}
(\pi,\bar\pi;\boldsymbol n)
:=
C_0
\sqrt{
\beta_{k+1}
\sum_{h,s,a}
\frac{
\overline\Gamma_{k,hsa}(\pi,\bar\pi)
}{
N_{k+1,hsa}^{+}(\boldsymbol n)
}
},
\]
where $C_0$ is universal and $\beta_{k+1}$ is the logarithmic
confidence factor defined in the appendix.

For budget $B\in \mathbb N$, let
\[
\mathcal N(B)
:=
\left\{
\boldsymbol n\in\mathbb N^{H|\mathcal S|}:
\sum_{h,s}n_{hs}\le B
\right\},
\]
and jointly optimize the reference policy and allocation:
\[
\widehat\Phi_k^\star(B)
:=
\min_{\bar\pi\in\widehat\Pi_k}
\inf_{\boldsymbol n\in\mathcal N(B)}
\max_{\pi\in\widehat\Pi_k\setminus\{\bar\pi\}}
\frac{
\widehat{\operatorname{rad}}_k^{+}
(\pi,\bar\pi;\boldsymbol n)
}{
\widehat b_{k,\mathrm{reg}}(\pi,\bar\pi)
}.
\]

Before solving the finite-budget problem, we compute the infinite-budget value of the current prospective design
\[
\widehat\Phi_{k,\infty}^\star
:=
\lim_{B\to\infty}
\widehat\Phi_k^\star(B).
\]
If
\[
\widehat\Phi_{k,\infty}^\star
\ge
\frac1{32},
\]
the algorithm stops, since even unlimited collaborator queries cannot
reduce the normalized uncertainty below the certification threshold.
Otherwise,
\[
\widehat\Phi_{k,\infty}^\star
<
\frac1{32}
<
\frac1{16},
\]
so some finite budget attains the target level $1/16$.
We choose the smallest
\[
B_k\in\{1,2,4,8,\ldots\}
\]
satisfying
\[
\widehat\Phi_k^\star(B_k)
\le
\frac1{16},
\]
store a corresponding optimizer
$(\bar\pi_k,\widehat{\boldsymbol n}_k)$,
and execute $\widehat{\boldsymbol n}_k$ at epoch $k+1$.

\paragraph{Stopping rules and output.}
Let $\mathsf{OffCover}_k$ denote the offline-coverage indicator defined
formally in Appendix~\ref{app:full_algorithm} which  checks whether every
collaborator-unsupported row has sufficient fresh offline coverage for
 inference at epoch $k$ and sufficient cumulative offline coverage for
historical sensitivity estimation.
Before entering the first epoch, we check the offline-coverage condition
$\mathsf{OffCover}_1$. If
\[
\mathsf{OffCover}_1=0,
\]
the algorithm stops and returns an arbitrary policy in $\Pi_d$.

After updating
$\widehat\Pi_k$, if
\(
|\widehat\Pi_k|=1
\),
we return the unique remaining policy. If instead
\(
\epsilon_k\le\epsilon,
\)
we return $\widehat\pi_k$.

If neither condition holds, we release the current fresh sample into the
historical pool and construct the sensitivity and variance estimates used
for the next-epoch design. Before performing that design, we evaluate the
next-epoch offline-coverage condition using the row counts recorded during
preprocessing. If
\[
\mathsf{OffCover}_{k+1}=0,
\]
some collaborator-unsupported row cannot satisfy the required fresh or
historical coverage at epoch $k+1$, and additional collaborator queries
cannot repair this deficiency. We therefore stop and return
$\widehat\pi_k$.

Otherwise, we compute
\[
\widehat\Phi_{k,\infty}^\star
:=
\lim_{B\to\infty}
\widehat\Phi_k^\star(B).
\]
If
\(
\widehat\Phi_{k,\infty}^\star
\ge
\frac1{32}
\),
the remaining uncertainty cannot be sufficiently reduced even with
unlimited collaborator queries, and we return $\widehat\pi_k$.
Otherwise, we select the next finite-budget allocation as described above
and proceed to epoch $k+1$.

\section{Theoretical Guarantees}
\label{sec:guarantees}

We now state the main theoretical guarantees of our algorithm.

\noindent\textbf{Stopping and correctness.}
Our first result shows that the algorithm terminates after finitely many
collaborator queries and characterizes the accuracy of the returned policy.

\begin{theorem}[Finite stopping and accuracy at termination]
\label{thm:stopping-correctness}
With probability at least $1-\delta$, the following statements hold.

\begin{enumerate}
    \item
    \textbf{Finite stopping.}
    The algorithm terminates after finitely many collaborator queries, at an
    epoch $k_{\rm stop}$ satisfying
    \[
    k_{\rm stop}
    =
    O\!\left(
    1+\log
    \frac{H}{
    \epsilon\vee\Delta_{\min}
    }
    \right).
    \]

    \item
    \textbf{Accuracy at termination.}
    If the algorithm terminates at epoch $k$ and returns
    $\widehat\pi_k$, then
    \[
    V_1^\star
    -
    V_1^{\widehat\pi_k}
    \le
    \epsilon_k
    =
    H2^{1-k}.
    \]

    In particular, if the algorithm stops with
    $\epsilon_k\le\epsilon$ or
    $|\widehat\Pi_k|=1$, then its output policy is $\epsilon$-optimal.
\end{enumerate}
\end{theorem}
Theorem~\ref{thm:stopping-correctness} guarantees finite-epoch termination
and bounds the suboptimality of the returned policy at the
terminal resolution.
The requested accuracy is certified when
$\epsilon_k\le\epsilon$ or the active policy class is a
singleton.

If the algorithm stops because the next offline fold cannot provide the
required coverage on collaborator-unsupported rows, the fixed offline
data fail a row-wise lower-order coverage requirement arising from the
construction in Appendix~\ref{app:full_algorithm}; see
Remark~\ref{rem:lower-order-coverage}. 
The next theorem addresses the other  stopping event, triggered
by the infinite-budget design criterion $\widehat\Phi_{k,\infty}^\star
\ge
\frac1{32}$.

Let
$\mathcal U_c:=\{(h,s,a):\pi_h^c(a\mid s)=0\}$ 
denote the collaborator-unsupported region. We restrict attention here to $\mathcal U_c\ne\varnothing$,
since otherwise $\widehat\Phi_{k,\infty}^\star=0$ and the
infinite-budget stopping criterion cannot be triggered. Define $\mathcal F_\infty(\eta;n)$ as
\begin{align*}
\inf_{\bar\pi\in\Pi_\eta(\mathcal M)}
\sup_{\pi\in\Pi_d\setminus\{\bar\pi\}}
\frac{2}{b_\eta(\pi,\bar\pi)^2}
\sum_{(h,s,a)\in\mathcal U_c}
\frac{\Gamma_{hsa}(\pi,\bar\pi)}{nd^\mu_{hsa}},
\end{align*}while using the extended-value conventions stated above.
Note that taking the infinite-query limit in \(T_{\mathrm{LQ},n}\) eliminates the collaborator-supported terms, leaving \(\mathcal F_\infty(\eta;n_0)\) as the local quadratic difficulty of certifying an \(\eta\)-optimal policy from the fixed offline data on \(\mathcal U_c\).

We define the local quadratic certification threshold to characterize the local quadratic accuracy limit induced by
collaborator-unsupported rows to be
\[
\epsilon^{\rm lim}(n_0)
:=
\inf\left(
\left\{
\eta\in(0,H]:
\mathcal F_\infty(\eta;n_0)\le1
\right\}
\cup\{H\}
\right).
\]
Thus $\epsilon^{\rm lim}(n_0)\in[0,H]$.
Here if $\left\{
\eta\in(0,H]:
\mathcal F_\infty(\eta;n_0)\le1
\right\}$ is empty, we set
$\epsilon^{\rm lim}(n_0)=H$, corresponding to the trivial
suboptimality bound.

\begin{theorem}[Nearly reaching the local quadratic certification threshold]
\label{thm:support-floor-stopping}
With probability at least $1-\delta$, if the algorithm stops at epoch $k$
because $\widehat\Phi_{k,\infty}^{\star}\ge1/32$, then
\[
\epsilon_k
\le
\epsilon^{\mathrm{lim}}\!\left(
\frac{n_0}
{H|\mathcal S|\,\operatorname{polylog}
 (H,|\mathcal S|,|\mathcal A|,1/(\epsilon\vee\Delta_{\min}),1/\delta)}
\right).
\]
\end{theorem}

Thus, the algorithm attains a terminal resolution no worse than the local-quadratic certification threshold evaluated at
the reduced effective offline sample size.

\noindent\textbf{Collaborator-query complexity.} 
We define the finite-sample counterpart of the local quadratic characteristic time $T_{LQ,\kappa}$ as \[
\begin{aligned}
T_{\mathrm{LQ},n_0}(\mathbf{\mathcal M})
:=
\min_{\bar\pi\in\Pi_\epsilon(\mathbf{\mathcal M})}
\inf_{\boldsymbol\lambda\in\Delta_{H|\mathcal S|}}
\inf\Bigg\{
t\ge0:
\max_{\pi\in\Pi_d\setminus\{\bar\pi\}}
\frac{2}{b(\pi,\bar\pi)^2}
\sum_{h,s,a}
\frac{
\Gamma_{hsa}(\pi,\bar\pi)
}{
n_0 d_{hsa}^{\mu}
+
tq_{hsa}^{\boldsymbol\lambda}
}
\le1
\Bigg\}.
\end{aligned}
\]
  Next we show that there is an upper bound of the collaborator-query complexity
that nearly matches it.
\begin{theorem}[Instance-dependent collaborator-query complexity]
\label{thm:query-complexity}
  With probability at least
$1-\delta$, the total number $N_q$ of collaborator queries satisfies
\[
    N_q
\le
\widetilde O
\left(
T_{{\rm LQ},n_0/L_\epsilon}(\mathcal M)
\right)
+
\operatorname{poly}
\left(
H,|\mathcal S|,|\mathcal A|,(p_{\min}^c)^{-1}
\right)
\widetilde O
\left(
\frac1{\epsilon\vee\Delta_{\min}}
\right),
\]  where
\[
L_\epsilon
=
\operatorname{polylog}
\left(
H,
|\mathcal S|,
|\mathcal A|,
\frac1\delta,
\frac{H}{\epsilon\vee\Delta_{\min}},
\left|\Pi_{3\epsilon_{k_0}}(\mathcal M)\right|
\right),
\qquad
k_0:=\left\lceil\sqrt{K_\star}\right\rceil,
\]
with
\[
K_\star
:=
\left\lceil
\log_2
\frac{6H}{\epsilon\vee\Delta_{\min}}
\right\rceil+1.
\]
\end{theorem}

\noindent\textbf{Worst-case vs.\ instance-dependent complexity.}
We finally illustrate the distinction between the worst-case statistical
difficulty of the problem and the instance-dependent complexity captured by
$T_{\rm LQ}$. For simplicity, consider the setting with no offline data and
a full-support collaborator.

\begin{theorem}[Worst-case and instance-dependent characteristic times]
\label{thm:instance-dependent-separation}

Suppose first that $|\mathcal S|\ge6$, $|\mathcal A|\ge2$, and that there exists an integer
$d\ge1$ such that
\[
|\mathcal S|
=
3+\frac{|\mathcal A|^d-1}{|\mathcal A|-1}.
\]
Assume moreover that $3\mid H$ and
\[
H
\ge
3\bigl(d\vee4\sqrt{2}\epsilon\bigr).
\]
Then, for every
\[
p_{\min}^c\in(0,1/|\mathcal A|],
\]
there exists a finite-horizon tabular MDP
$\mathbf{\mathcal M}_\epsilon$ with $|\mathcal S|$ states, $|\mathcal A|$ actions, and
horizon $H$, together with a full-support collaborator policy satisfying
\[
\min_{h,s,a}
\pi_h^c(a\mid s)
=
p_{\min}^c,
\]
such that
\[
T_{\epsilon,0}^\star
(\mathbf{\mathcal M}_\epsilon)
\gtrsim
\frac{
H^3|\mathcal S|
}{
p_{\min}^c\epsilon^2
}.
\]

On the other hand, let $\mathbf{\mathcal M}$ be any finite-horizon
tabular MDP and suppose that the collaborator has full support with
\[
\min_{h,s,a}
\pi_h^c(a\mid s)
\ge
p_{\min}^c.
\]
Then
\[
T_{\rm LQ,0}(\mathbf{\mathcal M})
\lesssim
\frac{
H^3|\mathcal S|
}{
p_{\min}^c
(\epsilon+\Delta_{\min})^2
}
\lesssim
\frac{
H^3|\mathcal S|
}{
p_{\min}^c\epsilon^2
}.
\]
In particular, when
$\Delta_{\min}\gg\epsilon$, the local-quadratic characteristic time can
be much smaller than the worst-case $\epsilon^{-2}$ rate.
\end{theorem}
Theorem~\ref{thm:instance-dependent-separation} shows that the
local-quadratic characteristic time is controlled by the worst-case
$H^3|\mathcal S|/(p_{\min}^c\epsilon^2)$ scaling, 
while automatically adapting to instances in which the optimal policy
is well separated in value from other policies. 
\section{Experiments}\label{sec:simulations}

First, we compare lower-bound-guided querying (\emph{adaptive}) with \emph{uniform} and \emph{coverage} allocation in three synthetic simulations. All three run the same algorithm on the same offline sample and stabilization batches, differing only in the query allocation. Each environment (Figure~\ref{fig:sim-environments}) has a few hard decisions: Simulation~1 is an eight-context bandit whose smallest gap lies in its most frequent context; Simulation~2 is an $H=5$ MDP with two close calls at different stages, $y$ and the noisier $p$; and Simulation~3 is an absorbing MDP whose decision at the root has a value gap that grows with the horizon $H$. The behavior policy rarely takes the optimal action at these decisions, while the collaborator takes it with probability $0.6$. We set $\varepsilon=0.01$ ($0.02$ in Simulation~2); every suboptimal policy is more than $\varepsilon$ from optimal. We present $n_0=500$ offline trajectories, $\delta=0.05$ and 300 paired trials (full simulation details and other $n_0$ in Appendix~\ref{app:simulations}).

\begin{figure}[!htbp]
    \centering
    \includegraphics[
    width=0.8\linewidth,
    height=0.4\textheight,
    keepaspectratio
]{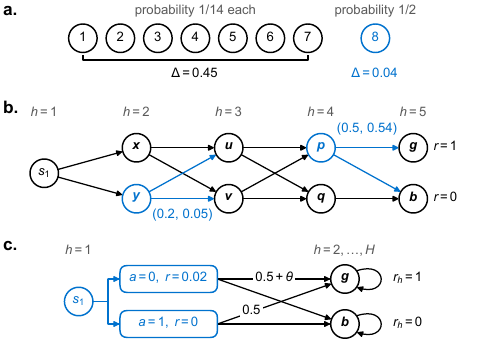}
    \caption{Simulation setups. \textbf{a.}~Contextual bandit with reward gap $\Delta$. \textbf{b.}~Sequential decisions ($H=5$) with transition probabilities under $(a=0,a=1)$. \textbf{c.}~Long-term consequences (varying $H$) where the value gap $0.02+(H-1)\theta$ grows with the horizon. Blue marks the hard decisions.}
    \label{fig:sim-environments}
\end{figure}

We run two versions of the algorithm (full details in Appendix~\ref{app:simulations}). The \emph{theory-aligned} version follows \Cref{sec:algorithm}. The \emph{efficient} version pools all data instead of using fresh folds, eliminates any candidate beaten by more than the design's confidence radius, and only tops up the stabilization batch each epoch, since earlier batches remain in use.

\noindent\textbf{Simulation Results.} In all three simulations, adaptive allocation does what the lower bound suggests and spends most of its queries on the difficult decisions (Figure~\ref{fig:sim-results}c). Between $59\%$ and $80\%$ of the theory-aligned version's queries go to Context~8, the close calls $p$ and $y$, and the root of Simulation~3, against only $3$-$22\%$ under uniform allocation. It therefore certifies far sooner (Figure~\ref{fig:sim-results}b). The efficient version stops after about $10^{4}$ queries, while uniform and coverage allocation need $5$-$22$ times as many. The gap is widest in Simulation~3, where uniform allocation spends nearly all its queries on absorbing states that carry no decision. Adaptive allocation also recovers the optimal policy earlier (Figure~\ref{fig:sim-results}a): after a few thousand queries the efficient version returns it in $96$-$98\%$ of trials, against $77$-$86\%$ for the baselines. The theory-aligned version shows the same pattern, with peak recovery gains of $10$-$21$ points, but at $16$-$44$ times the efficient version's cost. No stopping run returned a policy that is not $\varepsilon$-optimal.

\noindent\textbf{Drug-response application.} We next apply the method to the PRISM repurposing screen \citep{corsello2020discovering}. Contexts are six cancer lineages, actions are dasatinib and selumetinib ($2.5\,\mu$M), and the reward is 1 when a cell line's viability falls to at most half of control. A query chooses a lineage; the collaborator picks either drug with probability $1/2$ and reveals the response of a random line from a 116-line benchmark held out from the lines used to choose the drugs. The initial log has 116 records that all reveal one drug (dasatinib in half the trials, selumetinib in the rest), so it cannot compare the drugs. A policy assigns a drug to each lineage (64 policies), and its value is the lineage-weighted response rate. The weights act as the context distribution and are either equal or each lineage's share of new cases of these six cancers worldwide in 2022 \citep{bray2024global} (lung $42\%$, colorectal $33\%$, the other four $5$-$9\%$, with melanoma for skin). Dasatinib is better in four lineages, and selumetinib is narrowly better in the two close calls, colorectal ($7/14$ vs.\ $8/14$ responders) and skin ($10/18$ vs.\ $12/18$). With $\varepsilon=0.01$, only colorectal matters under incidence weights. We use the simulation constants, $\delta=0.05$ and 400 paired trials (Appendix~\ref{app:prism}).

Under both weightings (Figure~\ref{fig:prism-main}), the efficient version of adaptive allocation certifies an $\varepsilon$-optimal policy after about $6\times10^{3}$ queries, roughly a third of what uniform or coverage allocation needs, and no stopping run returned a policy that is not $\varepsilon$-optimal. At fixed budgets, after about $800$ queries under incidence weights it returns an $\varepsilon$-optimal policy in $89\%$ of trials, against about $80\%$ for both baselines, and under equal weights it pulls ahead of both by about $3\times10^{3}$ queries. Rather than spreading queries evenly, adaptive allocation concentrates them on the close calls that decide the policy (Figure~\ref{fig:prism-main}c). Over half go to colorectal under incidence weights, and most go to skin, colorectal and pancreas under equal weights. Lineages where one drug clearly wins receive few queries, even lung, which carries $42\%$ of the incidence weight. The theory-aligned version shows the same ordering at about $25$ times the cost; its recovery gains are smaller, though clear under incidence weights, and it still needs far fewer queries than the baselines.

\noindent
\begin{minipage}[t]{0.49\textwidth}
    \vspace{0pt}
    \centering
    \includegraphics[
        width=\linewidth,
        keepaspectratio
    ]{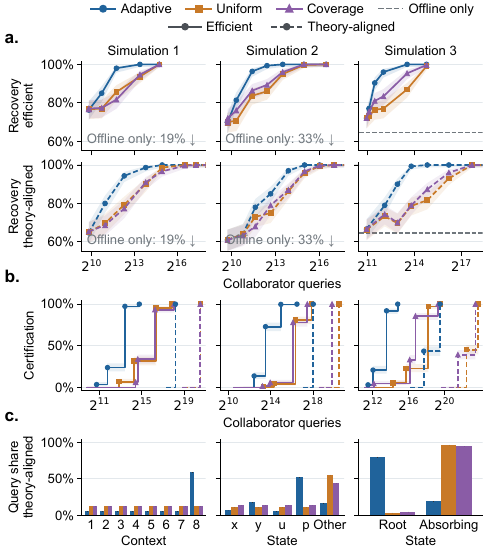}
    \captionof{figure}{
        Simulations~1--3 at $n_0=500$, 300 trials, and $H=15$
        for Simulation~3 for the efficient implementation (solid)
        and the theory-aligned algorithm of
        \Cref{sec:algorithm} (dashed).
        \textbf{a.}~Fraction of trials returning the optimal
        policy under fixed budget collaborator queries.
        \textbf{b.}~Fraction of trials whose stopping rule
        has returned, each method under its own allocation.
        \textbf{c.}~Allocation of the theory-aligned queries.
    }
    \label{fig:sim-results}
\end{minipage}
\hfill
\begin{minipage}[t]{0.49\textwidth}
    \vspace{0pt}
    \centering
    \includegraphics[
        width=\linewidth,
        keepaspectratio
    ]{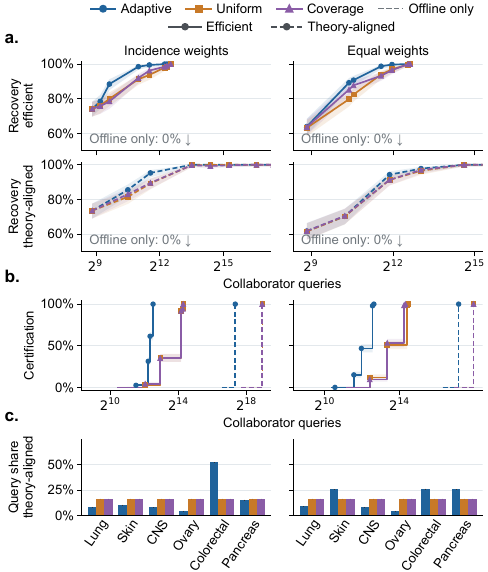}
    \captionof{figure}{
        PRISM application over 400 trials under incidence
        (left) and equal (right) lineage weights.
        \textbf{a.}~Fraction of trials returning an
        $\varepsilon$-optimal policy ($\varepsilon=0.01$)
        under fixed budgets.
        \textbf{b.}~Fraction of trials whose stopping
        rule has returned.
        \textbf{c.}~Each method's share of the
        theory-aligned version's queries by lineage.
    }
    \label{fig:prism-main}
\end{minipage}

\par\vspace{0.15in}

\section{Summary}
\label{sec:summary}

We introduced query-based collaborative reinforcement learning, a novel setting in
which a learner augments a fixed offline dataset by choosing state--time
queries to a collaborator-mediated simulator, while the sampled action is
determined by a fixed and possibly suboptimal collaborator policy. We derived
an instance-dependent information-theoretic lower bound and a local quadratic
characterization. It motivates an
adaptive algorithm that progressively localizes the policy class and allocates
queries through a lower-bound-guided minimax design. Our analysis establishes
finite stopping and high-probability accuracy guarantees and shows that, in
the certifiable regime, the leading query-complexity term nearly tracks the local
quadratic benchmark.

Our experiments illustrate the resulting instance dependence. Across three
synthetic environments and a semi-synthetic drug-response application,
adaptive allocation concentrates queries on the policy comparisons that
remain difficult after incorporating the offline data, rather than simply on
the least-covered state--action pairs. This substantially reduces the number
of collaborator queries needed for certification relative to uniform and
coverage-based allocation. 

The current theory focuses on finite tabular MDPs, a fixed collaborator,
and a theory-aligned implementation with fresh epoch-wise data splitting.
Natural directions for future work include establishing sharper guarantees
for data-reusing and efficient implementations, developing more computationally scalable
algorithms, extending the framework to function approximation, and studying
fixed-budget settings with heterogeneous query costs or state-time-specific
budget constraints.

\bibliographystyle{plainnat}
\bibliography{refs2}
\appendix
\newpage
\section{Full Algorithm and Auxiliary Definitions}
\label{app:full_algorithm}

This appendix additionally provides the complete version of our algorithm described in Section~\ref{sec:algorithm} and defines several auxiliary quantities and steps that were left implicit in the main text.

\subsection{Unknown collaborator policy}
\label{sec:unknown-collaborator}

For clarity, our main theoretical development assumes that the collaborator
policy $\pi^c$ is known, in which case we simply set
$\widehat{\pi}^c=\pi^c$ throughout the section. Importantly, this assumption
is not essential. In practice, $\pi^c$ may be unknown to the learner, but
its support and the quantities needed by our allocation rule can be learned
with a lightweight uniform-query initialization phase with negligible query
overhead.

Suppose that a known constant $p_{\min}>0$ satisfies
\[
\min_{(h,s,a):\,\pi_h^c(a\mid s)>0}
\pi_h^c(a\mid s)
\ge p_{\min}>0.
\]

Fix $\zeta\in(0,1)$ and $\delta_0\in(0,1)$. Before running the main
algorithm, query every state--time pair $(h,s)$ an equal number
$m_0$ of times, and let
\[
\widetilde\pi_h^c(a\mid s)
:=
\frac{1}{m_0}
\sum_{i=1}^{m_0}
\mathbf 1\{A_i^{h,s}=a\}
\]
be the resulting empirical collaborator distribution, where
$A_1^{h,s},\ldots,A_{m_0}^{h,s}$ are the actions returned by the
collaborator from the $m_0$ independent queries at $(h,s)$. Define
\[
\eta:=\frac{\zeta}{2+\zeta}.
\]

By a multiplicative Chernoff bound and a union bound over
$(h,s,a)$, it suffices to take
\[
m_0
=
\left\lceil
\frac{C}{\zeta^2p_{\min}}
\log\frac{2H|\mathcal S||\mathcal A|}{\delta_0}
\right\rceil,
\]
for a sufficiently large universal  constant $C>0$,
to ensure, with probability at least $1-\delta_0$, that simultaneously
for every $(h,s,a)$ with $\pi_h^c(a\mid s)>0$,
\[
(1-\eta)\pi_h^c(a\mid s)
\le
\widetilde\pi_h^c(a\mid s)
\le
(1+\eta)\pi_h^c(a\mid s).
\]

Since unsupported actions are never sampled, on the same event
the empirical unsupported set exactly recovers the true one:
\[
\widehat{\mathcal U}_c
:=
\{(h,s,a):\widetilde\pi_h^c(a\mid s)=0\}
=
\mathcal U_c
:=
\{(h,s,a):\pi_h^c(a\mid s)=0\}.
\]

For the subsequent algorithm, define the multiplicative upper estimate
\[
\widehat\pi_h^{\,\mathrm{up}}(a\mid s)
:=
\frac{\widetilde\pi_h^c(a\mid s)}{1-\eta},
\]
and the conservative proxy
\[
\widehat\pi_h^c(a\mid s)
:=
\frac{1}{1+\zeta}
\widehat\pi_h^{\,\mathrm{up}}(a\mid s).
\]
Our choice of $\eta$ gives
\[
\frac{1+\eta}{1-\eta}=1+\zeta,
\]
and therefore, simultaneously over all supported rows,
\[
\frac{1}{1+\zeta}\pi_h^c(a\mid s)
\le
\widehat\pi_h^c(a\mid s)
\le
\pi_h^c(a\mid s).
\]
In general, $\widehat\pi_h^c(\cdot\mid s)$ need not sum to one and should
be viewed as a conservative, generally sub-probability proxy for the
collaborator action probabilities, rather than as a policy distribution.
This causes no difficulty because it is used only in the query-allocation
and effective-count calculations, where a coordinatewise lower proxy is
sufficient.

The cost of this initialization is only
\[
N_{\mathrm{init}}
=
H|\mathcal S|m_0
=
O\left(
\frac{H|\mathcal S|}{\zeta^2p_{\min}}
\log\frac{H|\mathcal S||\mathcal A|}{\delta_0}
\right)
\]
collaborator queries, independent of the target accuracy $\epsilon$.
Moreover, replacing $\pi^c$ by $\widehat\pi^c$ incurs at most a
multiplicative $1+\zeta$ overhead in the subsequent query complexity.
Indeed,
\[
(1+\zeta)
\widehat\pi_h^{c}(a\mid s)
\ge
\pi_h^c(a\mid s), 
\]
so multiplying any allocation designed with the true collaborator policy
by $1+\zeta$ provides at least the same effective number of samples on
every $(h,s,a)$.

Consequently, if the main algorithm is run with confidence
parameter $\delta$, then the corresponding guarantees
derived under known $\pi^c$ continue to hold under unknown
$\pi^c$ with probability at least $1-\delta-\delta_0$,
with the following modifications to the query-complexity
bound: an additive initialization cost $N_{\mathrm{init}}$,
a multiplicative $1+\zeta$ factor on the leading
adaptive-query term, and replacement of $p_{\min}^c$ by
the known lower bound $p_{\min}$ in the stabilization
and other lower-order query terms.
In particular, to retain an overall confidence level
of $1-\delta$, one may run the initialization with
failure probability $\delta/2$ and the main algorithm
with confidence parameter $\delta/2$.

Thus exact knowledge of the collaborator action
probabilities is not required, provided that a known
positive lower bound $p_{\min}\le p_{\min}^c$ is available.
In the algorithmic description in this section, we retain
$\widehat{\pi}^c$ and $p_{\min}$.
For the known-policy theoretical analysis, we set
$\widehat{\pi}^c=\pi^c$ and $p_{\min}=p_{\min}^c$;
the unknown-policy guarantees follow with the
initialization cost and parameter substitutions
described above.
\subsection{Complete procedure}
\label{app:complete-procedure}

The procedure begins with a one-time preprocessing step applied to the offline
data. It then enters the epoch-wise iteration $k=1,2,\ldots$, where each epoch
consists of four stages: fresh-data acquisition, policy localization,
historical sensitivity estimation, and lower-bound-guided prospective query
design. All notation is summarized in
Table~\ref{tab:algorithm-notation}.

\paragraph{Offline preprocessing.}

We first preprocess the offline trajectories before any adaptive querying
takes place. If $n_0=0$, skip the  step and set all offline
subsamples and their row counts to be empty and zero,
respectively. Otherwise, apply the two-fold trimming procedure
\citep[Step~2 of Algorithm~3]{gen2024settling} with
\[
\delta_{\rm trim}
=
\frac{\delta}{48}.
\]
The retained observations are equally split into two disjoint parts,
$\mathcal D_0^\mathcal G$ and $\mathcal D_0^\mathcal D$.

We further partition
$\mathcal D_0^\mathcal D$ into disjoint fresh folds,
\[
\mathcal D_0^\mathcal D
=
\bigsqcup_{k\ge1}\mathcal D_{0,k}^{\mathcal D},
\qquad
w_k
=
\frac{6}{\pi^2k^2},
\]
where $w_k$ specifies the proportion of $\mathcal D_0^\mathcal D$ assigned to
the $k$th fresh fold. The row counts
\[
N_{0,k,hsa}^{\mathcal D}
:=
\#\left\{
\text{row-$(h,s,a)$ observations in }
\mathcal D_{0,k}^{\mathcal D}
\right\}
\]
are recorded at preprocessing time, but the corresponding reward and
next-state observations are not revealed to the inference procedure until
epoch $k$. Thus the prospective design at epoch $k-1$ may use the row counts
of the next offline fold without reusing its reward--transition observations
for estimation.

Similarly, let
\[
N_{0,hsa}^{\mathcal G}
:=
\#\left\{
\text{row-$(h,s,a)$ observations in }\mathcal D_0^\mathcal G
\right\}.
\]

The active policy class and the initial adaptive allocation are initialized as
\[
\widehat\Pi_0
:=
\Pi_d,
\qquad
\widehat{\mathbf n}_0
:=
\mathbf 0.
\]

\paragraph{Stage 1: Fresh data acquisition and stabilization.}

At epoch $k$, set
\[
\epsilon_k
:=
H2^{-(k-1)}
\]
and let
\[
q_{k-1}
:=
|\widehat\Pi_{k-1}|^2,
\]
which determines the epoch-dependent confidence parameter $\beta_k$ through
Table~\ref{tab:algorithm-notation}.

To guarantee sufficient accuracy of the historical sensitivity estimates,
for $\eta>0$ define
\begin{equation}
n^{\rm sen}(\eta)
:=
\min
\left\{
n\ge2:
\frac{
C_{\rm sen}^2\beta^\mathcal G(m)
}{
2m
}
\le
\eta
\quad
\text{for all }m\ge n
\right\},
\label{eq:app-sensitivity-threshold}
\end{equation}
and set
\begin{equation}
n_{k+1}^{\rm sen}
:=
n^{\rm sen}(\epsilon_{k+1})
=
n^{\rm sen}\!\left(
\frac{\epsilon_k}{2}
\right).
\label{eq:app-epoch-sensitivity-threshold}
\end{equation}
Together with the uniform-concentration threshold
$n_k^{\rm uc}$ defined in
Table~\ref{tab:algorithm-notation}, this determines the stabilization batch
size
\begin{equation}
m_k^{\rm st}
:=
\left\lceil
\frac{1}{p_{\min}}
\max
\left\{
2n_k^{\rm uc},
\,
2n_{k+1}^{\rm sen},
\,
8\beta_k,
\,
50w_k\ell_{\rm trim}+2
\right\}
\right\rceil.
\label{eq:app-stabilization-size}
\end{equation}

The stabilization sample can increase only rows that are supported by the
collaborator. Let $\mathcal U_c$ denote the set of collaborator-unsupported
rows. Such rows must therefore be supplied entirely by the offline data.
For each row define the cumulative historical offline count
\begin{equation}
N_{k,hsa}^{\mathcal G,\rm off}
:=
N_{0,hsa}^{\mathcal G}
+
\sum_{j=1}^k
N_{0,j,hsa}^{\mathcal D}.
\label{eq:app-cumulative-offline-historical-count}
\end{equation}
All counts on the right-hand side are recorded during preprocessing and are
therefore known before the corresponding observations are used for inference.

We define the offline-coverage indicator
\begin{equation}
\mathsf{OffCover}_k
:=
\mathbf 1
\left\{
\begin{array}{l}
N_{0,k,hsa}^{\mathcal D}
\ge
n_k^{\text{uc}}
\vee
\left(
50w_k\ell_{\text{trim}}+2
\right),
\\[1mm]
N_{k,hsa}^{\mathcal G,\text{off}}
\ge
n_{k+1}^{\text{sen}},
\end{array}
\quad
\text{for every }(h,s,a)\in\mathcal U_c
\right\}.
\label{eq:app-offline-coverage}
\end{equation}
When $\mathcal U_c=\varnothing$, we set
$\mathsf{OffCover}_k=1$.

\begin{remark}\label{rem:lower-order-coverage}
By construction, the coverage thresholds entering
$\mathsf{OffCover}_k$ are lower order relative to the traditional leading
 scale $O(1/\epsilon^2)$ \citep{gen2024settling}. Indeed,
\[
n_k^{\rm uc}
=
\widetilde O(\epsilon_k^{-1}),
\qquad
n_{k+1}^{\rm sen}
=
\widetilde O(\epsilon_k^{-1}).
\]
Thus the condition $\mathsf{OffCover}_k=1$ enforces only the auxiliary
coverage needed for the fresh-sample and sensitivity-control steps of
our construction.
\end{remark}

Before entering the first epoch, the algorithm evaluates
$\mathsf{OffCover}_1$ using the row counts recorded during preprocessing.
If
\[
\mathsf{OffCover}_1=0,
\]
the algorithm terminates and returns an arbitrary policy in $\Pi_d$.
For every $k\ge2$, the condition $\mathsf{OffCover}_k=1$ has already
been verified at the end of epoch $k-1$, before any epoch-$k$
observations are used for inference. Consequently, every epoch that is
actually entered satisfies
\begin{equation}
\mathsf{OffCover}_k=1.
\label{eq:app-current-offline-coverage}
\end{equation}

The algorithm then acquires three sources of fresh data. First, reveal the
offline observations in $\mathcal D_{0,k}^{\mathcal D}$. Second, for $k\ge2$,
execute the adaptive allocation
\[
\widehat{\mathbf n}_{k-1}
=
\{
\widehat n_{k-1,hs}
\}_{h,s}
\]
selected at the end of epoch $k-1$: for every $(h,s)$, the collaborator is
queried exactly $\widehat n_{k-1,hs}$ times, producing
$\mathcal Q_k^{\rm ad}$. At the first epoch there is no previous adaptive
allocation, so
\[
\mathcal Q_1^{\rm ad}
:=
\varnothing.
\]
Third, independently of the adaptive allocation, query every state--time
pair $(h,s)$ exactly $m_k^{\rm st}$ times and denote the resulting
stabilization observations by $\mathcal Q_k^{\rm st}$.

The fresh inference sample at epoch $k$ is therefore
\begin{equation}
\mathcal D_k^{\mathcal D}
:=
\mathcal D_{0,k}^{\mathcal D}
\cup
\mathcal Q_k^{\rm ad}
\cup
\mathcal Q_k^{\rm st},
\label{eq:app-fresh-inference-sample}
\end{equation}
with row counts
\begin{equation}
N_{k,hsa}^{\mathcal D}
:=
\#\left\{
\text{row-$(h,s,a)$ observations in }\mathcal D_k^{\mathcal D}
\right\}.
\label{eq:app-fresh-row-count}
\end{equation}

\paragraph{Stage 2: Fresh inference and localization of the policy class.}

Construct the fresh empirical MDP
\[
\widehat M_k^\mathcal D
=
(\widehat r_k^\mathcal D,\widehat P_k^\mathcal D)
\]
from $\mathcal D_k^{\mathcal D}$. For every deterministic policy $\pi$, let
$\widehat V_h^{\pi,k,\mathcal D}$ denote its value function in
$\widehat M_k^\mathcal D$. For a policy pair $(\pi,\bar\pi)$ define the
empirical value difference
\begin{equation}
\widehat D_k(\pi,\bar\pi)
:=
\widehat V_1^{\pi,k,\mathcal D}
-
\widehat V_1^{\bar\pi,k,\mathcal D}.
\label{eq:app-value-difference}
\end{equation}

For $k=1$, define
\begin{equation}
\widehat\Pi_1
:=
\Pi_d.
\label{eq:app-initial-active-class}
\end{equation}
For every $k\ge2$, the reference policy
$\bar \pi_{k-1}$ was selected by the minimax design at the end of
epoch $k-1$. The new active class is obtained by retaining policies whose
empirical value is within the prescribed localization radius of the best
empirical response to this reference:
\begin{equation}
\widehat\Pi_k
:=
\left\{
\pi\in\widehat\Pi_{k-1}:
\max_{\pi'\in\widehat\Pi_{k-1}}
\left[
\widehat D_k(
\pi',
\bar \pi_{k-1}
)
-
\widehat D_k(
\pi,
\bar \pi_{k-1}
)
\right]
\le
\frac52\epsilon_k
\right\}.
\label{eq:app-localized-class}
\end{equation}

After the active class has been updated, define the reported policy as
follows. For $k=1$, let $\widehat\pi_1$ be an arbitrary policy in
$\widehat\Pi_1=\Pi_d$. For every $k\ge2$, define
\begin{equation}
\widehat\pi_k
\in
\arg\max_{\pi\in\widehat\Pi_k}
\widehat V_1^{\pi,k,\mathcal D}.
\label{eq:app-empirical-best-policy}
\end{equation}

If
\[
|\widehat\Pi_k|=1,
\]
the unique remaining policy is returned and the algorithm terminates.
Likewise, if
\[
\epsilon_k\le\epsilon,
\]
the active class has reached the target localization scale and the algorithm
returns $\widehat\pi_k$. Otherwise, proceed to Stage 3.

\paragraph{Stage 3: Historical sensitivity and local-variance estimation.}

If the algorithm has not terminated after Stage~2, release the current fresh
observations into the historical sample and define
\begin{equation}
\mathcal Z_k^\mathcal G
:=
\mathcal D_0^\mathcal G
\cup
\bigcup_{j=1}^k
\mathcal D_j^{\mathcal D}.
\label{eq:app-historical-sample}
\end{equation}
For each row let
\begin{equation}
N_{k,hsa}^\mathcal G
:=
\#\left\{
\text{row-$(h,s,a)$ observations in }\mathcal Z_k^\mathcal G
\right\},
\label{eq:app-historical-row-count}
\end{equation}
and construct the historical empirical MDP
\[
\widehat M_k^\mathcal G
=
(\widehat r_k^\mathcal G,\widehat P_k^\mathcal G).
\]
For a deterministic policy $\pi$, let
$\widehat V_h^{\pi,k,\mathcal G}$ and
$\widehat d_{hsa}^{\pi,k,\mathcal G}$ denote its value function and
state--action occupancy under $\widehat M_k^\mathcal G$.

For every pair
$\pi,\bar\pi\in\widehat\Pi_k$, define the empirical reward sensitivity
\begin{equation}
\widehat G_{k,hsa}^{r}(\pi,\bar\pi)
:=
\widehat d_{hsa}^{\pi,k,\mathcal G}
-
\widehat d_{hsa}^{\bar\pi,k,\mathcal G},
\label{eq:app-empirical-sensitivities}
\end{equation}
and the empirical transition sensitivity
\begin{equation}
\begin{aligned}
\widehat G_{k,hsa}^{P}(\pi,\bar\pi)(s')
:={}&
\widehat d_{hsa}^{\pi,k,\mathcal G}
\widehat V_{h+1}^{\pi,k,\mathcal G}(s')
\\
&-
\widehat d_{hsa}^{\bar\pi,k,\mathcal G}
\widehat V_{h+1}^{\bar\pi,k,\mathcal G}(s').
\end{aligned}
\label{eq:app-transition-sensitivity}
\end{equation}
The associated plug-in influence score is
\begin{equation}
\widehat\Psi_{k,hsa}^{\pi,\bar\pi}(r,s')
:=
\widehat G_{k,hsa}^{r}(\pi,\bar\pi)\,r
+
\widehat G_{k,hsa}^{P}(\pi,\bar\pi)(s').
\label{eq:app-empirical-influence}
\end{equation}

To control the discrepancy between the empirical and population influence
scores, for $N_{k,hsa}^\mathcal G>0$ define
\begin{equation}
\alpha_{k,hsa}^{\mathcal G,r}
=
\alpha_{k,hsa}^{\mathcal G,P}
:=
\min
\left\{
1,\,
\sqrt{
\frac{
\beta^\mathcal G(N_{k,hsa}^\mathcal G)
}{
2N_{k,hsa}^\mathcal G
}
}
\right\}.
\label{eq:app-historical-row-radius}
\end{equation}
Let
\[
\alpha_k^{\mathcal G,r}
:=
\max_{h,s,a}\alpha_{k,hsa}^{\mathcal G,r},
\qquad
\alpha_k^{\mathcal G,P}
:=
\max_{h,s,a}\alpha_{k,hsa}^{\mathcal G,P}.
\]
Set $E_{k,H+1}^{V,\mathcal G}:=0$ and, for $h\le H$, define
\begin{equation}
E_{k,h}^{V,\mathcal G}
:=
(H-h+1)\alpha_k^{\mathcal G,r}
+
\frac{(H-h)(H-h+1)}{2}
|\mathcal S|\alpha_k^{\mathcal G,P},
\label{eq:app-value-sensitivity-envelope}
\end{equation}
and
\begin{equation}
\Delta_{k,hsa}^{\mathcal G}
:=
2(h-1)(H-h+1)|\mathcal S|\alpha_k^{\mathcal G,P}
+
2E_{k,h+1}^{V,\mathcal G}.
\label{eq:app-sensitivity-envelope}
\end{equation}

For unsupported rows,
\[
N_{k,hsa}^\mathcal G
=
N_{k,hsa}^{\mathcal G,\rm off},
\]
since collaborator querying can never generate observations from such rows.
By \eqref{eq:app-current-offline-coverage},
\[
N_{k,hsa}^\mathcal G
\ge
n_{k+1}^{\rm sen},
\qquad
\forall (h,s,a)\in\mathcal U_c.
\]

The sensitivity envelope above is used to construct an upper confidence
envelope for the local variance
$\Gamma_{hsa}(\pi,\bar\pi)$.
For $N_{k,hsa}^{\mathcal D}\ge2$, first define the fresh-sample mean of the
plug-in score,
\begin{equation}
\overline{\widehat\Psi}_{k,hsa}^{\pi,\bar\pi}
:=
\frac{1}{N_{k,hsa}^{\mathcal D}}
\sum_{i=1}^{N_{k,hsa}^{\mathcal D}}
\widehat\Psi_{k,hsa}^{\pi,\bar\pi}
\left(
R_{k,hsa}^{(i)},
S_{k,hsa}^{\prime(i)}
\right),
\label{eq:app-plugin-score-mean}
\end{equation}
and its empirical variance
\begin{equation}
\begin{aligned}
\widehat\Gamma_{k,hsa}^{\rm raw}(\pi,\bar\pi)
:=
\frac{1}{N_{k,hsa}^{\mathcal D}-1}
\sum_{i=1}^{N_{k,hsa}^{\mathcal D}}
\Bigl(
&
\widehat\Psi_{k,hsa}^{\pi,\bar\pi}
(
R_{k,hsa}^{(i)},
S_{k,hsa}^{\prime(i)}
)
-
\overline{\widehat\Psi}_{k,hsa}^{\pi,\bar\pi}
\Bigr)^2.
\end{aligned}
\label{eq:app-raw-variance}
\end{equation}
The local-variance envelope is then
\begin{equation}
\overline\Gamma_{k,hsa}(\pi,\bar\pi)
:=
\left[
\sqrt{
\widehat\Gamma_{k,hsa}^{\rm raw}(\pi,\bar\pi)
}
+
2H
\sqrt{
\frac{
2\beta_k
}{
N_{k,hsa}^{\mathcal D}-1
}
}
+
\sqrt{2}\Delta_{k,hsa}^{\mathcal G}
\right]^2
\label{eq:app-variance-envelope}
\end{equation}
whenever $N_{k,hsa}^{\mathcal D}\ge2$. If
$N_{k,hsa}^{\mathcal D}\le1$, set instead
\begin{equation}
\overline\Gamma_{k,hsa}(\pi,\bar\pi)
:=
H^2.
\label{eq:app-small-sample-variance-envelope}
\end{equation}

\paragraph{Stage 4: Lower-bound-guided prospective minimax design.}

The final stage uses the localized policy class
$\widehat\Pi_k$ and the local-variance envelopes
$\overline\Gamma_{k,hsa}$ to determine both a reference policy and an
adaptive collaborator-query allocation for epoch $k+1$.

For every
$\pi,\bar\pi\in\widehat\Pi_k$, define the regularized empirical margin
\begin{equation}
\widehat b_{k,\rm reg}(\pi,\bar\pi)
:=
\max
\left\{
\epsilon-\widehat D_k(\pi,\bar\pi),
\,
4\epsilon_k
\right\}.
\label{eq:app-regularized-margin}
\end{equation}

Set
\[
\epsilon_{k+1}
=
\frac{\epsilon_k}{2},
\qquad
q_k
=
|\widehat\Pi_k|^2,
\]
and compute the next-epoch quantities
\[
\beta_{k+1},
\qquad
n_{k+1}^{\rm uc},
\qquad
n_{k+2}^{\rm sen},
\qquad
m_{k+1}^{\rm st}
\]
from Table~\ref{tab:algorithm-notation},
\eqref{eq:app-epoch-sensitivity-threshold}, and
\eqref{eq:app-stabilization-size}.

Because all offline row counts are known from preprocessing,
$\mathsf{OffCover}_{k+1}$ can be evaluated before any epoch-$(k+1)$
observation is revealed. If
\[
\mathsf{OffCover}_{k+1}=0,
\]
then some collaborator-unsupported row cannot meet the fresh or historical
coverage requirement at the next epoch, and collaborator querying cannot
repair this deficiency. The algorithm therefore terminates and returns
$\widehat\pi_k$.

Hence every epoch that reaches the prospective minimax design satisfies
\begin{equation}
\mathsf{OffCover}_{k+1}=1.
\label{eq:app-next-offline-coverage}
\end{equation}

For a prospective adaptive allocation
\[
\mathbf n
=
\{n_{hs}\}_{h,s}
\in
\mathbb N^{H|\mathcal S|},
\]
the expected row count available for fresh inference at epoch $k+1$ is
\begin{equation}
N_{k+1,hsa}^{+}(\mathbf n)
:=
N_{0,k+1,hsa}^{\mathcal D}
+
\left(
m_{k+1}^{\rm st}
+
n_{hs}
\right)
\widehat\pi_h^c(a\mid s).
\label{eq:app-prospective-count}
\end{equation}
Here $N_{0,k+1,hsa}^{\mathcal D}$ is known from offline preprocessing even
though the corresponding observations are not revealed until epoch $k+1$.

Define the prospective pairwise radius
\begin{equation}
\widehat{\operatorname{rad}}_k^{+}
(\pi,\bar\pi;\mathbf n)
:=
C_0
\sqrt{
\beta_{k+1}
\sum_{h,s,a}
\frac{
\overline\Gamma_{k,hsa}(\pi,\bar\pi)
}{
N_{k+1,hsa}^{+}(\mathbf n)
}
},
\label{eq:app-prospective-radius}
\end{equation}
where we set $C_0=8$.

For a total adaptive-query budget $B\ge0$, define the feasible allocation set
\begin{equation}
\mathcal N(B)
:=
\left\{
\mathbf n
=
\{n_{hs}\}_{h,s}
\in
\mathbb N^{H|\mathcal S|}:
\sum_{h,s}n_{hs}
\le
B
\right\}.
\label{eq:app-budget-set}
\end{equation}
The lower-bound-guided design jointly selects a reference policy and an
allocation by solving
\begin{equation}
\widehat\Phi_k^\star(B)
:=
\min_{\bar\pi\in\widehat\Pi_k}
\inf_{\mathbf n\in\mathcal N(B)}
\max_{\pi\in\widehat\Pi_k\setminus\{\bar\pi\}}
\frac{
\widehat{\operatorname{rad}}_k^{+}
(\pi,\bar\pi;\mathbf n)
}{
\widehat b_{k,\rm reg}(\pi,\bar\pi)
}.
\label{eq:app-minimax-design}
\end{equation}
Thus the allocation is chosen to minimize the largest normalized uncertainty
over policies that remain statistically plausible after localization.

Because $\mathcal N(B_1)\subseteq\mathcal N(B_2)$ whenever
$B_1\le B_2$, the function
$\widehat\Phi_k^\star(B)$ is nonincreasing in $B$. Its infinite-budget
limit is
\begin{equation}
\widehat\Phi_{k,\infty}^\star
:=
\lim_{B\to\infty}
\widehat\Phi_k^\star(B)
=
\inf_{B<\infty}
\widehat\Phi_k^\star(B).
\label{eq:app-infinite-budget-limit}
\end{equation}
If
\[
\widehat\Phi_{k,\infty}^\star
\ge
\frac1{32},
\]
then the infinite-budget normalized uncertainty remains above the prescribed
certification-limit threshold, and the algorithm terminates and returns
$\widehat\pi_k$.

Otherwise,
\[
\widehat\Phi_{k,\infty}^\star
<
\frac1{32}
<
\frac1{16},
\]
so, by the definition of the infimum and monotonicity of
$\widehat\Phi_k^\star(B)$, some finite query budget attains the design
threshold $1/16$. Search over the dyadic grid
\[
\mathcal B_{\rm dyad}
:=
\{1,2,4,8,\ldots\}
\]
and choose
\begin{equation}
B_k
:=
\min
\left\{
B\in\mathcal B_{\rm dyad}:
\widehat\Phi_k^\star(B)
\le
\frac1{16}
\right\}.
\label{eq:app-dyadic-budget}
\end{equation}
Finally, select a corresponding minimax optimizer
\begin{equation}
(
\bar \pi_k,
\widehat{\mathbf n}_k
)
\in
\arg\min_{\substack{
\bar\pi\in\widehat\Pi_k\\
\mathbf n\in\mathcal N(B_k)
}}
\max_{\pi\in\widehat\Pi_k\setminus\{\bar\pi\}}
\frac{
\widehat{\operatorname{rad}}_k^{+}
(\pi,\bar\pi;\mathbf n)
}{
\widehat b_{k,\rm reg}(\pi,\bar\pi)
}.
\label{eq:app-design-optimizer}
\end{equation}
The reference policy $\bar \pi_k$ is used in the next localization
step, while the allocation
$\widehat{\mathbf n}_k$ is stored and executed only at epoch $k+1$.

\FloatBarrier

\begingroup
\footnotesize
\setlength{\tabcolsep}{4pt}
\renewcommand{\arraystretch}{1.10}
\setlength{\abovedisplayskip}{3pt}
\setlength{\belowdisplayskip}{3pt}
\setlength{\abovedisplayshortskip}{2pt}
\setlength{\belowdisplayshortskip}{2pt}
\setlength{\LTcapwidth}{\textwidth}
\setlength{\LTpre}{6pt}
\setlength{\LTpost}{6pt}

\begin{longtable}{
    @{}
    >{\raggedright\arraybackslash}p{0.22\textwidth}
    >{\raggedright\arraybackslash}
        p{\dimexpr0.78\textwidth-2\tabcolsep\relax}
    @{}
}

\caption{Notation and constants used in
Algorithm~\ref{alg:adaptive-collaborator-query-full}.}
\label{tab:algorithm-notation}\\

\toprule
Quantity & Definition and role \\
\midrule
\endfirsthead

\caption[]{Notation and constants used in
Algorithm~\ref{alg:adaptive-collaborator-query-full} (continued).}\\

\toprule
Quantity & Definition and role \\
\midrule
\endhead

\midrule
\multicolumn{2}{r}{\footnotesize Continued on next page.}\\
\endfoot

\bottomrule
\endlastfoot

$\delta_{\rm trim}$
&
Failure probability used in the two-fold trimming step:
$\delta_{\rm trim}:=\delta/48$.
\\

$\ell_{\rm trim}$
&
Logarithmic factor used in the trimming bound:
$\ell_{\rm trim}:=\log\!\left(
\frac{(n_0\vee1)H}{\delta_{\rm trim}}
\right)$.
\\

$w_k$
&
Fraction of the inference-side offline sample assigned to fold $k$:
$w_k:=6/(\pi^2k^2)$, with $\sum_{k\ge1}w_k=1$.
\\

$\mathcal D_0^\mathcal G,\mathcal D_{0,k}^{\mathcal D}$
&
Historical part of the trimmed offline sample and the $k$th fresh inference
fold. The complete table of row counts $N_{0,k,hsa}^{\mathcal D}$ is recorded
during preprocessing, while the corresponding reward--transition
observations are not used for inference until epoch $k$.
\\

$\mathcal U_c$
&
Collaborator-unsupported rows:
$\mathcal U_c:=\{(h,s,a):\pi_h^c(a\mid s)=0\}$.
\\

$p_{\min}$
&
A known lower bound on the minimum positive collaborator-action probability:
$0<p_{\min}\le
\min_{(h,s,a)\notin\mathcal U_c}\pi_h^c(a\mid s)$.
\\

$\epsilon_k$
&
Epoch resolution: $\epsilon_k:=H2^{-(k-1)}$.
\\

$\mathcal P_k$
&
Predictable parent pair class used for epoch-$k$ concentration:
$\mathcal P_k:=\widehat\Pi_{k-1}\times\widehat\Pi_{k-1}$.
It is measurable before the fresh epoch-$k$ observations are revealed.
\\

$q_k$
&
Pair-count bound: $q_k:=|\widehat\Pi_k|^2$.
Thus $|\mathcal P_k|\le q_{k-1}$.
\\

$\beta_k$
&
Epoch-wise logarithmic factor:
\[
\beta_k:=
\log\!\left(
\frac{2\pi^2k^2H|\mathcal S||\mathcal A|(|\mathcal S|+1)q_{k-1}}{\delta}
\right).
\]
\\

$\beta^\mathcal G(n)$
&
Anytime logarithmic factor for historical sensitivity estimation:
\[
\beta^\mathcal G(n):=
\log\!\left(
\frac{2\pi^2n^2H|\mathcal S||\mathcal A|(|\mathcal S|+1)}{\delta}
\right).
\]
\\

$C_{\rm sen}$
&
Sensitivity-envelope constant:
$C_{\rm sen}:=|\mathcal S|H^2\vee4|\mathcal S|$.
\\

$K_H$
&
Coefficient appearing in the second-order model remainder:
\[
K_H:=
H(H-1)|\mathcal S|
+
\frac{H(H-1)(H-2)}{3}|\mathcal S|^2.
\]
\\

$C_{\rm uc}$
&
Uniform-coverage constant:
\[
C_{\rm uc}:=
\max\left\{
64H^2,\,
8K_H,\,
67584H|\mathcal S||\mathcal A|
\right\}.
\]
\\

$C_0$
&
Numerical constant in the leading prospective radius:
$C_0:=8$.
\\

$n^{\rm sen}(\eta)$
&
Historical row count sufficient to guarantee that the squared sensitivity
envelope is at most $\eta$; defined in
\eqref{eq:app-sensitivity-threshold}.
\\

$n_k^{\rm uc}$
&
Fresh row-wise coverage threshold:
$n_k^{\rm uc}:=
\left\lceil
C_{\rm uc}\frac{\beta_k}{\epsilon_k}
\right\rceil$.
\\

$m_k^{\rm st}$
&
Number of uniform stabilization queries made at every state--time pair
during epoch $k$; defined in \eqref{eq:app-stabilization-size}.
\\

$\mathcal Q_k^{\rm ad},\mathcal Q_k^{\rm st}$
&
Adaptive and uniform-stabilization query batches collected at epoch $k$.
For $k=1$, $\mathcal Q_1^{\rm ad}:=\varnothing$.
\\

$\mathcal D_k^{\mathcal D}$
&
Fresh inference sample at epoch $k$:
$\mathcal D_k^{\mathcal D}:=
\mathcal D_{0,k}^{\mathcal D}
\cup\mathcal Q_k^{\rm ad}
\cup\mathcal Q_k^{\rm st}$.
\\

$N_{k,hsa}^{\mathcal D}$
&
Number of row-$(h,s,a)$ observations in $\mathcal D_k^{\mathcal D}$.
\\

$\widehat M_k^\mathcal D$
&
Empirical MDP constructed only from the fresh inference sample
$\mathcal D_k^{\mathcal D}$.
\\

$\widehat D_k(\pi,\bar\pi)$
&
Fresh empirical value difference:
$\widehat D_k(\pi,\bar\pi):=
\widehat V_1^{\pi,k,\mathcal D}
-
\widehat V_1^{\bar\pi,k,\mathcal D}$.
\\

$\widehat\Pi_k$
&
Localized active policy class. We initialize
$\widehat\Pi_0=\Pi_d$ and set $\widehat\Pi_1=\Pi_d$; for $k\ge2$ it is
defined in \eqref{eq:app-localized-class}.
\\

$\widehat\pi_k$
&
Reported policy: $\widehat\pi_1$ is an arbitrary policy in $\Pi_d$;
for $k\ge2$,
\[
\widehat\pi_k
\in
\arg\max_{\pi\in\widehat\Pi_k}
\widehat V_1^{\pi,k,\mathcal D}.
\]
For $k\ge2$, this is equivalently the empirical best response
to any fixed reference policy.
\\

$\mathcal Z_k^\mathcal G$
&
Cumulative historical sample available after epoch-$k$ inference:
$\mathcal Z_k^\mathcal G:=
\mathcal D_0^\mathcal G
\cup
\bigcup_{j=1}^k\mathcal D_j^{\mathcal D}$.
It is used for sensitivity estimation and subsequent query design.
\\

$N_{k,hsa}^\mathcal G$
&
Number of row-$(h,s,a)$ observations in $\mathcal Z_k^\mathcal G$.
\\

$\widehat M_k^\mathcal G$
&
Historical empirical MDP constructed from $\mathcal Z_k^\mathcal G$.
\\

$\widehat G_k^r,\widehat G_k^P,\widehat\Psi_k$
&
Empirical reward sensitivity, transition sensitivity, and plug-in influence
score defined in
\eqref{eq:app-empirical-sensitivities}--%
\eqref{eq:app-empirical-influence}.
\\

$\Delta_{k,hsa}^\mathcal G$
&
Uniform sensitivity error envelope satisfying
\[
\sup_{r,s'}
\left|
\widehat\Psi_{k,hsa}^{\pi,\bar\pi}(r,s')
-
\Psi_{hsa}^{\pi,\bar\pi}(r,s')
\right|
\le
\Delta_{k,hsa}^\mathcal G
\]
on the historical sensitivity event.
\\

$\widehat\Gamma_{k,hsa}^{\rm raw}$
&
Sample variance of the plug-in influence score $\widehat\Psi_k$
evaluated on the current fresh sample $\mathcal D_k^{\mathcal D}$.
The score and the sample need not be independent.
\\

$\overline\Gamma_{k,hsa}$
&
Reuse-safe upper confidence envelope for the population local variance
$\Gamma_{hsa}$; see \eqref{eq:app-variance-envelope}.
\\

$\mathsf{OffCover}_k$
&
Indicator that every collaborator-unsupported row has sufficient fresh
epoch-$k$ offline coverage for inference and sufficient cumulative offline
coverage for historical sensitivity estimation; see
\eqref{eq:app-offline-coverage}.
\\

$\widehat b_{k,\mathrm{reg}}$
&
Regularized empirical certification margin:
\[
\widehat b_{k,\mathrm{reg}}(\pi,\bar\pi):=
\max\left\{
\epsilon-\widehat D_k(\pi,\bar\pi),
4\epsilon_k
\right\}.
\]
\\

$N_{k+1,hsa}^{+}(\boldsymbol n)$
&
Predicted row count at epoch $k+1$.
\\

$\widehat{\operatorname{rad}}_k^+$
&
Leading prospective radius defined in
\eqref{eq:app-prospective-radius}.
\\

$\mathcal N(B)$
&
Integer state--time allocations using at most $B$ adaptive queries:
\[
\mathcal N(B):=
\left\{
\boldsymbol n=\{n_{hs}\}_{h,s}\in\mathbb N^{H|\mathcal S|}:
\sum_{h,s}n_{hs}\le B
\right\}.
\]
\\

$\widehat\Phi_k^\star(B)$
&
Joint reference-policy/allocation design objective defined in
\eqref{eq:app-minimax-design}.
\\

$\widehat\Phi_{k,\infty}^\star$
&
Infinite-budget value of the current prospective design:
$\widehat\Phi_{k,\infty}^\star:=
\lim_{B\to\infty}\widehat\Phi_k^\star(B)$.
\\

\end{longtable}
\endgroup

\FloatBarrier
\Needspace{12\baselineskip}
\subsection{Complete pseudocode}
\label{app:complete-pseudocode}

The complete pseudocode is given in
Algorithm~\ref{alg:adaptive-collaborator-query-full}.

\par\medskip
\begingroup
\small

\setlength{\abovedisplayskip}{3pt}
\setlength{\belowdisplayskip}{3pt}
\setlength{\abovedisplayshortskip}{2pt}
\setlength{\belowdisplayshortskip}{2pt}

\captionsetup{
    font=small,
    labelfont=bf,
    labelsep=space,
    format=plain,
    justification=raggedright,
    singlelinecheck=false,
    position=top,
    skip=3pt,
    hypcap=false
}

\hrule height 0.8pt
\nobreak
\captionof{algorithm}{Lower-Bound-Guided Adaptive Collaborator Querying}
\label{alg:adaptive-collaborator-query-full}
\nobreak
\hrule height 0.4pt
\nobreak

\begin{algorithmic}[1]

\Require Offline trajectories $\mathcal D_0$, target accuracy $\epsilon$,
confidence level $\delta$, collaborator access, a known lower bound $0<p_{\min}\le p_{\min}^c$ and policy proxy
$\widehat\pi^c$ (with $\widehat\pi^c=\pi^c$ when $\pi^c$ is known).

\State If $\mathcal D_0=\varnothing$, skip the trimming step; otherwise,
apply the two-fold trimming procedure
\citep[Step~2 in Algorithm~3]{gen2024settling} with
$\delta_{\rm trim}=\delta/48$.

\State Split the retained observations into
$\mathcal D_0^\mathcal G$ and $\mathcal D_0^\mathcal D$, and partition
$\mathcal D_0^\mathcal D$ into fresh folds
$\{\mathcal D_{0,k}^{\mathcal D}\}_{k\ge1}$ using weights
$w_k=6/(\pi^2k^2)$.

\State Set
\[
\widehat\Pi_0\gets\Pi_d,
\qquad
\widehat{\mathbf n}_0\gets\mathbf 0.
\]

\For{$k=1,2,\ldots$}

    \State Set
    \[
    \epsilon_k\gets H2^{-(k-1)},
    \qquad
    q_{k-1}\gets|\widehat\Pi_{k-1}|^2,
    \]
    and compute $\beta_k$, $n_k^{\rm uc}$,
    $n_{k+1}^{\rm sen}$, and $m_k^{\rm st}$ according to
    Table~\ref{tab:algorithm-notation},
    \eqref{eq:app-epoch-sensitivity-threshold}, and
    \eqref{eq:app-stabilization-size}.

    \If{$k=1$}
        \If{$\mathsf{OffCover}_1=0$}
            \State \Return any policy in $\Pi_d$.
        \EndIf
    \EndIf

    \State Reveal the observations in
    $\mathcal D_{0,k}^{\mathcal D}$.

    \If{$k=1$}
        \State Set
        $\mathcal Q_1^{\rm ad}\gets\varnothing$.
    \Else
        \State Execute $\widehat{\mathbf n}_{k-1}$ and denote the resulting
        observations by $\mathcal Q_k^{\rm ad}$.
    \EndIf

    \State Query every state--time pair $(h,s)$ exactly
    $m_k^{\rm st}$ times and denote the resulting observations by
    $\mathcal Q_k^{\rm st}$.

    \State Form $\mathcal D_k^{\mathcal D}$ according to
    \eqref{eq:app-fresh-inference-sample}, record
    $\{N_{k,hsa}^{\mathcal D}\}_{h,s,a}$, and construct
    $\widehat M_k^\mathcal D$.

    \If{$k=1$}
        \State Set
        \[
        \widehat\Pi_1\gets\Pi_d
        \]
        and choose any
        \[
        \widehat\pi_1\in\widehat\Pi_1.
        \]
    \Else
        \State Compute
        $\widehat D_k(\cdot,\bar \pi_{k-1})$
        in $\widehat M_k^\mathcal D$.

        \State Update $\widehat\Pi_k$ according to
        \eqref{eq:app-localized-class}.

        \State Choose
        \[
        \widehat\pi_k
        \in
        \arg\max_{\pi\in\widehat\Pi_k}
        \widehat V_1^{\pi,k,\mathcal D}.
        \]
    \EndIf

    \If{$|\widehat\Pi_k|=1$}
        \State \Return the unique policy in $\widehat\Pi_k$.
    \EndIf

    \If{$\epsilon_k\le\epsilon$}
        \State \Return $\widehat\pi_k$.
    \EndIf

    \State Release $\mathcal D_k^{\mathcal D}$ into the historical sample and
    form $\mathcal Z_k^\mathcal G$ according to
    \eqref{eq:app-historical-sample}.

    \State Construct $\widehat M_k^\mathcal G$ and compute
    $\widehat G_k^r$, $\widehat G_k^P$, $\widehat\Psi_k$,
    and $\Delta_k^\mathcal G$ according to
    \eqref{eq:app-empirical-sensitivities}--%
    \eqref{eq:app-sensitivity-envelope}.

    \State Compute
    $\widehat\Gamma_{k,hsa}^{\rm raw}(\pi,\bar\pi)$ and
    $\overline\Gamma_{k,hsa}(\pi,\bar\pi)$
    for all $\pi,\bar\pi\in\widehat\Pi_k$ using
    \eqref{eq:app-raw-variance}--%
    \eqref{eq:app-small-sample-variance-envelope}.

    \State Compute
    $\widehat b_{k,\rm reg}$ according to
    \eqref{eq:app-regularized-margin}.

    \State Set
    \[
    \epsilon_{k+1}\gets\frac{\epsilon_k}{2},
    \qquad
    q_k\gets|\widehat\Pi_k|^2,
    \]
    and compute $\beta_{k+1}$, $n_{k+1}^{\rm uc}$,
    $n_{k+2}^{\rm sen}$, and $m_{k+1}^{\rm st}$ according to
    Table~\ref{tab:algorithm-notation},
    \eqref{eq:app-epoch-sensitivity-threshold}, and
    \eqref{eq:app-stabilization-size}.

    \If{$\mathsf{OffCover}_{k+1}=0$}
        \State \Return $\widehat\pi_k$.
    \EndIf

    \State Form the prospective counts
    $N_{k+1,hsa}^{+}(\mathbf n)$ and prospective radii
    $\widehat{\operatorname{rad}}_k^+
    (\pi,\bar\pi;\mathbf n)$ according to
    \eqref{eq:app-prospective-count} and
    \eqref{eq:app-prospective-radius}.

    \State Compute
    $\widehat\Phi_{k,\infty}^\star$
    according to
    \eqref{eq:app-infinite-budget-limit}.

    \If{$\widehat\Phi_{k,\infty}^\star\ge1/32$}
        \State \Return $\widehat\pi_k$.
    \EndIf

    \State Choose the smallest dyadic $B_k$ satisfying
    \eqref{eq:app-dyadic-budget} and compute
    $(\bar \pi_k,\widehat{\boldsymbol n}_k)$ according to
    \eqref{eq:app-design-optimizer}.

    \State Store
    $\bar \pi_k$ and $\widehat{\boldsymbol n}_k$.

\EndFor

\end{algorithmic}

\nobreak
\hrule height 0.8pt
\endgroup
\par\medskip

\section{Proofs of Main Results}

\subsection{Proof of Theorem~\ref{thm:information-lower-bound}}

\begin{proof}

We first consider $T_{\epsilon,\kappa}^\star(\mathbf{\mathcal M})\in (0,+\infty)$.
Fix
\[
\bar\pi\in\Pi_\epsilon(\mathbf{\mathcal M}),
\]
and define
\[
\begin{aligned}
T_{\bar\pi}
:=
\inf_{\boldsymbol \lambda\in\Delta_{H|\mathcal S|}}
\inf\Bigg\{
t\ge0:\;
\inf_{\widetilde{\mathbf{\mathcal M}}
\in\mathcal A_\epsilon(\bar\pi)}
\Big[
\kappa
J_\mu(
\mathbf{\mathcal M},
\widetilde{\mathbf{\mathcal M}}
)
+
t
I_{\boldsymbol \lambda}(
\mathbf{\mathcal M},
\widetilde{\mathbf{\mathcal M}}
)
\Big]
\ge1
\Bigg\}.
\end{aligned}
\]
Then
\[
T_{\epsilon,\kappa}^\star(\mathbf{\mathcal M})
=
\min_{\bar\pi\in\Pi_\epsilon(\mathbf{\mathcal M})}
T_{\bar\pi}.
\]

We first show that restricting the alternatives to models
with finite row-wise KL divergences does not change
$T_{\bar\pi}$. Define
\[
\mathcal A_\epsilon^{\rm fin}(\bar\pi)
:=
\left\{
\widetilde{\mathcal M}\in\mathcal A_\epsilon(\bar\pi):
\mathrm{KL}_{hsa}(\mathcal M,\widetilde{\mathcal M})<\infty
\text{ for every }(h,s,a)
\right\}.
\]
Fix any $\lambda\in\Delta_{H|\mathcal S|}$ and finite $t\ge0$,
and write
\[
c_{hsa}:=\kappa d_{hsa}^{\mu}+tq_{hsa}^{\lambda},
\qquad
C_{\lambda,t}(\widetilde{\mathcal M})
:=
\sum_{h,s,a}
c_{hsa}\,
\mathrm{KL}_{hsa}(\mathcal M,\widetilde{\mathcal M}).
\]
Thus
\[
C_{\lambda,t}(\widetilde{\mathcal M})
=
\kappa J_\mu(\mathcal M,\widetilde{\mathcal M})
+
tI_\lambda(\mathcal M,\widetilde{\mathcal M}).
\]
We use the conventions $0\cdot(+\infty)=0$ and
$\inf\varnothing=+\infty$.

Let
\[
v:=
\inf_{\widetilde{\mathcal M}\in\mathcal A_\epsilon(\bar\pi)}
C_{\lambda,t}(\widetilde{\mathcal M}),
\qquad
v_{\rm fin}:=
\inf_{\widetilde{\mathcal M}\in
\mathcal A_\epsilon^{\rm fin}(\bar\pi)}
C_{\lambda,t}(\widetilde{\mathcal M}).
\]
The set inclusion gives $v\le v_{\rm fin}$.
If $v=+\infty$, equality is immediate.
Otherwise, for any $\xi>0$, choose
$\widetilde{\mathcal M}\in\mathcal A_\epsilon(\bar\pi)$
such that
\[
C_{\lambda,t}(\widetilde{\mathcal M})\le v+\xi<\infty.
\]

For each row $x=(h,s,a)$, let
$\mathbf R_x,\mathbf P_x$ and
$\widetilde{\mathbf R}_x,\widetilde{\mathbf P}_x$
denote the reward and next-state marginal laws under
$\mathcal M$ and $\widetilde{\mathcal M}$, respectively.
For $\alpha\in(0,1)$, construct
$\widetilde{\mathcal M}_\alpha$ by replacing these
alternative marginals with
\[
\widetilde{\mathbf R}_{x,\alpha}
:=
(1-\alpha)\widetilde{\mathbf R}_x+\alpha\mathbf R_x,
\qquad
\widetilde{\mathbf P}_{x,\alpha}
:=
(1-\alpha)\widetilde{\mathbf P}_x+\alpha \mathbf P_x.
\]
 Then
\[
\mathrm{KL}_x
(\mathcal M,\widetilde{\mathcal M}_\alpha)
\le2\log(1/\alpha)<\infty
\qquad\text{for every }x.
\]

As $\alpha\downarrow0$, the reward means and transition
probabilities of $\widetilde{\mathcal M}_\alpha$ converge
to those of $\widetilde{\mathcal M}$.
By continuity of finite-horizon policy values and the
strict inequality
\[
V_1^\star(\widetilde{\mathcal M})
-
V_1^{\bar\pi}(\widetilde{\mathcal M})
>\epsilon,
\]
we have
$\widetilde{\mathcal M}_\alpha
\in\mathcal A_\epsilon^{\rm fin}(\bar\pi)$
for all sufficiently small $\alpha>0$.

Moreover, finiteness of
$C_{\lambda,t}(\widetilde{\mathcal M})$ implies that
$\mathrm{KL}_x(\mathcal M,\widetilde{\mathcal M})<\infty$
whenever $c_x>0$. By the construction of $\widetilde M_\alpha$,
\[
\mathrm{KL}_x
(\mathcal M,\widetilde{\mathcal M}_\alpha)
\le
\mathrm{KL}_x(\mathcal M,\widetilde{\mathcal M})
-
2\log(1-\alpha).
\]
Since $\sum_x c_x=\kappa H+t$, it follows that, for all
sufficiently small $\alpha>0$,
\[
v_{\rm fin}
\le
C_{\lambda,t}(\widetilde{\mathcal M}_\alpha)
\le
v+\xi-2(\kappa H+t)\log(1-\alpha).
\]
Letting $\alpha\downarrow0$ and then $\xi\downarrow0$
yields $v_{\rm fin}\le v$, proving equality.

Thus the two inner infima agree for every finite $t\ge0$.
Their threshold-feasible sets are therefore identical, so
\begin{align*}
&\inf\left\{
t\ge0:
\inf_{\widetilde{\mathcal M}\in\mathcal A_\epsilon(\bar\pi)}
\left[
\kappa J_\mu(\mathcal M,\widetilde{\mathcal M})
+
tI_\lambda(\mathcal M,\widetilde{\mathcal M})
\right]\ge1
\right\}
\\
&\quad =
\inf\left\{
t\ge0:
\inf_{\widetilde{\mathcal M}\in
\mathcal A_\epsilon^{\rm fin}(\bar\pi)}
\left[
\kappa J_\mu(\mathcal M,\widetilde{\mathcal M})
+
tI_\lambda(\mathcal M,\widetilde{\mathcal M})
\right]\ge1
\right\}.
\end{align*}
Taking the infimum over $\lambda$ shows that
$T_{\bar\pi}$ is unchanged by this restriction.
Since $\bar\pi\in\Pi_\epsilon(\mathcal M)$ was arbitrary,
the same holds for $T_{\epsilon,\kappa}^{\star}(\mathcal M)$.

In the remainder of this proof, all alternative models
are taken from $\mathcal A_\epsilon^{\rm fin}(\bar\pi)$;
we omit the superscript ${\rm fin}$ for notational simplicity.

Fix $0<t<T_{\bar\pi}$.
By the definition of $T_{\bar\pi}$, for every
$\boldsymbol \lambda\in\Delta_{H|\mathcal S|}$,
\[
\inf_{\widetilde{\mathbf{\mathcal M}}
\in\mathcal A_\epsilon(\bar\pi)}
\Big[
\kappa
J_\mu(
\mathbf{\mathcal M},
\widetilde{\mathbf{\mathcal M}}
)
+
t
I_{\boldsymbol \lambda}(
\mathbf{\mathcal M},
\widetilde{\mathbf{\mathcal M}}
)
\Big]
<1.
\]
The map
\[
\boldsymbol \lambda
\mapsto
\inf_{\widetilde{\mathbf{\mathcal M}}
\in\mathcal A_\epsilon(\bar\pi)}
\Big[
\kappa
J_\mu(
\mathbf{\mathcal M},
\widetilde{\mathbf{\mathcal M}}
)
+
t
I_{\boldsymbol \lambda}(
\mathbf{\mathcal M},
\widetilde{\mathbf{\mathcal M}}
)
\Big]
\]
is upper semicontinuous, since it is the infimum of affine continuous
functions of $\boldsymbol \lambda$. Since $\Delta_{H|\mathcal S|}$ is compact, its
supremum is attained. Hence
\begin{equation}
\label{eq:proof-c-barpi-t}
c_{\bar\pi,t}
:=
\sup_{\boldsymbol \lambda\in\Delta_{H|\mathcal S|}}
\inf_{\widetilde{\mathbf{\mathcal M}}
\in\mathcal A_\epsilon(\bar\pi)}
\Big[
\kappa
J_\mu(
\mathbf{\mathcal M},
\widetilde{\mathbf{\mathcal M}}
)
+
t
I_{\boldsymbol \lambda}(
\mathbf{\mathcal M},
\widetilde{\mathbf{\mathcal M}}
)
\Big]
<1.
\end{equation}

For every alternative model
$\widetilde{\mathbf{\mathcal M}}
\in\mathcal A_\epsilon(\bar\pi)$, define
\[
I_{hs}^{c}
(
\mathbf{\mathcal M},
\widetilde{\mathbf{\mathcal M}}
)
:=
\sum_{a\in\mathcal A}
\pi_h^c(a\mid s)
\operatorname{KL}
\left(
\nu_{hsa}^{\mathbf{\mathcal M}}
\,\middle\|\,
\nu_{hsa}^{\widetilde{\mathbf{\mathcal M}}}
\right).
\]
Then
\[
I_{\boldsymbol \lambda}(
\mathbf{\mathcal M},
\widetilde{\mathbf{\mathcal M}}
)
=
\sum_{h,s}
\lambda_{hs}
I_{hs}^{c}
(
\mathbf{\mathcal M},
\widetilde{\mathbf{\mathcal M}}
).
\]

Then
define the information vector
\[
d(\widetilde{\mathbf{\mathcal M}})
:=
\left(
J_\mu(
\mathbf{\mathcal M},
\widetilde{\mathbf{\mathcal M}}
),
\left\{
I_{hs}^{c}
(
\mathbf{\mathcal M},
\widetilde{\mathbf{\mathcal M}}
)
\right\}_{h,s}
\right)
\in
\mathbb R_+^{1+H|\mathcal S|}.
\]
For
\[
d=
\left(
d_0,\{d_{hs}\}_{h,s}
\right)
\in
\mathbb R_+^{1+H|\mathcal S|},
\]
let
\[
F_t(\boldsymbol \lambda,d)
:=
\kappa d_0
+
t\sum_{h,s}\lambda_{hs}d_{hs}.
\]
By construction,
\[
F_t
\left(
\boldsymbol \lambda,
d(\widetilde{\mathbf{\mathcal M}})
\right)
=
\kappa
J_\mu(
\mathbf{\mathcal M},
\widetilde{\mathbf{\mathcal M}}
)
+
t
I_{\boldsymbol \lambda}(
\mathbf{\mathcal M},
\widetilde{\mathbf{\mathcal M}}
).
\]

Let
\[
\mathcal D_{\bar\pi}
:=
\left\{
d(\widetilde{\mathbf{\mathcal M}}):
\widetilde{\mathbf{\mathcal M}}
\in\mathcal A_\epsilon(\bar\pi)
\right\},
\]
and
\[
\mathcal C_{\bar\pi}
:=
\overline{
\operatorname{conv}
(
\mathcal D_{\bar\pi}
)
}.
\]
Since $F_t(\boldsymbol \lambda,\cdot)$ is continuous and linear, for every fixed
$\boldsymbol \lambda\in\Delta_{H|\mathcal S|}$,
\[
\inf_{d\in\mathcal C_{\bar\pi}}
F_t(\boldsymbol \lambda,d)
=
\inf_{d\in\mathcal D_{\bar\pi}}
F_t(\boldsymbol \lambda,d).
\]
Therefore,
\[
c_{\bar\pi,t}
=
\sup_{\boldsymbol \lambda\in\Delta_{H|\mathcal S|}}
\inf_{d\in\mathcal C_{\bar\pi}}
F_t(\boldsymbol \lambda,d).
\]

The set $\Delta_{H|\mathcal S|}$ is compact and convex,
$\mathcal C_{\bar\pi}$ is convex, and $F_t$ is continuous and affine in
each argument. Sion's minimax theorem therefore gives
\begin{equation}
\label{eq:proof-sion}
c_{\bar\pi,t}
=
\inf_{d\in\mathcal C_{\bar\pi}}
\sup_{\boldsymbol \lambda\in\Delta_{H|\mathcal S|}}
F_t(\boldsymbol \lambda,d).
\end{equation}
Choose constants
\[
c_{\bar\pi,t}
<
c_0<c_1<c_2<1.
\]
By \eqref{eq:proof-sion}, there exists
$d^\star\in\mathcal C_{\bar\pi}$ such that
\[
\sup_{\boldsymbol \lambda\in\Delta_{H|\mathcal S|}}
F_t(\boldsymbol \lambda,d^\star)
<
c_0.
\]
Moreover,
\[
\sup_{\boldsymbol \lambda\in\Delta_{H|\mathcal S|}}
F_t(\boldsymbol \lambda,d)
=
\kappa d_0+t\max_{h,s}d_{hs},
\]
so this supremum is continuous in $d$.

Since
\[
\mathcal C_{\bar\pi}
=
\overline{\operatorname{conv}(\mathcal D_{\bar\pi})},
\]
continuity then
implies that there exists
\[
\bar d\in\operatorname{conv}(\mathcal D_{\bar\pi})
\]
such that
\[
\sup_{\boldsymbol \lambda\in\Delta_{H|\mathcal S|}}
F_t(\boldsymbol \lambda,\bar d)
<
c_1.
\]

By Carath\'eodory's theorem, $\bar d$ can be represented as a convex
combination of at most
\(
H|\mathcal S|+2
\)
elements of $\mathcal D_{\bar\pi}$. Hence there exist
\(
m\le H|\mathcal S|+2
\)
alternative models
\[
\widetilde{\mathbf{\mathcal M}}_1,\ldots,
\widetilde{\mathbf{\mathcal M}}_m
\in
\mathcal A_\epsilon(\bar\pi),
\]
and weights
\[
\alpha_1,\ldots,\alpha_m>0,
\qquad
\sum_{i=1}^m\alpha_i=1,
\]
such that
\[
\bar d
=
\sum_{i=1}^m
\alpha_i
d(\widetilde{\mathbf{\mathcal M}}_i).
\]
Consequently, for every
$\boldsymbol \lambda\in\Delta_{H|\mathcal S|}$,
\begin{equation}
\label{eq:uniform-control-mixed-information}
\sum_{i=1}^m
\alpha_i
\Big[
\kappa
J_\mu(
\mathbf{\mathcal M},
\widetilde{\mathbf{\mathcal M}}_i
)
+
t
I_{\boldsymbol \lambda}(
\mathbf{\mathcal M},
\widetilde{\mathbf{\mathcal M}}_i
)
\Big]
<
c_1
<
1.
\end{equation}

Set
\(
n_\delta
:=
\left\lfloor
t\log\left(\frac1\delta\right)
\right\rfloor.
\)
Consider the history consisting of the offline dataset, the learner's
internal randomization, and all collaborator-query observations up to
query $n_\delta\wedge\tau_\delta$. 

For each $i\in[m]$, let
\[
\Lambda_{i,\delta}
:=
\log
\frac{
d\mathbb P_{\mathbf{\mathcal M}}
}{
d\mathbb P_{\widetilde{\mathbf{\mathcal M}}_i}
}
\]
denote the log-likelihood ratio of this truncated history, and define
the mixture law
\[
\mathbb Q_\delta
:=
\sum_{i=1}^m
\alpha_i
\mathbb P_{\widetilde{\mathbf{\mathcal M}}_i}.
\]
Since
\[
\frac{
d\mathbb Q_\delta
}{
d\mathbb P_{\mathbf{\mathcal M}}
}
=
\sum_{i=1}^m
\alpha_i e^{-\Lambda_{i,\delta}},
\]
 we have
\begin{equation}
\label{eq:mixture-likelihood-upper}
\log
\frac{
d\mathbb P_{\mathbf{\mathcal M}}
}{
d\mathbb Q_\delta
}
\le
\sum_{i=1}^m
\alpha_i\Lambda_{i,\delta}.
\end{equation}

For brevity, write
\[
J_i
:=
J_\mu(
\mathbf{\mathcal M},
\widetilde{\mathbf{\mathcal M}}_i
),
\qquad
I_{i,hs}^{c}
:=
I_{hs}^{c}
(
\mathbf{\mathcal M},
\widetilde{\mathbf{\mathcal M}}_i
).
\]

Let
\[
Z_1^0,\ldots,Z_{n_{0,\delta}}^0
\]
denote the independent offline trajectories generated by the behavior
policy $\mu$. For the $j$th trajectory define
\[
\ell_{j,i}^0
:=
\log
\frac{
d\mathbb P_{\mathbf{\mathcal M}}^\mu
}{
d\mathbb P_{\widetilde{\mathbf{\mathcal M}}_i}^\mu
}
(Z_j^0).
\]
By the chain rule for KL divergence and the definition of $J_\mu$,
\[
\mathbb E_{\mathbf{\mathcal M}}
[
\ell_{j,i}^0
]
=
J_i.
\]
Thus
\[
\sum_{j=1}^{n_{0,\delta}}
\ell_{j,i}^0
=
n_{0,\delta}J_i
+
M_{i,\delta}^0,
\]
where
\[
M_{i,\delta}^0
:=
\sum_{j=1}^{n_{0,\delta}}
\bigl(
\ell_{j,i}^0-J_i
\bigr)
\]
is a martingale.

We next consider the collaborator-query observations.
Let
\(
X_q=(H_q,S_q)
\)
be the state-time pair selected at the $q$th query and let
\[
Y_q=(A_q,R_q,S_q')
\]
be the corresponding observation. Let $\mathcal F_{q-1}$ contain the
offline data, the learner's internal randomization, and all information
available immediately before the $q$th query. Then both the decision to
issue the $q$th query and $X_q$ are
$\mathcal F_{q-1}$-measurable.

Conditional on $X_q=(h,s)$,
\[
A_q\sim\pi_h^c(\cdot\mid s)
\]
under both $\mathbf{\mathcal M}$ and
$\widetilde{\mathbf{\mathcal M}}_i$, so the collaborator's action
probability cancels from the likelihood ratio. Define
\[
\ell_{q,i}^{c}
:=
\mathbf 1\{q\le\tau_\delta\}
\log
\frac{
d\mathbb P_{\mathbf{\mathcal M}}
(
Y_q\mid X_q,\mathcal F_{q-1}
)
}{
d\mathbb P_{\widetilde{\mathbf{\mathcal M}}_i}
(
Y_q\mid X_q,\mathcal F_{q-1}
)
}.
\]
Then
\[
\mathbb E_{\mathbf{\mathcal M}}
\left[
\ell_{q,i}^{c}
\mid
\mathcal F_{q-1}
\right]
=
\mathbf 1\{q\le\tau_\delta\}
I_{i,H_qS_q}^{c}.
\]

Define
\[
N_{hs}^{(\delta)}
:=
\sum_{q=1}^{n_\delta}
\mathbf 1
\left\{
q\le\tau_\delta,\,
X_q=(h,s)
\right\}.
\]
Then
\[
\sum_{h,s}
N_{hs}^{(\delta)}
=
n_\delta\wedge\tau_\delta
\le
n_\delta.
\]
Summing the query-level log-likelihood increments gives
\[
\sum_{q=1}^{n_\delta}
\ell_{q,i}^{c}
=
\sum_{h,s}
N_{hs}^{(\delta)}
I_{i,hs}^{c}
+
M_{i,\delta}^{c},
\]
where
\[
M_{i,\delta}^{c}
:=
\sum_{q=1}^{n_\delta}
\left[
\ell_{q,i}^{c}
-
\mathbb E_{\mathbf{\mathcal M}}
\left[
\ell_{q,i}^{c}
\mid
\mathcal F_{q-1}
\right]
\right]
\]
is a martingale.

Combining the offline and collaborator-query parts yields
\[
\Lambda_{i,\delta}
=
n_{0,\delta}J_i
+
\sum_{h,s}
N_{hs}^{(\delta)}
I_{i,hs}^{c}
+
M_{i,\delta},
\]
where
\[
M_{i,\delta}
:=
M_{i,\delta}^0+M_{i,\delta}^{c}.
\]
Therefore,
\[
\sum_{i=1}^m
\alpha_i\Lambda_{i,\delta}
=
A_\delta+ M_\delta,
\]
where
\[
A_\delta
:=
\sum_{i=1}^m
\alpha_i
\left[
n_{0,\delta}J_i
+
\sum_{h,s}
N_{hs}^{(\delta)}
I_{i,hs}^{c}
\right]
\]
and
\[
\mathcal M_\delta
:=
\sum_{i=1}^m
\alpha_iM_{i,\delta}.
\]

We first control the predictable term $A_\delta$.
Fix an arbitrary
\[
\boldsymbol\lambda^0\in\Delta_{H|\mathcal S|}.
\]
For each realization define
\[
\lambda_{hs}^{(\delta)}
:=
\frac{
N_{hs}^{(\delta)}
}{
n_\delta
}
+
\left(
1-
\frac{
\sum_{h',s'}
N_{h's'}^{(\delta)}
}{
n_\delta
}
\right)
\lambda_{hs}^0.
\]
Then
\[
\boldsymbol \lambda^{(\delta)}
\in\Delta_{H|\mathcal S|}.
\]
Since
\[
\sum_{h,s}
N_{hs}^{(\delta)}I_{i,hs}^{c}
\le
n_\delta
\sum_{h,s}
\lambda_{hs}^{(\delta)}I_{i,hs}^{c}
=
n_\delta
I_{{\boldsymbol \lambda}^{(\delta)}}
(
\mathbf{\mathcal M},
\widetilde{\mathbf{\mathcal M}}_i
),
\]
we have
\[
\begin{aligned}
\frac{A_\delta}{\log(1/\delta)}
&\le
\sum_{i=1}^m
\alpha_i
\left[
\frac{n_{0,\delta}}{\log(1/\delta)}J_i
+
\frac{n_\delta}{\log(1/\delta)}
I_{\boldsymbol \lambda^{(\delta)}}
(
\mathbf{\mathcal M},
\widetilde{\mathbf{\mathcal M}}_i
)
\right]
\\
&=
\sum_{i=1}^m
\alpha_i
\left[
\kappa J_i
+
t
I_{\boldsymbol \lambda^{(\delta)}}
(
\mathbf{\mathcal M},
\widetilde{\mathbf{\mathcal M}}_i
)
\right]
+
R_\delta,
\end{aligned}
\]
where
\[
|R_\delta|
\le
\left|
\frac{n_{0,\delta}}{\log(1/\delta)}
-\kappa
\right|
\sum_{i=1}^m\alpha_iJ_i+
\left|
\frac{n_\delta}{\log(1/\delta)}
-t
\right|
\sum_{i=1}^m
\alpha_i
\sup_{\boldsymbol \lambda\in\Delta_{H|\mathcal S|}}
I_{\boldsymbol \lambda}
(
\mathbf{\mathcal M},
\widetilde{\mathbf{\mathcal M}}_i
).
\]
 Since
\[
\frac{n_{0,\delta}}{\log(1/\delta)}
\rightarrow\kappa,
\qquad
\frac{n_\delta}{\log(1/\delta)}
\rightarrow t,
\]
we have
\[
R_\delta\rightarrow0.
\]

Hence, by
\eqref{eq:uniform-control-mixed-information},
\[
\frac{A_\delta}{\log(1/\delta)}
\le
c_1+o(1).
\]
Thus, for all sufficiently small $\delta$,
\begin{equation}
\label{eq:upper-bound-predictable}
A_\delta
\le
\frac{c_1+c_2}{2}
\log(1/\delta).
\end{equation}

It remains to control the martingale term $\mathcal M_\delta$.
For any $P,Q$ with $\operatorname{KL}(P\|Q)<\infty$,
\[
\mathbb E_P\left|
\log\frac{dP}{dQ}
\right|
\le
\operatorname{KL}(P\|Q)+\frac{2}{e}
<\infty.
\]
Since only finitely many positive-weight alternatives and finitely many
query rows are involved, the centered log-likelihood increments forming
$\mathcal M_\delta$ are uniformly integrable. Moreover, after concatenating
the offline and collaborator-query increments, they form a
martingale-difference triangular array. Therefore, by Theorem 1 in
\citet{andrews1988laws}
\[
\frac{\mathcal M_\delta}
     {n_{0,\delta}+n_\delta}
\xrightarrow{\mathbb P_{\mathbf{\mathcal M}}}
0.
\]
Since
\[
\frac{n_{0,\delta}+n_\delta}{\log(1/\delta)}
\longrightarrow
\kappa+t,
\]
it follows that
\[
\frac{\mathcal M_\delta}{\log(1/\delta)}
\xrightarrow{\mathbb P_{\mathbf{\mathcal M}}}
0.
\]

Combining this with
\eqref{eq:mixture-likelihood-upper}
and
\eqref{eq:upper-bound-predictable},
we obtain
\begin{equation}
\label{eq:mixture-likelihood-tail}
\mathbb P_{\mathbf{\mathcal M}}
\left(
\log
\frac{
d\mathbb P_{\mathbf{\mathcal M}}
}{
d\mathbb Q_\delta
}
>
c_2\log(1/\delta)
\right)
\longrightarrow0.
\end{equation}

Now define
\[
C_{\bar\pi,\delta}
:=
\left\{
\widehat\pi=\bar\pi,\;
\tau_\delta\le n_\delta
\right\}.
\]
For every $i\in[m]$,
\[
\widetilde{\mathbf{\mathcal M}}_i
\in
\mathcal A_\epsilon(\bar\pi),
\]
so
\[
V_1^\star(\widetilde{\mathbf{\mathcal M}}_i)
-
V_1^{\bar\pi}(\widetilde{\mathbf{\mathcal M}}_i)
>
\epsilon.
\]
Thus $\bar\pi$ is not $\epsilon$-optimal under
$\widetilde{\mathbf{\mathcal M}}_i$, and the
$(\epsilon,\delta)$-PAC property gives
\[
\mathbb P_{\widetilde{\mathbf{\mathcal M}}_i}
(
C_{\bar\pi,\delta}
)
\le
\mathbb P_{\widetilde{\mathbf{\mathcal M}}_i}
(
\widehat\pi=\bar\pi
)
\le
\delta.
\]
Therefore,
\[
\mathbb Q_\delta(C_{\bar\pi,\delta})
=
\sum_{i=1}^m
\alpha_i
\mathbb P_{\widetilde{\mathbf{\mathcal M}}_i}
(
C_{\bar\pi,\delta}
)
\le
\delta.
\]

For any event $C$ and any $x>0$, by Proposition 18 in \citet{degenne2019pure},
\[
\mathbb P_{\mathbf{\mathcal M}}(C)
\le
e^x\mathbb Q_\delta(C)
+
\mathbb P_{\mathbf{\mathcal M}}
\left(
\log
\frac{
d\mathbb P_{\mathbf{\mathcal M}}
}{
d\mathbb Q_\delta
}
>x
\right).
\]
Taking
\[
x=c_2\log(1/\delta)
\]
and using \eqref{eq:mixture-likelihood-tail} yields
\begin{equation}
\label{eq:single-upper-bound-stopping-time}
\mathbb P_{\mathbf{\mathcal M}}
(
C_{\bar\pi,\delta}
)
\le
\delta^{\,1-c_2}
+
o(1)
\longrightarrow0.
\end{equation}

We now conclude the proof.
Suppose first that
\[
0<
T_{\epsilon,\kappa}^\star(\mathbf{\mathcal M})
<
\infty.
\]
Fix $\xi\in(0,1)$ and let
\[
t
=
(1-\xi)
T_{\epsilon,\kappa}^\star(\mathbf{\mathcal M}).
\]
Since
\[
T_{\epsilon,\kappa}^\star(\mathbf{\mathcal M})
=
\min_{\bar\pi\in\Pi_\epsilon(\mathbf{\mathcal M})}
T_{\bar\pi},
\]
we have
\[
t<T_{\bar\pi}
\]
for every
\(
\bar\pi\in\Pi_\epsilon(\mathbf{\mathcal M}).
\)
Therefore,
\eqref{eq:single-upper-bound-stopping-time}
holds for every such $\bar\pi$.

By the $(\epsilon,\delta)$-PAC property under
$\mathbf{\mathcal M}$,
\[
\mathbb P_{\mathbf{\mathcal M}}
\left(
\widehat\pi
\notin
\Pi_\epsilon(\mathbf{\mathcal M})
\right)
\le
\delta.
\]
Moreover, since $\tau_\delta$ is integer-valued,
\[
\tau_\delta
<
t\log(1/\delta)
\quad\Longrightarrow\quad
\tau_\delta
\le
n_\delta.
\]
Since $\Pi_d$ is finite,
\[
\begin{aligned}
\mathbb P_{\mathbf{\mathcal M}}
\left(
\tau_\delta
<
t\log(1/\delta)
\right)
&\le
\mathbb P_{\mathbf{\mathcal M}}
\left(
\widehat\pi
\notin
\Pi_\epsilon(\mathbf{\mathcal M})
\right)+
\sum_{\bar\pi\in\Pi_\epsilon(\mathbf{\mathcal M})}
\mathbb P_{\mathbf{\mathcal M}}
\left(
\widehat\pi=\bar\pi,\;
\tau_\delta\le n_\delta
\right)
\\
&\le
\delta+o(1)
\rightarrow0.
\end{aligned}
\]
It follows that
\[
\mathbb E_{\mathbf{\mathcal M}}[\tau_\delta]
\ge
t\log(1/\delta)
\,
\mathbb P_{\mathbf{\mathcal M}}
\left(
\tau_\delta
\ge
t\log(1/\delta)
\right),
\]
and therefore
\[
\liminf_{\delta\downarrow0}
\frac{
\mathbb E_{\mathbf{\mathcal M}}[\tau_\delta]
}{
\log(1/\delta)
}
\ge
t.
\]
Hence
\[
\liminf_{\delta\downarrow0}
\frac{
\mathbb E_{\mathbf{\mathcal M}}[\tau_\delta]
}{
\log(1/\delta)
}
\ge
(1-\xi)
T_{\epsilon,\kappa}^\star(\mathbf{\mathcal M}).
\]
Letting $\xi\downarrow0$ gives
\[
\liminf_{\delta\downarrow0}
\frac{
\mathbb E_{\mathbf{\mathcal M}}[\tau_\delta]
}{
\log(1/\delta)
}
\ge
T_{\epsilon,\kappa}^\star(\mathbf{\mathcal M}).
\]

If
\(
T_{\epsilon,\kappa}^\star(\mathbf{\mathcal M})=0,
\)
the desired inequality is immediate.

Finally, suppose
\(
T_{\epsilon,\kappa}^\star(\mathbf{\mathcal M})
=
+\infty.
\)
Since $\Pi_\epsilon(\mathbf{\mathcal M})$ is finite,
this implies
\[
T_{\bar\pi}=+\infty
\qquad
\forall
\bar\pi\in\Pi_\epsilon(\mathbf{\mathcal M}).
\]
Thus the preceding argument applies to every finite $t>0$, giving
\[
\liminf_{\delta\downarrow0}
\frac{
\mathbb E_{\mathbf{\mathcal M}}[\tau_\delta]
}{
\log(1/\delta)
}
\ge
t
\]
for every finite $t$. Therefore,
\[
\liminf_{\delta\downarrow0}
\frac{
\mathbb E_{\mathbf{\mathcal M}}[\tau_\delta]
}{
\log(1/\delta)
}
=
+\infty.
\]

This completes the proof.
\end{proof}

\begin{remark}
The proof of Theorem~\ref{thm:information-lower-bound} follows from the
 proof strategy of Theorem~1 in
\citet{al2023towards}.  The main difference in our setting is that
the information comes from two sources and we also retain the exact KL-divergence formulation of the
information quantities throughout the lower-bound argument.
\end{remark}

\subsection{Proof of Corollary~\ref{cor:local-quadratic-lower-bound}}

\begin{proof}
Fix
\[
\bar\pi\in\Pi_\epsilon(\mathbf{\mathcal M}),
\qquad
\boldsymbol\lambda\in\Delta_{H|\mathcal S|},
\qquad
t\ge0,
\]
and write
\[
c_{hsa}
=
\kappa d_{hsa}^{\mu}
+
tq_{hsa}^{\boldsymbol\lambda}.
\] 
By definition,
\begin{equation}
\kappa J_\mu(
\mathbf{\mathcal M},
\widetilde{\mathbf{\mathcal M}}
)
+
tI_{\boldsymbol\lambda}(
\mathbf{\mathcal M},
\widetilde{\mathbf{\mathcal M}}
)
=
\sum_{h,s,a}
c_{hsa}
\KL_{hsa}(
\mathbf{\mathcal M},
\widetilde{\mathbf{\mathcal M}}
).
\label{eq:lq-weighted-information}
\end{equation}

We start from the exact alternative problem. Since
\[
V_1^\star(\widetilde{\mathbf{\mathcal M}})
=
\max_{\pi\in\Pi_d}
V_1^\pi(\widetilde{\mathbf{\mathcal M}}),
\]
we have
\[
\mathcal A_\epsilon(\bar\pi)
=
\bigcup_{\pi\in\Pi_d\setminus\{\bar\pi\}}
\left\{
\widetilde{\mathbf{\mathcal M}}:
V_1^\pi(\widetilde{\mathbf{\mathcal M}})
-
V_1^{\bar\pi}(\widetilde{\mathbf{\mathcal M}})
>
\epsilon
\right\}.
\]
Hence, since $\Pi_d$ is finite,
\begin{equation}
\inf_{\widetilde{\mathbf{\mathcal M}}
\in\mathcal A_\epsilon(\bar\pi)}
\left[
\kappa J_\mu(
\mathbf{\mathcal M},
\widetilde{\mathbf{\mathcal M}}
)
+
tI_{\boldsymbol\lambda}(
\mathbf{\mathcal M},
\widetilde{\mathbf{\mathcal M}}
)
\right]
=
\min_{\pi\in\Pi_d\setminus\{\bar\pi\}}
\inf_{\substack{
\widetilde{\mathbf{\mathcal M}}:\\
V_1^\pi(\widetilde{\mathbf{\mathcal M}})
-
V_1^{\bar\pi}(\widetilde{\mathbf{\mathcal M}})
>
\epsilon}}
\sum_{h,s,a}
c_{hsa}
\KL_{hsa}(
\mathbf{\mathcal M},
\widetilde{\mathbf{\mathcal M}}
).
\label{eq:lq-exact-challenger-problem}
\end{equation}
It therefore suffices to derive the local approximation of the inner
problem for a fixed challenger
$\pi\in\Pi_d\setminus\{\bar\pi\}$.

For each row $(h,s,a)$, recall
\[
\nu^{\mathbf{\mathcal M}}_{hsa}
=
\mathcal L_{\mathbf{\mathcal M}}
\left(
R_h,S_{h+1}
\mid
S_h=s,A_h=a
\right).
\]

We consider bounded additive score directions
\[
g_{hsa}(R_h,S_{h+1})
=
g_{hsa}^r(R_h)+g_{hsa}^P(S_{h+1})
\]
satisfying
\begin{equation}
\mathbb E_{\mathbf{\mathcal M}}
[g_{hsa}\mid s,a]
=
0,
\qquad
\mathbb E_{\mathbf{\mathcal M}}
[g_{hsa}^2\mid s,a]
=
1.
\label{eq:lq-unit-score}
\end{equation}
The unit-variance normalization is imposed only on nonzero
perturbation directions.
An unperturbed row is represented by
$\eta_{hsa}=0$ and $g_{hsa}\equiv0$.
In particular, if a row admits no nonzero centered score,
it is left unchanged and contributes zero to both the
linearized constraint and the quadratic information cost.
All derivative and Taylor identities below for normalized
scores apply only to nonzero directions.

For such a direction and any $\eta$ in a neighborhood of zero, define
the perturbed row law explicitly by
\begin{equation}
\frac{
d\nu_{hsa}^{\eta}
}{
d\nu_{hsa}^{\mathbf{\mathcal M}}
}
=
\exp\left\{
\eta g_{hsa}
-
\psi_{hsa}(\eta)
\right\},
\qquad
\psi_{hsa}(\eta)
:=
\log
\mathbb E_{\mathbf{\mathcal M}}
\left[
e^{\eta g_{hsa}}
\mid s,a
\right].
\label{eq:lq-exponential-path}
\end{equation}
Because $g_{hsa}$ is bounded, $\psi_{hsa}(\eta)$ is finite for both
positive and negative $\eta$, so
\eqref{eq:lq-exponential-path} defines a two-sided local path through
$\nu_{hsa}^{\mathbf{\mathcal M}}$.

Importantly, every row distribution satisfying
\eqref{eq:lq-admissible-local-alternative} admits such a bounded
exponential-tilt representation. Indeed, let
\[
\ell_{hsa}
:=
\log
\frac{
d\nu_{hsa}^{\widetilde{\mathbf{\mathcal M}}}
}{
d\nu_{hsa}^{\mathbf{\mathcal M}}
}.
\]
By \eqref{eq:lq-admissible-local-alternative},
$\ell_{hsa}\in L^\infty(\nu_{hsa}^{\mathbf{\mathcal M}})$.
If
\(
\operatorname{Var}_{\mathbf{\mathcal M}}
(\ell_{hsa}\mid s,a)>0,
\)
set
\[
\eta_{hsa}
:=
\sqrt{
\operatorname{Var}_{\mathbf{\mathcal M}}
(\ell_{hsa}\mid s,a)
},
\qquad
g_{hsa}
:=
\frac{
\ell_{hsa}
-
\mathbb E_{\mathbf{\mathcal M}}
[\ell_{hsa}\mid s,a]
}{
\eta_{hsa}
}.
\]
Then $g_{hsa}$ is bounded and satisfies
\[
\mathbb E_{\mathbf{\mathcal M}}
[g_{hsa}\mid s,a]=0,
\qquad
\mathbb E_{\mathbf{\mathcal M}}
[g_{hsa}^2\mid s,a]=1.
\]
Moreover,
\[
\psi_{hsa}(\eta_{hsa})
=
\log
\mathbb E_{\mathbf{\mathcal M}}
\left[
e^{\eta_{hsa}g_{hsa}}
\mid s,a
\right]
=
-
\mathbb E_{\mathbf{\mathcal M}}
[\ell_{hsa}\mid s,a],
\]
because
\[
\mathbb E_{\mathbf{\mathcal M}}
[e^{\ell_{hsa}}\mid s,a]
=
\int
\frac{
d\nu_{hsa}^{\widetilde{\mathbf{\mathcal M}}}
}{
d\nu_{hsa}^{\mathbf{\mathcal M}}
}
\,d\nu_{hsa}^{\mathbf{\mathcal M}}
=
1.
\]
Therefore
\[
\eta_{hsa}g_{hsa}
-
\psi_{hsa}(\eta_{hsa})
=
\ell_{hsa},
\]
and hence
\[
\frac{
d\nu_{hsa}^{\widetilde{\mathbf{\mathcal M}}}
}{
d\nu_{hsa}^{\mathbf{\mathcal M}}
}
=
\exp\left\{
\eta_{hsa}g_{hsa}
-
\psi_{hsa}(\eta_{hsa})
\right\}.
\]
If
$\operatorname{Var}_{\mathbf{\mathcal M}}(\ell_{hsa}\mid s,a)=0$,
it implies
$\ell_{hsa}=0$ almost surely, so
$\nu_{hsa}^{\widetilde{\mathbf{\mathcal M}}}
=
\nu_{hsa}^{\mathbf{\mathcal M}}$,
which corresponds to $\eta_{hsa}=0$.

Thus the bounded exponential-tilt parametrization is exact for the
row-wise alternative class specified in
\eqref{eq:lq-admissible-local-alternative}.

By \eqref{eq:lq-unit-score},
\[
\psi_{hsa}(0)=0,
\qquad
\psi_{hsa}'(0)=0,
\qquad
\psi_{hsa}''(0)=1.
\]
Since $g_{hsa}$ is bounded, $\psi_{hsa}$ has bounded derivatives in a
neighborhood of zero, and therefore
\begin{equation}
\psi_{hsa}(\eta)
=
\frac{\eta^2}{2}
+
o(|\eta|^2).
\label{eq:lq-log-normalizer-expansion}
\end{equation}
It follows from
\eqref{eq:lq-exponential-path} that
\begin{equation}
\frac{
d\nu_{hsa}^{\eta}
}{
d\nu_{hsa}^{\mathbf{\mathcal M}}
}
=
\exp\left\{
\eta g_{hsa}
-
\psi_{hsa}(\eta)
\right\}=
1+\eta g_{hsa}+O(\eta^2),
\label{eq:lq-local-density-expansion}
\end{equation}
where the remainder is uniform over the support of the row law.

Consequently, for every bounded row-wise function
$\varphi(R_h,S_{h+1})$,
\begin{equation}
\mathbb E_{\nu_{hsa}^{\eta}}
[\varphi]
-
\mathbb E_{\nu_{hsa}^{\mathbf{\mathcal M}}}
[\varphi]
=
\mathbb E_{\mathbf{\mathcal M}}
\left[
\varphi
\left(
\frac{
d\nu_{hsa}^{\eta}
}{
d\nu_{hsa}^{\mathbf{\mathcal M}}
}
-1
\right)
\middle|s,a
\right]
=
\eta\,
\mathbb E_{\mathbf{\mathcal M}}
[
\varphi g_{hsa}\mid s,a
]
+
O(\eta^2).
\label{eq:lq-local-expectation-expansion}
\end{equation}

Finally,
\begin{align}
\KL\!\left(
\nu_{hsa}^{\mathbf{\mathcal M}},
\nu_{hsa}^{\eta}
\right)
&=
\mathbb E_{\mathbf{\mathcal M}}
\left[
\log
\frac{
d\nu_{hsa}^{\mathbf{\mathcal M}}
}{
d\nu_{hsa}^{\eta}
}
\middle|s,a
\right]
\nonumber\\
&=
-\eta\,
\mathbb E_{\mathbf{\mathcal M}}
[g_{hsa}\mid s,a]
+
\psi_{hsa}(\eta)
\nonumber\\
&=
\frac{\eta^2}{2}
+
o(|\eta|^2),
\label{eq:lq-row-kl}
\end{align}
where the last equality follows from
\eqref{eq:lq-unit-score} and
\eqref{eq:lq-log-normalizer-expansion}.

Applying \eqref{eq:lq-local-expectation-expansion} to $\varphi(R,S')=R$ and to
$\varphi(R,S')=\mathbf 1\{S'=s'\}$ yields
\[
\widetilde r_h(s,a)-r_h(s,a)
=
\eta_{hsa}
\mathbb E_{\mathbf{\mathcal M}}
[
R_hg_{hsa}\mid s,a
]
+
O(\eta_{hsa}^2),
\]
and
\[
\widetilde P_h(s'\mid s,a)-P_h(s'\mid s,a)
=
\eta_{hsa}
\mathbb E_{\mathbf{\mathcal M}}
[
\mathbf 1\{S_{h+1}=s'\}g_{hsa}
\mid s,a
]
+
O(\eta_{hsa}^2).
\]

We now linearize the exact alternative constraint in
\eqref{eq:lq-exact-challenger-problem}. By
Lemma~\ref{lem:model-difference-expansion},
\begin{align}
\bigl(
V_1^\pi(\widetilde{\mathbf{\mathcal M}}_\eta)
-
V_1^{\bar\pi}(\widetilde{\mathbf{\mathcal M}}_\eta)
\bigr)
-
\bigl(
V_1^\pi(\mathbf{\mathcal M})
-
V_1^{\bar\pi}(\mathbf{\mathcal M})
\bigr)
&=
\sum_{h,s,a}
G_{hsa}^r(\pi,\bar\pi)
\bigl(
\widetilde r_h(s,a)-r_h(s,a)
\bigr)
\nonumber
\\
&+
\sum_{h,s,a,s'}
G_{hsa}^P(\pi,\bar\pi)(s')
\bigl(
\widetilde P_h(s'\mid s,a)
-
P_h(s'\mid s,a)
\bigr)
+
O(\|\eta\|_\infty^2).
\label{eq:lq-value-expansion-start}
\end{align}
Here the remainder is second order because the explicit remainder in
Lemma~\ref{lem:model-difference-expansion} is a sum of products between
a model perturbation and an occupancy perturbation, and the latter is
$O(\|\eta\|_\infty)$ for a finite-horizon finite-state MDP.

Using
\[
\Psi_{hsa}^{\pi,\bar\pi}(R_h,S_{h+1})
=
G_{hsa}^r(\pi,\bar\pi)R_h
+
G_{hsa}^P(\pi,\bar\pi)(S_{h+1}),
\]
equation \eqref{eq:lq-value-expansion-start} becomes
\begin{align}
\bigl(
V_1^\pi(\widetilde{\mathbf{\mathcal M}}_\eta)
-
V_1^{\bar\pi}(\widetilde{\mathbf{\mathcal M}}_\eta)
\bigr)
-
\bigl(
V_1^\pi(\mathbf{\mathcal M})
-
V_1^{\bar\pi}(\mathbf{\mathcal M})
\bigr)
=
\sum_{h,s,a}
\eta_{hsa}
\mathbb E_{\mathbf{\mathcal M}}
\left[
\Psi_{hsa}^{\pi,\bar\pi}
g_{hsa}
\mid s,a
\right]
+
o(\|\eta\|_\infty).
\label{eq:lq-value-first-order}
\end{align}
Since the scores are centered,
\[
\mathbb E_{\mathbf{\mathcal M}}
[
\Psi_{hsa}^{\pi,\bar\pi}g_{hsa}\mid s,a
]
=
\mathbb E_{\mathbf{\mathcal M}}
\left[
\left(
\Psi_{hsa}^{\pi,\bar\pi}
-
\mathbb E_{\mathbf{\mathcal M}}
[
\Psi_{hsa}^{\pi,\bar\pi}\mid s,a
]
\right)
g_{hsa}
\mid s,a
\right].
\]

Recall
\[
b(\pi,\bar\pi)
=
\epsilon
-
\left(
V_1^\pi(\mathbf{\mathcal M})
-
V_1^{\bar\pi}(\mathbf{\mathcal M})
\right).
\]
Therefore, retaining only the first-order term in
\eqref{eq:lq-value-first-order}, we replace the exact
challenger constraint
\[
V_1^\pi(\widetilde{\mathbf{\mathcal M}})
-
V_1^{\bar\pi}(\widetilde{\mathbf{\mathcal M}})
>
\epsilon
\]
by the non-strict linearized constraint
\begin{equation}
\sum_{h,s,a}
\eta_{hsa}
\mathbb E_{\mathbf{\mathcal M}}
\left[
\left(
\Psi_{hsa}^{\pi,\bar\pi}
-
\mathbb E_{\mathbf{\mathcal M}}
[
\Psi_{hsa}^{\pi,\bar\pi}\mid s,a
]
\right)
g_{hsa}
\mid s,a
\right]
\ge
b(\pi,\bar\pi).
\label{eq:lq-linearized-alternative}
\end{equation}
Whenever the strict linearized constraint is feasible,
replacing $>$ by $\ge$ leaves the infimum of the quadratic
information cost unchanged. 

Likewise, substituting \eqref{eq:lq-row-kl} into the exact information
cost \eqref{eq:lq-weighted-information} and retaining only the leading
quadratic term gives
\begin{equation}
\frac12
\sum_{h,s,a}
c_{hsa}\eta_{hsa}^2.
\label{eq:lq-leading-information-cost}
\end{equation}

Thus, after replacing both the exact alternative constraint and the
exact information cost by their leading local terms, the
challenger-specific problem becomes
\begin{equation}
\inf_{\{g_{hsa},\eta_{hsa}\}}
\left\{
\frac12
\sum_{h,s,a}
c_{hsa}\eta_{hsa}^2\right\},
\label{eq:lq-score-length-program}
\end{equation}
under \eqref{eq:lq-unit-score} and
\eqref{eq:lq-linearized-alternative}.

We first optimize over the score directions. By Cauchy-Schwarz inequality,
\begin{equation}
\left|
\mathbb E_{\mathbf{\mathcal M}}
\left[
\left(
\Psi_{hsa}^{\pi,\bar\pi}
-
\mathbb E_{\mathbf{\mathcal M}}
[
\Psi_{hsa}^{\pi,\bar\pi}\mid s,a
]
\right)
g_{hsa}
\mid s,a
\right]
\right|
\le
\sqrt{
\Gamma_{hsa}(\pi,\bar\pi)
}.
\label{eq:lq-score-cs}
\end{equation}
Whenever $\Gamma_{hsa}(\pi,\bar\pi)>0$, equality is attained by
\begin{equation}
g_{hsa}^\star
=
\frac{
\Psi_{hsa}^{\pi,\bar\pi}
-
\mathbb E_{\mathbf{\mathcal M}}
[
\Psi_{hsa}^{\pi,\bar\pi}\mid s,a
]
}{
\sqrt{\Gamma_{hsa}(\pi,\bar\pi)}
}.
\label{eq:lq-optimal-score}
\end{equation}
This is an admissible local direction. Indeed,
$R_h\in[0,1]$, $S_{h+1}$ has finite support, and
$\Psi_{hsa}^{\pi,\bar\pi}$ is additive in $(R_h,S_{h+1})$; hence
$g_{hsa}^\star$ is bounded and additive, so the row-wise exponential
tilt described above remains within the reward--transition product
model. Its sign may be reversed without affecting either
\eqref{eq:lq-unit-score} or the quadratic information cost.

Consequently, after optimizing over $g_{hsa}$, the local problem is
\begin{equation}
\inf_{\{\eta_{hsa}\}}
\left\{
\frac12
\sum_{h,s,a}
c_{hsa}\eta_{hsa}^2:
\sum_{h,s,a}
\sqrt{\Gamma_{hsa}(\pi,\bar\pi)}
\,|\eta_{hsa}|
\ge
b(\pi,\bar\pi)
\right\}.
\label{eq:lq-length-program}
\end{equation}

Suppose first that $b(\pi,\bar\pi)>0$ and
$c_{hsa}>0$ whenever
$\Gamma_{hsa}(\pi,\bar\pi)>0$. Then
\[
\left(
\sum_{h,s,a}
\sqrt{\Gamma_{hsa}(\pi,\bar\pi)}
\,|\eta_{hsa}|
\right)^2
\le
\left(
\sum_{h,s,a}
\frac{
\Gamma_{hsa}(\pi,\bar\pi)
}{
c_{hsa}
}
\right)
\left(
\sum_{h,s,a}
c_{hsa}\eta_{hsa}^2
\right).
\]
Hence every feasible perturbation satisfies
\[
\frac12
\sum_{h,s,a}
c_{hsa}\eta_{hsa}^2
\ge
\frac{
b(\pi,\bar\pi)^2
}{
2\displaystyle
\sum_{h,s,a}
\frac{
\Gamma_{hsa}(\pi,\bar\pi)
}{
c_{hsa}
}
}.
\]
Equality is attained by choosing
\[
|\eta_{hsa}|
=
\frac{
b(\pi,\bar\pi)
}{
\displaystyle
\sum_{\tilde h,\tilde s,\tilde a}
\frac{
\Gamma_{\tilde h\tilde s\tilde a}(\pi,\bar\pi)
}{
c_{\tilde h\tilde s\tilde a}
}
}
\,
\frac{
\sqrt{\Gamma_{hsa}(\pi,\bar\pi)}
}{
c_{hsa}
}.
\]
Therefore the challenger-specific local quadratic information is
\begin{equation}
\frac{
b(\pi,\bar\pi)^2
}{
2\displaystyle
\sum_{h,s,a}
\frac{
\Gamma_{hsa}(\pi,\bar\pi)
}{
\kappa d_{hsa}^\mu
+
tq_{hsa}^{\boldsymbol\lambda}
}
}.
\label{eq:lq-challenger-information}
\end{equation}

If $c_{hsa}=0$ for some row with
$\Gamma_{hsa}(\pi,\bar\pi)>0$, that row can satisfy the 
constraint at zero quadratic information cost, and the local infimum is
zero.
If $\Gamma_{hsa}(\pi,\bar\pi)=0$ for every row and
$b(\pi,\bar\pi)>0$, the constraint is infeasible and the local
cost is $+\infty$. Finally, if $b(\pi,\bar\pi)=0$, the local infimum is
zero.

Taking the minimum over challengers, the local quadratic surrogate of
the exact inner information problem in
\eqref{eq:lq-exact-challenger-problem} is therefore
\[
\min_{\pi\in\Pi_d\setminus\{\bar\pi\}}
\frac{
b(\pi,\bar\pi)^2
}{
2\displaystyle
\sum_{h,s,a}
\frac{
\Gamma_{hsa}(\pi,\bar\pi)
}{
\kappa d_{hsa}^\mu
+
tq_{hsa}^{\boldsymbol\lambda}
}
}.
\]
Hence the exact information constraint
\[
\inf_{\widetilde{\mathbf{\mathcal M}}
\in\mathcal A_\epsilon(\bar\pi)}
\left[
\kappa J_\mu(
\mathbf{\mathcal M},
\widetilde{\mathbf{\mathcal M}}
)
+
tI_{\boldsymbol\lambda}(
\mathbf{\mathcal M},
\widetilde{\mathbf{\mathcal M}}
)
\right]
\ge1
\]
is locally replaced by
\[
\frac{
b(\pi,\bar\pi)^2
}{
2\displaystyle
\sum_{h,s,a}
\frac{
\Gamma_{hsa}(\pi,\bar\pi)
}{
\kappa d_{hsa}^\mu
+
tq_{hsa}^{\boldsymbol\lambda}
}
}
\ge1
\qquad
\forall
\pi\in\Pi_d\setminus\{\bar\pi\}.
\]
Equivalently,
\begin{equation}
\max_{\pi\in\Pi_d\setminus\{\bar\pi\}}
\frac{2}{b(\pi,\bar\pi)^2}
\sum_{h,s,a}
\frac{
\Gamma_{hsa}(\pi,\bar\pi)
}{
\kappa d_{hsa}^\mu
+
tq_{hsa}^{\boldsymbol\lambda}
}
\le1,
\label{eq:lq-information-constraint}
\end{equation}
where a term with $b(\pi,\bar\pi)=0$ is understood as $+\infty$.

Finally, substituting
\eqref{eq:lq-information-constraint}
into the outer optimization defining the characteristic threshold gives
\[
\begin{aligned}
T_{\mathrm{LQ},\kappa}(\mathbf{\mathcal M})
:=
\min_{\bar\pi\in\Pi_\epsilon(\mathbf{\mathcal M})}
\inf_{\boldsymbol\lambda\in\Delta_{H|\mathcal S|}}
\inf\Bigg\{
t\ge0:\;
\max_{\pi\in\Pi_d\setminus\{\bar\pi\}}
\frac{2}{b(\pi,\bar\pi)^2}
\sum_{h,s,a}
\frac{
\Gamma_{hsa}(\pi,\bar\pi)
}{
\kappa d_{hsa}^\mu
+
tq_{hsa}^{\boldsymbol\lambda}
}
\le1
\Bigg\},
\end{aligned}
\]
which is the claimed local quadratic characteristic threshold.
\end{proof}

\subsection{Proof of Lemma~\ref{lem:local-bernstein}}
\begin{proof}

By 
Lemma~\ref{lem:model-difference-expansion},
\[
\begin{aligned}
\widehat D(\pi,\bar\pi)
-
D(\pi,\bar\pi)
={}&
\sum_{h,s,a}
G_{hsa}^{r}(\pi,\bar\pi)
\left(
\widehat r_h(s,a)-r_h(s,a)
\right)
\\
&+
\sum_{h,s,a,s'}
G_{hsa}^{P}(\pi,\bar\pi)(s')
\left(
\widehat P_h(s'\mid s,a)
-
P_h(s'\mid s,a)
\right)
\\
&+
\operatorname{Rem}_{N}(\pi,\bar\pi).
\end{aligned}
\]
For each row, recall
\[
\Psi_{hsa}^{\pi,\bar\pi}(r,s')
=
G_{hsa}^{r}(\pi,\bar\pi)r
+
G_{hsa}^{P}(\pi,\bar\pi)(s').
\]
Since
\[
\widehat r_h(s,a)
=
\frac1{N_{hsa}}
\sum_{i=1}^{N_{hsa}}
R_{hsa}^{(i)}
\]
and
\[
\widehat P_h(s'\mid s,a)
=
\frac1{N_{hsa}}
\sum_{i=1}^{N_{hsa}}
\mathbf 1
\left\{
S_{hsa}^{\prime(i)}=s'
\right\},
\]
the linear part of the preceding expansion can be written exactly as
\[
\sum_{h,s,a}
\frac1{N_{hsa}}
\sum_{i=1}^{N_{hsa}}
\left[
\Psi_{hsa}^{\pi,\bar\pi}
\left(
Z_{hsa}^{(i)}
\right)
-
\mathbb E_{\mathbf{\mathcal M}}
\left[
\Psi_{hsa}^{\pi,\bar\pi}(Z)
\mid S_h=s,A_h=a
\right]
\right].
\]

Define
\[
X_{hsa}^{(i)}
:=
\frac1{N_{hsa}}
\left[
\Psi_{hsa}^{\pi,\bar\pi}
\left(
Z_{hsa}^{(i)}
\right)
-
\mathbb E_{\mathbf{\mathcal M}}
\left[
\Psi_{hsa}^{\pi,\bar\pi}(Z)
\mid S_h=s,A_h=a
\right]
\right].
\]
Conditional on the row counts, the variables
$\{X_{hsa}^{(i)}\}_{h,s,a,i}$ are independent and centered.

By the definition of
$\Gamma_{hsa}(\pi,\bar\pi)$,
\[
\operatorname{Var}_{\mathbf{\mathcal M}}
\left(
X_{hsa}^{(i)}
\right)
=
\frac{
\Gamma_{hsa}(\pi,\bar\pi)
}{
N_{hsa}^2
}.
\]
Therefore the variance of the full linear term is
\[
\begin{aligned}
\sum_{h,s,a}
\sum_{i=1}^{N_{hsa}}
\operatorname{Var}_{\mathbf{\mathcal M}}
\left(
X_{hsa}^{(i)}
\right)
&=
\sum_{h,s,a}
N_{hsa}
\frac{
\Gamma_{hsa}(\pi,\bar\pi)
}{
N_{hsa}^2
}
\\
&=
\sum_{h,s,a}
\frac{
\Gamma_{hsa}(\pi,\bar\pi)
}{
N_{hsa}
}.
\end{aligned}
\]

Moreover, letting
\[
L_{hsa}(\pi,\bar\pi)
:=
\operatorname{range}
\left(
\Psi_{hsa}^{\pi,\bar\pi}
\right),
\]
we have
\[
\left|
\Psi_{hsa}^{\pi,\bar\pi}(Z)
-
\mathbb E_{\mathbf{\mathcal M}}
\left[
\Psi_{hsa}^{\pi,\bar\pi}(Z)
\mid S_h=s,A_h=a
\right]
\right|
\le
L_{hsa}(\pi,\bar\pi),
\]
and consequently
\[
\left|
X_{hsa}^{(i)}
\right|
\le
\frac{
L_{hsa}(\pi,\bar\pi)
}{
N_{hsa}
}.
\]
Thus
\[
\max_{h,s,a,i}
\left|
X_{hsa}^{(i)}
\right|
\le
\max_{h,s,a}
\frac{
L_{hsa}(\pi,\bar\pi)
}{
N_{hsa}
}.
\]

Applying Bernstein's inequality to the sum of the independent centered
variables $\{X_{hsa}^{(i)}\}$ gives, with probability at least
$1-\delta$,
\[
\begin{aligned}
&
\left|
\sum_{h,s,a}
\frac1{N_{hsa}}
\sum_{i=1}^{N_{hsa}}
\left[
\Psi_{hsa}^{\pi,\bar\pi}
\left(
Z_{hsa}^{(i)}
\right)
-
\mathbb E_{\mathbf{\mathcal M}}
\left[
\Psi_{hsa}^{\pi,\bar\pi}(Z)
\mid S_h=s,A_h=a
\right]
\right]
\right|
\\
&\qquad\le
\sqrt{
2\log\frac{2}{\delta}
\sum_{h,s,a}
\frac{
\Gamma_{hsa}(\pi,\bar\pi)
}{
N_{hsa}
}
}
+
\frac13
\log\frac{2}{\delta}
\max_{h,s,a}
\frac{
L_{hsa}(\pi,\bar\pi)
}{
N_{hsa}
}.
\end{aligned}
\]
Combining this inequality with the plug-in decomposition yields
\[
\begin{aligned}
\left|
\widehat D(\pi,\bar\pi)
-
D(\pi,\bar\pi)
\right|
\le{}&
\sqrt{
2\log\frac{2}{\delta}
\sum_{h,s,a}
\frac{
\Gamma_{hsa}(\pi,\bar\pi)
}{
N_{hsa}
}
}
\\
&+
\frac13
\log\frac{2}{\delta}
\max_{h,s,a}
\frac{
L_{hsa}(\pi,\bar\pi)
}{
N_{hsa}
}
\\
&+
\left|
\operatorname{Rem}_{N}(\pi,\bar\pi)
\right|.
\end{aligned}
\]
\end{proof}

\subsection{Proof of Theorem~\ref{thm:stopping-correctness}}
\label{app:proof-stopping-correctness}

\begin{proof}
We first construct a global event on which all concentration statements used
by the algorithm hold simultaneously.

\paragraph{The trimming event.}
Let $\mathcal E_{\rm trim}$ denote the event in
Lemma~\ref{lem:two-fold-trimming-guarantee}.  Since
\[
\delta_{\rm trim}
=
\frac{\delta}{48},
\]
Lemma~\ref{lem:two-fold-trimming-guarantee} gives
\[
\mathbb P_{\mathbf{\mathcal M}}
\left(
\mathcal E_{\rm trim}^c
\right)
\le
8\delta_{\rm trim}
=
\frac{\delta}{6}.
\]

We use the coupling convention of Remark~\ref{rem:augmented-iid-analysis} throughout.

\paragraph{The stabilization event.}
For epoch $k$, let
\[
N_{k,hsa}^{\rm q}
:=
N_{k,hsa}^{\rm ad}
+
N_{k,hsa}^{\rm st}.
\]
Conditional on the history before epoch $k$,
\[
N_{k,hsa}^{\rm q}
\sim
\operatorname{Binomial}
\left(
\widehat n_{k-1,hs}+m_k^{\rm st},
\pi_h^c(a\mid s)
\right)
\]
for every collaborator-supported row. Write
\[
\mu_{k,hsa}
:=
\left(
\widehat n_{k-1,hs}+m_k^{\rm st}
\right)
\pi_h^c(a\mid s).
\]
By \eqref{eq:app-stabilization-size},
\[
\mu_{k,hsa}
\ge
m_k^{\rm st}p_{\min}
\ge
\max
\left\{
2n_k^{\rm uc},
2n_{k+1}^{\rm sen},
8\beta_k
\right\}.
\]
Let $\mathcal E_{\rm st}$ be the event that
\[
N_{k,hsa}^{\rm q}
\ge
\frac12\mu_{k,hsa}
\]
simultaneously over all epochs and collaborator-supported rows. The
multiplicative Chernoff bound gives
\[
\mathbb P_{\mathbf{\mathcal M}}
\left(
N_{k,hsa}^{\rm q}
<
\frac12\mu_{k,hsa}
\mid
\mathcal F_{k-1}
\right)
\le
\exp
\left(
-\frac{\mu_{k,hsa}}8
\right)
\le
e^{-\beta_k}.
\]
Consequently,
\[
\mathbb P_{\mathbf{\mathcal M}}
\left(
\mathcal E_{\rm st}^c
\right)
\le
\frac{\delta}{12}.
\]

On $\mathcal E_{\rm st}$, every supported row satisfies
\begin{equation}
N_{k,hsa}^{\mathcal D}
\ge
N_{k,hsa}^{\rm q}
\ge
\frac12\mu_{k,hsa}
\ge
n_k^{\rm uc}\vee n_{k+1}^{\rm sen}.
\label{eq:thm-supported-uc}
\end{equation}
Moreover, under the theoretical convention
$\widehat\pi^c=\pi^c$, if $\widehat{\mathbf n}_{k-1}$ is the
allocation selected at epoch $k-1$, then
\begin{equation}
N_{k,hsa}^{\mathcal D}
\ge
\frac12
N_{k,hsa}^{+}
\left(
\widehat{\mathbf n}_{k-1}
\right).
\label{eq:thm-realized-prospective-count}
\end{equation}
\paragraph{The historical sensitivity event.}
For each row $(h,s,a)$, enumerate the observations entering the historical
sample in their row-wise order of arrival as
\[
Z_{hsa,1},Z_{hsa,2},\ldots,
\qquad
Z_{hsa,i}
=
(R_{hsa,i},S'_{hsa,i}),
\]
and let $\mathcal G_{hsa,i}$ denote the sigma-field containing all
observations and query decisions made before the $(i+1)$st observation from
that row is revealed.

  Since the query decision and the
collaborator action are determined before the corresponding reward and next
state are observed, for every $i\ge1$,
\[
R_{hsa,i}-r_h(s,a)
\]
and
\[
\mathbf 1\{S'_{hsa,i}=s'\}
-
P_h(s'\mid s,a)
\]
are bounded martingale difference sequences.

The historical sample size $N_{k,hsa}^{\mathcal G}$ is random and depends on the
adaptive querying history.  We therefore require a concentration event that
holds simultaneously over all possible row-wise sample sizes, rather
than only at a fixed deterministic sample size.  For $n\ge1$, define
\[
\beta^\mathcal G(n)
:=
\log
\left(
\frac{
2\pi^2 n^2 H|\mathcal S||\mathcal A|(|\mathcal S|+1)
}{
\delta
}
\right).
\]
For a fixed row $(h,s,a)$ and a fixed $n$, the martingale Hoeffding
inequality gives
\[
\mathbb P_{\mathbf{\mathcal M}}
\left(
\left|
\frac1n
\sum_{i=1}^n
R_{hsa,i}
-
r_h(s,a)
\right|
>
\sqrt{
\frac{
\beta^\mathcal G(n)
}{
2n
}
}
\right)
\le
2e^{-\beta^\mathcal G(n)}.
\]
Similarly for all $s'\in\mathcal S$,
\[
\mathbb P_{\mathbf{\mathcal M}}
\left(
\left|
\frac1n
\sum_{i=1}^n
\mathbf 1\{S'_{hsa,i}=s'\}
-
P_h(s'\mid s,a)
\right|
>
\sqrt{
\frac{
\beta^\mathcal G(n)
}{
2n
}
}
\right)
\le
2e^{-\beta^\mathcal G(n)}.
\]

Let $\mathcal E_G$ denote the event that the preceding inequalities hold
simultaneously for every row $(h,s,a)$, every reward or transition
coordinate, and every integer $n\ge1$.  A union bound over the $H|\mathcal S||\mathcal A|$ rows,
the reward coordinate and the $|\mathcal S|$ transition coordinates, and all possible
sample sizes gives
\[
\begin{aligned}
\mathbb P_{\mathbf{\mathcal M}}
\left(
\mathcal E_G^c
\right)
&\le
2H|\mathcal S||\mathcal A|(|\mathcal S|+1)
\sum_{n\ge1}
e^{-\beta^\mathcal G(n)}
\\
&=
2H|\mathcal S||\mathcal A|(|\mathcal S|+1)
\sum_{n\ge1}
\frac{
\delta
}{
2\pi^2n^2H|\mathcal S||\mathcal A|(|\mathcal S|+1)
}
\\
&=
\frac{\delta}{\pi^2}
\sum_{n\ge1}\frac1{n^2}
\\
&=
\frac{\delta}{6}.
\end{aligned}
\]
Thus, on $\mathcal E_G$, for every deterministic $n\ge1$,
\[
\left|
\widehat r_{hsa,n}
-
r_h(s,a)
\right|
\le
\min
\left\{
1,\,
\sqrt{
\frac{\beta^\mathcal G(n)}{2n}
}
\right\},
\]
and
\[
\left\|
\widehat P_{hsa,n}(\cdot)
-
P_h(\cdot\mid s,a)
\right\|_\infty
\le
\min
\left\{
1,\,
\sqrt{
\frac{\beta^\mathcal G(n)}{2n}
}
\right\}.
\]
Because these inequalities hold simultaneously for all $n\ge1$, they may
be evaluated at the random historical count
$n=N_{k,hsa}^{\mathcal G}$.  Consequently, simultaneously for every epoch $k$ and
every row,
\[
\left|
\widehat r_{k,h}^{\mathcal G}(s,a)
-
r_h(s,a)
\right|
\le
\alpha_{k,hsa}^{\mathcal G,r},
\]
and
\[
\left\|
\widehat P_{k,h}^{\mathcal G}(\cdot\mid s,a)
-
P_h(\cdot\mid s,a)
\right\|_\infty
\le
\alpha_{k,hsa}^{\mathcal G,P},
\]
whenever $N_{k,hsa}^{\mathcal G}>0$.

Finally, by definition
\[
\Delta_{k,hsa}^{\mathcal G}
\le
C_{\rm sen}
\max
\left\{
\alpha_k^{\mathcal G,r},
\alpha_k^{\mathcal G,P}
\right\}.
\]
 Hence, whenever every row
contains at least $n^{\rm sen}_{k+1}$ observations, we have
\begin{equation}\label{eq:bound for delta}
\left(
\Delta_{k,hsa}^{\mathcal G}
\right)^2
\le
\epsilon_{k+1},
\qquad
\text{for every }(h,s,a).
\end{equation}

Also on $\mathcal E_G$, 
for any fixed \((h,s,a)\) and policy pair
\((\pi,\bar\pi)\), recall that
\[
\Psi_{hsa}^{\pi,\bar\pi}(r,s')
=
G_{hsa}^{r}(\pi,\bar\pi)r
+
G_{hsa}^{P}(\pi,\bar\pi)(s'),
\]
and analogously
\[
\widehat\Psi_{k,hsa}^{\pi,\bar\pi}(r,s')
=
\widehat G_{k,hsa}^{r}(\pi,\bar\pi)r
+
\widehat G_{k,hsa}^{P}(\pi,\bar\pi)(s').
\]
Hence, since \(r\in[0,1]\),
\[
\begin{aligned}
\left|
\widehat\Psi_{k,hsa}^{\pi,\bar\pi}(r,s')
-
\Psi_{hsa}^{\pi,\bar\pi}(r,s')
\right|\le
\left|
\widehat G_{k,hsa}^{r}(\pi,\bar\pi)
-
G_{hsa}^{r}(\pi,\bar\pi)
\right|
+
\left|
\widehat G_{k,hsa}^{P}(\pi,\bar\pi)(s')
-
G_{hsa}^{P}(\pi,\bar\pi)(s')
\right|.
\end{aligned}
\]

For the reward-sensitivity term,
\[
\begin{aligned}
&
\left|
\widehat G_{k,hsa}^{r}(\pi,\bar\pi)
-
G_{hsa}^{r}(\pi,\bar\pi)
\right|
\\
&=
\left|
\left(
\widehat d_{hsa}^{\pi,k,\mathcal G}
-
\widehat d_{hsa}^{\bar\pi,k,\mathcal G}
\right)
-
\left(
d_{hsa}^{\pi}
-
d_{hsa}^{\bar\pi}
\right)
\right|
\\
&\le
\left|
\widehat d_{hsa}^{\pi,k,\mathcal G}
-
d_{hsa}^{\pi}
\right|
+
\left|
\widehat d_{hsa}^{\bar\pi,k,\mathcal G}
-
d_{hsa}^{\bar\pi}
\right|
\\
&\le
2(h-1)|\mathcal S|\alpha_k^{\mathcal G,P}.
\end{aligned}
\]

For the transition-sensitivity term, first, for either
\(\rho\in\{\pi,\bar\pi\}\),
\[
\begin{aligned}
&
\left|
\widehat d_{hsa}^{\rho,k,\mathcal G}
\widehat V_{h+1}^{\rho,k,\mathcal G}(s')
-
d_{hsa}^{\rho}
V_{h+1}^{\rho}(s')
\right|
\\
&\le
\left|
\widehat d_{hsa}^{\rho,k,\mathcal G}
-
d_{hsa}^{\rho}
\right|
\left|
\widehat V_{h+1}^{\rho,k,\mathcal G}(s')
\right|
+
d_{hsa}^{\rho}
\left|
\widehat V_{h+1}^{\rho,k,\mathcal G}(s')
-
V_{h+1}^{\rho}(s')
\right|
\\
&\le
(h-1)(H-h)|\mathcal S|\alpha_k^{\mathcal G,P}
+
E_{k,h+1}^{V,\mathcal G},
\end{aligned}
\]
where we used
\(
0\le
\widehat V_{h+1}^{\rho,k,\mathcal G}
\le H-h
\)
and \(0\le d_{hsa}^{\rho}\le1\).
Therefore,
\[
\begin{aligned}
&
\left|
\widehat G_{k,hsa}^{P}(\pi,\bar\pi)(s')
-
G_{hsa}^{P}(\pi,\bar\pi)(s')
\right|
\le
2(h-1)(H-h)|\mathcal S|\alpha_k^{\mathcal G,P}
+
2E_{k,h+1}^{V,\mathcal G}.
\end{aligned}
\]

Combining the two bounds gives
\[
\begin{aligned}
&
\left|
\widehat\Psi_{k,hsa}^{\pi,\bar\pi}(r,s')
-
\Psi_{hsa}^{\pi,\bar\pi}(r,s')
\right|
\\
&\le
2(h-1)|\mathcal S|\alpha_k^{\mathcal G,P}
+
2(h-1)(H-h)|\mathcal S|\alpha_k^{\mathcal G,P}
+
2E_{k,h+1}^{V,\mathcal G}
\\
&=
2(h-1)(H-h+1)|\mathcal S|\alpha_k^{\mathcal G,P}
+
2E_{k,h+1}^{V,\mathcal G}
\\
&=
\Delta_{k,hsa}^{\mathcal G}.
\end{aligned}
\]
Since the right-hand side does not depend on \(r\) or \(s'\), we conclude that
\begin{equation}
\sup_{r,s'}
\left|
\widehat\Psi_{k,hsa}^{\pi,\bar\pi}(r,s')
-
\Psi_{hsa}^{\pi,\bar\pi}(r,s')
\right|
\le
\Delta_{k,hsa}^{\mathcal G}.
\label{eq:thm-sensitivity-good}
\end{equation}

\paragraph{The fresh sample concentration event.}
For each reached epoch $k$ and row $(h,s,a)$, define
\[
\alpha_{k,hsa}^{\mathcal D,r}
=
\alpha_{k,hsa}^{\mathcal D,P}
:=
\begin{cases}
\displaystyle
\min\left\{
1,\sqrt{\frac{\beta_k}{2N_{k,hsa}^{\mathcal D}}}
\right\},
& N_{k,hsa}^{\mathcal D}>0,\\[1.2ex]
1,
& N_{k,hsa}^{\mathcal D}=0.
\end{cases}
\]
Let
\[
\alpha_k^{\mathcal D,r}
:=
\max_{h,s,a}\alpha_{k,hsa}^{\mathcal D,r},
\qquad
\alpha_k^{\mathcal D,P}
:=
\max_{h,s,a}\alpha_{k,hsa}^{\mathcal D,P},
\]
and define the uniform remainder bound
\[
\overline{\operatorname{Rem}}_k
:=
K_H
\left(
\alpha_k^{\mathcal D,r}
\vee
\alpha_k^{\mathcal D,P}
\right)^2.
\]
Let $\mathcal E_D$ denote the event that, simultaneously over all epochs and
rows,
\[
\left|
\widehat r_{k,h}^{\mathcal D}(s,a)-r_h(s,a)
\right|
\le
\alpha_{k,hsa}^{\mathcal D,r},
\]
and
\[
\left\|
\widehat P_{k,h}^{\mathcal D}(\cdot\mid s,a)
-
P_h(\cdot\mid s,a)
\right\|_\infty
\le
\alpha_{k,hsa}^{\mathcal D,P}.
\]
Conditional on the history preceding epoch $k$ and on the realized row
counts, the observations in the fresh sample $\mathcal D_k^\mathcal D$ are row-wise
iid with the correct row laws.  Therefore Hoeffding's inequality and the
choice of $\beta_k$ give
\[
\mathbb P_{\mathbf{\mathcal M}}
\left(
\mathcal E_D^c
\right)
\le
\frac{\delta}{6}.
\]
On $\mathcal E_D$, Lemma~\ref{lem:iid-remainder} gives
\begin{equation}
\left|
\operatorname{Rem}_{\widehat{\mathbf{\mathcal M}}_k^{\mathcal D}}(\pi,\bar\pi)
\right|
\le
\overline{\operatorname{Rem}}_k
\label{eq:thm-remainder-good}
\end{equation}
simultaneously over all relevant policy pairs.

\paragraph{The local-variance and Bernstein events.}

Let
\[
\mathcal P_k
:=
\widehat\Pi_{k-1}^2.
\]
Since
$\widehat\Pi_k\subseteq\widehat\Pi_{k-1}$,
every policy pair used after the epoch-$k$ localization belongs to
$\mathcal P_k$.
Let $\mathcal E_\Gamma$ denote the intersection of the
standard-deviation concentration events $\mathcal E_{\rm var}$
from Lemma~\ref{lem:variance-envelope-general}, applied to the fixed
population scores $\Psi_{hsa}^{\pi,\bar\pi}$,
simultaneously over all reached epochs, rows with
$N_{k,hsa}^\mathcal D\ge2$, and policy pairs
$(\pi,\bar\pi)\in\mathcal P_k$.
Then by Lemma~\ref{lem:variance-envelope-general}, \[
\mathbb P_{\mathbf{\mathcal M}}
\left(
\mathcal E_\Gamma^c
\right)
\le
\frac{\delta}{6},
\] and under $\mathcal E_\Gamma\cap \mathcal E_G$, by Lemma~\ref{lem:variance-envelope-general} and \eqref{eq:thm-sensitivity-good}, 
\begin{equation}\label{eq:thm-variance-control}
\Gamma_{hsa}(\pi,\bar\pi)
\le
\overline\Gamma_{k,hsa}(\pi,\bar\pi)
\le
3\left[
\Gamma_{hsa}(\pi,\bar\pi)
+
\frac{32H^2\beta_k}{N_{k,hsa}^\mathcal D-1}
+
9\bigl(\Delta_{k,hsa}^\mathcal G\bigr)^2
\right]
\end{equation}
holds.

Let $\mathcal E_B$ be the event on which
Lemma~\ref{lem:iid-linear-bernstein} holds simultaneously over all epochs
and all policy pairs in the predictable parent classes.  The same confidence
allocation gives
\[
\mathbb P_{\mathbf{\mathcal M}}
\left(
\mathcal E_B^c
\right)
\le
\frac{\delta}{6}.
\]

\paragraph{The global good event.}
Define
\[
\mathcal E
:=
\mathcal E_{\rm trim}
\cap
\mathcal E_G
\cap
\mathcal E_D
\cap
\mathcal E_\Gamma
\cap
\mathcal E_B
\cap
\mathcal E_{\rm st}.
\]
The preceding bounds imply
\[
\begin{aligned}
\mathbb P_{\mathbf{\mathcal M}}(\mathcal E^c)
&\le
\frac{\delta}{6}
+
\frac{\delta}{6}
+
\frac{\delta}{6}
+
\frac{\delta}{6}
+
\frac{\delta}{6}
+
\frac{\delta}{12}
=
\frac{11\delta}{12}
<
\delta.
\end{aligned}
\]
Hence
\[
\mathbb P_{\mathbf{\mathcal M}}(\mathcal E)
\ge
1-\delta.
\]
We work on $\mathcal E$ for the remainder of the proof.

\paragraph{Policy Localization.}
We prove simultaneously by induction that every epoch $k$ that completes
the localization step satisfies
\begin{equation}
\Pi_{\epsilon_k}
\subseteq
\widehat\Pi_k
\subseteq
\Pi_{3\epsilon_k},
\label{eq:thm-policy-localization}
\end{equation}
and, whenever epoch $k$ reaches the prospective design step,
\begin{equation}
\widehat b_{k,\rm reg}(\pi,\bar\pi)
=
4\epsilon_k,
\qquad
\forall \pi,\bar\pi\in\widehat\Pi_k.
\label{eq:thm-regularized-margin-induction}
\end{equation}

At the initial epoch, we have
\(
\epsilon_1=H
\).
Whenever epoch $1$ proceeds past the offline-coverage check, recall that we set
\(
\widehat\Pi_1=\Pi_d
\).
Since every deterministic policy has value in $[0,H]$,
\[
\Pi_{\epsilon_1}
=
\Pi_d
=
\widehat\Pi_1,
\]
so \eqref{eq:thm-policy-localization} holds for $k=1$.
Moreover, if epoch \(1\) reaches the prospective design step, then
\(\epsilon<\epsilon_1=H\). Since empirical policy values also lie in
\([0,H]\),
\[
\widehat D_1(\pi,\bar\pi)\ge -H=-\epsilon_1
\]
for every \(\pi,\bar\pi\in\widehat\Pi_1\). Therefore
\[
\epsilon-\widehat D_1(\pi,\bar\pi)
\le
\epsilon+\epsilon_1
<
2\epsilon_1
<
4\epsilon_1,
\]
and hence, by the definition of the regularized margin,
\[
\widehat b_{1,\rm reg}(\pi,\bar\pi)
=
4\epsilon_1.
\]
Thus both statements hold at the initial epoch.

Suppose that the two induction statements hold at epoch \(k\), that epoch
\(k\) reaches the prospective design step, and that
\[
\mathsf{OffCover}_{k+1}=1.
\]
By \eqref{eq:thm-regularized-margin-induction},
\[
\widehat b_{k,\rm reg}(\pi,\bar\pi)
=
4\epsilon_k,
\qquad
\pi,\bar\pi\in\widehat\Pi_k.
\]
Since \(B_k\) is chosen so that
\[
\widehat\Phi_k^\star(B_k)
\le
\frac1{16},
\]
the optimizer
\((\bar\pi_k,\widehat{\mathbf n}_k)\) satisfies
\begin{equation}
\begin{aligned}
\max_{\pi\in\widehat\Pi_k}
\widehat{\operatorname{rad}}_k^+
\left(
\pi,\bar\pi_k;
\widehat{\mathbf n}_k
\right)
&\le
\frac{4\epsilon_k}{16}
=
\frac{\epsilon_{k+1}}2.
\end{aligned}
\label{eq:thm-design-radius}
\end{equation}

Together with
\(\mathsf{OffCover}_{k+1}=1\) and
\eqref{eq:thm-supported-uc}, this implies
\begin{equation}
N_{k+1,hsa}^{\mathcal D}
\ge
n_{k+1}^{\rm uc}
\qquad
\forall(h,s,a).
\label{eq:thm-all-row-uc}
\end{equation}

Fix \(\pi\in\widehat\Pi_k\). On
\(\mathcal E_{\rm st}\cap\mathcal E_\Gamma\),
\eqref{eq:thm-realized-prospective-count},
\eqref{eq:thm-variance-control}, and
\eqref{eq:thm-design-radius} give
\begin{equation}
\begin{aligned}
&
\sqrt{
2\beta_{k+1}
\sum_{h,s,a}
\frac{
\Gamma_{hsa}
(
\pi,\bar\pi_k
)
}{
N_{k+1,hsa}^{\mathcal D}
}
}
\\
&\le
2
\sqrt{
\beta_{k+1}
\sum_{h,s,a}
\frac{
\overline\Gamma_{k,hsa}
(
\pi,\bar\pi_k
)
}{
N_{k+1,hsa}^{+}
(
\widehat{\mathbf n}_k
)
}
}
\\
&=
\frac14
\widehat{\operatorname{rad}}_k^+
\left(
\pi,\bar\pi_k;
\widehat{\mathbf n}_k
\right)
\\
&\le
\frac{\epsilon_{k+1}}8.
\end{aligned}
\label{eq:thm-leading-bernstein}
\end{equation}

It remains to control the Bernstein range term and the nonlinear remainder.

By the definition
\[
n_{k+1}^{\rm uc}
=
\left\lceil
\frac{
C_{\rm uc}\beta_{k+1}
}{
\epsilon_{k+1}
}
\right\rceil,
\qquad
C_{\rm uc}
\ge
\max
\left\{
64H^2,
8K_H
\right\},
\]
\eqref{eq:thm-all-row-uc} yields
\begin{equation}
\frac{2H}{3}
\beta_{k+1}
\max_{h,s,a}
\frac1{N_{k+1,hsa}^{\mathcal D}}
\le
\frac{
2H
}{
3C_{\rm uc}
}
\epsilon_{k+1}
\le
\frac{\epsilon_{k+1}}{96}.
\label{eq:thm-bernstein-range}
\end{equation}

Furthermore,
\[
\alpha_{k+1}^{\mathcal D,r}
\vee
\alpha_{k+1}^{\mathcal D,P}
\le
\sqrt{
\frac{
\beta_{k+1}
}{
2n_{k+1}^{\rm uc}
}
}.
\]
Therefore Lemma~\ref{lem:iid-remainder} and
\(C_{\rm uc}\ge8K_H\) imply
\begin{equation}
\begin{aligned}
\left|
\operatorname{Rem}_{\widehat{\mathbf{\mathcal M}}_{k+1}^{\mathcal D}}
\left(
\pi,\bar\pi_k
\right)
\right|
&\le
K_H
\frac{
\beta_{k+1}
}{
2n_{k+1}^{\rm uc}
}
\\
&\le
\frac{\epsilon_{k+1}}{16}.
\end{aligned}
\label{eq:thm-remainder-sixteenth}
\end{equation}
where $\operatorname{Rem}_{\widehat{\mathbf{\mathcal M}}_{k+1}^{\mathcal D}}
\left(
\pi,\bar\pi_k
\right)
$ is defined in Lemma~\ref{lem:model-difference-expansion}.
Combining
\eqref{eq:thm-bernstein-range}
and
\eqref{eq:thm-remainder-sixteenth},
\begin{equation}
\frac{2H}{3}
\beta_{k+1}
\max_{h,s,a}
\frac1{N_{k+1,hsa}^{\mathcal D}}
+
\left|
\operatorname{Rem}_{\widehat{\mathbf{\mathcal M}}_{k+1}^{\mathcal D}}
\left(
\pi,\bar\pi_k
\right)
\right|
<
\frac{\epsilon_{k+1}}8.
\label{eq:thm-lower-order-eighth}
\end{equation}

Applying Lemma~\ref{lem:model-difference-expansion} and
Lemma~\ref{lem:iid-linear-bernstein}, and combining
\eqref{eq:thm-leading-bernstein} with
\eqref{eq:thm-lower-order-eighth}, we obtain
\begin{equation}
\sup_{\pi\in\widehat\Pi_k}
\left|
\widehat D_{k+1}
\left(
\pi,\bar\pi_k
\right)
-
D
\left(
\pi,\bar\pi_k
\right)
\right|
\le
\frac{\epsilon_{k+1}}4.
\label{eq:thm-referencewise-quarter}
\end{equation}

First take any
\(\pi\in\Pi_{\epsilon_{k+1}}\).
Since
\(\epsilon_{k+1}\le\epsilon_k\),
\[
\pi
\in
\Pi_{\epsilon_k}
\subseteq
\widehat\Pi_k.
\]
For every \(\pi'\in\widehat\Pi_k\),
\[
\begin{aligned}
D
\left(
\pi',\bar\pi_k
\right)
-
D
\left(
\pi,\bar\pi_k
\right)
=
V_1^{\pi'}-V_1^\pi \le
V_1^\star-V_1^\pi
\le
\epsilon_{k+1}.
\end{aligned}
\]
Using \eqref{eq:thm-referencewise-quarter} twice,
\[
\begin{aligned}
\widehat D_{k+1}
\left(
\pi',\bar\pi_k
\right)
-
\widehat D_{k+1}
\left(
\pi,\bar\pi_k
\right)
\le
D
\left(
\pi',\bar\pi_k
\right)
-
D
\left(
\pi,\bar\pi_k
\right)
+
\frac{\epsilon_{k+1}}2
<
\frac52\epsilon_{k+1}.
\end{aligned}
\]
Hence, by the localization rule
\eqref{eq:app-localized-class},
\[
\Pi_{\epsilon_{k+1}}
\subseteq
\widehat\Pi_{k+1}.
\]

Conversely, take
\(\pi\in\widehat\Pi_{k+1}\).
Since
\[
\pi^\star
\in
\Pi_{\epsilon_k}
\subseteq
\widehat\Pi_k,
\]
the localization rule
\eqref{eq:app-localized-class} gives
\[
\widehat D_{k+1}
\left(
\pi^\star,\bar\pi_k
\right)
-
\widehat D_{k+1}
\left(
\pi,\bar\pi_k
\right)
\le
\frac52\epsilon_{k+1}.
\]
Again using \eqref{eq:thm-referencewise-quarter} twice,
\[
\begin{aligned}
V_1^\star-V_1^\pi
&=
D
\left(
\pi^\star,\bar\pi_k
\right)
-
D
\left(
\pi,\bar\pi_k
\right)
\\
&\le
\widehat D_{k+1}
\left(
\pi^\star,\bar\pi_k
\right)
-
\widehat D_{k+1}
\left(
\pi,\bar\pi_k
\right)
+
\frac{\epsilon_{k+1}}2
\\
&\le
3\epsilon_{k+1}.
\end{aligned}
\]
Thus
\[
\widehat\Pi_{k+1}
\subseteq
\Pi_{3\epsilon_{k+1}},
\]
and therefore
\[
\Pi_{\epsilon_{k+1}}
\subseteq
\widehat\Pi_{k+1}
\subseteq
\Pi_{3\epsilon_{k+1}}.
\]

It remains to propagate the regularized-margin statement. Suppose that epoch
\(k+1\) reaches the prospective design step. Then, by the stopping rule,
\[
\epsilon<\epsilon_{k+1}.
\]
For any
\(\pi,\bar\pi\in\widehat\Pi_{k+1}\), we have
\(\bar\pi\in\widehat\Pi_k\), and the localization rule for \(\pi\) gives
\[
\max_{\pi'\in\widehat\Pi_k}
\left[
\widehat D_{k+1}
\left(
\pi',\bar\pi_k
\right)
-
\widehat D_{k+1}
\left(
\pi,\bar\pi_k
\right)
\right]
\le
\frac52\epsilon_{k+1}.
\]
In particular, taking \(\pi'=\bar\pi\),
\[
\widehat D_{k+1}
\left(
\bar\pi,\bar\pi_k
\right)
-
\widehat D_{k+1}
\left(
\pi,\bar\pi_k
\right)
\le
\frac52\epsilon_{k+1}.
\]
Since empirical value differences satisfy
\[
\widehat D_{k+1}(\pi,\bar\pi)
=
\widehat D_{k+1}
\left(
\pi,\bar\pi_k
\right)
-
\widehat D_{k+1}
\left(
\bar\pi,\bar\pi_k
\right),
\]
it follows that
\[
\widehat D_{k+1}(\pi,\bar\pi)
\ge
-\frac52\epsilon_{k+1}.
\]
Therefore
\[
\epsilon-\widehat D_{k+1}(\pi,\bar\pi)
\le
\epsilon+\frac52\epsilon_{k+1}
<
4\epsilon_{k+1}.
\]
By the definition of the regularized margin,
\[
\widehat b_{k+1,\rm reg}(\pi,\bar\pi)
=
4\epsilon_{k+1},
\qquad
\forall\pi,\bar\pi\in\widehat\Pi_{k+1}.
\]
Thus both induction statements hold at epoch \(k+1\), completing the
induction. 
\paragraph{Accuracy of the output at a stopping epoch.}

First suppose that the algorithm terminates at the initial offline-coverage
check because
\[
\mathsf{OffCover}_1=0.
\]
The algorithm then returns an arbitrary policy in $\Pi_d$. Since all policy
values lie in $[0,H]$ and $\epsilon_1=H$,
\[
V_1^\star-V_1^{\widehat\pi_1}
\le
H
=
\epsilon_1.
\]

Now suppose that the algorithm stops after completing the localization step
at some epoch $k\ge1$.

If $k=1$, then again
\[
V_1^\star-V_1^{\widehat\pi_1}
\le
H
=
\epsilon_1,
\]
so the desired guarantee holds.

It remains to consider $k\ge2$. Since epoch $k$ is entered, epoch $k-1$
completed prospective design. Hence
\eqref{eq:thm-referencewise-quarter}, applied with $k-1$ in place of $k$,
gives
\begin{equation}
\sup_{\pi\in\widehat\Pi_{k-1}}
\left|
\widehat D_k
\left(
\pi,\bar\pi_{k-1}
\right)
-
D
\left(
\pi,\bar\pi_{k-1}
\right)
\right|
\le
\frac{\epsilon_k}{4}.
\label{eq:thm-current-referencewise-quarter}
\end{equation}

If
\[
|\widehat\Pi_k|=1,
\]
then, by \eqref{eq:thm-policy-localization},
\[
\pi^\star
\in
\Pi_{\epsilon_k}
\subseteq
\widehat\Pi_k.
\]
Therefore
\[
\widehat\Pi_k
=
\{\pi^\star\},
\]
and the returned policy is optimal.

For every other stopping rule occurring after the epoch-$k$ localization
step, including
\[
\epsilon_k\le\epsilon,
\qquad
\mathsf{OffCover}_{k+1}=0,
\qquad\text{or}\qquad
\widehat\Phi_{k,\infty}^\star\ge\frac1{32},
\]
the algorithm returns
\[
\widehat\pi_k
\in
\arg\max_{\pi\in\widehat\Pi_k}
\widehat V_1^{\pi,k,\mathcal D}.
\]
By \eqref{eq:thm-policy-localization},
\[
\pi^\star
\in
\Pi_{\epsilon_k}
\subseteq
\widehat\Pi_k
\subseteq
\widehat\Pi_{k-1}.
\]
Moreover,
\[
\widehat\pi_k
\in
\widehat\Pi_k
\subseteq
\widehat\Pi_{k-1}.
\]
Since $\widehat\pi_k$ maximizes the fresh empirical value over
$\widehat\Pi_k$,
\[
\widehat D_k
\left(
\pi^\star,\bar\pi_{k-1}
\right)
-
\widehat D_k
\left(
\widehat\pi_k,\bar\pi_{k-1}
\right)
\le
0.
\]
Using
\eqref{eq:thm-current-referencewise-quarter}
for $\pi^\star$ and $\widehat\pi_k$, respectively, gives
\[
\begin{aligned}
V_1^\star
-
V_1^{\widehat\pi_k}
&=
D
\left(
\pi^\star,\bar\pi_{k-1}
\right)
-
D
\left(
\widehat\pi_k,\bar\pi_{k-1}
\right)
\\
&\le
\widehat D_k
\left(
\pi^\star,\bar\pi_{k-1}
\right)
-
\widehat D_k
\left(
\widehat\pi_k,\bar\pi_{k-1}
\right)
+
\frac{\epsilon_k}{2}
\\
&\le
\frac{\epsilon_k}{2}
\\
&\le
\epsilon_k.
\end{aligned}
\]

Therefore every stopping output covered by the theorem is
$\epsilon_k$-optimal.
\paragraph{Finite termination and number of epochs.}
It remains to show that the algorithm cannot continue indefinitely.

At any epoch $k$, the algorithm may terminate immediately because of one of
its stopping conditions.  Otherwise, it reaches the prospective design step
with
\[
\widehat\Phi_{k,\infty}^\star
<
\frac1{32}.
\]
Since $\widehat\Phi_k^\star(B)$ is nonincreasing in $B$ and
\[
\widehat\Phi_{k,\infty}^\star
=
\inf_{B<\infty}
\widehat\Phi_k^\star(B),
\]
there exists a finite $B_0$ such that
\[
\widehat\Phi_k^\star(B_0)
\le
\frac1{32}.
\]
Hence the dyadic search defining $B_k$ terminates after finitely many steps,
and every nonterminal epoch advances to the next epoch.

Since
\[
\epsilon_k
=
H2^{-(k-1)},
\]
the first epoch satisfying
\[
\epsilon_k
\le
\epsilon
\]
obeys
\[
k
\le
2+\log_2\frac{H}{\epsilon}.
\]
If the algorithm has not stopped earlier, it therefore stops at this epoch.
Consequently,
\[
k_{\rm stop}
=
O
\left(
\log
\frac{H}{\epsilon}
\right).
\]

Whenever
\[
3\epsilon_k
<
\Delta_{\min},
\]
\eqref{eq:thm-policy-localization} implies
\[
\widehat\Pi_k
\subseteq
\Pi_{3\epsilon_k}
=
\{\pi^\star\}.
\]
Together with
$\pi^\star\in\widehat\Pi_k$, this gives
\[
\widehat\Pi_k
=
\{\pi^\star\},
\]
so the singleton stopping rule terminates the algorithm at epoch $k$.

Thus the algorithm stops no later than the first epoch for which either
\[
\epsilon_k\le\epsilon
\qquad\text{or}\qquad
3\epsilon_k<\Delta_{\min}.
\]
Since $\epsilon_k=H2^{-(k-1)}$,
\[
k_{\rm stop}
=
O
\left(
1+
\log_2
\frac{
H
}{
\epsilon\vee\Delta_{\min}
}
\right).
\]
This proves both claims.
\end{proof}

\subsection{Proof of Theorem~\ref{thm:support-floor-stopping}}

\begin{proof}
Condition on the global good event used in the proof of
Theorem~\ref{thm:stopping-correctness}, which holds with probability at
least $1-\delta$. If $n_0=0$, then $\mathrm{OffCover}_1=0$ since
$\mathcal U_c\ne\varnothing$, so the infinite-budget stopping
event cannot occur. We therefore assume $n_0>0$ below.

Suppose that the algorithm stops at epoch $k$ because
\[
\widehat\Phi_{k,\infty}^{\star}
\ge
\frac1{32}.
\]
Every epoch reaching the computation of
$\widehat\Phi_{k,\infty}^{\star}$ satisfies
$\epsilon_k>\epsilon$.  Moreover, by
\eqref{eq:thm-policy-localization},
\[
\Pi_{\epsilon_k}(M)
\subseteq
\widehat\Pi_k
\subseteq
\Pi_{3\epsilon_k}(M).
\]

Fix any
\[
\bar\pi\in\Pi_{\epsilon_k}(M).
\]
Since $\bar\pi\in\widehat\Pi_k$ and
$\widehat\Phi_{k,\infty}^{\star}$ minimizes over the reference policy,
\begin{equation}
\frac1{32}
\le
\max_{\pi\in\widehat\Pi_k}
\frac{
8\sqrt{
\beta_{k+1}
\displaystyle\sum_{x\in\mathcal U_c}
\frac{
\overline\Gamma_{k,x}(\pi,\bar\pi)
}{
N_{0,k+1,x}^{\mathcal D}
}
}
}{
\widehat b_{k,\mathrm{reg}}(\pi,\bar\pi)
}.
\label{eq:accuracy-gamma-bar-lower-bound}
\end{equation}
Here the infinite-query limit removes all collaborator-supported rows,
leaving only $x\in\mathcal U_c$.

We first compare the empirical variance envelope with the population
variance. By Lemma~\ref{lem:variance-envelope-general},
we obtain
\[
\overline\Gamma_{k,x}(\pi,\bar\pi)
\le
3
\left[
\Gamma_x(\pi,\bar\pi)
+
\frac{
32H^2\beta_k
}{
N_{k,x}^{\mathcal D}-1
}
+
9
\left(
\Delta_{k,x}^{\mathcal G}
\right)^2
\right].
\]

Because epoch $k$ reaches prospective design, every row used in this
variance envelope satisfies the uniform fresh-row coverage condition
\[
N_{k,x}^{\mathcal D}
\ge
n_k^{\rm uc}
=
\left\lceil
C_{\rm uc}\frac{\beta_k}{\epsilon_k}
\right\rceil.
\]
Since $C_{\rm uc}\beta_k/\epsilon_k\ge2$,
\[
N_{k,x}^{\mathcal D}-1
\ge
n_k^{\rm uc}-1
\ge
\frac{C_{\rm uc}\beta_k}{2\epsilon_k}.
\]
Therefore, using $C_{\rm uc}\ge64H^2$,
\[
\frac{32H^2\beta_k}
{N_{k,x}^{\mathcal D}-1}
\le
\epsilon_k.
\]

The historical sensitivity guarantee on the same good event gives
\[
\left(
\Delta_{k,x}^{\mathcal G}
\right)^2
\le
\epsilon_{k+1}
=
\frac{\epsilon_k}{2}.
\]
Hence, uniformly over
$\pi\in\widehat\Pi_k$ and $x$,
\begin{equation}
\overline\Gamma_{k,x}(\pi,\bar\pi)
\le
3\Gamma_x(\pi,\bar\pi)
+
\frac{33}{2}\epsilon_k.
\label{eq:local-variance-envelope-upper}
\end{equation}

We next compare the regularized empirical margin with the population
$\epsilon_k$-margin.  Recall that
\[
b_{\epsilon_k}(\pi,\bar\pi)
:=
\epsilon_k
-
D(\pi,\bar\pi)
=
\epsilon_k
-
V_1^\pi
+
V_1^{\bar\pi}.
\]
Since
$\bar\pi\in\Pi_{\epsilon_k}(M)$,
\[
D(\pi,\bar\pi)
\le
V_1^\star-V_1^{\bar\pi}
\le
\epsilon_k,
\]
and therefore
\[
b_{\epsilon_k}(\pi,\bar\pi)\ge0.
\]
On the other hand,
$\pi\in\widehat\Pi_k\subseteq\Pi_{3\epsilon_k}(M)$ implies
\[
D(\pi,\bar\pi)
=
\left(
V_1^\star-V_1^{\bar\pi}
\right)
-
\left(
V_1^\star-V_1^\pi
\right)
\ge
-3\epsilon_k.
\]
Thus
\[
0
\le
b_{\epsilon_k}(\pi,\bar\pi)
\le
4\epsilon_k.
\]
By definition,
\[
\widehat b_{k,\mathrm{reg}}(\pi,\bar\pi)
=
\max
\left\{
\epsilon-\widehat D_k(\pi,\bar\pi),
4\epsilon_k
\right\},
\]
so in particular
\begin{equation}
\widehat b_{k,\mathrm{reg}}(\pi,\bar\pi)
\ge
4\epsilon_k
\ge
b_{\epsilon_k}(\pi,\bar\pi).
\label{eq:local-margin-comparison}
\end{equation}

We now control the additive part in
\eqref{eq:local-variance-envelope-upper}.
Since the algorithm has reached prospective design,
$\mathsf{OffCover}_{k+1}=1$ on every
$x\in\mathcal U_c$, and hence
\[
N_{0,k+1,x}^{\mathcal D}
\ge
n_{k+1}^{\rm uc}
=
\left\lceil
C_{\rm uc}
\frac{\beta_{k+1}}{\epsilon_{k+1}}
\right\rceil,
\qquad
x\in\mathcal U_c.
\]
Using
$\epsilon_{k+1}=\epsilon_k/2$ and
$|\mathcal U_c|\le H|\mathcal S||\mathcal A|$,
\[
\frac{\beta_{k+1}}{\epsilon_k}
\sum_{x\in\mathcal U_c}
\frac1{N_{0,k+1,x}^{\mathcal D}}
\le
\frac{H|\mathcal S||\mathcal A|}{2C_{\rm uc}}.
\]
Combining this with
\eqref{eq:local-variance-envelope-upper} and
\eqref{eq:local-margin-comparison} gives
\[
\begin{aligned}
&
\max_{\pi\in\widehat\Pi_k}
\frac{
8\sqrt{
\beta_{k+1}
\displaystyle\sum_{x\in\mathcal U_c}
\frac{
\overline\Gamma_{k,x}(\pi,\bar\pi)
}{
N_{0,k+1,x}^{\mathcal D}
}
}
}{
\widehat b_{k,\mathrm{reg}}(\pi,\bar\pi)
}
\\
&\le
8\sqrt3
\max_{\pi\in\widehat\Pi_k}
\frac{
\sqrt{
\beta_{k+1}
\displaystyle\sum_{x\in\mathcal U_c}
\frac{
\Gamma_x(\pi,\bar\pi)
}{
N_{0,k+1,x}^{\mathcal D}
}
}
}{
\widehat b_{k,\mathrm{reg}}(\pi,\bar\pi)
}
\\
&\qquad
+
2
\sqrt{
\frac{33}{2}
\frac{\beta_{k+1}}{\epsilon_k}
\sum_{x\in\mathcal U_c}
\frac1{N_{0,k+1,x}^{\mathcal D}}
}
\\
&\le
8\sqrt3
\max_{\pi\in\widehat\Pi_k}
\frac{
\sqrt{
\beta_{k+1}
\displaystyle\sum_{x\in\mathcal U_c}
\frac{
\Gamma_x(\pi,\bar\pi)
}{
N_{0,k+1,x}^{\mathcal D}
}
}
}{
\widehat b_{k,\mathrm{reg}}(\pi,\bar\pi)
}
+
\sqrt{
\frac{
33H|\mathcal S||\mathcal A|
}{
C_{\rm uc}
}
}.
\end{aligned}
\]
With
\[
C_{\rm uc}
\ge
67584H|\mathcal S||\mathcal A|
=
4096\cdot\frac{33}{2}H|\mathcal S||\mathcal A|,
\]
the last term is at most
\[
\frac1{32\sqrt2}.
\]
Combining this with
\eqref{eq:accuracy-gamma-bar-lower-bound} yields
\[
\max_{\pi\in\widehat\Pi_k}
\frac{
\sqrt{
\beta_{k+1}
\displaystyle\sum_{x\in\mathcal U_c}
\frac{
\Gamma_x(\pi,\bar\pi)
}{
N_{0,k+1,x}^{\mathcal D}
}
}
}{
\widehat b_{k,\mathrm{reg}}(\pi,\bar\pi)
}
\ge
c_0,
\]
where
\[
c_0
:=
\frac{
1-2^{-1/2}
}{
256\sqrt3
}
>0.
\]
By \eqref{eq:local-margin-comparison},
\[
\frac1{
b_{\epsilon_k}(\pi,\bar\pi)
}
\ge
\frac1{
\widehat b_{k,\mathrm{reg}}(\pi,\bar\pi)
},
\]
and therefore
\begin{equation}
c_0
\le
\max_{\pi\in\widehat\Pi_k}
\frac{
\sqrt{
\beta_{k+1}
\displaystyle\sum_{x\in\mathcal U_c}
\frac{
\Gamma_x(\pi,\bar\pi)
}{
N_{0,k+1,x}^{\mathcal D}
}
}
}{
b_{\epsilon_k}(\pi,\bar\pi)
}.
\label{eq:local-population-leading-lower}
\end{equation}

It remains to replace the realized offline counts by their population
counterparts. On the two-fold trimming event,
Lemma~\ref{lem:two-fold-trimming-guarantee} and the deterministic
epoch-wise split give
\[
N_{0,k+1,x}^{\mathcal D}
\ge
\left[
\frac{w_{k+1}}{2}
\left[
\frac{n_0d_x^\mu}{8}
-
5
\sqrt{
n_0d_x^\mu\ell_{\rm trim}
}
\right]_+
-
2
\right]_+.
\]
Using
\[
5\sqrt{y\ell_{\rm trim}}
\le
\frac{y}{16}
+
100\ell_{\rm trim},
\qquad
y\ge0,
\]
with $y=n_0d_x^\mu$, we obtain
\begin{equation}
N_{0,k+1,x}^{\mathcal D}
\ge
\left[
\frac{
w_{k+1}n_0d_x^\mu
}{32}
-
50w_{k+1}\ell_{\rm trim}
-
2
\right]_+.
\label{eq:local-trim-preliminary}
\end{equation}

On the other hand, because the algorithm has reached the prospective
design at epoch $k$, we have
\[
\mathsf{OffCover}_{k+1}=1.
\]
Hence, for every $x\in\mathcal U_c$,
\begin{equation}
N_{0,k+1,x}^{\mathcal D}
\ge
50w_{k+1}\ell_{\rm trim}+2.
\label{eq:local-offcover-trim-threshold}
\end{equation}

Then
\eqref{eq:local-trim-preliminary} and
\eqref{eq:local-offcover-trim-threshold} imply

\begin{equation}
N_{0,k+1,x}^{\mathcal D}
\ge
\frac{
w_{k+1}n_0d_x^\mu
}{64},
\qquad
\forall x\in\mathcal U_c.
\label{eq:local-trim-population-comparison}
\end{equation}

It follows from
\eqref{eq:local-population-leading-lower} and
\eqref{eq:local-trim-population-comparison} that
\[
c_0
\le
\sqrt{
\frac{64
\beta_{k+1}
}{
w_{k+1}
}
}
\max_{\pi\in\widehat\Pi_k}
\frac{
\sqrt{
\displaystyle\sum_{x\in\mathcal U_c}
\frac{
\Gamma_x(\pi,\bar\pi)
}{
n_0d_x^\mu
}
}
}{
b_{\epsilon_k}(\pi,\bar\pi)
}.
\]
Hence, for some universal constant $c_1>0$,
\[
\sup_{\pi\in\Pi_d\setminus\{\bar\pi\}}
\frac{
2
}{
b_{\epsilon_k}(\pi,\bar\pi)^2
}
\sum_{x\in\mathcal U_c}
\frac{
\Gamma_x(\pi,\bar\pi)
}{
n_0d_x^\mu
}
\ge
c_1
\frac{
w_{k+1}
}{
\beta_{k+1}
}.
\]
Here we used
$\widehat\Pi_k\subseteq\Pi_d$; moreover, the policy witnessing the
positive lower bound cannot equal $\bar\pi$, since
$\Gamma_x(\bar\pi,\bar\pi)=0$ for every $x$.

The preceding inequality holds for every
$\bar\pi\in\Pi_{\epsilon_k}(M)$.  Taking the infimum over
$\bar\pi$ and recalling the definition of
$\mathcal F_\infty$ therefore gives
\[
\mathcal F_\infty(\epsilon_k;n_0)
\ge
c_1
\frac{
w_{k+1}
}{
\beta_{k+1}
}.
\]
Since
\[
w_{k+1}
=
\frac{
6
}{
\pi^2(k+1)^2
},
\]
there exists a universal constant $c_2>0$ such that
\begin{equation}
\mathcal F_\infty(\epsilon_k;n_0)
\ge
\frac{
2c_2
}{
\beta_{k+1}(k+1)^2
}.
\label{eq:local-F-infty-lower}
\end{equation}

By homogeneity in the effective offline sample size,
\[
\mathcal F_\infty(\eta;n)
=
\frac{n_0}{n}
\mathcal F_\infty(\eta;n_0).
\]
Set
\[
n_k^{\rm lim}
:=
\frac{
c_2n_0
}{
\beta_{k+1}(k+1)^2
}.
\]
Then \eqref{eq:local-F-infty-lower} implies
\[
\mathcal F_\infty
\left(
\epsilon_k;
n_k^{\rm lim}
\right)
\ge
2
>
1.
\]
Because
$\mathcal F_\infty(\eta;n)$ is nonincreasing in $\eta$, and
by the definition of $\epsilon^{\lim}(n)$, 
we conclude that
\begin{equation}
\epsilon_k
\le
\epsilon^{\rm lim}
\left(
\frac{
c_2n_0
}{
\beta_{k+1}(k+1)^2
}
\right).
\label{eq:local-epsilon-lim-intermediate}
\end{equation}

Finally, Theorem~\ref{thm:stopping-correctness} implies that every epoch
reaching prospective design satisfies
\[
k
=
O\left(
\log
\frac{
H
}{
\epsilon\vee\Delta_{\min}
}
\right).
\]
Moreover, from the definition of the confidence parameter and
$\widehat\Pi_k\subseteq\Pi_d$,
\[
\beta_{k+1}
=
O\left(
H|\mathcal S|\log (|\mathcal A|)
+
\log
\frac{
(k+1)H|\mathcal S||\mathcal A|
}{
\delta
}
\right).
\]
Consequently,
\[
\beta_{k+1}(k+1)^2
\le
H|\mathcal S|\,
\operatorname{polylog}
\left(
H,
|\mathcal S|,
|\mathcal A|,
\frac1{\epsilon\vee\Delta_{\min}},
\frac1\delta
\right).
\]
  Since
$\epsilon^{\rm lim}(n)$ is nonincreasing in $n$,
\eqref{eq:local-epsilon-lim-intermediate} yields
\[
\epsilon_k
\le
\epsilon^{\rm lim}
\left(
\frac{
n_0
}{
H|\mathcal S|\,
\operatorname{polylog}
\left(
H,
|\mathcal S|,
|\mathcal A|,
1/(\epsilon\vee\Delta_{\min}),
1/\delta
\right)
}
\right),
\]
\end{proof}

\subsection{Proof of Theorem~\ref{thm:query-complexity}}
\label{app:proof-query-complexity}

\begin{proof}
We work throughout on the global good event
$\mathcal E$ constructed in the proof of
Theorem~\ref{thm:stopping-correctness}. By construction,
\[
\mathbb P_{\mathbf{\mathcal M}}(\mathcal E)\ge 1-\delta.
\]
On this event, the proof of Theorem~\ref{thm:stopping-correctness} shows every epoch $k$ that reaches the prospective
allocation step satisfies
\begin{equation}
\Pi_{\epsilon_k}(\mathbf{\mathcal M})
\subseteq
\widehat\Pi_k
\subseteq
\Pi_{3\epsilon_k}(\mathbf{\mathcal M}),
\label{eq:qc-localization}
\end{equation}
and
\begin{equation}
\widehat b_{k,\rm reg}(\pi,\bar\pi)
=
4\epsilon_k,
\qquad
\pi,\bar\pi\in\widehat\Pi_k.
\label{eq:qc-margin}
\end{equation}

By Theorem~\ref{thm:stopping-correctness}, the algorithm stops no
later than the first epoch at which either
$\epsilon_k\le\epsilon$ or
$3\epsilon_k<\Delta_{\min}$.
Hence
\[
\epsilon_{K_\star}
=
H2^{-\lceil\log_2(6H/(\epsilon\vee\Delta_{\min}))\rceil}
\le
\frac{\epsilon\vee\Delta_{\min}}{6},
\]
If $\epsilon\ge\Delta_{\min}$, this implies
$\epsilon_{K_\star}<\epsilon$; if
$\Delta_{\min}>\epsilon$, then $\Delta_{\min}>0$, and hence the
optimal deterministic policy is unique, so
\eqref{eq:qc-localization} and
$3\epsilon_{K_\star}<\Delta_{\min}$ imply that
$\widehat\Pi_{K_\star}$ is a singleton.
Therefore
\begin{equation}
\epsilon_{k_{\rm stop}}
\ge
\epsilon_{K_\star}
\ge
\frac{\epsilon\vee\Delta_{\min}}{12} \qquad \text{and}\qquad k_{\rm stop}\le K_\star
\label{eq:qc-terminal-resolution}
\end{equation}

\paragraph{A finite budget controlled by the epoch-wise local quantity.}

For every collaborator-supported row,
the definition of $m_k^{\rm st}$ gives
\[
m_k^{\rm st}\pi_h^c(a\mid s)
\ge
50w_k\ell_{\rm trim}+2.
\]
Hence by \eqref{eq:local-trim-preliminary},
\begin{equation}
N_{0,k,hsa}^{\mathcal D}
+
m_k^{\rm st}\pi_h^c(a\mid s)
\ge
\frac{
n_0w_kd_{hsa}^{\mu}
}{32}.
\label{eq:qc-effective-offline-mass}
\end{equation}

Fix a design epoch $k<K_\star$. For this paragraph, let
\begin{equation}
\begin{aligned}
T_{{\rm LQ}}^k(\mathbf{\mathcal M})
:=
\inf_{\bar\pi\in\Pi_{\epsilon_k}(\mathbf{\mathcal M})}
\inf_{\boldsymbol \lambda\in\Delta_{H|\mathcal S|}}
\inf
\Bigg\{
T\ge0:\;
\sup_{\pi\in\Pi_{3\epsilon_k}(\mathbf{\mathcal M})}
\frac{\overline\beta_{k+1}}{\epsilon_k^2}
\sum_{h,s,a}
\frac{
\Gamma_{hsa}(\pi,\bar\pi)
}{
n_0w_{k+1}d_{hsa}^{\mu}
+
T\lambda_{hs}\pi_h^c(a\mid s)
}
\le
c_{\rm LQ}
\Bigg\},
\end{aligned}
\label{eq:qc-epochwise-LQ}
\end{equation}
where
\[
\overline\beta_{k+1}
:=
\log\!\left(
\frac{
2\pi^2 (k+1)^2 H|\mathcal S||\mathcal A|(|\mathcal S|+1)
\left|\Pi_{3\epsilon_k}(\mathbf{\mathcal M})\right|^2
}{
\delta
}
\right)
\] with $c_{\rm LQ}=2^{-20}$.
By \eqref{eq:qc-localization},
\[
\widehat\Pi_k
\subseteq
\Pi_{3\epsilon_k}(\mathbf{\mathcal M}),
\]
and hence
\(
\beta_{k+1}
\le
\overline\beta_{k+1}.
\)

Take any
\[
T>T_{{\rm LQ}}^k(\mathbf{\mathcal M}).
\]
Then there exist
\[
\bar\pi^0\in\Pi_{\epsilon_k}(\mathbf{\mathcal M}),
\qquad
\boldsymbol \lambda\in\Delta_{H|\mathcal S|}
\]
such that
\begin{equation}
\sup_{\pi\in\Pi_{3\epsilon_k}(\mathbf{\mathcal M})}
\frac{\overline\beta_{k+1}}{\epsilon_k^2}
\sum_{h,s,a}
\frac{
\Gamma_{hsa}(\pi,\bar\pi^0)
}{
n_0w_{k+1}d_{hsa}^{\mu}
+
T\lambda_{hs}\pi_h^c(a\mid s)
}
\le
c_{\rm LQ}.
\label{eq:qc-LQ-feasible}
\end{equation}
By \eqref{eq:qc-localization},
$\bar\pi^0\in\widehat\Pi_k$.

Since the algorithm reaches the design step and does not stop at the
infinite-budget criterion,
\[
\widehat\Phi_{k,\infty}^\star<\frac1{32}.
\]
Because $\widehat\Pi_k$ is finite, there exists
$\bar\pi_k^\infty\in\widehat\Pi_k$ such that
\begin{equation}
\max_{\pi\in\widehat\Pi_k}
\frac{C_0}{4\epsilon_k}
\sqrt{
\beta_{k+1}
\sum_{(h,s,a)\in\mathcal U_c}
\frac{
\overline\Gamma_{k,hsa}
(\pi,\bar\pi_k^\infty)
}{
N_{0,k+1,hsa}^{\mathcal D}
}
}
<
\frac1{32}.
\label{eq:qc-infinite-unsupported}
\end{equation}
Since $C_0=8$,
\begin{equation}
\frac{\beta_{k+1}}{\epsilon_k^2}
\sum_{(h,s,a)\in\mathcal U_c}
\frac{
\overline\Gamma_{k,hsa}
(\pi,\bar\pi_k^\infty)
}{
N_{0,k+1,hsa}^{\mathcal D}
}
<
\frac1{4096},
\qquad
\forall\pi\in\widehat\Pi_k.
\label{eq:qc-unsupported-envelope}
\end{equation}

Now allocate
\[
n_{hs}
:=
\left\lceil T\lambda_{hs}\right\rceil,
\qquad
\mathbf n:=\{n_{hs}\}_{h,s}.
\]
Then
\begin{equation}
\sum_{h,s}n_{hs}
\le
T+H|\mathcal S|.
\label{eq:qc-candidate-total}
\end{equation}
For every collaborator-supported row,
\eqref{eq:qc-effective-offline-mass} gives
\[
\begin{aligned}
N_{k+1,hsa}^{+}(\mathbf n)
&=
N_{0,k+1,hsa}^{\mathcal D}
+
\bigl(m_{k+1}^{\rm st}+n_{hs}\bigr)
\pi_h^c(a\mid s)
\\
&\ge
\frac{
n_0w_{k+1}d_{hsa}^{\mu}
}{32}
+
T\lambda_{hs}\pi_h^c(a\mid s)
\\
&\ge
\frac1{32}
\left(
n_0w_{k+1}d_{hsa}^{\mu}
+
T\lambda_{hs}\pi_h^c(a\mid s)
\right).
\end{aligned}
\]
For every
$\pi\in\widehat\Pi_k\subseteq
\Pi_{3\epsilon_k}(\mathbf{\mathcal M})$,
we have
\[
\Psi_{hsa}^{\pi,\bar\pi_k^\infty}
=
\Psi_{hsa}^{\pi,\bar\pi^0}
-
\Psi_{hsa}^{\bar\pi_k^\infty,\bar\pi^0}.
\]
Therefore
\[
\Gamma_{hsa}(\pi,\bar\pi_k^\infty)
\le
2\Gamma_{hsa}(\pi,\bar\pi^0)
+
2\Gamma_{hsa}(\bar\pi_k^\infty,\bar\pi^0).
\]
Both $\pi$ and $\bar\pi_k^\infty$ belong to
$\Pi_{3\epsilon_k}(\mathbf{\mathcal M})$, so
\eqref{eq:qc-LQ-feasible}, together with
$\beta_{k+1}\le\overline\beta_{k+1}$, yields
\begin{equation}
\frac{\beta_{k+1}}{\epsilon_k^2}
\sum_{(h,s,a)\notin\mathcal U_c}
\frac{
\Gamma_{hsa}(\pi,\bar\pi_k^\infty)
}{
N_{k+1,hsa}^{+}(\mathbf n)
}
\le
128c_{\rm LQ}.
\label{eq:qc-supported-leading}
\end{equation}

On the global good event, the variance envelope \eqref{eq:local-variance-envelope-upper} satisfies
\begin{equation}
\overline\Gamma_{k,hsa}
(\pi,\bar\pi_k^\infty)
\le
3\Gamma_{hsa}(\pi,\bar\pi_k^\infty)
+
C_{\rm env}\epsilon_k.
\label{eq:qc-envelope}
\end{equation} with $C_{\rm env}=33/2$.
Moreover, every collaborator-supported row satisfies
\[
N_{k+1,hsa}^{+}(\mathbf n)
\ge
m_{k+1}^{\rm st}p_{\min}
\ge
2n_{k+1}^{\rm uc}.
\]
Since
\[
n_{k+1}^{\rm uc}
\ge
\frac{
C_{\rm uc}\beta_{k+1}
}{
\epsilon_{k+1}
}
=
\frac{
2C_{\rm uc}\beta_{k+1}
}{
\epsilon_k
},
\]
we have
\begin{equation}
\frac{\beta_{k+1}}{\epsilon_k}
\sum_{(h,s,a)\notin\mathcal U_c}
\frac1{
N_{k+1,hsa}^{+}(\mathbf n)
}
\le
\frac{H|\mathcal S||\mathcal A|}{4C_{\rm uc}}.
\label{eq:qc-additive-envelope}
\end{equation}

Combining
\eqref{eq:qc-supported-leading},
\eqref{eq:qc-unsupported-envelope},
\eqref{eq:qc-envelope}, and
\eqref{eq:qc-additive-envelope}, we obtain, for every
$\pi\in\widehat\Pi_k$,
\[
\begin{aligned}
\frac{\beta_{k+1}}{\epsilon_k^2}
\sum_{h,s,a}
\frac{
\overline\Gamma_{k,hsa}
(\pi,\bar\pi_k^\infty)
}{
N_{k+1,hsa}^{+}(\mathbf n)
}
&\le
384c_{\rm LQ}
+
\frac1{4096}
+
\frac{C_{\rm env}H|\mathcal S||\mathcal A|}{4C_{\rm uc}}
\\
&\le
\frac3{8192}
+
\frac2{8192}
+
\frac2{8192}
\\
&=
\frac7{8192}.
\end{aligned}
\]
Here we used $c_{\rm LQ}=2^{-20}$ and 
\(
C_{\rm uc}
\ge
1024C_{\rm env}H|\mathcal S||\mathcal A|.
\)
Using \eqref{eq:qc-margin},
\[
\begin{aligned}
\frac{
\widehat{\operatorname{rad}}_k^+
(\pi,\bar\pi_k^\infty;\mathbf n)
}{
\widehat b_{k,\rm reg}
(\pi,\bar\pi_k^\infty)
}
&=
\frac{8}{4\epsilon_k}
\sqrt{
\beta_{k+1}
\sum_{h,s,a}
\frac{
\overline\Gamma_{k,hsa}
(\pi,\bar\pi_k^\infty)
}{
N_{k+1,hsa}^{+}(\mathbf n)
}
}
\\
&\le
2\sqrt{\frac7{8192}}
<
\frac1{16}.
\end{aligned}
\]
Thus the allocation $\mathbf n$ is feasible for the empirical
minimax program with total size at most
$T+H|\mathcal S|$. Since $B_k$ is the smallest dyadic budget
attaining the threshold $1/16$, letting
$T\downarrow T_{{\rm LQ}}^k(\mathbf{\mathcal M})$ gives
\begin{equation}
B_k
\le
C
\left(
T_{{\rm LQ}}^k(\mathbf{\mathcal M})
+
H|\mathcal S|
\right)
\label{eq:qc-budget-vs-epoch-LQ}
\end{equation}
for a universal constant $C$.

\paragraph{Early epochs.}

Let
\[
k_0:=\left\lceil\sqrt{K_\star}\right\rceil.
\]
For $k\le k_0$, we use a crude uniform allocation.
Note that
\begin{equation}
\Gamma_{hsa}(\pi,\bar\pi)
\le
4H^2.
\label{eq:qc-Gamma-uniform}
\end{equation}
As above, because the epoch reaches the design step,
\[
\widehat\Phi_{k,\infty}^\star<\frac1{32},
\]
so we may choose $\bar\pi_k^\infty$ satisfying
\eqref{eq:qc-unsupported-envelope}.

Let
\[
B'
:=
12288\,C_{\rm wc}
\frac{\overline\beta_{k+1}}{\epsilon_k^2},
\qquad
C_{\rm wc}
:=
\frac{4H^4|\mathcal S|^2|\mathcal A|}{p^c_{\min}},
\]
and take
\[
n_{hs}
=
\left\lceil
\frac{B'}{H|\mathcal S|}
\right\rceil.
\]
The actual size of this allocation is
\[
\sum_{h,s}n_{hs}
\le
B'+H|\mathcal S|.
\]
For every collaborator-supported row,
\[
N_{k+1,hsa}^{+}(\mathbf n)
\ge
n_{hs}\pi_h^c(a\mid s)
\ge
\frac{B'p^c_{\min}}{H|\mathcal S|}.
\]
Therefore, by \eqref{eq:qc-Gamma-uniform},
\[
\sum_{(h,s,a)\notin\mathcal U_c}
\frac{
\Gamma_{hsa}(\pi,\bar\pi_k^\infty)
}{
N_{k+1,hsa}^{+}(\mathbf n)
}
\le
\frac{C_{\rm wc}}{B'}.
\]
Since $\beta_{k+1}\le\overline\beta_{k+1}$,
\[
\frac{\beta_{k+1}}{\epsilon_k^2}
\sum_{(h,s,a)\notin\mathcal U_c}
\frac{
3\Gamma_{hsa}(\pi,\bar\pi_k^\infty)
}{
N_{k+1,hsa}^{+}(\mathbf n)
}
\le
\frac1{4096}.
\]
The unsupported-row contribution is at most $1/4096$ by
\eqref{eq:qc-unsupported-envelope}, and the additive part of the
variance envelope is at most $1/4096$ by
\eqref{eq:qc-additive-envelope} and the choice of $C_{\rm uc}$.
Hence
\[
\frac{\beta_{k+1}}{\epsilon_k^2}
\sum_{h,s,a}
\frac{
\overline\Gamma_{k,hsa}
(\pi,\bar\pi_k^\infty)
}{
N_{k+1,hsa}^{+}(\mathbf n)
}
\le
\frac3{4096}.
\]
Consequently,
\[
\frac{
\widehat{\operatorname{rad}}_k^+
(\pi,\bar\pi_k^\infty;\mathbf n)
}{
\widehat b_{k,\rm reg}
(\pi,\bar\pi_k^\infty)
}
\le
2\sqrt{\frac3{4096}}
<
\frac1{16}.
\]
After the dyadic search,
\begin{equation}
B_k
\le
C_{\rm early}
\frac{\overline\beta_{k+1}}{\epsilon_k^2}
+
2H|\mathcal S|,
\qquad
k\le k_0,
\label{eq:qc-early-budget}
\end{equation}
where $C_{\rm early}$ depends polynomially only on
$H,|\mathcal S|,|\mathcal A|,(p^c_{\min})^{-1}$.

Since
\[
\epsilon_k=H2^{-(k-1)},
\]
\[
\sum_{k=1}^{k_0}\frac1{\epsilon_k^2}
\le
\frac{4^{k_0}}{3H^2}.
\]
If $K_\star\ge7$, then
\[
2k_0
=
2\left\lceil\sqrt{K_\star}\right\rceil
\le
K_\star-1,
\]
so
\[
4^{k_0}
\le
2^{K_\star-1}
<
\frac{12H}{\epsilon\vee\Delta_{\min}}.
\]
The finitely many cases $K_\star<7$ are absorbed into the constant.
Since $\overline\beta_{k+1}$ is logarithmic,
\begin{equation}
\sum_{k\le k_0}B_k
=
\widetilde O
\left(
\frac1{\epsilon\vee\Delta_{\min}}
\right).
\label{eq:qc-early-sum}
\end{equation}

\paragraph{Tail epochs.}

For $n\ge0$, recall
\begin{equation}
\begin{aligned}
T_{{\rm LQ},n}(\mathbf{\mathcal M})
:=
\min_{\bar\pi\in\Pi_\epsilon(\mathbf{\mathcal M})}
\inf_{\boldsymbol\lambda\in\Delta_{H|\mathcal S|}}
\inf
\Bigg\{
t\ge0:\;
\max_{\pi\in\Pi_d\setminus\{\bar\pi\}}
\frac{2}{b(\pi,\bar\pi)^2}
\sum_{h,s,a}
\frac{
\Gamma_{hsa}(\pi,\bar\pi)
}{
nd_{hsa}^{\mu}
+
t\lambda_{hs}\pi_h^c(a\mid s)
}
\le1
\Bigg\}.
\end{aligned}
\label{eq:qc-finite-LQ}
\end{equation}
As usual, the ratio is interpreted in the extended-value sense when
$b(\pi,\bar\pi)=0$. The map
$n\mapsto T_{{\rm LQ},n}(\mathbf{\mathcal M})$
is nonincreasing.
Suppose $k_{\rm stop}>k_0$. Otherwise, the following tail-epoch analysis is vacuous.

For a tail epoch $k_0<k<k_{\rm stop}$, set
\[
a_k
:=
\frac{8\overline\beta_{k+1}}{c_{\rm LQ}},
\qquad
n_k^{\rm eff}
:=
\frac{n_0w_{k+1}}{a_k}.
\]
Take any
\[
t>
T_{{\rm LQ},n_k^{\rm eff}}(\mathbf{\mathcal M}).
\]
Then there exist
\[
\bar\pi\in\Pi_\epsilon(\mathbf{\mathcal M}),
\qquad
\boldsymbol \lambda\in\Delta_{H|\mathcal S|}
\]
such that
\begin{equation}
\max_{\pi\in\Pi_d\setminus\{\bar\pi\}}
\frac{2}{b(\pi,\bar\pi)^2}
\sum_{h,s,a}
\frac{
\Gamma_{hsa}(\pi,\bar\pi)
}{
n_k^{\rm eff}d_{hsa}^{\mu}
+
t\lambda_{hs}\pi_h^c(a\mid s)
}
\le1.
\label{eq:qc-finite-LQ-feasible}
\end{equation}

Because epoch $k$ reaches the design step,
$\epsilon_k>\epsilon$, and hence
\[
\bar\pi
\in
\Pi_\epsilon(\mathbf{\mathcal M})
\subseteq
\Pi_{\epsilon_k}(\mathbf{\mathcal M}).
\]
For every
$\pi\in\Pi_{3\epsilon_k}(\mathbf{\mathcal M})$,
\[
b(\pi,\bar\pi)
\in [0,
4\epsilon_k].
\]
Therefore \eqref{eq:qc-finite-LQ-feasible} implies
\begin{equation}
\sum_{h,s,a}
\frac{
\Gamma_{hsa}(\pi,\bar\pi)
}{
n_k^{\rm eff}d_{hsa}^{\mu}
+
t\lambda_{hs}\pi_h^c(a\mid s)
}
\le
8\epsilon_k^2,
\qquad
\pi\in\Pi_{3\epsilon_k}(\mathbf{\mathcal M}).
\label{eq:qc-tail-gap}
\end{equation}

Set $T=a_kt$. Since
\[
n_0w_{k+1}d_{hsa}^{\mu}
=
a_kn_k^{\rm eff}d_{hsa}^{\mu},
\]
\eqref{eq:qc-tail-gap} gives
\[
\begin{aligned}
&
\sup_{\pi\in\Pi_{3\epsilon_k}(\mathbf{\mathcal M})}
\frac{\overline\beta_{k+1}}{\epsilon_k^2}
\sum_{h,s,a}
\frac{
\Gamma_{hsa}(\pi,\bar\pi)
}{
n_0w_{k+1}d_{hsa}^{\mu}
+
T\lambda_{hs}\pi_h^c(a\mid s)
}
\\
&\qquad\le
\frac{
8\overline\beta_{k+1}
}{
a_k
}
=
c_{\rm LQ}.
\end{aligned}
\]
Hence
\begin{equation}
T_{{\rm LQ}}^k(\mathbf{\mathcal M})
\le
a_k
T_{{\rm LQ},n_k^{\rm eff}}(\mathbf{\mathcal M}).
\label{eq:qc-tail-epoch-comparison}
\end{equation}

Now choose 
\begin{equation}
L_\epsilon
\ge
1\vee
\max_{k_0<k<K_\star}
\frac{a_k}{w_{k+1}}.
\label{eq:qc-Lepsilon-choice}
\end{equation}
The maximum over an empty index set is understood as zero.
Since
$w_k=6/(\pi^2k^2)$,
$L_\epsilon$ is polynomial in logarithms of
\[
H,\quad
|\mathcal S|,\quad
|\mathcal A|,\quad
\frac{H}{\epsilon\vee\Delta_{\min}},
\quad
\frac1\delta,
\quad
\left|
\Pi_{3\epsilon_{k_0}}(\mathbf{\mathcal M})
\right|.
\]
For every tail epoch,
\[
n_k^{\rm eff}
=
\frac{n_0w_{k+1}}{a_k}
\ge
\frac{n_0}{L_\epsilon}.
\]
By monotonicity of $T_{{\rm LQ},n}$,
\begin{equation}
T_{{\rm LQ},n_k^{\rm eff}}(\mathbf{\mathcal M})
\le
T_{{\rm LQ},n_0/L_\epsilon}(\mathbf{\mathcal M}).
\label{eq:qc-LQ-monotonicity}
\end{equation}

Combining
\eqref{eq:qc-budget-vs-epoch-LQ},
\eqref{eq:qc-tail-epoch-comparison}, and
\eqref{eq:qc-LQ-monotonicity},
\[
\begin{aligned}
\sum_{k=k_0+1}^{k_{\rm stop}-1}B_k
\le
C
\left(
\sum_{k=k_0+1}^{K_\star-1}a_k
\right)
T_{{\rm LQ},n_0/L_\epsilon}(\mathbf{\mathcal M})
+
CH|\mathcal S|K_\star.
\end{aligned}
\]
The sum of the $a_k$'s is polylogarithmic, and
\[
K_\star
=
O\left(
1+
\log
\frac{H}{\epsilon\vee\Delta_{\min}}
\right).
\]
Therefore
\begin{equation}
\sum_{k=k_0+1}^{k_{\rm stop}-1}B_k
=
\widetilde O
\left(
T_{{\rm LQ},n_0/L_\epsilon}(\mathbf{\mathcal M})
\right)
+
\widetilde O
\left(
\frac1{\epsilon\vee\Delta_{\min}}
\right).
\label{eq:qc-tail-sum}
\end{equation}

\paragraph{Stabilization queries.}

Only the allocations selected before the terminal epoch are executed,
whereas a stabilization batch is collected at every reached epoch.
Thus
\begin{equation}
N_q(k_{\rm stop})
\le
\sum_{k=1}^{k_{\rm stop}-1}B_k
+
H|\mathcal S|
\sum_{k=1}^{k_{\rm stop}}m_k^{\rm st}.
\label{eq:qc-total}
\end{equation}
By the definition of $m_k^{\rm st}$,
\[
m_k^{\rm st}
\le
1+
\frac1{p_{\min}}
\left(
2n_k^{\rm uc}
+
2n_{k+1}^{\rm sen}
+
8\beta_k
+
50w_k\ell_{\rm trim}
+
2
\right).
\]
Furthermore,
\[
n_k^{\rm uc}
=
O\left(
\frac{C_{\rm uc}\beta_k}{\epsilon_k}
\right),
\qquad
n_{k+1}^{\rm sen}
=
\widetilde O
\left(
\frac{C_{\rm sen}^2}{\epsilon_{k+1}}
\right).
\]
Since $\epsilon_k=H2^{-(k-1)}$,
\[
\sum_{k=1}^{k_{\rm stop}}\frac1{\epsilon_k}
\le
\frac2{\epsilon_{k_{\rm stop}}},
\]
and therefore
\[
\sum_{k=1}^{k_{\rm stop}}
\frac{\beta_k}{\epsilon_k}
=
\widetilde O
\left(
\frac1{\epsilon_{k_{\rm stop}}}
\right),
\qquad
\sum_{k=1}^{k_{\rm stop}}n_{k+1}^{\rm sen}
=
\widetilde O
\left(
\frac1{\epsilon_{k_{\rm stop}}}
\right).
\]
Moreover,
\[
\sum_{k\ge1}w_k=1,
\]
so
\[
\sum_{k=1}^{k_{\rm stop}}
w_k\ell_{\rm trim}
\le
\ell_{\rm trim},
\]
while
\[
\sum_{k=1}^{k_{\rm stop}}(\beta_k+1)
=
\widetilde O(k_{\rm stop}).
\]
Consequently,
\[
H|\mathcal S|
\sum_{k=1}^{k_{\rm stop}}m_k^{\rm st}
=
\widetilde O
\left(
\frac1{\epsilon_{k_{\rm stop}}}
\right).
\]
Using \eqref{eq:qc-terminal-resolution},
\begin{equation}
H|\mathcal S|
\sum_{k=1}^{k_{\rm stop}}m_k^{\rm st}
=
\widetilde O
\left(
\frac1{\epsilon\vee\Delta_{\min}}
\right).
\label{eq:qc-stabilization-final}
\end{equation}

Finally, combining
\eqref{eq:qc-early-sum},
\eqref{eq:qc-tail-sum},
\eqref{eq:qc-total}, and
\eqref{eq:qc-stabilization-final}, we obtain
\[
N_q(k_{\rm stop})
\le
\widetilde O
\left(
T_{{\rm LQ},n_0/L_\epsilon}(\mathbf{\mathcal M})
\right)
+
\widetilde O
\left(
\frac1{\epsilon\vee\Delta_{\min}}
\right),
\]
which proves the theorem.
\end{proof}

\subsection{Proof of Theorem~\ref{thm:instance-dependent-separation}}
\begin{proof}
First we use the hard family introduced in the proof of
Theorem~7 of \citet{domingues2021episodic}.

The state space consists of three special states,
a waiting state $s_{\rm w}$, a good absorbing state $s_{\rm g}$,
and a bad absorbing state $s_{\rm b}$, together with a full $|\mathcal A|$-ary tree
of depth $d-1$ containing the remaining $|\mathcal S|-3$ states.
Let $s_{\rm root}$ denote the root of the tree and let
$\mathcal L$ be the set of its leaves, with
\[
L:=|\mathcal L|.
\]
As shown in \citet{domingues2021episodic}, this construction satisfies
\[
L\ge \frac{|\mathcal S|}{4}.
\]

The initial state is $s_{\rm w}$.
Choose
\[
\bar H:=\frac{H}{3},
\qquad
\widetilde\epsilon
:=
\frac{2\epsilon}{H-\bar H-d}.
\]
The waiting state $s_w$ contains a distinguished action $a_{\rm w}$.
For every stage $h$ and action $a$,
\[
P_h(s_{\rm w}\mid s_{\rm w},a)
=
\mathbf 1\{a=a_{\rm w},\,h<\bar H\},
\]
and
\[
P_h(s_{\rm root}\mid s_{\rm w},a)
=
1-
P_h(s_{\rm w}\mid s_{\rm w},a).
\]

Transitions inside the tree are deterministic.
At every internal tree state, action $a\in[|\mathcal A|]$ leads to the $a$-th child.
Since the depth of the tree is $d-1$, by choosing when to leave
$s_{\rm w}$ and which actions to take along the tree, a deterministic
policy can reach any leaf $\ell\in\mathcal L$ at any stage
\[
h
\in
\{1+d,\ldots,\bar H+d\}.
\]
Accordingly, define the set of hard triples
\begin{equation}
\mathcal I
:=
\{1+d,\ldots,\bar H+d\}
\times
\mathcal L
\times
[|\mathcal A|].
\label{eq:domingues-hard-index-set}
\end{equation}

For every
$i=(h^\star,\ell^\star,a^\star)\in\mathcal I$,
define an MDP
$M_i=M(h^\star,\ell^\star,a^\star)$
as follows.
For a leaf $\ell\in\mathcal L$ and action $a\in[|\mathcal A|]$, let
\[
P_h^{M_i}(s_{\rm g}\mid\ell,a)
=
\frac12
+
\Delta_i(h,\ell,a),
\]
and
\[
P_h^{M_i}(s_{\rm b}\mid\ell,a)
=
\frac12
-
\Delta_i(h,\ell,a),
\]
where
\[
\Delta_i(h,\ell,a)
:=
\widetilde\epsilon\,
\mathbf 1
\left\{
(h,\ell,a)
=
(h^\star,\ell^\star,a^\star)
\right\}.
\]
Thus $M_i$ differs from the baseline dynamics only at the single
leaf--stage--action triple
$(h^\star,\ell^\star,a^\star)$.

Both $s_{\rm g}$ and $s_{\rm b}$ are absorbing:
\[
P_h(s_{\rm g}\mid s_{\rm g},a)=1,
\qquad
P_h(s_{\rm b}\mid s_{\rm b},a)=1
\]
for every $h$ and $a$.
The reward is deterministic, depends only on the current state, and is
given by
\[
r_h(s,a)
=
\mathbf 1
\left\{
s=s_{\rm g},
\;
h\ge\bar H+d+1
\right\}.
\]
Hence entering $s_{\rm g}$ through any admissible leaf--stage triple
produces exactly
\[
H-\bar H-d
\]
rewarding stages, independently of the time at which the tree was
entered.

Finally, define the reference MDP $M_0$ by setting
\[
\Delta_0(h,\ell,a)=0
\qquad
\text{for every }(h,\ell,a).
\]
Equivalently, under $M_0$,
\[
P_h^{M_0}(s_{\rm g}\mid\ell,a)
=
P_h^{M_0}(s_{\rm b}\mid\ell,a)
=
\frac12
\]
for every leaf--stage--action triple.
The resulting hard family is therefore
\[
\mathcal C_{\bar H,\widetilde\epsilon}
=
\{M_0\}
\cup
\left\{
M_i:
i\in\mathcal I
\right\}.
\]

Under $M_0$, the probability of reaching $s_{\rm g}$ is $1/2$
regardless of the leaf, stage, and action selected by the policy.
Consequently, every deterministic policy has value
\[
V_1^\pi(M_0)
=
\frac{H-\bar H-d}{2},
\qquad
\forall\pi\in\Pi_d,
\]
and therefore
\begin{equation}
\Pi_\epsilon(M_0)
=
\Pi_d.
\label{eq:all-policies-interior-M0}
\end{equation}

Now fix an arbitrary
$\bar\pi\in\Pi_\epsilon(M_0)$.
Because all transitions prior to the leaf are deterministic,
$\bar\pi$ reaches exactly one leaf at exactly one stage and takes exactly
one action there.
Denote the resulting triple by
\[
i_{\bar\pi}
=
(h_{\bar\pi},\ell_{\bar\pi},a_{\bar\pi})
\in\mathcal I.
\]
For every
\[
i
\in
\mathcal I\setminus\{i_{\bar\pi}\},
\]
the policy $\bar\pi$ never visits the perturbed row of $M_i$.
Hence
\[
V_1^{\bar\pi}(M_i)
=
V_1^{\bar\pi}(M_0)
=
\frac{H-\bar H-d}{2}.
\]
On the other hand, there exists a deterministic policy that waits for the
appropriate number of stages, follows the unique path to the perturbed
leaf, and chooses the perturbed action.
Its value under $M_i$ is
\[
(H-\bar H-d)
\left(
\frac12+\widetilde\epsilon
\right).
\]
Therefore
\[
V_1^\star(M_i)-V_1^{\bar\pi}(M_i)
=
(H-\bar H-d)\widetilde\epsilon
=
2\epsilon
>
\epsilon.
\]
It follows that
\begin{equation}
M_i
\in
\mathcal A_{\epsilon}(\bar \pi),
\qquad
\forall
i\in
\mathcal I\setminus\{i_{\bar\pi}\}.
\label{eq:domingues-hard-alternatives}
\end{equation}

For each hard state-time pair
\[
(h,\ell)
\in
\mathcal J
:=
\{1+d,\ldots,\bar H+d\}
\times
\mathcal L,
\]
choose a distinguished action
$a_{h\ell}^{\rm r}\in \mathcal A$.
Fix any
\[
p_{\min}^c\in(0,1/|\mathcal A|],
\]
and equip every MDP in the family with a full-support collaborator satisfying
\[
\pi_h^c(a_{h\ell}^{\rm r}\mid \ell)
=
p_{\min}^c
\]
on every hard pair $(h,\ell)$, while distributing the remaining mass
$1-p_{\min}^c$ equally over the other $|\mathcal A|-1$ actions so that every action has
probability at least $p_{\min}^c$.
Thus the minimum collaborator probability is exactly $p_{\min}^c$.

It is enough to restrict attention to the subfamily
\[
\mathcal I_e
:=
\left\{
(h,\ell,a_{h\ell}^{\rm r})
:
(h,\ell)\in\mathcal J
\right\}.
\]
For a fixed
$\bar\pi\in\Pi_\epsilon(M_0)=\Pi_d$,
at most one element of $\mathcal I_e$ coincides with the unique
leaf--stage--action triple visited by $\bar\pi$.
Hence, by \eqref{eq:domingues-hard-alternatives}, all the remaining
$|\mathcal I_e|-1$ alternatives belong to
$\mathcal A_{\epsilon}(\bar \pi)$.

For
$i=(h,\ell,a_{h\ell}^{\rm r})\in\mathcal I_e$,
$M_i$ differs from $M_0$ only at this row, and therefore
\[
I_{\boldsymbol \lambda}(M_0,M_i)
=
\lambda_{h\ell}
p_{\min}^c
\operatorname{kl}
\left(
\frac12,
\frac12+\widetilde\epsilon
\right).
\]
Consequently,
\begin{align}
T_{\epsilon,0}^\star(M_0)
&=
\min_{\bar\pi\in\Pi_\epsilon(M_0)}
\inf_{\boldsymbol \lambda\in\Delta_{H|\mathcal S|}}
\left[
\inf_{\widetilde M\in
\mathcal A_{\epsilon}(\bar \pi)}
I_{\boldsymbol \lambda}(M_0,\widetilde M)
\right]^{-1}
\nonumber\\
&\ge
\min_{\bar\pi\in\Pi_d}
\inf_{\boldsymbol \lambda\in\Delta_{H|\mathcal S|}}
\left[
\min_{(h,\ell,a_{h\ell}^{\rm r})
\in
\mathcal I_e\setminus\{i_{\bar\pi}\}}
\lambda_{h\ell}
p_{\min}^c
\operatorname{kl}
\left(
\frac12,
\frac12+\widetilde\epsilon
\right)
\right]^{-1}
\nonumber\\
&\ge
\frac{
|\mathcal I_e|-1
}{
p_{\min}^c
\operatorname{kl}
\left(
\frac12,
\frac12+\widetilde\epsilon
\right)
}
\nonumber\\
&\ge
\frac{
|\mathcal I_e|-1
}{
4p_{\min}^c\widetilde\epsilon^2
}
\nonumber\\
&=
\frac{
(|\mathcal I_e|-1)
(H-\bar H-d)^2
}{
16p_{\min}^c\epsilon^2
}
\nonumber\\
&\gtrsim
\frac{
H^3|\mathcal S|
}{
p_{\min}^c\epsilon^2
},
\label{eq:worst-case-characteristic-time}
\end{align}
where we use
\[
\min_{(h,\ell)\in
\mathcal J\setminus\{(h_{\bar\pi},\ell_{\bar\pi})\}}
\lambda_{h\ell}
\le
\frac{1}{|\mathcal I_e|-1}.
\]
and
\[
\operatorname{kl}
\left(
\frac12,
\frac12+x
\right)
\le
4x^2,
\qquad
\forall |x|\le \frac 1{2\sqrt{2}}.
\]

Taking $M_\epsilon:=M_0$ proves the worst-case claim.

Then we prove the second statement.
Now let $M$ be an arbitrary finite-horizon MDP.
Fix any
\[
\pi^\star\in\Pi^\star(M).
\]
Since
$\pi^\star\in\Pi_\epsilon(M)$,
\[
T_{\mathrm{LQ},0}(M)
\le
2
\inf_{\boldsymbol \lambda\in\Delta_{H|\mathcal S|}}
\max_{\pi\in\Pi_d\setminus\{\pi^\star\}}
\frac{
\displaystyle
\sum_{h,s,a}
\frac{\Gamma_{hsa}(\pi,\pi^\star)}
{\lambda_{hs}\pi_h^c(a\mid s)}
}{
b(\pi,\pi^\star)^2
}.
\]
Moreover, for every
$\pi\neq\pi^\star$,
\[
b(\pi,\pi^\star)
=
\epsilon-
\bigl(V_1^\pi-V_1^{\pi^\star}\bigr)
=
\epsilon+\Delta_M(\pi)
\ge
\epsilon+\Delta_{\min}.
\]
In particular, when the optimal deterministic policy is not unique,
$\Delta_{\min}=0$, so the same bound remains valid.

It remains to control the numerator uniformly over $\pi$.  Recall that
\[
\Gamma_{hsa}(\pi,\pi^\star)
=
\bigl(
d_{hsa}^{\pi}-d_{hsa}^{\pi^\star}
\bigr)^2
\operatorname{Var}(R_h(s,a))
+
\operatorname{Var}_{S'\sim P_h(\cdot\mid s,a)}
\left[
d_{hsa}^{\pi}V_{h+1}^{\pi}(S')
-
d_{hsa}^{\pi^\star}V_{h+1}^{\pi^\star}(S')
\right].
\]
Since rewards are bounded in $[0,1]$,
$\operatorname{Var}(R_h(s,a))\le 1$, and for every $h$,
\[
\sum_{s,a}d_{hsa}^{\pi}=1,
\qquad
0\le d_{hsa}^{\pi}\le1,
\]
we obtain
\[
\begin{aligned}
\sum_{h,s,a}
\bigl(
d_{hsa}^{\pi}-d_{hsa}^{\pi^\star}
\bigr)^2
\operatorname{Var}(R_h(s,a))
&\le
2\sum_{h,s,a}
\left[
(d_{hsa}^{\pi})^2+
(d_{hsa}^{\pi^\star})^2
\right]                                                     \\
&\le 4H.
\end{aligned}
\]

For the transition term, we have
\[
\begin{aligned}
&\sum_{h,s,a}
\operatorname{Var}_{P_h}
\left[
d_{hsa}^{\pi}V_{h+1}^{\pi}(S')
-
d_{hsa}^{\pi^\star}V_{h+1}^{\pi^\star}(S')
\right]
\\
&\qquad\le
2\sum_{h,s,a}
(d_{hsa}^{\pi})^2
\operatorname{Var}_{P_h}
\left(
V_{h+1}^{\pi}(S')
\right)
+
2\sum_{h,s,a}
(d_{hsa}^{\pi^\star})^2
\operatorname{Var}_{S'\sim P_h(\cdot\mid s,a)}
\left(
V_{h+1}^{\pi^\star}(S')
\right)
\\
&\qquad\le
2\sum_{h,s,a}
d_{hsa}^{\pi}
\operatorname{Var}_{P_h}
\left(
V_{h+1}^{\pi}(S')
\right)
+
2\sum_{h,s,a}
d_{hsa}^{\pi^\star}
\operatorname{Var}_{P_h}
\left(
V_{h+1}^{\pi^\star}(S')
\right).
\end{aligned}
\]
For any policy $\pi$, Lemma 3.4 in \cite{yin2020asymptotically} yields
\[
\sum_{h,s,a}
d_{hsa}^{\pi}
\operatorname{Var}_{P_h(\cdot\mid s,a)}
\left(
V_{h+1}^{\pi}(S')
\right)
\le H^2,
\]

Thus
\[
\sum_{h,s,a}
\Gamma_{hsa}(\pi,\pi^\star)
\le
4H+4H^2
.
\]

Therefore
\begin{align}
T_{\mathrm{LQ},0}(M)
&=
\min_{\bar\pi\in\Pi_\epsilon(M)}
\inf_{\boldsymbol \lambda\in\Delta_{H|\mathcal S|}}
\max_{\pi\neq\bar\pi}
\frac{2}{
b(\pi,\bar\pi)^2
}
\sum_{h,s,a}
\frac{
\Gamma_{hsa}(\pi,\bar\pi)
}{
\lambda_{hs}\pi_h^c(a\mid s)
}
\nonumber\\
&\le
\inf_{\boldsymbol \lambda\in\Delta_{H|\mathcal S|}}
\max_{\pi\neq\pi^\star}
\frac{2}{
(\epsilon+\Delta_{\min})^2
}
\sum_{h,s,a}
\frac{
\Gamma_{hsa}(\pi,\pi^\star)
}{
\lambda_{hs}\pi_h^c(a\mid s)
}
\nonumber\\
&\le
\frac{2H|\mathcal S|}{
p_{\min}^c(\epsilon+\Delta_{\min})^2
}
\max_{\pi\neq\pi^\star}
\sum_{h,s,a}
\Gamma_{hsa}(\pi,\pi^\star)
\nonumber\\
&\lesssim
\frac{
H^3|\mathcal S|
}{
p_{\min}^c(\epsilon+\Delta_{\min})^2
}
\nonumber\\
&\le
C
\frac{
H^3|\mathcal S|
}{
p_{\min}^c\epsilon^2
}.
\label{eq:lq-worst-case-upper}
\end{align}
In particular, when $\Delta_{\min}\gg \epsilon$, the bound no longer scales
with $\epsilon^{-2}$ and is instead controlled by
$H^3|\mathcal S|/(p_{\min}^c\Delta_{\min}^2)$.
\end{proof}

\section{Robustness to a Shifted Collaborator Simulator}
\label{sec:shifted-simulator}

In practice, the collaborator simulator need not reproduce the true
reward--transition law exactly.  Unlike sampling noise, a systematic
simulator error cannot be removed merely by issuing more queries, and an
error introduced at an early localization epoch may propagate to later
epochs.  We show that Algorithm~\ref{alg:adaptive-collaborator-query-full}
remains stable under such a shift without any modification.  The shift
parameters introduced below are used only for analysis and are neither
known to the learner nor used by the algorithm.

For a row
\[
x=(h,s,a),
\]
write
\[
\nu_x^{\mathbf{\mathcal M}}
:=
\mathcal L_{\mathbf{\mathcal M}}(R_h,S_{h+1}\mid S_h=s,A_h=a)
\]
for the true reward--transition law.  Offline observations continue to
follow $\nu_x^{\mathbf{\mathcal M}}$.  Conditional on a collaborator query selecting row $x$,
the reward--next-state observation is instead drawn independently from a
fixed law $\nu_x^{\rm sim}$ satisfying
\begin{equation}
\operatorname{TV}
\left(
\nu_x^{\rm sim},
\nu_x^{\mathbf{\mathcal M}}
\right)
\le
\eta_x,
\label{eq:robust-shift-assumption}
\end{equation}
for every row $x$, where $\eta_x\in[0,1]$.
The collaborator action law $\pi^c$, the initial-state distribution, and
the reward range $[0,1]$ are unchanged.

\subsection{Additional quantities}

For every reached epoch $j$ and row $x$, let
\[
N_{j,x}^{Q,\mathcal D}
:=
N_{j,x}^{\rm ad}
+
N_{j,x}^{\rm st},
\qquad
N_{j,x}^{\mathcal D}
=
N_{0,j,x}^{\mathcal D}
+
N_{j,x}^{Q,\mathcal D}.
\]

Whenever $\mathsf{OffCover}_j=1$, define
\[
\alpha_{j,x}^{\mathcal D}
:=
\frac{N_{j,x}^{Q,\mathcal D}}{N_{j,x}^{\mathcal D}},
\qquad
\mathsf a_{j,x}
:=
\min
\left\{
1,
\sqrt{\frac{\beta_j}{2N_{j,x}^{\mathcal D}}}
\right\},
\qquad
\mathsf a_j
:=
\max_x \mathsf a_{j,x}.
\]

\subsection{A robust global event}\label{subsec:A robust global event}

We first define the concentration events used throughout the robustness
analysis.  The trimming and stabilization events
$\mathcal E_{\rm trim}$ and $\mathcal E_{\rm st}$ are exactly the same as
those in the proof of Theorem~\ref{thm:stopping-correctness}.  The
historical, fresh-sample, variance, and Bernstein events are replaced by
source-centered analogues, since an observation may now have law
$\nu_x^{\mathbf{\mathcal M}}$ or $\nu_x^{\rm sim}$.

\paragraph{Historical source-centered concentration event.}
Fix a row $x=(h,s,a)$ and enumerate the observations entering its
historical sample in their order of arrival as
\[
Z_{x,1},Z_{x,2},\ldots,
\qquad
Z_{x,i}
=
(R_{x,i},S'_{x,i}).
\]
Let $\mathcal G_{x,i-1}$ denote the information available immediately
before $Z_{x,i}$ is revealed.  Define the corresponding source means by
\[
r_{x,i}^{\rm src}
:=
\mathbb E
\left[
R_{x,i}
\mid
\mathcal G_{x,i-1}
\right],
\]
and, for $s'\in\mathcal S$,
\[
P_{x,i}^{\rm src}(s')
:=
\mathbb P
\left(
S'_{x,i}=s'
\mid
\mathcal G_{x,i-1}
\right).
\]
Thus these quantities are the reward and transition means under
$\nu_x^{\mathbf{\mathcal M}}$ for an offline observation and under $\nu_x^{\rm sim}$ for a
simulator observation.

Let $\mathcal E_G^{\rm sh}$ be the event that, simultaneously for every
row $x$, every $n\ge1$, and every $s'\in\mathcal S$,
\begin{equation}
\left|
\frac1n
\sum_{i=1}^n
\left(
R_{x,i}
-
r_{x,i}^{\rm src}
\right)
\right|
\le
\sqrt{
\frac{
\beta^\mathcal G(n)
}{
2n
}
},
\label{eq:robust-EG-reward}
\end{equation}
and
\begin{equation}
\left|
\frac1n
\sum_{i=1}^n
\left[
\mathbf 1\{S'_{x,i}=s'\}
-
P_{x,i}^{\rm src}(s')
\right]
\right|
\le
\sqrt{
\frac{
\beta^\mathcal G(n)
}{
2n
}
}.
\label{eq:robust-EG-transition}
\end{equation}

After source-wise centering, the summands in
\eqref{eq:robust-EG-reward} and
\eqref{eq:robust-EG-transition} are bounded martingale differences.
Therefore the same anytime Hoeffding argument used for
$\mathcal E_G$ in the proof of
Theorem~\ref{thm:stopping-correctness} gives
\[
\Pr
\left(
(\mathcal E_G^{\rm sh})^c
\right)
\le
\frac{\delta}{6}.
\]

\medskip
\noindent
\textbf{Fresh source-centered concentration event.}
For a fresh epoch $k$ and row $x$, enumerate the observations in the
fresh sample from row $x$ as
\[
Z_{k,x,1},\ldots,Z_{k,x,N_{k,x}^{\mathcal D}},
\qquad
Z_{k,x,i}
=
(R_{k,x,i},S'_{k,x,i}).
\]
For each $i\in\{1,\ldots,N_{k,x}^{\mathcal D}\}$, define the source indicator
\[
L_{k,x,i}
:=
\begin{cases}
0,
&
Z_{k,x,i}\in\mathcal D_{0,k}^{\mathcal D},
\\[1mm]
1,
&
Z_{k,x,i}\in
\mathcal Q_k^{\rm ad}\cup\mathcal Q_k^{\rm st}.
\end{cases}
\]
Thus
\[
\sum_{i=1}^{N_{k,x}^{\mathcal D}}
L_{k,x,i}
=
N_{k,x}^{Q,\mathcal D}.
\]
Let
\[
\mathcal L_{k,x}
:=
\sigma
\left(
N_{k,x}^{\mathcal D},
L_{k,x,1},\ldots,L_{k,x,N_{k,x}^{\mathcal D}}
\right)
\]
denote the realized source information in row $x$, and let
$\mathcal H_k$ denote the information available before the fresh
reward--transition observations at epoch $k$ are revealed.

For $i=1,\ldots,N_{k,x}^{\mathcal D}$, define
\[
r_{k,x,i}^{\rm src}
:=
\mathbb E
\left[
R_{k,x,i}
\mid
\mathcal H_k,\mathcal L_{k,x}
\right],
\]
and
\[
P_{k,x,i}^{\rm src}(s')
:=
\mathbb P
\left(
S'_{k,x,i}=s'
\mid
\mathcal H_k,\mathcal L_{k,x}
\right).
\]
Equivalently, conditional on
$\mathcal H_k$ and $\mathcal L_{k,x}$,
\[
Z_{k,x,i}
\sim
(1-L_{k,x,i})\nu_x^{\mathbf{\mathcal M}}
+
L_{k,x,i}\nu_x^{\rm sim}.
\]

Define the average source means
\[
\bar r_{k,x}^{\rm src}
:=
\frac1{N_{k,x}^{\mathcal D}}
\sum_{i=1}^{N_{k,x}^{\mathcal D}}
r_{k,x,i}^{\rm src},
\qquad
\bar P_{k,x}^{\rm src}(s')
:=
\frac1{N_{k,x}^{\mathcal D}}
\sum_{i=1}^{N_{k,x}^{\mathcal D}}
P_{k,x,i}^{\rm src}(s').
\]

Let $\mathcal E_D^{\rm sh}$ be the event that, simultaneously over all
reached epochs $k$, all rows $x$, and all $s'\in\mathcal S$,
\[
\left|
\widehat r_{k,x}^{\mathcal D}
-
\bar r_{k,x}^{\rm src}
\right|
\le
a_{k,x},
\]
and
\[
\left|
\widehat P_{k,x}^{\mathcal D}(s')
-
\bar P_{k,x}^{\rm src}(s')
\right|
\le
a_{k,x}.
\]

Conditional on $\mathcal H_k$ and the source information
$\{\mathcal L_{k,x}\}_x$, the fresh observations are independent,
although not necessarily identically distributed.  Hoeffding's
inequality therefore gives, with the same confidence allocation as in
the definition of $\mathcal E_D$ in the proof of
Theorem~\ref{thm:stopping-correctness},
\[
\Pr
\left(
(\mathcal E_D^{\rm sh})^c
\right)
\le
\frac{\delta}{6}.
\]

\medskip
\noindent
\textbf{Auxiliary true-law variance event.}
For every reached epoch $k$, row $x$, and observation
$Z_{k,x,i}$ in the fresh sample, construct an auxiliary variable
$Y_{k,x,i}$ as follows.
If $Z_{k,x,i}$ is an offline observation, set
\[
Y_{k,x,i}
=
Z_{k,x,i}.
\]
If $Z_{k,x,i}$ is a simulator observation, construct
$Y_{k,x,i}\sim\nu_x^{\mathbf{\mathcal M}}$ using a maximal coupling of
$\nu_x^{\rm sim}$ and $\nu_x^{\mathbf{\mathcal M}}$.  That is,
\[
Z_{k,x,i}\sim\nu_x^{\rm sim},
\qquad
Y_{k,x,i}\sim\nu_x^{\mathbf{\mathcal M}},
\]
and
\[
\Pr
\left(
Y_{k,x,i}\neq Z_{k,x,i}
\mid
\mathcal H_k, L_{k,x,i}=1
\right)
=
\operatorname{TV}
\left(
\nu_x^{\rm sim},
\nu_x^{\mathbf{\mathcal M}}
\right)
\le
\eta_x.
\]
Conditional on $\mathcal H_k$ and the source information
$\{\mathcal L_{k,x}\}_x$, these maximal couplings are chosen
independently across observations.  Consequently,
\[
Y_{k,x,1},\ldots,Y_{k,x,N_{k,x}^{\mathcal D}}
\overset{\rm iid}{\sim}
\nu_x^{\mathbf{\mathcal M}}
\qquad
\text{conditionally on }
\mathcal H_k,\mathcal L_{k,x}.
\]

For a policy pair
$p=(\pi,\bar\pi)$, define
\[
f_{p,x}
:=
\Psi_x^{\pi,\bar\pi},
\qquad
\Gamma_x^{\mathbf{\mathcal M}}(p)
:=
\Gamma_x^{\mathbf{\mathcal M}}(\pi,\bar\pi)
:=
\operatorname{Var}_{\nu_x^{\mathbf{\mathcal M}}}
\left(
f_{p,x}(Z)
\right),
\]
and the auxiliary empirical variance
\[
\widehat\Gamma_{k,x}^{Y}(p)
:=
\frac1{N_{k,x}^{\mathcal D}-1}
\sum_{i=1}^{N_{k,x}^{\mathcal D}}
\left(
f_{p,x}(Y_{k,x,i})
-
\frac1{N_{k,x}^{\mathcal D}}
\sum_{\ell=1}^{N_{k,x}^{\mathcal D}}
f_{p,x}(Y_{k,x,\ell})
\right)^2.
\]

Recall that
\[
\mathcal P_k
:=
\widehat\Pi_{k-1}^2
\]
is the predictable parent class.
Define $\mathcal E_\Gamma^{\rm sh}$ by
\[
\begin{aligned}
\mathcal E_\Gamma^{\rm sh}
:=
\bigcap_{k,x}
\bigcap_{p\in\mathcal P_k}
\left\{
\left|
\sqrt{
\widehat\Gamma_{k,x}^{Y}(p)
}
-
\sqrt{
\Gamma_x^{\mathbf{\mathcal M}}(p)
}
\right|
\le
2H
\sqrt{
\frac{
2\beta_k
}{
N_{k,x}^{\mathcal D}-1
}
}
\right\},
\end{aligned}
\]
where the event is imposed only for rows satisfying
$N_{k,x}^{\mathcal D}\ge2$.

By \eqref{eq:true-score-sd-concentration} in the proof of
Lemma~\ref{lem:variance-envelope-general}, together with the same
predictable-class confidence allocation used in the proof of
Theorem~\ref{thm:stopping-correctness},
\[
\Pr
\left(
(\mathcal E_\Gamma^{\rm sh})^c
\right)
\le
\frac{\delta}{6}.
\]

\medskip
\noindent
\textbf{Source-centered Bernstein event.}
For a fresh epoch $k$, row $x$, pair
$p=(\pi,\bar\pi)\in\mathcal P_k$, and observation $i$, define
\[
m_{k,x,i}(p)
:=
\mathbb E
\left[
f_{p,x}(Z_{k,x,i})
\mid
\mathcal H_k,\mathcal L_{k,x}
\right],
\]
and
\[
v_{k,x,i}(p)
:=
\operatorname{Var}
\left(
f_{p,x}(Z_{k,x,i})
\mid
\mathcal H_k,\mathcal L_{k,x}
\right).
\]

Define the centered linear statistic
\begin{equation}
L_k(p)
:=
\sum_x
\frac1{N_{k,x}^{\mathcal D}}
\sum_{i=1}^{N_{k,x}^{\mathcal D}}
\left[
f_{p,x}(Z_{k,x,i})
-
m_{k,x,i}(p)
\right].
\label{eq:robust-centered-linear-statistic}
\end{equation}

Let $\mathcal E_B^{\rm sh}$ be the event that, simultaneously over all
epochs $k$ and all $p\in\mathcal P_k$,
\[
\left|
L_k(p)
\right|
\le
\sqrt{
2\beta_k
\sum_x
\frac1{(N_{k,x}^{\mathcal D})^2}
\sum_{i=1}^{N_{k,x}^{\mathcal D}}
v_{k,x,i}(p)
}
+
\frac{2H}{3}
\beta_k
\max_x
\frac1{N_{k,x}^{\mathcal D}}.
\]

Conditional on $\mathcal H_k$ and the source information
$\{\mathcal L_{k,x}\}_x$, the summands in
\eqref{eq:robust-centered-linear-statistic}
are independent and centered, with range at most $2H$.
Bernstein's inequality therefore gives
\[
\Pr
\left(
(\mathcal E_B^{\rm sh})^c
\right)
\le
\frac{\delta}{6}.
\]

\medskip
\noindent
\textbf{Coupling event.}
For every epoch $k$ and row $x$, define
\[
M_{k,x}
:=
\sum_{i=1}^{N_{k,x}^{\mathcal D}}
L_{k,x,i}
\mathbf 1
\left\{
Y_{k,x,i}
\neq
Z_{k,x,i}
\right\}.
\]
Let
\[
\rho_x
:=
\operatorname{TV}
\left(
\nu_x^{\rm sim},
\nu_x^{\mathbf{\mathcal M}}
\right)
\le
\eta_x.
\]
By the independent construction of the maximal couplings across
simulator observations, conditional on
$\mathcal H_k$ and $\mathcal L_{k,x}$,
\[
M_{k,x}
\sim
\operatorname{Binomial}
\left(
N_{k,x}^{Q,\mathcal D},
\rho_x
\right)
\preceq_{\rm st}
\operatorname{Binomial}
\left(
N_{k,x}^{Q,\mathcal D},
\eta_x
\right).
\]
where $\preceq_{\rm st}$ denotes first-order stochastic domination:
$X\preceq_{\rm st}Y$ if
$\mathbb P(X>t)\le\mathbb P(Y>t)$ for all $t$.
For every $k\ge1$, define the deterministic confidence
parameter
\[
\gamma_k
:=
\log\frac{
2\pi^2k^2H|\mathcal S||\mathcal A|
}{\delta}.
\]
At every reached epoch, define
\[
u_{k,x}
:=
\min\left\{
m\in\{0,\ldots,N_{k,x}^{Q,\mathcal D}\}:
\Pr\left(
\operatorname{Binomial}
(N_{k,x}^{Q,\mathcal D},\eta_x)>m
\right)
\le e^{-\gamma_k}
\right\},
\]
where the binomial tail is evaluated at the realized
value of $N_{k,x}^{Q,\mathcal D}$.
For unreached epochs, set $M_{k,x}=u_{k,x}=0$.

By the conditional binomial domination above,
\[
\Pr\left(
M_{k,x}>u_{k,x}
\mid \mathcal H_k,\mathcal L_{k,x}
\right)
\le e^{-\gamma_k}.
\]
Since $\gamma_k$ is deterministic, taking expectations gives
\[
\Pr(M_{k,x}>u_{k,x})\le e^{-\gamma_k}.
\]
Define
\[
\mathcal E_{\rm coup}
:=
\bigcap_{k\ge1}\bigcap_x
\{M_{k,x}\le u_{k,x}\}.
\]
A union bound yields
\[
\begin{aligned}
\Pr(\mathcal E_{\rm coup}^c)
&\le
\sum_{k\ge1}\sum_x e^{-\gamma_k}\\
&=
\frac{\delta}{2\pi^2}
\sum_{k\ge1}\frac1{k^2}\\
&=
\frac{\delta}{12}.
\end{aligned}
\]

\medskip
\noindent
\textbf{Global robust good event.}
Finally, define
\[
\mathcal E_{\rm rob}
:=
\mathcal E_{\rm trim}
\cap
\mathcal E_{\rm st}
\cap
\mathcal E_G^{\rm sh}
\cap
\mathcal E_D^{\rm sh}
\cap
\mathcal E_\Gamma^{\rm sh}
\cap
\mathcal E_B^{\rm sh}
\cap
\mathcal E_{\rm coup}.
\]
Then
\[
\begin{aligned}
\Pr
\left(
\mathcal E_{\rm rob}^c
\right)
&\le
\frac{\delta}{6}
+
\frac{\delta}{12}
+
\frac{\delta}{6}
+
\frac{\delta}{6}
+
\frac{\delta}{6}
+
\frac{\delta}{6}
+
\frac{\delta}{12}
\\
&=
\delta.
\end{aligned}
\]
Therefore
\begin{equation}
\Pr
\left(
\mathcal E_{\rm rob}
\right)
\ge
1-\delta.
\label{eq:robust-global-probability}
\end{equation}

\subsection{Auxiliary Lemmas}

\begin{lemma}[Historical sensitivity under simulator shift]
\label{lem:robust-historical-sensitivity}
On $\mathcal E_{\rm rob}$, fix an epoch $k$ reaching prospective design.
For every row $x$, let $N_{k,x}^{Q,\mathcal G}$ denote the number of collaborator
observations from row $x$ contained in the historical sample
$\mathcal Z_k^\mathcal G$, and define
\[
\theta_k^\mathcal G
:=
\max_x
\frac{N_{k,x}^{Q,\mathcal G}}{N_{k,x}^{\mathcal G}}
\eta_x.
\]
For $x=(h,s,a)$, define
\[
C_h^{\rm sh}
:=
2(h-1)(H-h+1)
+
2(H-h)
+
(H-h)(H-h-1).
\]
Then, for every relevant policy pair $(\pi,\bar\pi)$ and every row
$x=(h,s,a)$,
\begin{equation}
\left\|
\widehat\Psi_{k,x}^{\pi,\bar\pi}
-
\Psi_x^{\pi,\bar\pi}
\right\|_\infty
\le
\Delta_{k,x}^\mathcal G
+
C_h^{\rm sh}\theta_k^\mathcal G.
\label{eq:robust-score-error}
\end{equation}
\end{lemma}

\begin{proof}
For a fixed row $x$, let
$\bar\nu_{k,x}^\mathcal G$ denote the average source law of the observations in
$\mathcal Z_k^\mathcal G$.  Since exactly $N_{k,x}^{Q,\mathcal G}$ of these observations
come from the simulator,
\[
\bar\nu_{k,x}^\mathcal G
=
\left(
1-\frac{N_{k,x}^{Q,\mathcal G}}{N_{k,x}^\mathcal G}
\right)
\nu_x^{\mathbf{\mathcal M}}
+
\frac{N_{k,x}^{Q,\mathcal G}}{N_{k,x}^\mathcal G}
\nu_x^{\rm sim}.
\]
Let
$\bar{\mathbf{\mathcal M}}_k^\mathcal G$ denote the MDP formed by these average source
laws.  By the convexity of total variation,
\begin{equation}\label{eq:upper bound for r}
\left|
\bar r_{k,x}^\mathcal G-r_x
\right|
\le
\theta_k^\mathcal G
\end{equation}
and
\begin{equation}\label{eq:upper bound for p}
\operatorname{TV}
\left(
\bar P_{k,x}^\mathcal G,
P_x
\right)
\le
\theta_k^\mathcal G.
\end{equation}

Note that
$\left\|
\widehat\Psi_{k,x}^{\pi,\bar\pi}
-
\Psi_x^{\pi,\bar\pi;\bar{\mathbf{\mathcal M}}_k^\mathcal G}
\right\|_\infty$ can be controlled exactly as in the unshifted proof.
Hence on $\mathcal E_G^{\rm sh}$, the same argument gives
\begin{equation}
\left\|
\widehat\Psi_{k,x}^{\pi,\bar\pi}
-
\Psi_x^{\pi,\bar\pi;\bar{\mathbf{\mathcal M}}_k^\mathcal G}
\right\|_\infty
\le
\Delta_{k,x}^\mathcal G.
\label{eq:robust-score-statistical-part}
\end{equation}
For any deterministic policy $\pi$, the Bellman recursion together with
\eqref{eq:upper bound for r}--\eqref{eq:upper bound for p} gives
\[
\left\|
V_h^{\pi,\bar{\mathbf{\mathcal M}}_k^\mathcal G}
-
V_h^{\pi,\mathbf{\mathcal M}}
\right\|_\infty
\le
\frac{
(H-h+1)(H-h+2)
}{2}
\theta_k^\mathcal G.
\]

Similarly, the forward occupancy recursion and the total-variation bound in
\eqref{eq:upper bound for p} give, for every state $s$,
\[
\left|
d_h^{\pi,\bar{\mathbf{\mathcal M}}_k^\mathcal G}(s)
-
d_h^{\pi,\mathbf{\mathcal M}}(s)
\right|
\le
(h-1)\theta_k^\mathcal G.
\]
Since $\pi$ is deterministic, this also implies
\[
\left|
d_{hsa}^{\pi,\bar{\mathbf{\mathcal M}}_k^\mathcal G}
-
d_{hsa}^{\pi,\mathbf{\mathcal M}}
\right|
\le
(h-1)\theta_k^\mathcal G.
\]
Hence
\[
\begin{aligned}
&
\left\|
\Psi_x^{\pi,\bar\pi;\bar{\mathbf{\mathcal M}}_k^\mathcal G}
-
\Psi_x^{\pi,\bar\pi;\mathbf{\mathcal M}}
\right\|_\infty
\\
&\le
\left|
d_{hsa}^{\pi,\bar{\mathbf{\mathcal M}}_k^\mathcal G}
-
d_{hsa}^{\pi,\mathbf{\mathcal M}}
\right|
(H-h+1)
+
d_{hsa}^{\pi,\mathbf{\mathcal M}}
\left\|
V_{h+1}^{\pi,\bar{\mathbf{\mathcal M}}_k^\mathcal G}
-
V_{h+1}^{\pi,\mathbf{\mathcal M}}
\right\|_\infty
\\
&\qquad
+
\left|
d_{hsa}^{\bar\pi,\bar{\mathbf{\mathcal M}}_k^\mathcal G}
-
d_{hsa}^{\bar\pi,\mathbf{\mathcal M}}
\right|
(H-h+1)
+
d_{hsa}^{\bar\pi,\mathbf{\mathcal M}}
\left\|
V_{h+1}^{\bar\pi,\bar{\mathbf{\mathcal M}}_k^\mathcal G}
-
V_{h+1}^{\bar\pi,\mathbf{\mathcal M}}
\right\|_\infty
\\
&\le
2(h-1)(H-h+1)\theta_k^\mathcal G
+
(H-h)(H-h+1)\theta_k^\mathcal G
\\
&=
\left[
2(h-1)(H-h+1)
+
2(H-h)
+
(H-h)(H-h-1)
\right]
\theta_k^\mathcal G
\\
&=
C_h^{\rm sh}\theta_k^\mathcal G.
\end{aligned}
\]
Combining this with
\eqref{eq:robust-score-statistical-part}
proves
\eqref{eq:robust-score-error}.
\end{proof}

\begin{lemma}[Variance transfer under simulator shift]
\label{lem:robust-variance-transfer}
On $\mathcal E_{\rm rob}$, fix an epoch $k$ reaching prospective design,
a row $x=(h,s,a)$, and a policy pair
$(\pi,\bar\pi)\in\mathcal P_k$.  Then
\begin{equation}
\sqrt{
\Gamma_x^{\mathbf{\mathcal M}}(\pi,\bar\pi)
}
\le
\sqrt{
\overline\Gamma_{k,x}(\pi,\bar\pi)
}
+
\sqrt{2}\,
C_h^{\rm sh}\theta_k^\mathcal G
+
2H
\sqrt{
\frac{
u_{k,x}
}{
N_{k,x}^{\mathcal D}-1
}
}.
\label{eq:robust-true-variance-transfer}
\end{equation}

Moreover, define
\[
\Gamma_x^{\rm sim}(\pi,\bar\pi)
:=
\operatorname{Var}_{\nu_x^{\rm sim}}
\left(
\Psi_x^{\pi,\bar\pi}(Z)
\right).
\]
For the next fresh epoch $j=k+1$, define
\[
\widetilde\Gamma_{j,x}(\pi,\bar\pi)
:=
(1-\alpha_{j,x}^{\mathcal D})
\Gamma_x^{\mathbf{\mathcal M}}(\pi,\bar\pi)
+
\alpha_{j,x}^{\mathcal D}
\Gamma_x^{\rm sim}(\pi,\bar\pi).
\]
Also define, for $x=(h,s,a)$,
\begin{equation}
d_{j,x}
:=
\sqrt{2}\,
C_h^{\rm sh}\theta_{j-1}^\mathcal G
+
2H
\sqrt{
\frac{
u_{j-1,x}
}{
N_{j-1,x}^{\mathcal D}-1
}
}
+
2H
\sqrt{
\alpha_{j,x}^{\mathcal D}\eta_x
}.
\label{eq:robust-variance-transfer-term}
\end{equation}
Then
\begin{equation}
\sqrt{
\widetilde\Gamma_{j,x}(\pi,\bar\pi)
}
\le
\sqrt{
\overline\Gamma_{k,x}(\pi,\bar\pi)
}
+
d_{j,x}.
\label{eq:robust-next-variance-transfer}
\end{equation}
\end{lemma}

\begin{proof}
Write
\[
f
:=
\Psi_x^{\pi,\bar\pi},
\qquad
\widehat f
:=
\widehat\Psi_{k,x}^{\pi,\bar\pi}.
\]

Set
\[
N:=N_{k,x}^{\mathcal D},
\qquad
Y_i:=Y_{k,x,i},
\qquad
Z_i:=Z_{k,x,i}.
\]
Arguing as in the proof of
Lemma~\ref{lem:variance-envelope-general}, for
\[
s(v)
:=
\left[
\frac1{N-1}
\sum_{i=1}^N
\left(
v_i-\frac1N\sum_{\ell=1}^Nv_\ell
\right)^2
\right]^{1/2},
\]
the reverse triangle inequality for the centered Euclidean norm gives
\[
|s(v)-s(w)|
\le
\frac{\|v-w\|_2}{\sqrt{N-1}}.
\]
Since $f=\Psi_x^{\pi,\bar\pi}$ has range at most $2H$,
\[
|f(Y_i)-f(Z_i)|
\le
2H\,\mathbf 1\{Y_i\neq Z_i\},
\]
while
\[
|f(Z_i)-\widehat f(Z_i)|
\le
\|\widehat f-f\|_\infty.
\]
Therefore the triangle inequality in $\ell_2$ yields
\[
\left\|
\bigl(
f(Y_i)-\widehat f(Z_i)
\bigr)_{i=1}^N
\right\|_2
\le
2H\sqrt{M_{k,x}}
+
\sqrt N\,\|\widehat f-f\|_\infty,
\]
and hence
\begin{align}
&
s
\left(
f(Y_1),\ldots,f(Y_N)
\right)
\nonumber\\
&\qquad\le
s
\left(
\widehat f(Z_1),\ldots,\widehat f(Z_N)
\right)
+
2H
\sqrt{
\frac{M_{k,x}}{N-1}
}
+
\sqrt{\frac{N}{N-1}}
\|\widehat f-f\|_\infty.
\label{eq:robust-sd-vector-transfer}
\end{align}

On $\mathcal E_\Gamma^{\rm sh}$,
\[
\sqrt{
\Gamma_x^{\mathbf{\mathcal M}}(\pi,\bar\pi)
}
\le
s
\left(
f(Y_1),\ldots,f(Y_N)
\right)
+
2H
\sqrt{
\frac{
2\beta_k
}{
N-1
}
}.
\]
Combining this with
\eqref{eq:robust-sd-vector-transfer},
Lemma~\ref{lem:robust-historical-sensitivity},
 and the definition of
$\overline\Gamma_{k,x}$ gives, on $\mathcal E_{\rm coup}$ where
$M_{k,x}\le u_{k,x}$,
\[
\sqrt{
\Gamma_x^{\mathbf{\mathcal M}}(\pi,\bar\pi)
}
\le
\sqrt{
\overline\Gamma_{k,x}(\pi,\bar\pi)
}
+
\sqrt2\,
C_h^{\rm sh}\theta_k^\mathcal G
+
2H
\sqrt{
\frac{
u_{k,x}
}{
N_{k,x}^{\mathcal D}-1
}
},
\]
which proves
\eqref{eq:robust-true-variance-transfer}.

Next, by the variational characterization of variance,
\[
\operatorname{Var}_P(f)
=
\inf_{c\in\mathbb R}
\mathbb E_P[(f-c)^2].
\]
Let $c_P=\mathbb E_P[f]$. Since $f$ has range at most $2H$,
\[
0\le (f-c_P)^2\le 4H^2.
\]
Hence, by total-variation duality,
\[
\begin{aligned}
\Gamma_x^{\rm sim}(\pi,\bar\pi)
-
\Gamma_x^{\mathbf{\mathcal M}}(\pi,\bar\pi)
&\le
\mathbb E_{\nu_x^{\rm sim}}\!\left[(f-\mathbb E_{\mathbf{\mathcal M}}[f])^2\right]
-
\mathbb E_{\mathbf{\mathcal M}}\!\left[(f-\mathbb E_{\mathbf{\mathcal M}}[f])^2\right]
\\
&\le
4H^2\eta_x.
\end{aligned}
\]
Exchanging the two models gives the reverse inequality, and therefore
\[
\left|
\Gamma_x^{\rm sim}(\pi,\bar\pi)
-
\Gamma_x^{\mathbf{\mathcal M}}(\pi,\bar\pi)
\right|
\le
4H^2\eta_x.
\]
Therefore
\[
\widetilde\Gamma_{j,x}
\le
\Gamma_x^{\mathbf{\mathcal M}}
+
4H^2
\alpha_{j,x}^{\mathcal D}\eta_x,
\]
and hence
\[
\sqrt{\widetilde\Gamma_{j,x}}
\le
\sqrt{\Gamma_x^{\mathbf{\mathcal M}}}
+
2H
\sqrt{
\alpha_{j,x}^{\mathcal D}\eta_x
}.
\]
Combining this with
\eqref{eq:robust-true-variance-transfer}
and the definition
\eqref{eq:robust-variance-transfer-term}
proves
\eqref{eq:robust-next-variance-transfer}.
\end{proof}

\begin{lemma}[Shifted fresh referencewise comparison]
\label{lem:robust-referencewise-comparison}
On $\mathcal E_{\rm rob}$, let $j\ge2$ be a reached epoch whose preceding
epoch $j-1$ completed prospective design.

Define the fresh-sample shift quantities
\[
\theta_j^{\mathcal D}
:=
\max_x
\alpha_{j,x}^{\mathcal D}\eta_x,
\qquad
\zeta_j
:=
\sum_{h=1}^H
(H-h+1)
\max_{s,a}
\alpha_{j,hsa}^{\mathcal D}\eta_{hsa}.
\]
Define
\[
\mathcal V_j
:=
\sqrt{
2\beta_j
\sum_x
\frac{
d_{j,x}^2
}{
N_{j,x}^{\mathcal D}
}
}.
\]
and
\begin{equation}
\mathcal B_j
:=
2\zeta_j
+
\mathcal V_j
+
K_H
\left(
2\mathsf a_j\theta_j^{\mathcal D}
+
(\theta_j^{\mathcal D})^2
\right),
\label{eq:robust-epoch-penalty}
\end{equation}
where $K_H$ is the remainder coefficient defined in
Table~\ref{tab:algorithm-notation}.
Then, uniformly over
$\pi\in\widehat\Pi_{j-1}$,
\begin{equation}
\left|
\widehat D_j
\left(
\pi,\bar\pi_{j-1}
\right)
-
D
\left(
\pi,\bar\pi_{j-1}
\right)
\right|
\le
\frac{\epsilon_j}{4}
+
\mathcal B_j.
\label{eq:robust-referencewise-quarter}
\end{equation}
\end{lemma}

\begin{proof}
We first control the centered linear term.  Conditional on the
pre-epoch history, the fresh row counts, and the collaborator actions,
the observations are independent.  For a fixed pair
$(\pi,\bar\pi_{j-1})$,
the linear term in
Lemma~\ref{lem:model-difference-expansion}
can be decomposed exactly as
\begin{align}
&
\sum_x
\frac1{N_{j,x}^{\mathcal D}}
\sum_{i=1}^{N_{j,x}^{\mathcal D}}
\left[
\Psi_x(Z_i)
-
\mathbb E_i\Psi_x(Z_i)
\right]
+
\sum_x
\alpha_{j,x}^{\mathcal D}
\left(
\mathbb E_{\nu_x^{\rm sim}}
\Psi_x
-
\mathbb E_{\nu_x^{\mathbf{\mathcal M}}}
\Psi_x
\right),
\label{eq:robust-linear-decomposition}
\end{align}
where $\mathbb E_i$ denotes expectation under the source law of the
$i$th observation.

By the definition of $\mathcal E_B^{\rm sh}$,
\begin{align}
&
\left|
\sum_x
\frac1{N_{j,x}^{\mathcal D}}
\sum_{i=1}^{N_{j,x}^{\mathcal D}}
\left[
\Psi_x(Z_i)
-
\mathbb E_i\Psi_x(Z_i)
\right]
\right|\le
\sqrt{
2\beta_j
\sum_x
\frac1{(N_{j,x}^{\mathcal D})^2}
\sum_{i=1}^{N_{j,x}^{\mathcal D}}
v_{j,x,i}
}
+
\frac{2H}{3}
\beta_j
\max_x
\frac1{N_{j,x}^{\mathcal D}}.
\label{eq:robust-centered-bernstein}
\end{align}
Since exactly $N_{j,x}^{Q,\mathcal D}$ observations in row $x$ come
from the simulator,
\[
\frac1{N_{j,x}^{\mathcal D}}
\sum_{i=1}^{N_{j,x}^{\mathcal D}}
v_{j,x,i}
=
\widetilde\Gamma_{j,x}
\left(
\pi,\bar\pi_{j-1}
\right).
\]
Consequently,
\begin{align}
&
\left|
\sum_x
\frac1{N_{j,x}^{\mathcal D}}
\sum_{i=1}^{N_{j,x}^{\mathcal D}}
\left[
\Psi_x(Z_i)
-
\mathbb E_i\Psi_x(Z_i)
\right]
\right|
\le
\sqrt{
2\beta_j
\sum_x
\frac{
\widetilde\Gamma_{j,x}
(
\pi,\bar\pi_{j-1}
)
}{
N_{j,x}^{\mathcal D}
}
}
+
\frac{2H}{3}
\beta_j
\max_x
\frac1{N_{j,x}^{\mathcal D}}.
\label{eq:robust-centered-bernstein-mixture}
\end{align}

We next connect the first term to the prospective design performed at
epoch $j-1$.  Since epoch $j-1$ reached prospective design, necessarily
$\epsilon_{j-1}>\epsilon$.  Moreover, by the localization rule, any
$\pi,\bar\pi\in\widehat\Pi_{j-1}$ have empirical values differing by at
most $5\epsilon_{j-1}/2$.  Hence
\[
\epsilon
-
\widehat D_{j-1}(\pi,\bar\pi)
\le
\epsilon+\frac52\epsilon_{j-1}
<
\frac72\epsilon_{j-1}
<
4\epsilon_{j-1},
\]
and therefore
\[
\widehat b_{j-1,\rm reg}
(\pi,\bar\pi)
=
4\epsilon_{j-1},
\qquad
\forall
\pi,\bar\pi\in\widehat\Pi_{j-1}.
\]
Consequently, by
\eqref{eq:thm-design-radius},
the optimizer selected at epoch $j-1$ satisfies
\begin{equation}
\max_{\pi\in\widehat\Pi_{j-1}}
\widehat{\operatorname{rad}}_{j-1}^{+}
\left(
\pi,
\bar\pi_{j-1};
\widehat{\mathbf n}_{j-1}
\right)
\le
\frac{\epsilon_j}{2}.
\label{eq:robust-design-radius}
\end{equation}

On $\mathcal E_{\rm st}$,
\eqref{eq:thm-realized-prospective-count} gives
\[
N_{j,x}^{\mathcal D}
\ge
\frac12
N_{j,x}^{+}
\left(
\widehat{\mathbf n}_{j-1}
\right).
\]
Since the leading prospective radius has numerical constant $C_0=8$,
\eqref{eq:robust-design-radius} implies
\[
\begin{aligned}
&
\sqrt{
2\beta_j
\sum_x
\frac{
\overline\Gamma_{j-1,x}
(
\pi,\bar\pi_{j-1}
)
}{
N_{j,x}^{\mathcal D}
}
}
\\
&\qquad\le
\frac14
\widehat{\operatorname{rad}}_{j-1}^{+}
\left(
\pi,
\bar\pi_{j-1};
\widehat{\mathbf n}_{j-1}
\right)
\le
\frac{\epsilon_j}{8}.
\end{aligned}
\]

\eqref{eq:robust-next-variance-transfer}
therefore gives
\begin{equation}
\sqrt{
2\beta_j
\sum_x
\frac{
\widetilde\Gamma_{j,x}
(
\pi,\bar\pi_{j-1}
)
}{
N_{j,x}^{\mathcal D}
}
}
\le
\frac{\epsilon_j}{8}
+
\mathcal V_j.
\label{eq:robust-designed-bernstein}
\end{equation}

We now control the second term in
\eqref{eq:robust-linear-decomposition}.  Recall that
\[
\Psi_x^{\pi,\bar\pi}(r,s')
=
d_{hsa}^{\pi}
\bigl(r+V_{h+1}^{\pi}(s')\bigr)
-
d_{hsa}^{\bar\pi}
\bigl(r+V_{h+1}^{\bar\pi}(s')\bigr).
\]
Since
\[
0\le r+V_{h+1}^{\pi}(s'),
\,
r+V_{h+1}^{\bar\pi}(s')
\le H-h+1,
\]
total-variation duality yields
\[
\begin{aligned}
&
\left|
\sum_x
\alpha_{j,x}^{\mathcal D}
\left(
\mathbb E_{\nu_x^{\rm sim}}
\Psi_x^{\pi,\bar\pi}
-
\mathbb E_{\nu_x^{\mathbf{\mathcal M}}}
\Psi_x^{\pi,\bar\pi}
\right)
\right|
\\
&\le
\sum_{h,s,a}
\alpha_{j,hsa}^{\mathcal D}
(H-h+1)\eta_{hsa}
\left(
d_{hsa}^{\pi}
+
d_{hsa}^{\bar\pi}
\right)
\\
&\le
\sum_{h=1}^H
(H-h+1)
\max_{s,a}
\left\{
\alpha_{j,hsa}^{\mathcal D}\eta_{hsa}
\right\}
\left(
\sum_{s,a}d_{hsa}^{\pi}
+
\sum_{s,a}d_{hsa}^{\bar\pi}
\right)
\\
&=
2
\sum_{h=1}^H
(H-h+1)
\max_{s,a}
\left\{
\alpha_{j,hsa}^{\mathcal D}\eta_{hsa}
\right\}
\\
&=
2\zeta_j.
\end{aligned}
\]

It remains to bound the higher-order model remainder.  On
$\mathcal E_D^{\rm sh}$, centering each observation by its own source
mean gives
\[
\|
\widehat r_j^{\mathcal D}-r
\|_\infty
\vee
\|
\widehat P_j^{\mathcal D}-P
\|_\infty
\le
\mathsf a_j+\theta_j^{\mathcal D}.
\]
Applying the deterministic remainder bound established
in the proof of Lemma~\ref{lem:iid-remainder} gives
\begin{equation}
\left|
\operatorname{Rem}_{\widehat{\mathbf{\mathcal M}}_j^{\mathcal D}}
\left(
\pi,\bar\pi_{j-1}
\right)
\right|
\le
K_H
(\mathsf a_j+\theta_j^{\mathcal D})^2
=
K_H\mathsf a_j^2
+
K_H
\left(
2\mathsf a_j\theta_j^{\mathcal D}
+
(\theta_j^{\mathcal D})^2
\right).
\label{eq:robust-remainder}
\end{equation}

By the same deterministic uniform-coverage calculation as in
\eqref{eq:thm-bernstein-range} and
\eqref{eq:thm-remainder-sixteenth},
\begin{equation}
\frac{2H}{3}\beta_j
\max_x\frac1{N_{j,x}^{\mathcal D}}
+
K_H\mathsf a_j^2
<
\frac{\epsilon_j}{8}.
\label{eq:robust-lower-order-unshifted}
\end{equation}
Combining
\eqref{eq:robust-linear-decomposition},
\eqref{eq:robust-centered-bernstein-mixture},
\eqref{eq:robust-designed-bernstein},
\eqref{eq:robust-remainder}, and
\eqref{eq:robust-lower-order-unshifted}
gives
\[
\left|
\widehat D_j
\left(
\pi,\bar\pi_{j-1}
\right)
-
D
\left(
\pi,\bar\pi_{j-1}
\right)
\right|
\le
\frac{\epsilon_j}{4}
+
2\zeta_j
+
\mathcal V_j
+
K_H
\left(
2\mathsf a_j\theta_j^{\mathcal D}
+
(\theta_j^{\mathcal D})^2
\right),
\]
which is exactly
\eqref{eq:robust-referencewise-quarter}.
\end{proof}

\subsection{The robustness theorem}
In this subsection we state and prove the main results on robustness.
For $\eta<0$, we use the convention $\Pi_\eta=\varnothing$.

\begin{theorem}[Stopping accuracy under an unknown simulator shift]
\label{thm:robust-unknown-shift}
Suppose
\eqref{eq:robust-shift-assumption}
holds and run
Algorithm~\ref{alg:adaptive-collaborator-query-full}
without modification.

Initialize
\[
\ell_1:=H,
\qquad
G_1:=0.
\]
For every reached epoch $j\ge2$ that completes localization, recall that
$\mathcal B_j$ is the shift penalty defined in
Lemma~\ref{lem:robust-referencewise-comparison}, and define recursively
\begin{equation}
\ell_j
:=
\min
\left\{
\ell_{j-1},
2\epsilon_j-2\mathcal B_j
\right\},
\qquad
G_j
:=
G_{j-1}
+
\left(
2\mathcal B_j-2\epsilon_j
\right)_+.
\label{eq:robust-localization-accounts}
\end{equation}

Then, with probability at least $1-\delta$, the following statements
hold simultaneously at every reached epoch.

For every $j\ge2$ that completes localization,
\begin{equation}
\Pi_{\ell_j}
\subseteq
\widehat\Pi_j
\subseteq
\Pi_{
G_{j-1}
+
3\epsilon_j
+
2\mathcal B_j
}.
\label{eq:robust-shifted-localization}
\end{equation}
Moreover,
\begin{equation}
V_1^\star
-
\max_{\pi\in\widehat\Pi_j}
V_1^\pi
\le
G_j.
\label{eq:robust-best-active-value}
\end{equation}

If the algorithm stops after completing localization at epoch $k\ge2$
under any of its post-localization stopping rules, including the
next-epoch offline-coverage gate
\[
\mathsf{OffCover}_{k+1}=0,
\]
and returns the corresponding empirical best response $\widehat\pi_k$, then
\begin{equation}
V_1^\star
-
V_1^{\widehat\pi_k}
\le
\min
\left\{
H,\,
G_{k-1}
+
\frac{\epsilon_k}{2}
+
2\mathcal B_k
\right\}.
\label{eq:robust-shifted-output}
\end{equation}
At epoch $1$, as well as when the initial gate
$\mathsf{OffCover}_1=0$, the trivial bound $H$ applies.

In particular, if
\[
\mathcal B_j\le \frac{\epsilon_j}{2}
\qquad
\text{for every completed localization epoch } 2\le j\le k,
\]
then
\[
G_j=0,
\qquad
\ell_j\ge\epsilon_j,
\]
and hence
\begin{equation}
\Pi_{\epsilon_j}
\subseteq
\widehat\Pi_j
\subseteq
\Pi_{3\epsilon_j+2\mathcal B_j}.
\label{eq:robust-shifted-stable-localization}
\end{equation}
Any post-localization stopping rule at epoch $k\ge2$ then returns a policy
satisfying
\[
V_1^\star
-
V_1^{\widehat\pi_k}
\le
\frac{\epsilon_k}{2}
+
2\mathcal B_k.
\]

When $\eta_x=0$ for every row $x$,
\[
\mathcal B_j=0
\]
for every completed localization epoch $j\ge2$, and
\eqref{eq:robust-shifted-stable-localization}
reduces to the original localization guarantee
\[
\Pi_{\epsilon_j}
\subseteq
\widehat\Pi_j
\subseteq
\Pi_{3\epsilon_j}.
\]
\end{theorem}

\begin{proof}
Condition on the event $\mathcal E_{\rm rob}$ from
\eqref{eq:robust-global-probability}.

By Lemma~\ref{lem:robust-referencewise-comparison}, for every completed
localization epoch $j\ge2$,
\[
\left|
\widehat D_j
\left(
\pi,\bar\pi_{j-1}
\right)
-
D
\left(
\pi,\bar\pi_{j-1}
\right)
\right|
\le
\frac{\epsilon_j}{4}
+
\mathcal B_j
\]
uniformly over
$\pi\in\widehat\Pi_{j-1}$.
Subtracting the same referencewise bound for two policies gives
\[
\left|
\left(
\widehat V_1^{\pi,j,\mathcal D}
-
\widehat V_1^{\pi',j,\mathcal D}
\right)
-
\left(
V_1^\pi-V_1^{\pi'}
\right)
\right|
\le
\frac{\epsilon_j}{2}
+
2\mathcal B_j,
\qquad
\pi,\pi'\in\widehat\Pi_{j-1}.
\]

We prove
\eqref{eq:robust-shifted-localization}
and
\eqref{eq:robust-best-active-value}
by induction.

At epoch $1$,
\[
\widehat\Pi_1
=
\Pi_d
=
\Pi_H,
\qquad
V_1^\star
-
\max_{\pi\in\widehat\Pi_1}
V_1^\pi
=
0,
\]
so the claims hold because
\[
\ell_1=H,
\qquad
G_1=0.
\]

Now let $j\ge2$ and assume the induction statements hold at epoch
$j-1$.  Let
\[
p
\in
\arg\max_{\pi\in\widehat\Pi_{j-1}}
V_1^\pi
\]
be an underlying best policy in the previous active class, and let
\[
q
\in
\arg\max_{\pi\in\widehat\Pi_{j-1}}
\widehat V_1^{\pi,j,\mathcal D}
\]
be the empirical best response.  By
\eqref{eq:app-localized-class}, $q$ is always retained.

First consider the lower inclusion.  Let
\[
\pi\in\Pi_{\ell_j}.
\]
If $\ell_j<0$, the statement is vacuous.  Otherwise,
\[
\ell_j
\le
\ell_{j-1},
\]
so the induction hypothesis implies
\[
\pi\in\widehat\Pi_{j-1}.
\]
Moreover,
\[
\begin{aligned}
\widehat V_1^{q,j,\mathcal D}
-
\widehat V_1^{\pi,j,\mathcal D}
&\le
V_1^q-V_1^\pi
+
\frac{\epsilon_j}{2}
+
2\mathcal B_j
\\
&\le
V_1^\star-V_1^\pi
+
\frac{\epsilon_j}{2}
+
2\mathcal B_j
\\
&\le
\ell_j
+
\frac{\epsilon_j}{2}
+
2\mathcal B_j.
\end{aligned}
\]
By
\eqref{eq:robust-localization-accounts},
\[
\ell_j
\le
2\epsilon_j-2\mathcal B_j,
\]
and hence
\[
\ell_j
+
\frac{\epsilon_j}{2}
+
2\mathcal B_j
\le
\frac52\epsilon_j.
\]
Thus $\pi$ is retained, proving
\[
\Pi_{\ell_j}
\subseteq
\widehat\Pi_j.
\]

For the upper inclusion, take any
$\pi\in\widehat\Pi_j$.
Since $q$ is empirically optimal and $\pi$ survives localization,
\[
\widehat V_1^{p,j,\mathcal D}
-
\widehat V_1^{\pi,j,\mathcal D}
\le
\frac52\epsilon_j.
\]
Therefore
\[
\begin{aligned}
V_1^\star-V_1^\pi
&=
\left(
V_1^\star-V_1^p
\right)
+
\left(
V_1^p-V_1^\pi
\right)
\\
&\le
\left(
V_1^\star
-
\max_{\nu\in\widehat\Pi_{j-1}}
V_1^\nu
\right)
+
\frac52\epsilon_j
+
\frac{\epsilon_j}{2}
+
2\mathcal B_j
\\
&\le
G_{j-1}
+
3\epsilon_j
+
2\mathcal B_j.
\end{aligned}
\]
Hence
\[
\widehat\Pi_j
\subseteq
\Pi_{
G_{j-1}
+
3\epsilon_j
+
2\mathcal B_j
}.
\]

It remains to control the loss of the optimal policy in the active
class, namely
\[
V_1^\star-
\max_{\pi\in\widehat\Pi_j}V_1^\pi.
\]
If $p$ is retained, then
\[
\max_{\pi\in\widehat\Pi_j}
V_1^\pi
\ge
V_1^p,
\]
so no additional loss is incurred.  If $p$ is removed, then
\[
\widehat V_1^{q,j,\mathcal D}
-
\widehat V_1^{p,j,\mathcal D}
>
\frac52\epsilon_j.
\]
The pairwise comparison bound therefore implies
\[
V_1^q-V_1^p
>
\frac52\epsilon_j
-
\left(
\frac{\epsilon_j}{2}
+
2\mathcal B_j
\right).
\]
Since $q$ is retained, the loss in the best true value of the active
class is at most
\[
\left(
\frac{\epsilon_j}{2}
+
2\mathcal B_j
-
\frac52\epsilon_j
\right)_+
=
\left(
2\mathcal B_j
-
2\epsilon_j
\right)_+.
\]
Thus
\[
V_1^\star
-
\max_{\pi\in\widehat\Pi_j}
V_1^\pi
\le
G_{j-1}
+
\left(
2\mathcal B_j
-
2\epsilon_j
\right)_+
=
G_j.
\]
This proves the induction.

Now suppose that the algorithm stops after completing localization at
epoch $k\ge2$.  The empirical best response returned at that epoch is
also an empirical maximizer over $\widehat\Pi_{k-1}$.  Let
\[
p
\in
\arg\max_{\pi\in\widehat\Pi_{k-1}}
V_1^\pi.
\]
The pairwise comparison bound gives
\[
V_1^p
-
V_1^{\widehat\pi_k}
\le
\frac{\epsilon_k}{2}
+
2\mathcal B_k.
\]
Therefore
\[
V_1^\star
-
V_1^{\widehat\pi_k}
\le
G_{k-1}
+
\frac{\epsilon_k}{2}
+
2\mathcal B_k.
\]
The value range gives the additional cap at $H$, proving
\eqref{eq:robust-shifted-output}.

The same argument applies when termination is triggered by
$\mathsf{OffCover}_{k+1}=0$, since this gate is evaluated only after the
epoch-$k$ localization has been completed.  Hence this case is already
covered by \eqref{eq:robust-shifted-output}.  At epoch $1$, or if
$\mathsf{OffCover}_1=0$ initially, the value range gives the trivial
bound $H$.

Finally, suppose
\[
\mathcal B_j
\le
\frac{\epsilon_j}{2}
\]
at every completed localization epoch.  Then
\[
\left(
2\mathcal B_j
-
2\epsilon_j
\right)_+
=
0,
\]
so
\[
G_j=0.
\]
Moreover,
\[
2\epsilon_j-2\mathcal B_j
\ge
\epsilon_j.
\]
Since
$\epsilon_{j-1}=2\epsilon_j$,
induction in
\eqref{eq:robust-localization-accounts}
gives
\[
\ell_j\ge\epsilon_j.
\]
The stable localization and output statements follow immediately.

When $\eta_x=0$ for all $x$,
\[
\theta_j^{\mathcal D}
=
\theta_j^\mathcal G
=
\zeta_j
=
0
\]
and
\[
u_{k,x}=0.
\]
Hence
\[
d_{j,x}
=
\mathcal V_j
=
\mathcal B_j
=
0,
\]
recovering the unshifted result.
\end{proof}

The theorem quantifies the effect of simulator
misspecification through the current-epoch penalty
$\mathcal B_k$ and the accumulated localization loss
$G_{k-1}$.

To give an explicit rate, we first control
$\bar u_{k,x}(\eta)$.

\begin{lemma}[Binomial coupling-quantile bound]
\label{lem:binomial-coupling-quantile}
Let $n\in\mathbb N$, $\eta\in[0,1]$, and let
$\gamma_k>0$ be deterministic. Define
\[
\bar u_{k,x}(\eta)
:=
\min\left\{
m\in\{0,\ldots,n\}:
\Pr\left(
\operatorname{Binomial}(n,\eta)>m
\right)
\le e^{-\gamma_k}
\right\}.
\]
If $n\eta=0$, then $\bar u_{k,x}(\eta)=0$.
Otherwise,
\[
\bar u_{k,x}(\eta)
\le
\min\left\{
n,\,
\left\lfloor
\frac{2\gamma_k}{
\log\left(1+\frac{\gamma_k}{n\eta}\right)
}
\right\rfloor
\right\}.
\]
\end{lemma}

\begin{proof}
If $n\eta=0$, the binomial variable is zero almost surely,
so the claim is immediate.
Let
\[
X\sim\operatorname{Binomial}(n,\eta).
\]
For every $s>0$,
\[
\mathbb E[e^{sX}]
=
\left(1+\eta(e^s-1)\right)^n
\le
\exp\left(n\eta(e^s-1)\right).
\]
Set
\[
s:=\log\left(1+\frac{\gamma_k}{n\eta}\right),
\qquad
r:=\frac{2\gamma_k}{s}.
\]
By Markov's inequality,
\[
\Pr(X\ge r)
\le
e^{-sr}\mathbb E[e^{sX}]
\le
\exp\left(-sr+n\eta(e^s-1)\right)
=
e^{-\gamma_k}.
\]
Since $X$ is integer-valued,
\[
\Pr(X>\lfloor r\rfloor)
=
\Pr(X\ge\lfloor r\rfloor+1)
\le
\Pr(X\ge r)
\le
e^{-\gamma_k}.
\]
If $\lfloor r\rfloor\le n$, it is therefore feasible
in the definition of $\bar u_{k,x}(\eta)$.
Otherwise, the trivial feasible choice $m=n$ applies.
This proves the stated bound.
\end{proof}

\begin{corollary}[Uniform simulator shift]
\label{cor:robust-uniform-shift}
Let
\[
\eta:=\max_x\eta_x.
\]
For every $j\ge2$, define
\[
\chi_j(\eta)
:=
\begin{cases}
\displaystyle
\left[
1+\log\left(
1+\frac{\epsilon_j}{H|\mathcal S||\mathcal A|\,\eta}
\right)
\right]^{-1/2},
& \eta>0,\\[1.2ex]
0,
& \eta=0.
\end{cases}
\]
Then, on $\mathcal E_{\rm rob}$, for every completed
localization epoch $j\ge2$, there exists an upper envelope
$\overline{\mathcal B}_j(\eta)$ such that
\[
\mathcal B_j
\le
\overline{\mathcal B}_j(\eta),
\]
and
\begin{equation}
\overline{\mathcal B}_j(\eta)
=
O\left(
H^2\eta\left(1+|\mathcal S|\sqrt{\epsilon_j}\right)
+
H\sqrt{\eta\epsilon_j}
+
\sqrt{\frac{H}{|\mathcal S||\mathcal A|}}\,
\epsilon_j\chi_j(\eta)
+
H^3|\mathcal S|^2\eta^2
\right).
\label{eq:robust-uniform-penalty-rate}
\end{equation}
Here $O(\cdot)$ hides universal numerical constants only.

Consequently, all conclusions of
Theorem~\ref{thm:robust-unknown-shift}
remain valid with $\mathcal B_j$ replaced throughout by
$\overline{\mathcal B}_j(\eta)$, including in the recursive
definitions of $\ell_j$ and $G_j$.

In particular, if
\[
\overline{\mathcal B}_j(\eta)
\le
\frac{\epsilon_j}{2}
\]
at every completed localization epoch $2\le j\le k$, then
\[
\Pi_{\epsilon_j}
\subseteq
\widehat\Pi_j
\subseteq
\Pi_{3\epsilon_j+2\overline{\mathcal B}_j(\eta)}
\]
for every such $j$, and any post-localization output at
epoch $k\ge2$ satisfies
\[
V_1^\star-V_1^{\widehat\pi_k}
\le
\frac{\epsilon_k}{2}
+
2\overline{\mathcal B}_k(\eta).
\]
Moreover,
\[
\overline{\mathcal B}_j(0)=0.
\]
For fixed $H,|\mathcal S|,|\mathcal A|$ and $\epsilon_j>0$, the right-hand side
of \eqref{eq:robust-uniform-penalty-rate} converges to zero
as $\eta\downarrow0$.
\end{corollary}

\begin{proof}
We work on $\mathcal E_{\rm rob}$ and first construct an
explicit upper bound on the epoch-wise penalty of
Theorem~\ref{thm:robust-unknown-shift}.

For every reached epoch $k$ and row $x$, define
\[
\bar u_{k,x}(\eta)
:=
\min\left\{
m\in\{0,\ldots,N_{k,x}^{Q,\mathcal D}\}:
\Pr\left(
\operatorname{Binomial}
(N_{k,x}^{Q,\mathcal D},\eta)>m
\right)
\le e^{-\gamma_k}
\right\},
\]
where the binomial tail is evaluated at the realized
count $N_{k,x}^{Q,\mathcal D}$.

Fix a completed localization epoch $j\ge2$.
The preceding epoch $j-1$ completed prospective design.
On $\mathcal E_{\rm rob}$, the stabilization and
offline-coverage conditions therefore give, for
$r\in\{j-1,j\}$ and every row $x$,
\[
N_{r,x}^{\mathcal D}
\ge
n_r^{\rm uc}
\ge
\frac{C_{\rm uc}\beta_r}{\epsilon_r},
\qquad
N_{r,x}^{\mathcal D}\ge2.
\]
Thus all denominators used below are positive.

For every row $x=(h,s,a)$, let
\[
\begin{aligned}
\bar d_{j,x}(\eta)
:={}&
\sqrt{2}\,C_h^{\rm sh}\eta
+
2H\sqrt{
\frac{\bar u_{j-1,x}(\eta)}
{N_{j-1,x}^{\mathcal D}-1}
}
\\
&+
2H\sqrt{\alpha_{j,x}^{\mathcal D}\eta},
\end{aligned}
\]
and define
\[
\begin{aligned}
\overline{\mathcal B}_j(\eta)
:={}&
H(H+1)\eta
+
\sqrt{
2\beta_j
\sum_x
\frac{\bar d_{j,x}(\eta)^2}
{N_{j,x}^{\mathcal D}}
}
\\
&+
K_H\left(2\mathsf a_j\eta+\eta^2\right).
\end{aligned}
\]

By the definition of $\eta$,
\[
\theta_j^{\mathcal D}\le\eta,
\qquad
\theta_{j-1}^\mathcal G\le\eta.
\]
Moreover,
\[
\begin{aligned}
2\zeta_j
&\le
2\eta\sum_{h=1}^H(H-h+1)\\
&=
H(H+1)\eta.
\end{aligned}
\]
Since the binomial distribution is stochastically increasing
in its success probability, and both quantiles use the same
tail level $e^{-\gamma_k}$,
\[
u_{k,x}\le\bar u_{k,x}(\eta).
\]
Therefore,
\[
d_{j,x}\le\bar d_{j,x}(\eta),
\]
and substitution into
\eqref{eq:robust-epoch-penalty} gives
\begin{equation}
\mathcal B_j
\le
\overline{\mathcal B}_j(\eta).
\label{eq:robust-uniform-majorization}
\end{equation}

We next bound $\overline{\mathcal B}_j(\eta)$.
If $\eta=0$, then
\[
\bar u_{k,x}(0)=0,
\qquad
\bar d_{j,x}(0)=0,
\]
and hence
\[
\overline{\mathcal B}_j(0)=0.
\]
Then consider $\eta>0$.
If $N_{j-1,x}^{Q,\mathcal D}=0$, then
$\bar u_{j-1,x}(\eta)=0$.
Otherwise, Lemma~\ref{lem:binomial-coupling-quantile} gives
\[
\bar u_{j-1,x}(\eta)
\le
\frac{
2\gamma_{j-1}
}{
\log\left(
1+\frac{\gamma_{j-1}}
{N_{j-1,x}^{Q,\mathcal D}\eta}
\right)
}.
\]
Since
\[
N_{j-1,x}^{Q,\mathcal D}
\le
N_{j-1,x}^{\mathcal D}
\]
and
\[
N_{j-1,x}^{\mathcal D}-1
\ge
\frac{N_{j-1,x}^{\mathcal D}}{2},
\]
we have
\[
\frac{\bar u_{j-1,x}(\eta)}
{N_{j-1,x}^{\mathcal D}-1}
\le
4\,
\frac{
\gamma_{j-1}/N_{j-1,x}^{\mathcal D}
}{
\log\left(
1+
\frac{
\gamma_{j-1}/N_{j-1,x}^{\mathcal D}
}{\eta}
\right)
}.
\]

For fixed $\eta>0$, the function
\[
f_\eta(z)
:=
\frac{z}{\log(1+z/\eta)}
\]
is increasing on $(0,\infty)$.
The prescribed choice of $C_{\rm uc}$ satisfies
$C_{\rm uc}\ge2H|\mathcal S||\mathcal A|$.
Using $\gamma_{j-1}\le\beta_{j-1}$,
the uniform coverage bound, and
$\epsilon_{j-1}=2\epsilon_j$, we obtain
\[
\frac{\gamma_{j-1}}
{N_{j-1,x}^{\mathcal D}}
\le
\frac{
\gamma_{j-1}\epsilon_{j-1}
}{
C_{\rm uc}\beta_{j-1}
}
\le
\frac{2\epsilon_j}{C_{\rm uc}}
\le
\frac{\epsilon_j}{H|\mathcal S||\mathcal A|}.
\]
Consequently,
\[
\frac{\bar u_{j-1,x}(\eta)}
{N_{j-1,x}^{\mathcal D}-1}
\le
\frac{
4\epsilon_j
}{
H|\mathcal S||\mathcal A|\log\left(
1+\frac{\epsilon_j}{H|\mathcal S||\mathcal A|\eta}
\right)
}.
\]

For every $y>0$, the elementary inequality
\[
\frac{y}{\log(1+y)}
\le
3\left(
1+\frac{y}{1+\log(1+y)}
\right)
\]
holds. 
Applying this inequality with
$y=\epsilon_j/(H|\mathcal S||\mathcal A|\eta)$ yields
\[
\frac{\bar u_{j-1,x}(\eta)}
{N_{j-1,x}^{\mathcal D}-1}
\le
12\left(
\eta+\frac{\epsilon_j}{H|\mathcal S||\mathcal A|}\chi_j(\eta)^2
\right).
\]
This bound also holds when
$N_{j-1,x}^{Q,\mathcal D}=0$.

Using
\[
C_h^{\rm sh}\le H(H-1),
\qquad
\alpha_{j,x}^{\mathcal D}\le1,
\]
we therefore obtain
\[
\max_x\bar d_{j,x}(\eta)
\lesssim
H^2\eta
+
H\sqrt{\eta}
+
H\sqrt{\frac{\epsilon_j}{H|\mathcal S||\mathcal A|}}\chi_j(\eta).
\]

Then
\[
\sqrt{
2\beta_j\sum_x
\frac{\bar d_{j,x}(\eta)^2}
{N_{j,x}^{\mathcal D}}
}
\le
\sqrt{
\frac{
2\beta_j H|\mathcal S||\mathcal A|
}{
\min_x N_{j,x}^{\mathcal D}
}
}
\max_x\bar d_{j,x}(\eta),
\]
and
\[
\mathsf a_j
\le
\sqrt{
\frac{
\beta_j
}{
2\min_x N_{j,x}^{\mathcal D}
}
}.
\]
The uniform fresh-row coverage gives
\[
\sqrt{
\frac{
2\beta_j H|\mathcal S||\mathcal A|
}{
\min_x N_{j,x}^{\mathcal D}
}
}
\le
\sqrt{
\frac{
2H|\mathcal S||\mathcal A|\epsilon_j
}{
C_{\text{uc}}
}
}
\le
\sqrt{\epsilon_j},
\qquad
\mathsf a_j
\le
\sqrt{
\frac{\epsilon_j}{2C_{\text{uc}}}
}.
\]
Substitution yields
\begin{equation}
\begin{aligned}
\overline{\mathcal B}_j(\eta)
=
O\Bigg(
&
H^2\eta
+
H^2\eta\sqrt{\epsilon_j}
+
H\sqrt{\eta\epsilon_j}
+
\frac{H}{\sqrt{H|\mathcal S||\mathcal A|}}\,
\epsilon_j\chi_j(\eta)
+
K_H\eta\sqrt{\frac{\epsilon_j}{C_{\rm uc}}}
+
K_H\eta^2
\Bigg).
\end{aligned}
\label{eq:robust-uniform-intermediate-rate}
\end{equation}

We now substitute the fixed constants.
By construction,
\[
C_{\rm uc}\ge8K_H,
\qquad
K_H\lesssim H^3|\mathcal S|^2.
\]
Therefore,
\[
K_H\eta\sqrt{\frac{\epsilon_j}{C_{\rm uc}}}
\lesssim
\sqrt{K_H}\,\eta\sqrt{\epsilon_j}
\lesssim
H^2|\mathcal S|\eta\sqrt{\epsilon_j},
\]
and
\[
K_H\eta^2\lesssim H^3|\mathcal S|^2\eta^2.
\]
It follows from
\eqref{eq:robust-uniform-intermediate-rate} that
\[
\overline{\mathcal B}_j(\eta)
=
O\left(
H^2\eta\left(1+|\mathcal S|\sqrt{\epsilon_j}\right)
+
H\sqrt{\eta\epsilon_j}
+
\sqrt{\frac{H}{|\mathcal S||\mathcal A|}}\,
\epsilon_j\chi_j(\eta)
+
H^3|\mathcal S|^2\eta^2
\right).
\]
Together with $\overline{\mathcal B}_j(0)=0$
and $\chi_j(0)=0$, this proves
\eqref{eq:robust-uniform-penalty-rate}.
For fixed $H,|\mathcal S|,|\mathcal A|$ and $\epsilon_j>0$,
$\chi_j(\eta)\to0$ as $\eta\downarrow0$,
so the displayed rate also converges to zero.

Finally, \eqref{eq:robust-uniform-majorization}
allows every comparison bound in
Theorem~\ref{thm:robust-unknown-shift}
to be weakened by replacing $\mathcal B_j$ with
$\overline{\mathcal B}_j(\eta)$.
This replacement can only decrease $\ell_j$ and increase
$G_j$ in the corresponding recursions.
Hence the localization and stopping-accuracy conclusions
remain valid with the substituted penalties.
The stated stability conclusions follow directly from
Theorem~\ref{thm:robust-unknown-shift}.
\end{proof}

Corollary~\ref{cor:robust-uniform-shift} shows that simulator
misspecification affects the guarantee only through the additive penalty
$\overline{\mathcal B}_j(\eta)$ if $\eta$ is sufficiently small.  For fixed $H,|\mathcal S|,|\mathcal A|$ and $\epsilon_j>0$, the upper bound
in \eqref{eq:robust-uniform-penalty-rate} vanishes
as $\eta\downarrow0$.
Hence the guarantee degrades gracefully as the simulator departs from the
true reward-transition law.  In particular, whenever
$\overline{\mathcal B}_j(\eta)\le\epsilon_j/2$ over the relevant epochs,
the original localization argument continues to hold, with only an
additive $2\overline{\mathcal B}_j(\eta)$ loss in the returned policy's
suboptimality.  Moreover, $\overline{\mathcal B}_j(0)=0$, so the original
unshifted guarantee is recovered exactly.

\section{Technical Lemmas}

\begin{lemma}[Model-difference identity and policy-difference decomposition]
\label{lem:model-difference-expansion}
Fix a deterministic policy $\pi\in\Pi_d$. For any MDP
$\widetilde{\mathbf{\mathcal M}}$ with the same state and action spaces, horizon, and
initial-state distribution as $\mathbf{\mathcal M}$, and with reward means
$\widetilde{\boldsymbol r}$ and transition kernel $\widetilde{\boldsymbol P}$,
\begin{equation}
\begin{aligned}
V_1^\pi(\mathbf{\mathcal M})-V_1^\pi(\widetilde{\mathbf{\mathcal M}})
=
\sum_{h,s,a}
\widetilde d_{hsa}^\pi
\Bigg[
&r_h(s,a)-\widetilde r_h(s,a)
\\
&+
\sum_{s'}
\bigl(
P_h(s'\mid s,a)-\widetilde P_h(s'\mid s,a)
\bigr)
V_{h+1}^\pi(s')
\Bigg],
\end{aligned}
\label{eq:model-difference-identity}
\end{equation}
where $\widetilde d^\pi$ denotes the state--action occupancy of $\pi$
under $\widetilde{\mathbf{\mathcal M}}$, while $V^\pi$ denotes the value function under
$\mathbf{\mathcal M}$.

Consequently, for any $\pi,\bar\pi\in\Pi_d$,
\begin{equation}
\begin{aligned}
&
\bigl(
V_1^\pi( \widetilde{\mathbf{\mathcal M}})-V_1^{\bar\pi}(\widetilde{\mathbf{\mathcal M}})
\bigr)
-
\bigl(
V_1^\pi(\mathbf{\mathcal M})-V_1^{\bar\pi}(\mathbf{\mathcal M})
\bigr)
\\
&\qquad=
\sum_{h,s,a}
G^r_{hsa}(\pi,\bar\pi)
\bigl(
\widetilde r_h(s,a)-r_h(s,a)
\bigr)
\\
&\qquad\quad+
\sum_{h,s,a,s'}
G^P_{hsa}(\pi,\bar\pi)(s')
\bigl(
\widetilde P_h(s'\mid s,a)-P_h(s'\mid s,a)
\bigr)
\\
&\qquad\quad+
\operatorname{Rem}_{\widetilde{\mathbf{\mathcal M}}}(\pi,\bar\pi),
\end{aligned}
\label{eq:policy-difference-decomposition}
\end{equation}
where
\begin{equation}
\begin{aligned}
\operatorname{Rem}_{\widetilde{\mathbf{\mathcal M}}}(\pi,\bar\pi)
:={}&
\sum_{h,s,a}
\Big[
\bigl(\widetilde d_{hsa}^\pi-d_{hsa}^\pi\bigr)
-
\bigl(\widetilde d_{hsa}^{\bar\pi}-d_{hsa}^{\bar\pi}\bigr)
\Big]
\bigl(
\widetilde r_h(s,a)-r_h(s,a)
\bigr)
\\
&+
\sum_{h,s,a,s'}
\Big[
\bigl(\widetilde d_{hsa}^\pi-d_{hsa}^\pi\bigr)
V_{h+1}^\pi(s')
\\
&\hspace{31mm}
-
\bigl(\widetilde d_{hsa}^{\bar\pi}-d_{hsa}^{\bar\pi}\bigr)
V_{h+1}^{\bar\pi}(s')
\Big]
\\
&\hspace{22mm}\times
\bigl(
\widetilde P_h(s'\mid s,a)-P_h(s'\mid s,a)
\bigr).
\end{aligned}
\label{eq:policy-difference-remainder}
\end{equation}
\end{lemma}

\begin{proof}
We first prove \eqref{eq:model-difference-identity}.
For a fixed policy $\pi$, the Bellman equations under $\mathbf{\mathcal M}$ and
$\widetilde{\mathbf{\mathcal M}}$ give
\begin{equation}
\begin{aligned}
V_h^\pi(s)-V_h^\pi(\widetilde{\mathbf{\mathcal M}};s)
&=
r_h(s,\pi_h(s))
-
\widetilde r_h(s,\pi_h(s))
\\
&\quad+
\sum_{s'}
\Bigl(
P_h(s'\mid s,\pi_h(s))
-
\widetilde P_h(s'\mid s,\pi_h(s))
\Bigr)
V_{h+1}^\pi(s')
\\
&\quad+
\sum_{s'}
\widetilde P_h(s'\mid s,\pi_h(s))
\Bigl(
V_{h+1}^\pi(s')
-
V_{h+1}^\pi(\widetilde{\mathbf{\mathcal M}};s')
\Bigr).
\end{aligned}
\label{eq:bellman-model-difference}
\end{equation}
Here $V_h^\pi(\widetilde{\mathbf{\mathcal M}};s)$ denotes the value function under
$\widetilde{\mathbf{\mathcal M}}$. Taking expectation in
\eqref{eq:bellman-model-difference} with respect to the state
distribution induced by $\pi$ under $\widetilde{\mathbf{\mathcal M}}$ and summing over
$h=1,\ldots,H$ makes the last term telescope. Since
$V_{H+1}^\pi\equiv0$, this yields
\eqref{eq:model-difference-identity}.

Applying \eqref{eq:model-difference-identity} to $\pi$ and $\bar\pi$,
reversing signs, and subtracting gives
\begin{equation}
\begin{aligned}
&
\bigl(
V_1^\pi(\widetilde{\mathbf{\mathcal M}})-V_1^{\bar\pi}(\widetilde{\mathbf{\mathcal M}})
\bigr)
-
\bigl(
V_1^\pi(\mathbf{\mathcal M})-V_1^{\bar\pi}(\mathbf{\mathcal M})
\bigr)
\\
&=
\sum_{h,s,a}
\bigl(
\widetilde d_{hsa}^\pi
-
\widetilde d_{hsa}^{\bar\pi}
\bigr)
\bigl(
\widetilde r_h(s,a)-r_h(s,a)
\bigr)
\\
&\quad+
\sum_{h,s,a,s'}
\Bigl(
\widetilde d_{hsa}^\pi V_{h+1}^\pi(s')
-
\widetilde d_{hsa}^{\bar\pi}V_{h+1}^{\bar\pi}(s')
\Bigr)
\\
&\hspace{31mm}\times
\bigl(
\widetilde P_h(s'\mid s,a)-P_h(s'\mid s,a)
\bigr).
\end{aligned}
\label{eq:policy-difference-exact}
\end{equation}

For the reward term in \eqref{eq:policy-difference-exact}, write
\begin{equation}
\begin{aligned}
\widetilde d_{hsa}^\pi
-
\widetilde d_{hsa}^{\bar\pi}
={}&
\bigl(
d_{hsa}^\pi-d_{hsa}^{\bar\pi}
\bigr)
\\
&+
\bigl(
\widetilde d_{hsa}^\pi-d_{hsa}^\pi
\bigr)
-
\bigl(
\widetilde d_{hsa}^{\bar\pi}-d_{hsa}^{\bar\pi}
\bigr).
\end{aligned}
\label{eq:reward-occupancy-split}
\end{equation}
Likewise, for the transition term,
\begin{equation}
\begin{aligned}
&
\widetilde d_{hsa}^\pi V_{h+1}^\pi(s')
-
\widetilde d_{hsa}^{\bar\pi}V_{h+1}^{\bar\pi}(s')
\\
&\qquad=
d_{hsa}^\pi V_{h+1}^\pi(s')
-
d_{hsa}^{\bar\pi}V_{h+1}^{\bar\pi}(s')
\\
&\qquad\quad+
\bigl(
\widetilde d_{hsa}^\pi-d_{hsa}^\pi
\bigr)V_{h+1}^\pi(s')
-
\bigl(
\widetilde d_{hsa}^{\bar\pi}-d_{hsa}^{\bar\pi}
\bigr)
V_{h+1}^{\bar\pi}(s').
\end{aligned}
\label{eq:transition-occupancy-split}
\end{equation}

Substituting \eqref{eq:reward-occupancy-split} and
\eqref{eq:transition-occupancy-split} into
\eqref{eq:policy-difference-exact}, and using
\[
G^r_{hsa}(\pi,\bar\pi)
=
d_{hsa}^\pi-d_{hsa}^{\bar\pi},
\]
and
\[
G^P_{hsa}(\pi,\bar\pi)(s')
=
d_{hsa}^\pi V_{h+1}^\pi(s')
-
d_{hsa}^{\bar\pi}V_{h+1}^{\bar\pi}(s'),
\]
gives \eqref{eq:policy-difference-decomposition}, with the remaining
terms exactly equal to
\eqref{eq:policy-difference-remainder}.
\end{proof}

\begin{lemma}[Two-fold trimming guarantee]
\label{lem:two-fold-trimming-guarantee}
Let
\[
\mathcal D_0
=
\left\{
(S_1^{(j)},A_1^{(j)},R_1^{(j)},\ldots,
S_H^{(j)},A_H^{(j)},R_H^{(j)})
\right\}_{j=1}^{n_0}
\]
consist of $n_0$ independent offline trajectories generated by the
behavior policy $\mu$.
Apply the two-fold subsampling procedure of
\citet[step 
2 of Algorithm 3]{gen2024settling}: split the trajectories
into two halves $\mathcal D_0^{\rm main}$ and
$\mathcal D_0^{\rm aux}$, use the auxiliary half to construct the
trimming levels, and uniformly subsample the main half accordingly.

Let $N_{0,hsa}^{\rm trim}$ denote the number of retained observations
from row $(h,s,a)$. Then, for every
$\delta_{\rm trim}\in(0,1)$, with probability at least
$1-8\delta_{\rm trim}$,
\begin{equation}
N_{0,hsa}^{\rm trim}
\ge
\left[
\frac{
n_0 d_{hsa}^\mu
}{
8
}
-
5
\sqrt{
n_0 d_{hsa}^\mu
\log
\frac{
n_0H
}{
\delta_{\rm trim}
}
}
\right]_+
\qquad
\forall(h,s,a).
\label{eq:trimmed-row-count-lower-bound}
\end{equation}

Moreover, one can construct
an augmented dataset
\begin{equation}
\mathcal D_0^{\rm aug}
=
\bigsqcup_{h,s,a}
\left\{
Z_{hsa}^{(i)}
\right\}_{i=1}^{N_{0,hsa}^{\rm aug}},
\qquad
Z_{hsa}^{(i)}
=
\left(R_{hsa}^{(i)},S_{hsa}^{\prime(i)}\right),
\label{eq:trimmed-iid-representation}
\end{equation}
such that, conditional on the augmented row counts
$\{N_{0,hsa}^{\rm aug}\}_{h,s,a}$,
\begin{equation}
Z_{hsa}^{(1)},\ldots,
Z_{hsa}^{(N_{0,hsa}^{\rm aug})}
\overset{\rm iid}{\sim}
\nu_{hsa}^{\mathbf{\mathcal M}},
\label{eq:trimmed-row-iid-law}
\end{equation}
independently across rows.

There exists an event $\mathcal E_{\rm trim}$ with
$\mathbb P_{\mathbf{\mathcal M}}(\mathcal E_{\rm trim})\ge1-8\delta_{\rm trim}$
on which both
\eqref{eq:trimmed-row-count-lower-bound} and
\[
\mathcal D_0^{\rm trim}
=
\mathcal D_0^{\rm aug},
\qquad
N_{0,hsa}^{\rm trim}
=
N_{0,hsa}^{\rm aug}
\quad\text{for all }(h,s,a)
\]
hold.
\end{lemma}

\begin{proof}
The result is exactly from
\citet[Lemma~3 \& Lemma~C.2]{gen2024settling}.
\end{proof}
\begin{remark}
\label{rem:augmented-iid-analysis}
By Lemma~\ref{lem:two-fold-trimming-guarantee}, the trimmed offline data can be coupled with
data that are iid within each row and independent across rows
conditional on the row counts. The two datasets coincide on
$\mathcal E_{\rm trim}$. 

In the proof of Theorem~\ref{thm:stopping-correctness}
and the analysis in Section~\ref{sec:shifted-simulator},
all distributional and concentration arguments involving
the offline data are actually carried out for the algorithm run
on the augmented data of
Lemma~\ref{lem:two-fold-trimming-guarantee}. Since the final guarantees are established on good events
contained in $\mathcal E_{\rm trim}$, we retain the original notation
throughout these arguments for simplicity. This follows the same argument as in
\citep{gen2024settling}.

\end{remark}

\begin{lemma}[Variance envelope under a data-dependent perturbed score]
\label{lem:variance-envelope-general}
Let
\[
Z,Z_1,\ldots,Z_N
\overset{\rm iid}{\sim}
P,
\qquad
N\ge2,
\]
and let
\[
f:\mathcal Z\to\mathbb R
\]
be a fixed measurable function satisfying
\[
\operatorname{range}(f)
\le
L.
\]
Let
\[
\widehat f=\widehat f(Z_1,\ldots,Z_N)
\]
be an arbitrary, possibly sample-dependent, random measurable function.

Suppose that, on an event $\mathcal E_\Delta$,
\[
\|f-\widehat f\|_\infty
\le
\Delta.
\]
Define
\[
\Gamma
:=
\operatorname{Var}\!\left(f(Z)\right)
\]
and
\[
\widehat\Gamma
:=
\frac{1}{N-1}
\sum_{i=1}^N
\left(
\widehat f(Z_i)
-
\frac1N\sum_{j=1}^N\widehat f(Z_j)
\right)^2.
\]
Let
\[
\overline\Gamma
:=
\left[
\sqrt{\widehat\Gamma}
+
L
\sqrt{
\frac{2\log(2/\delta)}{N-1}
}
+
\sqrt{2}\,\Delta
\right]^2.
\]
Then there exists an event $\mathcal E_{\rm var}$ with
\[
\mathbb P(\mathcal E_{\rm var})
\ge
1-\delta
\]
such that, on
$\mathcal E_{\rm var}\cap\mathcal E_\Delta$,
\begin{equation}
\Gamma
\le
\overline\Gamma
\le
3
\left[
\Gamma
+
4L^2
\frac{2\log(2/\delta)}{N-1}
+
9\Delta^2
\right].
\label{eq:general-variance-sandwich}
\end{equation}
In particular, no independence between $\widehat f$ and
$Z_1,\ldots,Z_N$ is required.
\end{lemma}

\begin{proof}
For a measurable function $g:\mathcal Z\to\mathbb R$, write
\[
s_N(g)
:=
\left[
\frac{1}{N-1}
\sum_{i=1}^N
\left(
g(Z_i)
-
\frac1N\sum_{j=1}^N g(Z_j)
\right)^2
\right]^{1/2}.
\]
Thus
\[
s_N(\widehat f)
=
\sqrt{\widehat\Gamma}.
\]

Since $f$ is fixed and
$\operatorname{range}(f)\le L$,
Theorem~10 of \citet{maurer2009empirical} implies that there exists
an event $\mathcal E_{\rm var}$ with probability at least
$1-\delta$ on which
\begin{equation}
\left|
s_N(f)
-
\sqrt{\Gamma}
\right|
\le
L
\sqrt{
\frac{2\log(2/\delta)}{N-1}
}.
\label{eq:true-score-sd-concentration}
\end{equation}

We next compare the empirical standard deviations of $f$ and
$\widehat f$.  Let
\[
\mathbf f
:=
\bigl(
f(Z_1),\ldots,f(Z_N)
\bigr)^\top,
\qquad
\widehat{\mathbf f}
:=
\bigl(
\widehat f(Z_1),\ldots,\widehat f(Z_N)
\bigr)^\top,
\]
and let
\[
\mathbf P_N
:=
I_N-\frac1N\mathbf 1\mathbf 1^\top
\]
denote the centering projection.  Then
\[
s_N(f)
=
\frac{
\|\mathbf P_N\mathbf f\|_2
}{
\sqrt{N-1}
},
\qquad
s_N(\widehat f)
=
\frac{
\|\mathbf P_N\widehat{\mathbf f}\|_2
}{
\sqrt{N-1}
}.
\]
Hence, by the triangle inequality and the fact that
$\|\mathbf P_N\|_{\rm op}\le1$,
\[
\begin{aligned}
\left|
s_N(\widehat f)-s_N(f)
\right|
&\le
\frac{
\|\mathbf P_N(
\widehat{\mathbf f}-\mathbf f
)\|_2
}{
\sqrt{N-1}
}
\\
&\le
\frac{
\|\widehat{\mathbf f}-\mathbf f\|_2
}{
\sqrt{N-1}
}.
\end{aligned}
\]
On $\mathcal E_\Delta$,
\[
\|\widehat{\mathbf f}-\mathbf f\|_2
\le
\sqrt N\,\Delta,
\]
and therefore
\begin{equation}
\left|
s_N(\widehat f)-s_N(f)
\right|
\le
\sqrt{\frac{N}{N-1}}\Delta
\le
\sqrt2\,\Delta.
\label{eq:empirical-sd-perturbation}
\end{equation}

Combining
\eqref{eq:true-score-sd-concentration}
and
\eqref{eq:empirical-sd-perturbation},
on
$\mathcal E_{\rm var}\cap\mathcal E_\Delta$,
\[
\begin{aligned}
\sqrt{\Gamma}
&\le
s_N(f)
+
L
\sqrt{
\frac{2\log(2/\delta)}{N-1}
}
\\
&\le
s_N(\widehat f)
+
L
\sqrt{
\frac{2\log(2/\delta)}{N-1}
}
+
\sqrt2\,\Delta
\\
&=
\sqrt{\overline\Gamma}.
\end{aligned}
\]
Thus
\[
\Gamma
\le
\overline\Gamma.
\]

Conversely,
\[
\begin{aligned}
s_N(\widehat f)
&\le
s_N(f)+\sqrt2\,\Delta
\\
&\le
\sqrt{\Gamma}
+
L
\sqrt{
\frac{2\log(2/\delta)}{N-1}
}
+
\sqrt2\,\Delta.
\end{aligned}
\]
Therefore
\[
\begin{aligned}
\sqrt{\overline\Gamma}
&=
s_N(\widehat f)
+
L
\sqrt{
\frac{2\log(2/\delta)}{N-1}
}
+
\sqrt2\,\Delta
\\
&\le
\sqrt{\Gamma}
+
2L
\sqrt{
\frac{2\log(2/\delta)}{N-1}
}
+
2\sqrt2\,\Delta
\\
&\le
\sqrt{\Gamma}
+
2L
\sqrt{
\frac{2\log(2/\delta)}{N-1}
}
+
3\Delta.
\end{aligned}
\]
Hence
we obtain
\[
\overline\Gamma
\le
3
\left[
\Gamma
+
4L^2
\frac{2\log(2/\delta)}{N-1}
+
9\Delta^2
\right].
\]
This proves
\eqref{eq:general-variance-sandwich}.
\end{proof}

\begin{lemma}[Control of the higher-order remainder]
\label{lem:iid-remainder}
Suppose that, conditional on the row counts
$\{N_{hsa}\}$, the empirical MDP
$\widehat{\mathbf{\mathcal M}}=(\widehat r,\widehat P)$
is constructed from independent row-wise samples.
Assume that, on an event $\mathcal E_{\rm mod}$,
\[
\|\widehat r-r\|_\infty
\le
\alpha^r,
\qquad
\|\widehat P-P\|_\infty
\le
\alpha^P.
\]
Then, simultaneously for all
$\pi,\bar\pi\in\Pi_d$,
\begin{equation}
\begin{aligned}
\left|
\operatorname{Rem}_{\widehat{\mathbf{\mathcal M}}}(\pi,\bar\pi)
\right|
\le{}&
H(H-1)|\mathcal S|
\alpha^P\alpha^r
\\
&+
\frac{H(H-1)(H-2)}{3}
|\mathcal S|^2
(\alpha^P)^2.
\end{aligned}
\label{eq:iid-remainder-bound}
\end{equation}
\end{lemma}

\begin{proof}
We condition on the event $\mathcal E_{\rm mod}$.
Recall that the remainder is
\[
\begin{aligned}
\operatorname{Rem}_{\widehat{\mathbf{\mathcal M}}}(\pi,\bar\pi)
={}&
\sum_{h,s,a}
\left(
\widehat d_{hsa}^{\pi}-d_{hsa}^{\pi}
\right)
\left(
\widehat r_h(s,a)-r_h(s,a)
\right)
\\
&+
\sum_{h,s,a,s'}
\left(
\widehat d_{hsa}^{\pi}-d_{hsa}^{\pi}
\right)
\left(
\widehat P_h(s'\mid s,a)-P_h(s'\mid s,a)
\right)
V_{h+1}^{\pi}(s')
\\
&-
\sum_{h,s,a}
\left(
\widehat d_{hsa}^{\bar\pi}-d_{hsa}^{\bar\pi}
\right)
\left(
\widehat r_h(s,a)-r_h(s,a)
\right)
\\
&-
\sum_{h,s,a,s'}
\left(
\widehat d_{hsa}^{\bar\pi}-d_{hsa}^{\bar\pi}
\right)
\left(
\widehat P_h(s'\mid s,a)-P_h(s'\mid s,a)
\right)
V_{h+1}^{\bar\pi}(s').
\end{aligned}
\]

For every deterministic policy $\pi$, write
\[
d_h^\pi
=
\left(
d_{hsa}^\pi
\right)_{s,a}
\]
for the state--action occupancy vector at step $h$, and define the
policy-induced state--action transition kernel by
\[
P_h^\pi\bigl((s',a')\mid(s,a)\bigr)
:=
P_h(s'\mid s,a)
\mathbf 1
\left\{
a'=\pi_{h+1}(s')
\right\}.
\]
Define $\widehat P_h^\pi$ analogously using $\widehat P_h$. Then
\[
d_{h+1}^\pi
=
d_h^\pi P_h^\pi,
\qquad
\widehat d_{h+1}^\pi
=
\widehat d_h^\pi\widehat P_h^\pi.
\]
Since
\[
\left\|
\widehat P_h^\pi(\cdot\mid s,a)
-
P_h^\pi(\cdot\mid s,a)
\right\|_1
=
\sum_{s'}
\left|
\widehat P_h(s'\mid s,a)
-
P_h(s'\mid s,a)
\right|
\le
|\mathcal S|\alpha^P,
\]
we obtain
\[
\begin{aligned}
\left\|
\widehat d_{h+1}^\pi
-
d_{h+1}^\pi
\right\|_1
&=
\left\|
\widehat d_h^\pi\widehat P_h^\pi
-
d_h^\pi P_h^\pi
\right\|_1
\\
&\le
\left\|
\left(
\widehat d_h^\pi-d_h^\pi
\right)
\widehat P_h^\pi
\right\|_1
+
\left\|
d_h^\pi
\left(
\widehat P_h^\pi-P_h^\pi
\right)
\right\|_1
\\
&\le
\left\|
\widehat d_h^\pi-d_h^\pi
\right\|_1
+
|\mathcal S|\alpha^P.
\end{aligned}
\]
Since the initial-state distribution is the same under $\mathbf{\mathcal M}$ and
$\widehat{\mathbf{\mathcal M}}$,
\[
\widehat d_1^\pi=d_1^\pi,
\]
and induction yields
\[
\left\|
\widehat d_h^\pi-d_h^\pi
\right\|_1
\le
(h-1)|\mathcal S|\alpha^P.
\]

Therefore, the reward part of the remainder for one policy satisfies
\[
\begin{aligned}
&
\left|
\sum_{h,s,a}
\left(
\widehat d_{hsa}^{\pi}
-
d_{hsa}^{\pi}
\right)
\left(
\widehat r_h(s,a)-r_h(s,a)
\right)
\right|
\\
&\qquad\le
\alpha^r
\sum_{h=1}^H
\left\|
\widehat d_h^\pi-d_h^\pi
\right\|_1
\\
&\qquad\le
|\mathcal S|\alpha^P\alpha^r
\sum_{h=1}^H(h-1)
\\
&\qquad=
\frac{H(H-1)}{2}
|\mathcal S|\alpha^P\alpha^r.
\end{aligned}
\]

Similarly, since
\[
0\le V_{h+1}^\pi(s')\le H-h,
\]
the transition part of the remainder satisfies
\[
\begin{aligned}
&
\left|
\sum_{h,s,a,s'}
\left(
\widehat d_{hsa}^{\pi}
-
d_{hsa}^{\pi}
\right)
\left(
\widehat P_h(s'\mid s,a)
-
P_h(s'\mid s,a)
\right)
V_{h+1}^{\pi}(s')
\right|
\\
&\qquad\le
\sum_{h=1}^H
\sum_{s,a}
\left|
\widehat d_{hsa}^{\pi}
-
d_{hsa}^{\pi}
\right|
\sum_{s'}
\left|
\widehat P_h(s'\mid s,a)
-
P_h(s'\mid s,a)
\right|
V_{h+1}^{\pi}(s')
\\
&\qquad\le
|\mathcal S|\alpha^P
\sum_{h=1}^H
(H-h)
\left\|
\widehat d_h^\pi-d_h^\pi
\right\|_1
\\
&\qquad\le
|\mathcal S|^2(\alpha^P)^2
\sum_{h=1}^H
(h-1)(H-h)
\\
&\qquad=
\frac{H(H-1)(H-2)}{6}
|\mathcal S|^2(\alpha^P)^2.
\end{aligned}
\]
Hence
\[
\left|
\operatorname{Rem}_{\widehat{\mathbf{\mathcal M}}}(\pi,\bar\pi)
\right|
\le
H(H-1)|\mathcal S|\alpha^P\alpha^r
+
\frac{H(H-1)(H-2)}{3}
|\mathcal S|^2(\alpha^P)^2,
\]
simultaneously for all
$\pi,\bar\pi\in\Pi_d$.
\end{proof}

\begin{lemma}[Bernstein bound]
\label{lem:iid-linear-bernstein}
Suppose that, conditional on the row counts
$\{N_{hsa}\}_{h,s,a}$ with $N_{hsa}\ge1$,
\[
Z_{hsa}^{(1)},\ldots,Z_{hsa}^{(N_{hsa})}
\overset{\rm iid}{\sim}
\nu_{hsa}^{\mathbf{\mathcal M}}
\]
independently across rows. Then, for every fixed policy pair
$(\pi,\bar\pi)$, with probability at least $1-\delta$,
\begin{equation}
\begin{aligned}
&
\left|
\sum_{h,s,a}
\frac1{N_{hsa}}
\sum_{i=1}^{N_{hsa}}
\left[
\Psi_{hsa}^{\pi,\bar\pi}
\left(Z_{hsa}^{(i)}\right)
-
\mathbb E_{\mathbf{\mathcal M}}
\left[
\Psi_{hsa}^{\pi,\bar\pi}(Z)
\mid s,a
\right]
\right]
\right|
\\
&\qquad\le
\sqrt{
2\log\frac{2}{\delta}
\sum_{h,s,a}
\frac{
\Gamma_{hsa}(\pi,\bar\pi)
}{
N_{hsa}
}
}
+
\frac{2H}{3}
\log\frac{2}{\delta}
\max_{h,s,a}
\frac1{N_{hsa}}.
\end{aligned}
\label{eq:iid-linear-bernstein}
\end{equation}
\end{lemma}

\begin{proof}
Condition on the realized row counts.

For each row $(h,s,a)$,
\[
\operatorname{Var}_{\mathbf{\mathcal M}}
\left(
\Psi_{hsa}^{\pi,\bar\pi}(Z)
\mid s,a
\right)
=
\Gamma_{hsa}(\pi,\bar\pi).
\]
Therefore, the variance of the full sum is
\[
\sum_{h,s,a}
N_{hsa}
\frac{
\Gamma_{hsa}(\pi,\bar\pi)
}{
N_{hsa}^2
}
=
\sum_{h,s,a}
\frac{
\Gamma_{hsa}(\pi,\bar\pi)
}{
N_{hsa}
}.
\]

Moreover, we have
\[
\left|
G_{hsa}^{r}(\pi,\bar\pi)
\right|
\le1,
\]
and, since
$0\le V_{h+1}^{\pi},V_{h+1}^{\bar\pi}\le H-h$,
\[
\operatorname{range}
\left(
G_{hsa}^{P}(\pi,\bar\pi)
\right)
\le
2(H-h).
\]
Thus
\[
\operatorname{range}
\left(
\Psi_{hsa}^{\pi,\bar\pi}
\right)
\le
1+2(H-h)
\le
2H.
\]

Applying Bernstein's inequality
gives, with conditional probability at least $1-\delta$,
\[
\begin{aligned}
&
\left|
\sum_{h,s,a}
\frac1{N_{hsa}}
\sum_{i=1}^{N_{hsa}}
\left[
\Psi_{hsa}^{\pi,\bar\pi}
\left(Z_{hsa}^{(i)}\right)
-
\mathbb E_{\mathbf{\mathcal M}}
\left[
\Psi_{hsa}^{\pi,\bar\pi}(Z)
\mid s,a
\right]
\right]
\right|
\\
&\qquad\le
\sqrt{
2\log\frac{2}{\delta}
\sum_{h,s,a}
\frac{
\Gamma_{hsa}(\pi,\bar\pi)
}{
N_{hsa}
}
}
+
\frac{2H}{3}
\log\frac{2}{\delta}
\max_{h,s,a}
\frac1{N_{hsa}}.
\end{aligned}
\]
\end{proof}

\section{Additional Simulation Details and Results}\label{app:simulations}

Here we describe the implementation and protocol behind Section~\ref{sec:simulations} and report every setting: offline sample sizes $n_0\in\{100,250,500,1000,5000\}$ in Simulations~1 and~2, and these sizes combined with $H\in\{5,10,15\}$ in Simulation~3, for 25 settings with 300 paired trials each under both versions of the algorithm. All results use $C_0=0.5$, except in Appendix~\ref{app:sim-c01}, which repeats every setting with $C_0=1$. Figure~\ref{fig:app-environments} specifies the environments in full. Figures~\ref{fig:app-sim1}-\ref{fig:app-sim3-h15} show recovery, regret, certification and query allocation in every setting, and Figure~\ref{fig:app-stopping} the certification costs.

\begin{figure}[ht!]
\centering
\includegraphics[width=0.90\linewidth]{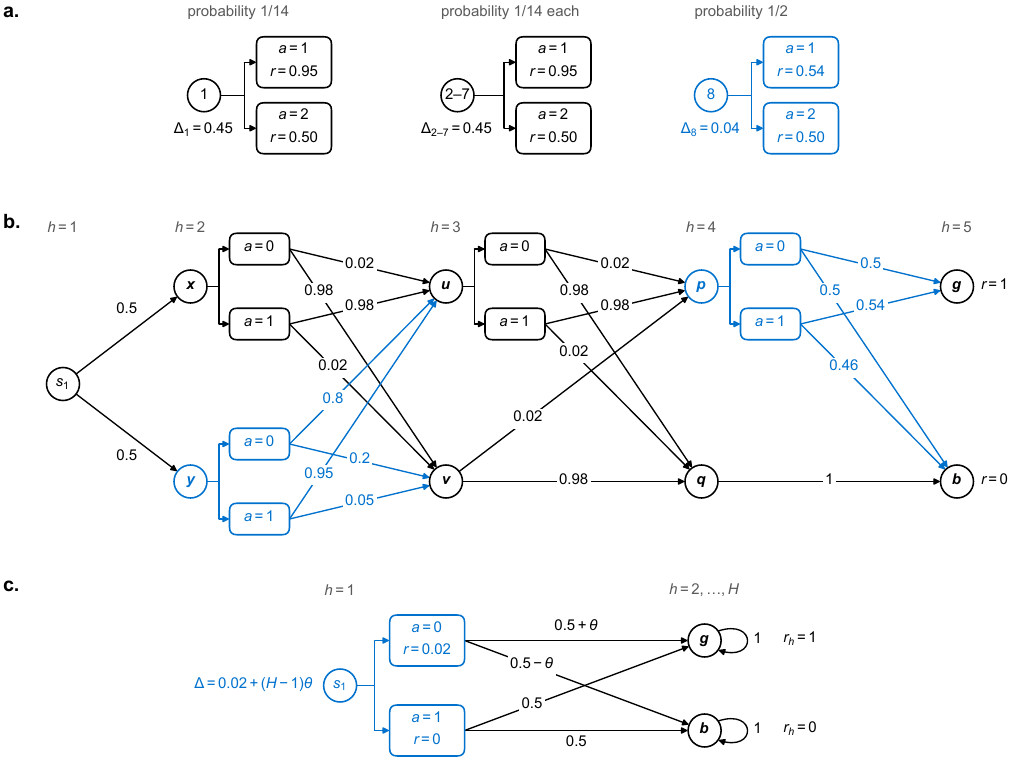}
\caption{Simulation environments in full, in the style of Figure~\ref{fig:sim-environments}. Circles are states, boxes are actions, and arrows are labelled with transition probabilities. \textbf{a.}~Simulation~1: each context's probability, the Bernoulli reward means of its two actions, and the gap $\Delta_s$. Contexts~2-7 are identical and drawn once; Context~1 is drawn separately because the behavior policy selects its optimal action with probability $0.01$, compared with $0.20$ elsewhere. \textbf{b.}~Simulation~2: $v$ and $q$ have a single action; reward is~1 only at $g$, and $q$ never reaches $g$. \textbf{c.}~Simulation~3: the absorbing states $g$ and $b$ pay 1 and 0 at every step until the horizon $H$. Blue marks the hard decisions.}
\label{fig:app-environments}
\end{figure}

\subsection{Implementation and Protocol}\label{app:sim-protocol}

\paragraph{Theory-aligned version.}
The theory-aligned version retains the data-splitting and epoch
structure of the algorithm in Section~\ref{sec:algorithm}, on tabular layered MDPs
($H=1$ in Simulation~1), with the experimental parameter choices
and stabilization rule described below. The offline trajectories are trimmed with the two-fold procedure of~\cite{gen2024settling} at failure probability $\delta/48$, and the retained data are split equally into the historical sample $\mathcal{D}^\mathcal G_0$ and the inference sample $\mathcal{D}^\mathcal D_0$; the latter is revealed in fresh folds with weights $w_k=6/(\pi^2k^2)$. At epoch $k$, with resolution $\varepsilon_k=H2^{1-k}$ and the confidence factor $\beta_k$ computed from $q_{k-1}=|\widehat\Pi_{k-1}|^2$, the algorithm collects a uniform stabilization batch at every state--time pair, executes the allocation chosen at the previous epoch, and localizes the candidate set at $5\varepsilon_k/2$ using that epoch's fresh data only. Occupancies, values and sensitivities are computed in the empirical MDP of the historical data, and the variance envelope on the fresh sample. The next allocation solves the minimax design with the regularized margin $\max\{\varepsilon-\widehat D_k,4\varepsilon_k\}$ at the smallest dyadic budget $B_k$ with $\widehat\Phi^\star_k(B_k)\le1/16$. For each candidate reference policy, the continuous minimum budget and the min--max allocation are convex programs with one variable per state--time pair, which we solve with SLSQP; integer allocations use largest-remainder rounding, and the $1/16$ criterion is rechecked after rounding. A run stops when a single candidate remains or $\varepsilon_k\le\varepsilon$, and returns the empirical best policy on that epoch's fresh data. Because the collaborator plays every action with positive probability, the stops for insufficient coverage of collaborator-unsupported rows and for the certification limit never occur. The policy classes contain all deterministic policies: $256$, $16$ and $2$ in Simulations~1-3.

\paragraph{Efficient version.}
The efficient version makes three changes, applied to every method alike. First, it pools the data: every step uses all data collected so far, namely the full untrimmed offline sample and every collaborator query from every epoch, in place of the epoch's fresh sample and the historical sample. This applies to the value estimates and the localization, the reported policy, the empirical model, the scores and variance envelope, and the counts already held when the next design is solved. Second, it eliminates candidates by confidence: after the localization, a candidate $\pi$ is also removed when another candidate $\pi'$ satisfies
\[
\widehat D_k(\pi',\pi)\;>\;C_0\Big(\beta_k\sum_{h,s,a}\Gamma_{k,hsa}(\pi,\pi')\big/N_{k,hsa}\Big)^{1/2},
\]
where $\widehat D_k(\pi',\pi)$ is the estimated value difference, $\Gamma_{k,hsa}(\pi,\pi')$ the variance envelope of the pair's influence score at $(h,s,a)$ and $N_{k,hsa}$ the number of observations of $(h,s,a)$. The right-hand side is the realized counterpart of the confidence radius that the design controls, computed from the same data. Third, it tops up the stabilization batches: at epoch~$k$ each state-time pair receives only the shortfall between $m^{\mathrm{st}}_k$ and the stabilization queries it has received in earlier epochs, instead of $m^{\mathrm{st}}_k$ new ones, since with pooled data the earlier batches remain in use.

\paragraph{Constants.}
All methods and both versions use the same constants. The trimming constant of~\cite{gen2024settling} is $1$, and $C_0=0.5$ ($C_0=1$ in Appendix~\ref{app:sim-c01}). We set $C_{\mathrm{uc}}=C_{\mathrm{sen}}=0$. For stabilization, we directly choose the target
\(
m_k^{\rm st}:=\lceil 5\beta_k\rceil
\)
per state-time pair.
This is an empirical simplification of
\eqref{eq:app-stabilization-size}. In particular, the
trimming correction $50w_k\ell_{\rm trim}+2$ is omitted
when setting this target. The variance envelope uses the exact range of the influence score in place of $2H$, and the sensitivity envelope $\Delta^\mathcal G$ is set to zero, which is exact at $H=1$.

\paragraph{Choice of $C_0$.}
$C_0$ scales the confidence radius $C_0\big(\beta_k\sum_{h,s,a}\Gamma_{k,hsa}/N_{k,hsa}\big)^{1/2}$, which sizes each design batch and, in the efficient version, sets the elimination threshold. The analysis takes $C_0=8$, so that the radius holds simultaneously over all candidate pairs and epochs with worst-case concentration constants. Because $\Gamma_{k,hsa}$ envelopes the variance of the influence score, the radius is at least $C_0\sqrt{\beta_k}$ estimated standard errors of the value difference, and $\sqrt{\beta_k}$ lies between $3.5$ and $5.0$ in our runs. With $C_0=0.5$ the radius is therefore about two standard errors, the half-width of a pointwise 95\% normal interval; a normal-approximation union bound over all candidate pairs and epochs would correspond to $C_0\approx1$. A pointwise radius suffices in these instances because a stopping failure requires a comparison against the optimal policy to be wrong by more than the radius plus the competitor's regret, which exceeds $\varepsilon$, and only a few competitors are that close. All methods share the stopping rule, so $C_0$ rescales their costs together: with $C_0=1$ (Appendix~\ref{app:sim-c01}), adaptive allocation needs $2.2$ to $3.7$ times as many queries in the efficient version and $3.2$ to $4.6$ times as many in the theory-aligned version, and it remains the cheapest in every setting.

\paragraph{Baselines.}
Uniform and coverage run the same version of the algorithm on the same offline sample and folds as the paired adaptive run, with the same stabilization, localization, elimination (efficient version), stopping and reporting rules. Only the allocation of design queries differs: uniform spreads each design batch evenly over the state--time pairs, and coverage splits it in proportion to $\sum_a\pi^c_h(a\mid s)/(1+N_{hsa})^2$, where $N_{hsa}$ counts all observations of $(h,s,a)$ collected so far, offline and queried; neither uses gaps, values or variances. In the fixed-budget protocol both baselines also use the paired adaptive run's epoch batch sizes and stabilization draws. In the stopping protocol each baseline solves the same design problem with its allocation restricted to its own shape, and so selects its own batch sizes. The offline-only reference returns the policy with the highest value in the empirical model of the full, untrimmed offline sample, with unobserved rows contributing zero.

\paragraph{Protocols and metrics.}
In the fixed-budget protocol the stop at $\varepsilon_k\le\varepsilon$ is disabled, and each setting runs until every adaptive trial would have stopped; a run whose candidate set has shrunk to a single policy keeps returning it. We record the policy returned at each epoch end, which all three methods reach with the same number of queries, and report its recovery (whether it is optimal) and regret $V^\star_1-V^{\widehat\pi}_1$ at the mean cumulative number of queries. In the theory-aligned version this policy is selected on the epoch's fresh data, which in early epochs contain only a fold of the offline sample ($61\%$ of $\mathcal{D}^\mathcal D_0$ in epoch~1 and $15\%$ in epoch~2). Its recovery can therefore fall from one epoch to the next, and at large $n_0$ the first epochs can fall below the offline-only reference, which uses the full sample. The efficient version selects it on all data collected so far. Every returned policy that is not optimal errs only at the hard decisions: its regret is $0.02$ in Simulation~1, $0.038$ ($p$), $0.039$ ($y$) or $0.074$ (both) in Simulation~2, and the value gap in Simulation~3. Mean regret therefore mirrors recovery, and we discuss recovery below. In the stopping protocol, the certification cost is the total number of collaborator queries, stabilization included, until a method's stopping rule returns; a stopping run fails if the returned policy is not $\varepsilon$-optimal. The certification panels give the fraction of trials whose stopping rule has returned by each budget, each method under its own allocation.

\paragraph{Statistics, seeds and computation.}
Each setting uses 300 trials. Within a trial, all methods and both versions share the offline sample with its trimming and fold split, and each trial's random streams are derived from the simulation's seed base ($20261101$, $20261161$ and $20261171$), the setting and the trial index; both values of $C_0$ use the same seed bases. Proportions have pointwise 95\% Wilson intervals, mean regret and mean query costs have intervals of the mean $\pm1.96$ standard errors, and paired differences in recovery have intervals of the mean $\pm1.96$ standard errors of the per-trial differences. No multiplicity adjustment is applied. The runs used six to nine worker processes on a single Apple-silicon laptop (Python~3.13.5, NumPy~2.5.3, SciPy~1.18.1). For each value of $C_0$, Simulation~1 took under 7 minutes for the efficient version and under 1.5 hours for the theory-aligned version, and Simulations~2 and~3 under 3 minutes each.

\subsection{Summary of Results}\label{app:sim-summary}

\paragraph{Certification cost.}
Adaptive allocation needs the fewest queries to certify in every setting under both versions (Figure~\ref{fig:app-stopping}). In the efficient version it certifies after $6.1\times10^{3}$ to $2.2\times10^{4}$ queries on average, and uniform and coverage allocation need $5.9$-$24.6$ and $3.6$-$17.0$ times as many. In the theory-aligned version it certifies after $2.6\times10^{5}$ to $5.5\times10^{5}$ queries, and the baselines need $5.1$-$21.2$ and $3.4$-$15.8$ times as many; there, in every setting, the last adaptive trial certified before the first baseline trial. The efficient version's costs are lowest at $n_0=5{,}000$ in every simulation, and the theory-aligned costs barely change with $n_0$.

\begin{figure}[p]
\centering
\includegraphics[width=0.95\linewidth]{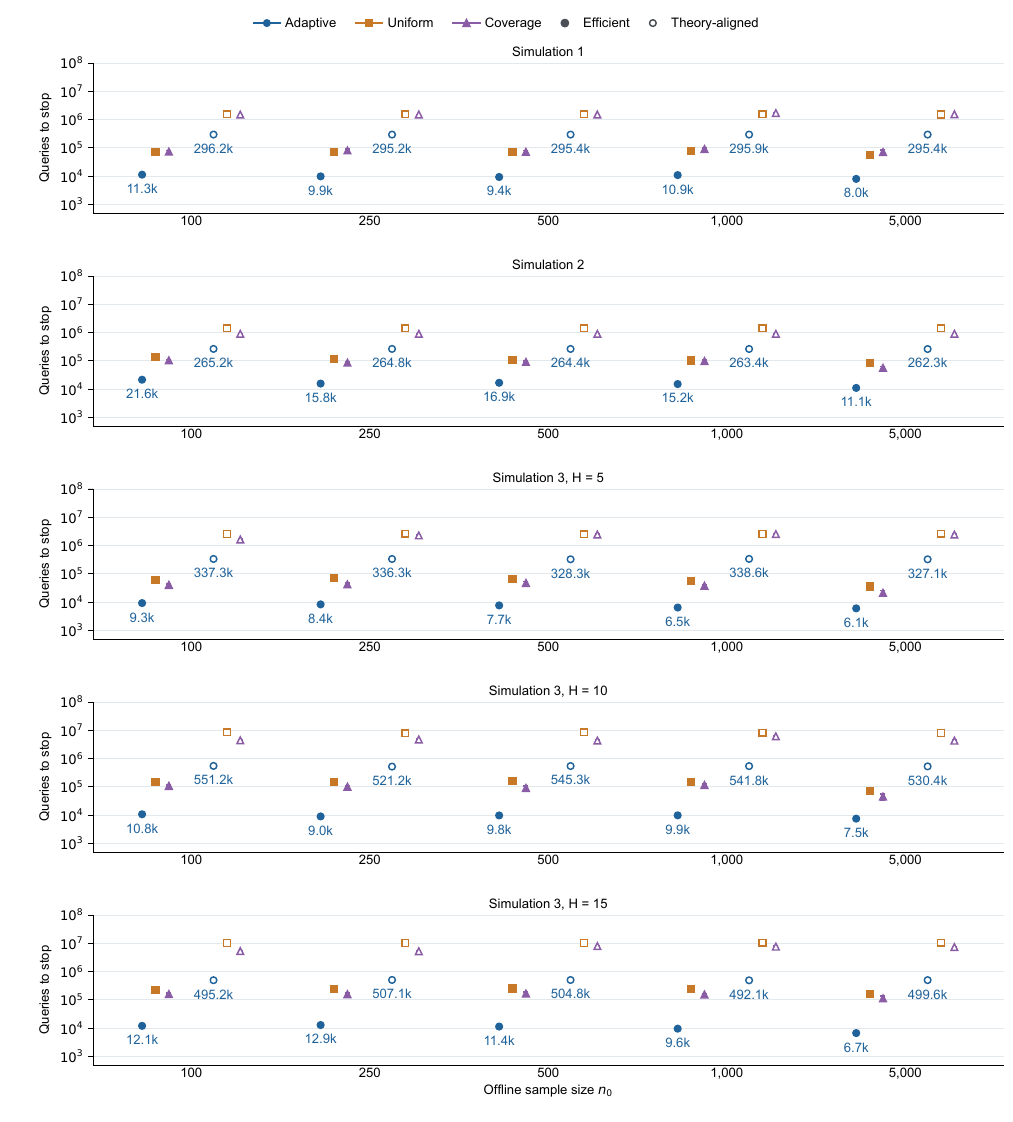}
\caption{Mean collaborator queries until each method's stopping rule returns (log scale; $\pm1.96$ standard errors, mostly smaller than the markers), for every setting and offline sample size, under the efficient (filled) and theory-aligned (open) versions of the algorithm of Section~\ref{sec:algorithm}, both with $C_0=0.5$; numbers below the markers give adaptive's mean cost.  No stopping run returned a policy that is not $\varepsilon$-optimal.}
\label{fig:app-stopping}
\end{figure}

\paragraph{Fixed budgets.}
At the epoch end where adaptive allocation's smaller recovery gain over the two baselines is largest, the efficient version's adaptive allocation recovers the optimal policy more often than uniform allocation in 23 of the 25 settings and more often than coverage in 22, at the 95\% level; all exceptions are at $n_0=5{,}000$ in Simulations~1 and~3, where the offline data alone recover the optimal policy in $90$-$93\%$ of trials. The theory-aligned version does so in all 25 settings against both baselines. These gains are the largest of several epoch-end comparisons, so they overstate the typical gain; the figures show every epoch end. Across all epoch ends, settings and both versions, no paired difference in recovery is significantly below zero.

\paragraph{Efficient and theory-aligned versions.}
At $n_0=500$ the efficient version recovers the optimal policy more often than the theory-aligned one at equal budgets under all three allocations, because it decides on all data collected so far. For example, in Simulation~1 adaptive, uniform and coverage allocation recover $98.0\%$, $85.7\%$ and $81.7\%$ after $3.4\times10^{3}$ queries under the efficient version, and $94.3\%$, $79.3\%$ and $77.3\%$ after $5.0\times10^{3}$ under the theory-aligned version; in Simulation~3 with $H=15$ they recover $96.0\%$, $77.0\%$ and $83.3\%$ after $4.1\times10^{3}$ queries under the efficient version, and $79.0\%$, $73.3\%$ and $74.3\%$ after $4.3\times10^{3}$ under the theory-aligned version. The theory-aligned version's largest gains are mostly as large or larger, because its baselines, which decide on each epoch's fresh data, recover less, but at $n_0=500$ they occur at budgets $1.5$ to $5.6$ times larger.

The theory-aligned version needs $12$ to $74$ times as many queries as the efficient one to certify, because only the localization shrinks its candidate set. A policy with regret $\Delta$ leaves the set once $5\varepsilon_k/2<\Delta$, up to estimation error, and with $\varepsilon_k=H2^{1-k}$ this first happens at epoch~8 in Simulation~1 ($5\varepsilon_8/2\approx0.0195$ against $\Delta=0.02$) and at epoch~7 or~8 in Simulation~3. In Simulation~2 the run stops at epoch~9, when $\varepsilon_9\le\varepsilon$, before $5\varepsilon_k/2$ falls below $0.038$. The theory-aligned stopping epoch is therefore set by the schedule and barely depends on $n_0$; the allocation determines how many queries those epochs need. Elimination removes a candidate as soon as its estimated gap exceeds the confidence radius, so the efficient version usually stops at epoch~5, 6 or~7; at $n_0=5{,}000$ up to $17\%$ of its trials stop at epoch~1 or~2, when the offline data and the first stabilization batch already separate the policies.

Adding the three changes one at a time to the theory-aligned version at $n_0=500$ separates their effects. Pooling alone lowers the cost of certification $1.6$- to $2.2$-fold and still stops at epochs~7 to~9; adding elimination lowers it a further $7.6$- to $20$-fold, and topping up the stabilization batches a further $1.1$- to $1.7$-fold. Adaptive allocation certifies first at every step. With pooling alone, adaptive's largest fixed-budget gains over uniform allocation are never larger than at any other step, and the top-up raises its largest gains over both baselines in all five settings: after epoch~1 it adds at most seven stabilization queries per state--time pair and epoch (medians), whereas without it every epoch adds $\lceil5\beta_k\rceil$ new ones, so more of each budget goes to the design.

\clearpage
\subsection{Simulation 1: Contextual Bandit}\label{app:sim1}

\paragraph{Environment.}
The bandit has $|\mathcal S|=8$ contexts and $|\mathcal A|=2$ actions with Bernoulli rewards (Figure~\ref{fig:app-environments}a). Contexts~1-7 have probability $1/14$ each and reward means $(0.95,0.50)$; Context~8 has probability $1/2$ and means $(0.54,0.50)$. The behavior policy selects the optimal action with probability $0.01$ in Context~1 and $0.20$ in every other context, and the collaborator selects it with probability $0.6$ in every context. A wrong action in Context~8 costs $0.5\times0.04=0.02$, the smallest nonzero regret, which exceeds $\varepsilon=0.01$.

\paragraph{Fixed budgets.}
The offline data alone recover the optimal policy in $1.3\%$ of trials at $n_0=100$ and $93.0\%$ at $n_0=5{,}000$ (Figure~\ref{fig:app-sim1}). In the efficient version the methods nearly coincide in the first epochs: after the epoch-1 stabilization batch ($888$ queries), the design batch of epoch~2 is a single query at every $n_0$, as is that of epoch~3 in almost every trial for $n_0\ge500$. Adaptive allocation then recovers the optimal policy more often than both baselines for $n_0\le1{,}000$, by $9.7$ to $16.3$ percentage points over uniform and $16.3$ to $21.3$ over coverage at the epoch of largest gain, after $2.5\times10^{3}$ to $3.8\times10^{3}$ queries. At $n_0=5{,}000$ all three methods recover $95.3\%$ through epoch~4, as the pooled data contain the offline sample, and the largest gains ($1.3$ points over uniform and $2.0$ over coverage, at epoch~5) are small; only the latter is significant. Adaptive returns the optimal policy in every trial after $7.7\times10^{3}$ to $1.1\times10^{4}$ queries. In the theory-aligned version the largest gains occur at epoch~3 for $n_0\le1{,}000$, after about $5.0\times10^{3}$ queries: $11.3$ to $17.0$ points over uniform and $16.0$ to $24.0$ over coverage; at $n_0=5{,}000$ they are $10.7$ and $15.7$ points, at epoch~4. At $n_0=5{,}000$ its first three epoch ends fall below the offline-only reference, and coverage stays at $71$-$84\%$ through epoch~4. Adaptive returns the optimal policy in every trial after about $3.2\times10^{4}$ queries ($9.6\times10^{4}$ at $n_0=1{,}000$), and all three methods after about $9.6\times10^{4}$ ($1.6\times10^{5}$ at $n_0=500$).

\paragraph{Certification cost.}
In the efficient version adaptive allocation certified after $8.0\times10^{3}$ to $1.1\times10^{4}$ queries on average. Every trial stopped with a single candidate left, most often at epoch~6 ($52$-$76\%$ of trials); at $n_0=5{,}000$, $17\%$ stopped at epoch~1 or~2. Uniform allocation needed $6.4$ to $7.4$ times as many queries and coverage $6.5$ to $8.9$ times. The cost distributions overlap: when the last adaptive trial certified, $23$-$37\%$ of uniform and $18$-$46\%$ of coverage trials had also certified. In the theory-aligned version every trial stopped at epoch~8, with a single candidate left ($58$-$66\%$ of trials) or at $\varepsilon_8\le\varepsilon$, after $2.95\times10^{5}$ to $2.96\times10^{5}$ queries on average; uniform allocation needed $5.1$ to $5.2$ times as many and coverage $5.1$ to $5.7$ times. Every adaptive trial certified within $1.23$ times the adaptive mean cost, and no baseline trial before $5.08$ times it.

\paragraph{Allocation.}
In the efficient version, adaptive allocation had spent $68$-$75\%$ of its queries on Context~8 after epoch~5 for $n_0\le1{,}000$ ($3.3\times10^{3}$ to $4.9\times10^{3}$ queries), compared with $12.5\%$ for uniform and $5.9$-$11.4\%$ for coverage, and $3.6$-$4.6\%$ on Context~1 (Figure~\ref{fig:app-sim1}d). At $n_0=5{,}000$ the epoch-1 stabilization batch makes up most of the $1.1\times10^{3}$ queries by then, and the share of Context~8 is $26\%$. From epoch~2 on, $76$-$100\%$ of its design queries went to Context~8; the top-up adds no stabilization queries after epoch~1 (medians), because elimination shrinks the candidate set and with it $\beta_k$. At $n_0=500$ the median number of candidates after epochs~1 to~4 is $8$, $4$, $2$ and~$2$. In the theory-aligned version, adaptive allocation had spent $59\%$ of its queries on Context~8 and $5.8\%$ on Context~1 after epoch~7 (about $1.6\times10^{5}$ queries). Its design batches are spread more widely while many candidates remain: the share assigned to Context~8 is $45$-$65\%$ in epochs~3 to~6 and rises to $72\%$ in epoch~7 and $86\%$ in epoch~8, as the median number of candidates falls from $256$ in epochs~1 to~4 to $198$, $37$, $9$ and~$1$.

\begin{figure}[p]
\centering
\includegraphics[width=0.95\linewidth]{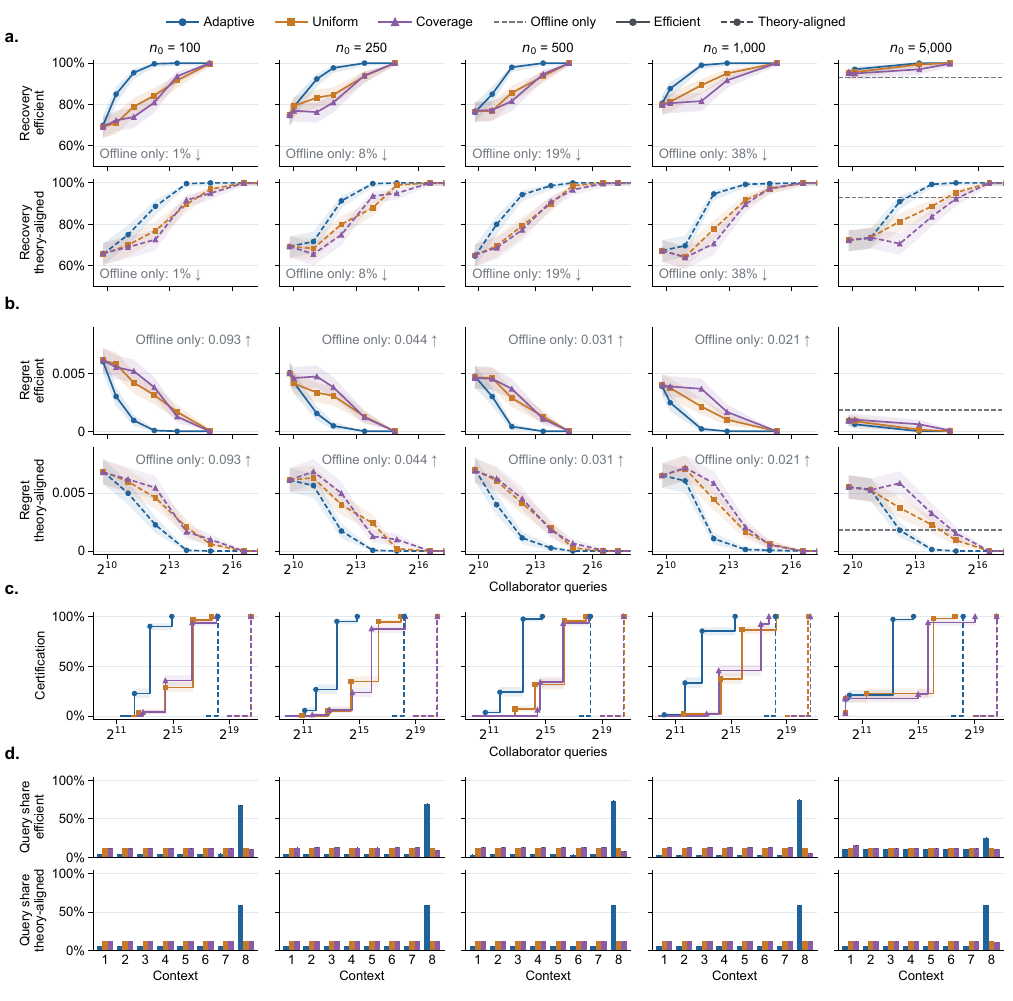}
\caption{Simulation~1 at every offline sample size $n_0$ (columns; 300 paired trials each), for the efficient (solid) and theory-aligned (dashed) versions of the algorithm of Section~\ref{sec:algorithm}, both with $C_0=0.5$. \textbf{a.}~Fraction of trials returning the optimal policy and \textbf{b.}~mean regret $V^\star_1-V^{\widehat\pi}_1$ of the returned policy at each epoch end, against mean collaborator queries (top: efficient; bottom: theory-aligned; the four rows of a column share the $x$-axis); offline-only levels outside the axis are printed in the panel. \textbf{c.}~Fraction of trials whose stopping rule has returned, each method under its own allocation; a curve starts at the last epoch with no certified trial and ends once all have certified. \textbf{d.}~Each method's share of the efficient (top) and theory-aligned (bottom) version's queries after that version's last epoch before adaptive typically stops (epochs~5 and~7). Bands: pointwise 95\% Wilson intervals (recovery, certification) and mean $\pm1.96$ standard errors (regret).}
\label{fig:app-sim1}
\end{figure}

\clearpage
\subsection{Simulation 2: Multiple Sequential Decisions}\label{app:sim2}

\paragraph{Environment.}
The MDP has horizon $H=5$ and the layered structure of Figure~\ref{fig:app-environments}b. The initial state $s_1$ moves to $x$ or $y$ with probability $1/2$ each. Actions at $x$ and $y$ lead to $u$ or $v$; actions at $u$, and the single action at $v$, lead to $p$ or $q$; actions at $p$, and the single action at $q$, lead to $g$ or $b$. Reward is $1$ only on reaching $g$. Action~$1$ leads to the first of the two next states ($u$, $p$ or $g$) with probability $0.98$ at $x$ and $u$, $0.95$ at $y$ and $0.54$ at $p$; action~$0$ does so with probability $0.02$, $0.80$ and $0.50$, respectively. The state $v$ moves to $p$ with probability $0.02$, and $q$ never reaches $g$. The optimal policy takes action~$1$ at every decision state and has value $0.511$. A single wrong choice costs $0.249$ at $x$, $0.039$ at $y$, $0.500$ at $u$ and $0.038$ at $p$, and every suboptimal policy has regret at least $0.038>\varepsilon=0.02$. The behavior policy takes action~$1$ with probability $0.9$ at $x$ and $u$ and $0.005$ at $y$ and $p$; the collaborator takes it with probability $0.6$ at every decision state. All nine state--time pairs can be queried, and $n_0$ counts offline trajectories.

\paragraph{Fixed budgets.}
Offline-only recovery rises from $4.3\%$ at $n_0=100$ to $58.0\%$ at $n_0=5{,}000$ (Figure~\ref{fig:app-sim2}). In the efficient version the methods coincide in epoch~1 and nearly so through epoch~3: the pooled data already meet the early design targets, so the design batches of epochs~2 and~3 are a single query in almost every trial at $n_0\ge250$. Adaptive's largest gains are then $10.7$ to $14.7$ percentage points over uniform and $6.0$ to $10.0$ over coverage, after $1.0\times10^{3}$ to $3.5\times10^{3}$ queries, all significant. In the theory-aligned version the largest gains are $8.7$ to $14.0$ points over uniform and $7.7$ to $12.3$ over coverage, after $6.2\times10^{3}$ to $1.5\times10^{4}$ queries. Adaptive returns the optimal policy in every trial after $3.4\times10^{3}$ to $3.7\times10^{4}$ queries in the efficient version and $3.2\times10^{4}$ to $6.6\times10^{4}$ in the theory-aligned version.

\paragraph{Certification cost.}
In the efficient version every trial stopped with a single candidate left, at epochs~5 to~9 (mostly~6 to~8), after $1.1\times10^{4}$ to $2.2\times10^{4}$ queries on average. Uniform allocation needed $6.4$ to $7.7$ times as many queries and coverage $4.8$ to $6.5$ times. When the last adaptive trial certified, $5$-$20\%$ of uniform and $6$-$15\%$ of coverage trials had also certified, except at $n_0=500$, where the last adaptive trial needed epoch~9 and $5.4$ times the adaptive mean cost, and $81\%$ of uniform and $75\%$ of coverage trials had certified by then. In the theory-aligned version every trial stopped at epoch~9, when $\varepsilon_9=5\cdot2^{-8}\le\varepsilon$, after $2.6\times10^{5}$ queries on average at every $n_0$; uniform allocation needed $5.4$ to $5.5$ times as many and coverage $3.4$ to $3.5$ times, and no baseline trial certified before $3.4$ times the adaptive mean cost.

\paragraph{Allocation.}
In the efficient version, adaptive allocation had spent $58$-$63\%$ of its queries on $p$ and $21$-$23\%$ on $y$ after epoch~6 ($3.4\times10^{3}$ to $1.0\times10^{4}$ queries), compared with $11.1\%$ each for uniform and $12$-$30\%$ for coverage (Figure~\ref{fig:app-sim2}d); from epoch~5 on, $98$-$100\%$ of its design queries went to $p$ and $y$. In the theory-aligned version it had spent $52\%$ on $p$ and $18\%$ on $y$ after epoch~8 (about $1.3\times10^{5}$ queries), compared with $14\%$ each for coverage. The theory-aligned design batches assign $44$-$54\%$ of their queries to $p$ and $y$ in epochs~4 to~7, while the median number of candidates falls from $16$ to~$4$, and $90$-$94\%$ from epoch~8 on. Earlier batches also cover $x$, $u$ and the states without a choice, whose transition estimates enter the estimated value differences.

\begin{figure}[p]
\centering
\includegraphics[width=0.95\linewidth]{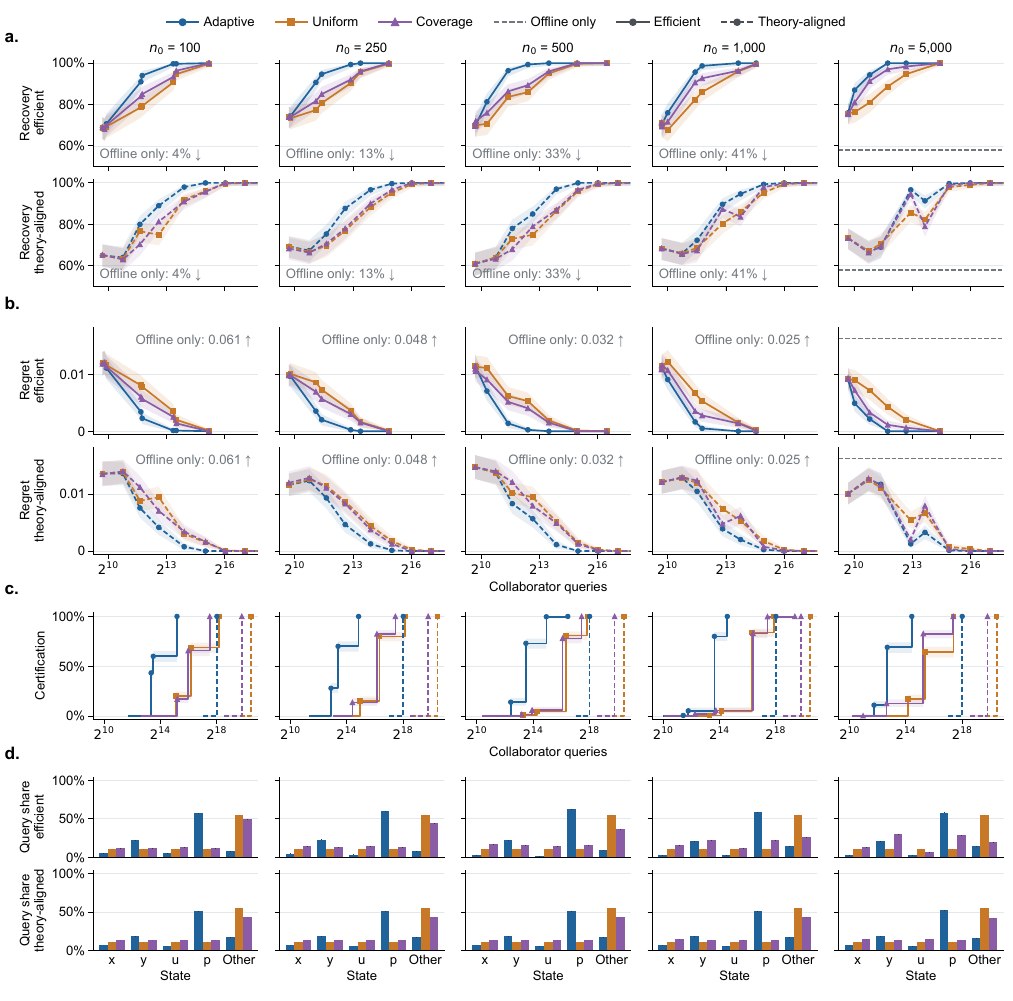}
\caption{Simulation~2 at every offline sample size; layout and conventions as in Figure~\ref{fig:app-sim1}. All methods can query the nine state--time pairs. Query shares (\textbf{d.}) after epochs~6 (efficient) and~8 (theory-aligned).}
\label{fig:app-sim2}
\end{figure}

\clearpage
\subsection{Simulation 3: Decisions with Long-Term Consequences}\label{app:sim3}

\paragraph{Environment.}
At the root, action~$0$ yields a Bernoulli reward with mean $0.02$ and reaches the absorbing good state with probability $0.5+\theta$; action~$1$ yields no reward and reaches it with probability $0.5$. The good and bad states pay $1$ and $0$ at every later step (Figure~\ref{fig:app-environments}c). With $\theta=0.04$, the value gap $0.02+(H-1)\theta$ is $0.18$, $0.38$ and $0.58$ for $H=5$, $10$ and $15$, far above $\varepsilon=0.01$. Only the root carries a choice, so there are two policies. The behavior policy selects the optimal action~$0$ with probability $0.05$, and the collaborator with probability $0.6$. All $2H-1$ state--time pairs ($9$, $19$ and $29$) can be queried; querying an absorbing pair returns its reward and transition.

\paragraph{Fixed budgets.}
Offline-only recovery ranges from $57.3$-$61.3\%$ at $n_0=100$ to $89.7$-$92.7\%$ at $n_0=5{,}000$ across horizons (Figures~\ref{fig:app-sim3-h5}-\ref{fig:app-sim3-h15}). In the efficient version adaptive's largest gains over uniform are $8.3$ to $23.0$ percentage points for $n_0\le1{,}000$ and $0.7$ to $6.3$ at $n_0=5{,}000$ (significant for $H=5$ and~$15$), after $1.2\times10^{3}$ to $8.3\times10^{3}$ queries; over coverage they are $3.0$ to $16.7$ points for $n_0\le1{,}000$ and not significant at $n_0=5{,}000$. In the theory-aligned version they are $14.7$ to $29.0$ points over uniform and $6.3$ to $25.0$ over coverage. Adaptive returns the optimal policy in every trial after $3.8\times10^{3}$ to $2.4\times10^{4}$ queries in the efficient version and $1.4\times10^{4}$ to $2.9\times10^{4}$ in the theory-aligned version. With only two policies, mean regret is the value gap times the error rate.

\paragraph{Certification cost.}
In the efficient version every trial stopped with a single candidate left, most often at epoch~5 ($44$-$76\%$ of trials), after $6.1\times10^{3}$ to $1.3\times10^{4}$ queries on average. Uniform allocation needed $5.9$-$8.8$, $9.7$-$17.1$ and $18.7$-$24.6$ times as many queries for $H=5$, $10$ and $15$, and coverage $3.6$-$6.2$, $6.1$-$11.8$ and $12.3$-$17.0$ times. As in Simulation~1 the cost distributions overlap, mainly because a few adaptive trials need an extra epoch: when the last adaptive trial certified, up to $34\%$ of uniform and $35\%$ of coverage trials had also certified for $n_0\le1{,}000$. At $n_0=5{,}000$, where the offline data already cover the root well and coverage allocation sends $20$-$41\%$ of its queries there, $81$-$88\%$ of coverage trials had. In the theory-aligned version every trial stopped with a single candidate left at epoch~7 or~8, after $3.3\times10^{5}$ to $3.4\times10^{5}$ ($H=5$), $5.2\times10^{5}$ to $5.5\times10^{5}$ ($H=10$) and $4.9\times10^{5}$ to $5.1\times10^{5}$ ($H=15$) queries on average. The share stopping at epoch~7 grows with the horizon, from $15$-$19\%$ to $31$-$37\%$ and $43$-$46\%$, because localization removes the suboptimal policy at epoch~7 only when its estimated gap exceeds $5\varepsilon_7/2$, which is $0.195$, $0.391$ and $0.586$, closer and closer to the gap. Uniform allocation needed $7.5$-$7.9$, $15.1$-$15.6$ and $20.2$-$21.2$ times as many queries and coverage $4.9$-$7.6$, $7.9$-$11.3$ and $10.3$-$15.8$ times, and no baseline trial certified before the last adaptive trial.

\paragraph{Allocation.}
In the efficient version, adaptive allocation had spent $47$-$74\%$, $38$-$64\%$ and $31$-$53\%$ of its queries on the root after epoch~4 for $H=5$, $10$ and $15$ ($1.2\times10^{3}$ to $5.0\times10^{3}$ queries). Uniform spends about $1/(2H-1)$ of its queries there ($11.1\%$, $5.3\%$ and $3.4\%$), and coverage $16$-$41\%$, $8$-$29\%$ and $5.5$-$20\%$. From epoch~2 on, $94$-$100\%$ of adaptive's design queries went to the root. The rest of its queries are mostly the epoch-1 stabilization batch at all $2H-1$ pairs ($61$ to $67$ queries per pair), which the top-up raises by at most seven queries per pair and epoch; their share grows with $H$ and with $n_0$, as the design batches shrink when the offline data are informative. In the theory-aligned version adaptive had spent $92\%$, $86\%$ and $80\%$ of its queries on the root after epoch~7, against $16$-$17\%$, $8\%$ and $5$-$6\%$ for coverage. The baselines' disadvantage therefore includes the queries they spend on absorbing pairs, where no choice remains. The number of such pairs grows with $H$ together with the value gap, so the horizon sweep does not isolate a single factor.

\begin{figure}[p]
\centering
\includegraphics[width=0.95\linewidth]{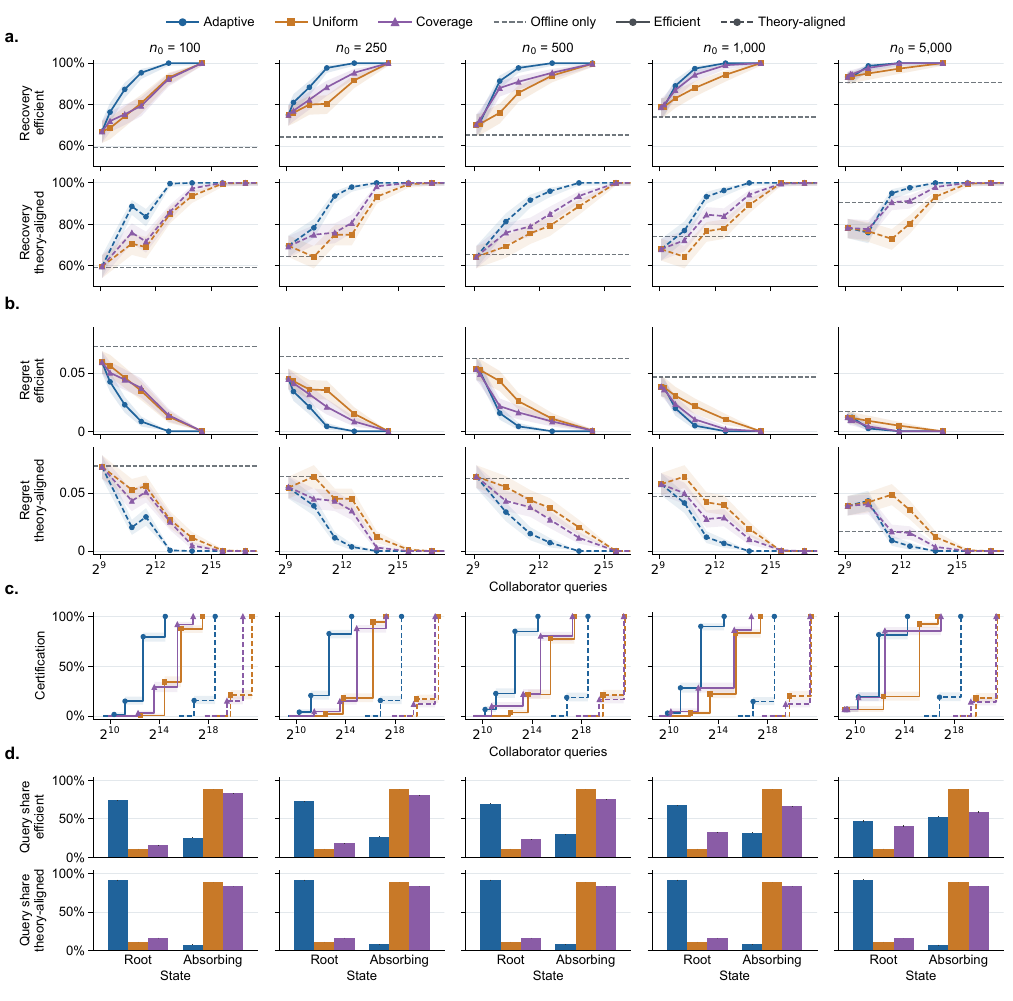}
\caption{Simulation~3 with $H=5$ at every offline sample size; layout and conventions as in Figure~\ref{fig:app-sim1}. All methods can query the nine state--time pairs. Query shares (\textbf{d.}) after epochs~4 (efficient) and~7 (theory-aligned).}
\label{fig:app-sim3-h5}
\end{figure}

\begin{figure}[p]
\centering
\includegraphics[width=0.95\linewidth]{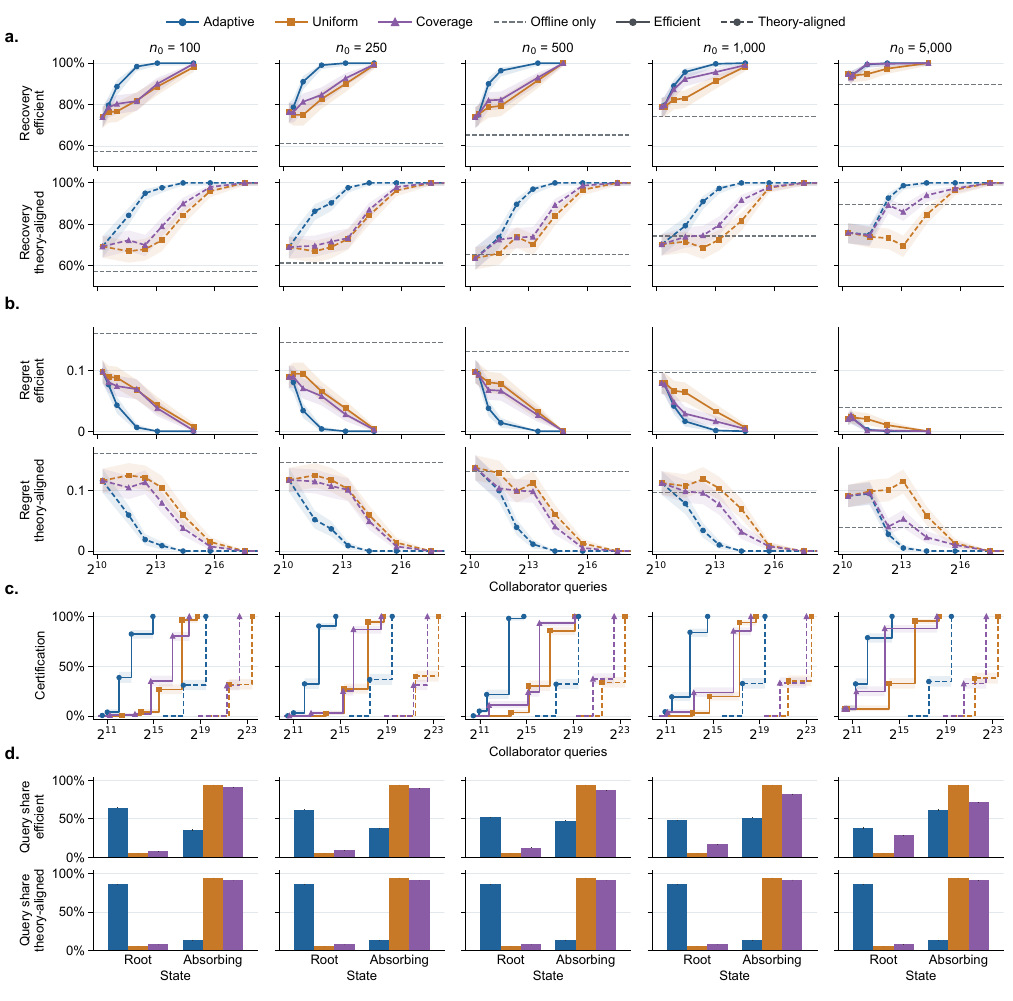}
\caption{Simulation~3 with $H=10$ at every offline sample size; layout and conventions as in Figure~\ref{fig:app-sim1}. All methods can query the 19 state--time pairs. Query shares (\textbf{d.}) after epochs~4 (efficient) and~7 (theory-aligned).}
\label{fig:app-sim3-h10}
\end{figure}

\begin{figure}[p]
\centering
\includegraphics[width=\linewidth]{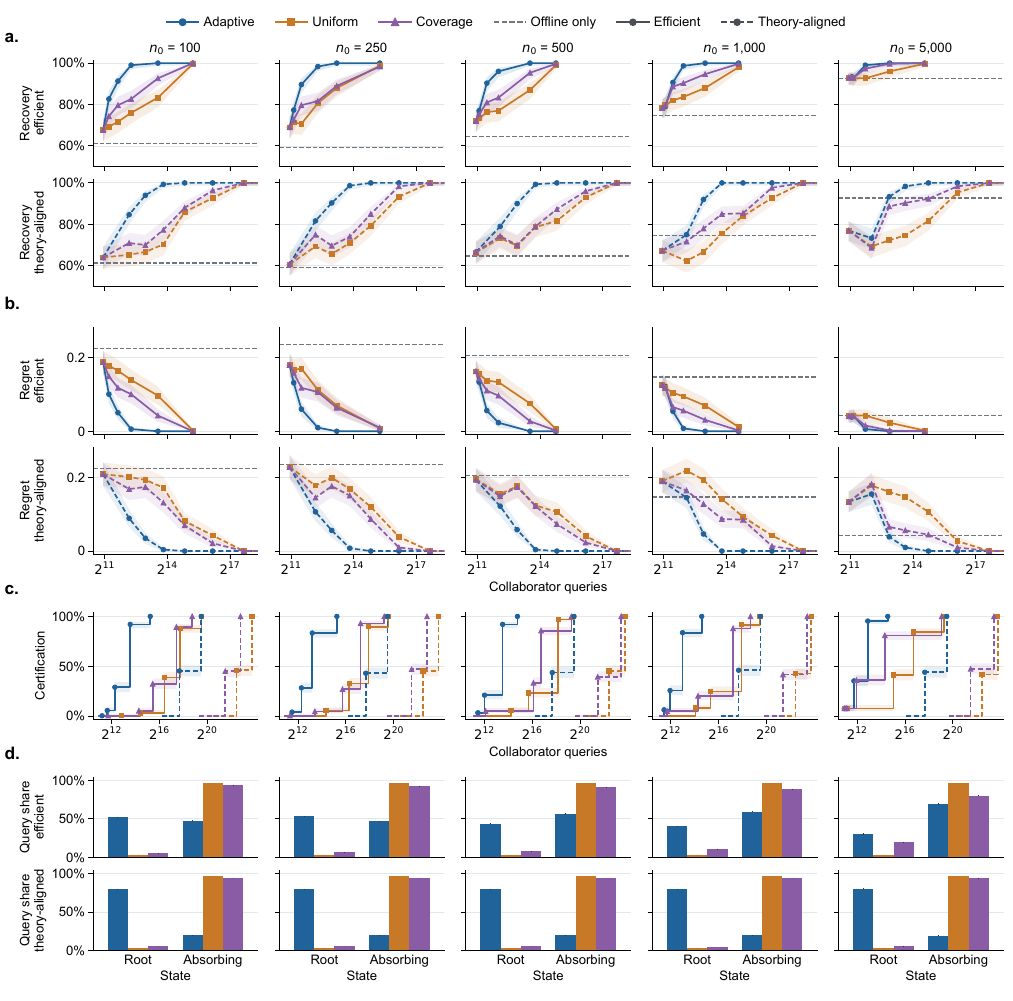}
\caption{Simulation~3 with $H=15$ at every offline sample size; layout and conventions as in Figure~\ref{fig:app-sim1}. All methods can query the 29 state--time pairs. Query shares (\textbf{d.}) after epochs~4 (efficient) and~7 (theory-aligned).}
\label{fig:app-sim3-h15}
\end{figure}

\clearpage
\subsection{Results with \texorpdfstring{$C_0=1$}{C0 = 1}}\label{app:sim-c01}

Figures~\ref{fig:app-c01-stopping}-\ref{fig:app-c01-sim3-h15} repeat every setting with $C_0=1$, which doubles the confidence radius; everything else, including the seed bases, is as above. All methods share the stopping rule, so the larger radius raises the certification cost of every method. Adaptive allocation needs $2.2$ to $3.7$ times as many queries as with $C_0=0.5$ in the efficient version ($1.9\times10^{4}$ to $4.9\times10^{4}$ on average) and $3.2$ to $4.6$ times as many in the theory-aligned version ($8.5\times10^{5}$ to $2.4\times10^{6}$). It remains the cheapest in every setting: uniform and coverage allocation need $4.6$-$28.3$ and $4.7$-$22.8$ times as many queries in the efficient version and $4.1$-$18.0$ and $4.0$-$21.0$ times in the theory-aligned version, where the last adaptive trial again certified before the first baseline trial in every setting. In the efficient version the cost distributions overlap less than with $C_0=0.5$ in Simulation~2, where at most $3\%$ of baseline trials had certified when the last adaptive trial did, and in Simulation~3 with $H=15$, where no uniform trial had.

At fixed budgets, the efficient version's adaptive allocation recovers the optimal policy more often than uniform allocation in 23 of the 25 settings and more often than coverage in 21, at the epoch end of largest gain and the 95\% level; all exceptions are at $n_0=5{,}000$. For $n_0\le1{,}000$ its largest gains are $8.0$ to $21.7$ percentage points over uniform and $3.7$ to $20.0$ over coverage. The theory-aligned version's largest gains are significant in all 25 settings against uniform and in 24 against coverage, the exception being Simulation~3 with $H=5$ at $n_0=5{,}000$. As with $C_0=0.5$, no paired difference in recovery at any epoch end is significantly below zero, and no stopping run of any method returned a policy that is not $\varepsilon$-optimal.

\begin{figure}[H]
\centering
\includegraphics[
        width=\linewidth,
        height=0.57\textheight,
        keepaspectratio
    ]{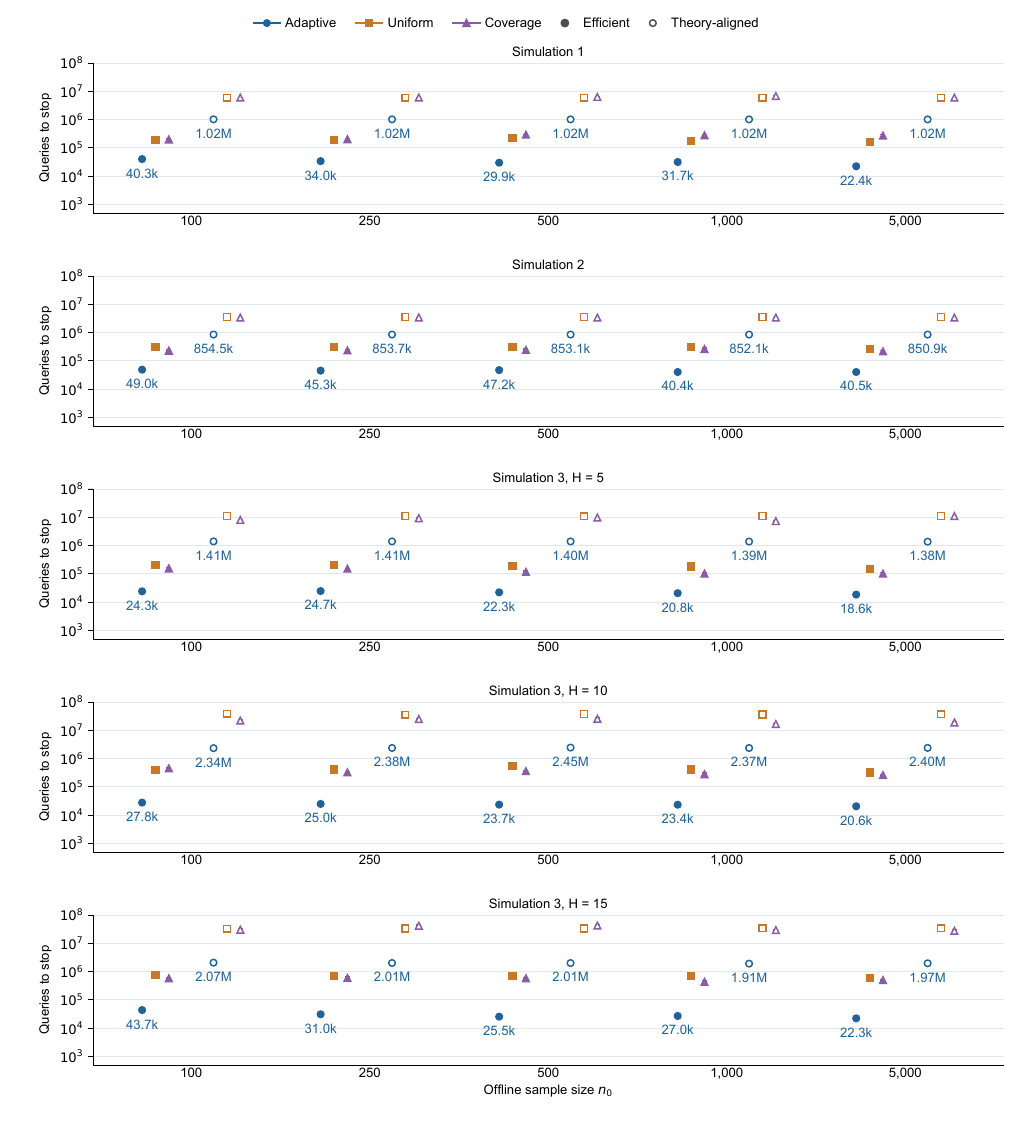}
\caption{Mean collaborator queries until each method's stopping rule returns with $C_0=1$; layout and conventions as in Figure~\ref{fig:app-stopping}. No stopping run returned a policy that is not $\varepsilon$-optimal.}
\label{fig:app-c01-stopping}
\end{figure}

\begin{figure}[H]
\centering
\includegraphics[width=0.95\linewidth]{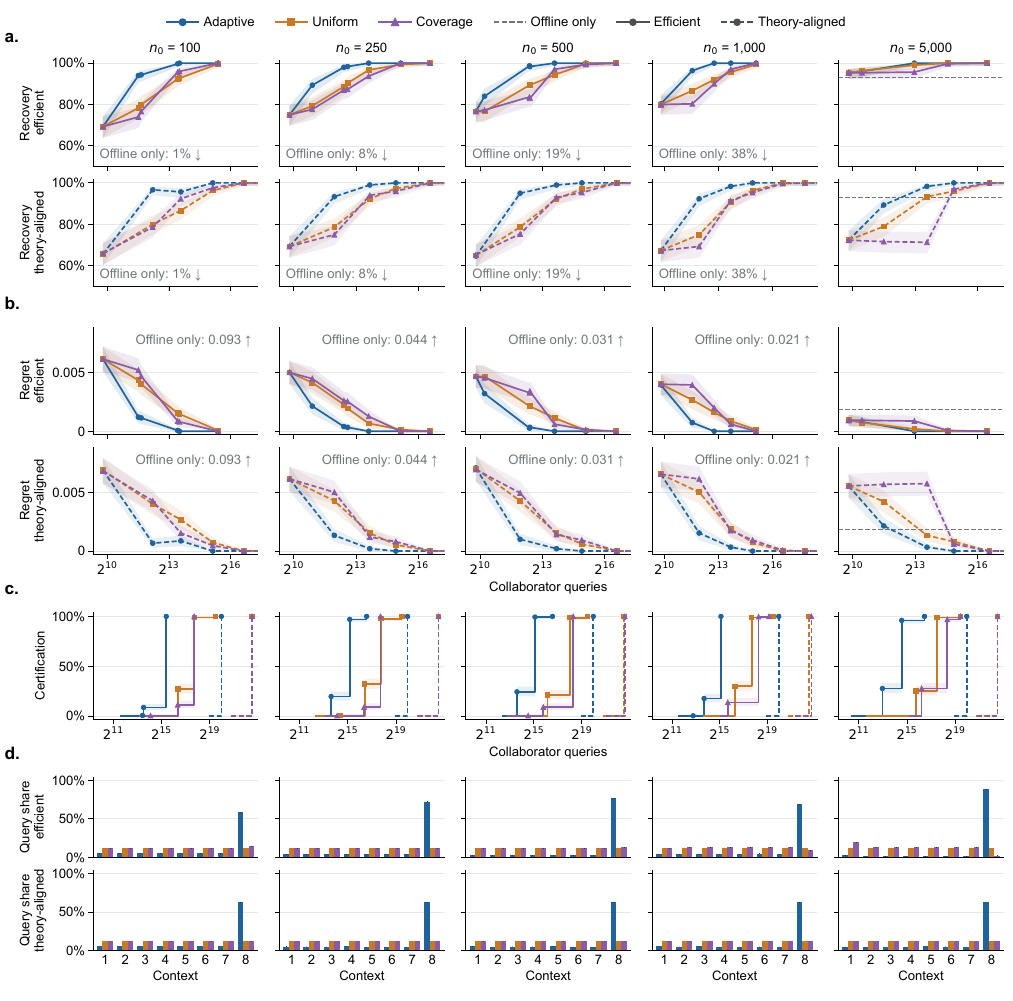}
\caption{Simulation~1 with $C_0=1$ at every offline sample size; layout and conventions as in Figure~\ref{fig:app-sim1}. Query shares (\textbf{d.}) after epochs~5 (efficient) and~7 (theory-aligned).}
\label{fig:app-c01-sim1}
\end{figure}

\begin{figure}[p]
\centering
\includegraphics[width=0.95\linewidth]{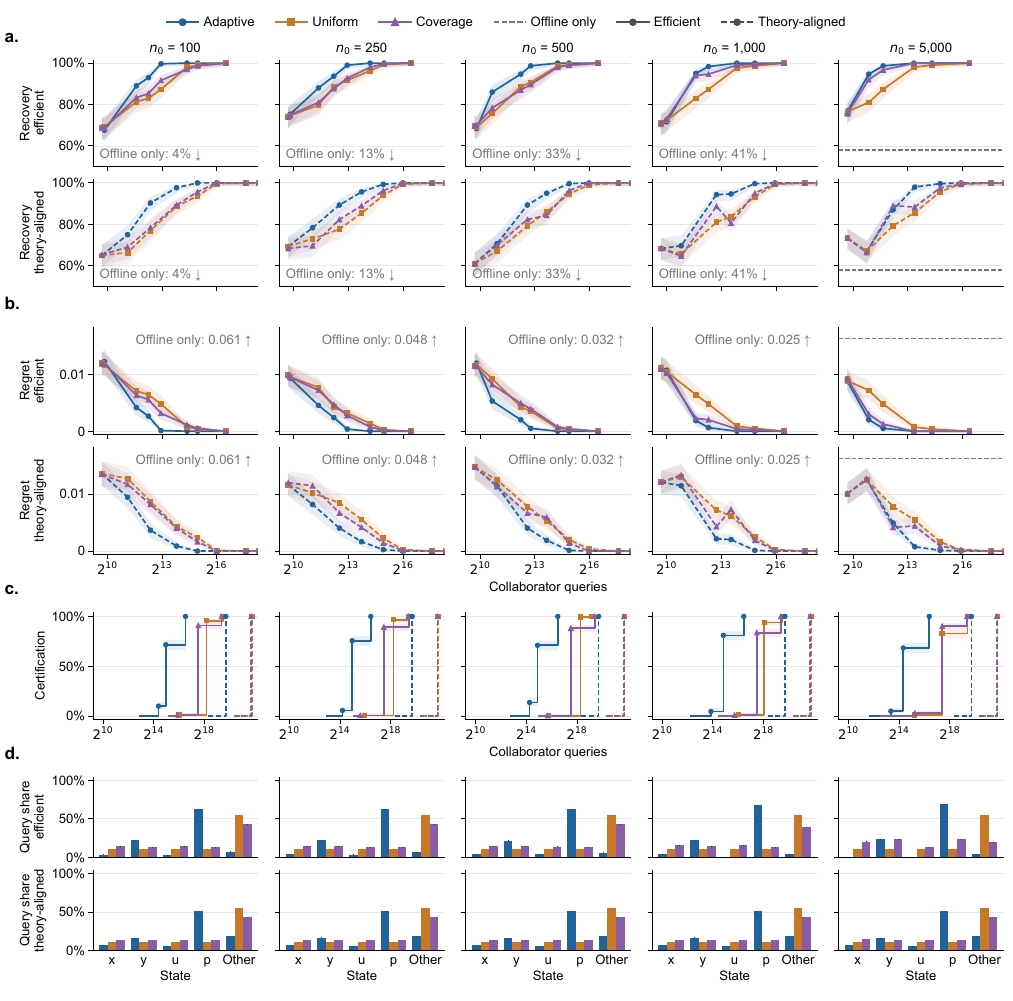}
\caption{Simulation~2 with $C_0=1$ at every offline sample size; layout and conventions as in Figure~\ref{fig:app-sim1}. Query shares (\textbf{d.}) after epochs~6 (efficient) and~8 (theory-aligned).}
\label{fig:app-c01-sim2}
\end{figure}

\begin{figure}[p]
\centering
\includegraphics[width=0.95\linewidth]{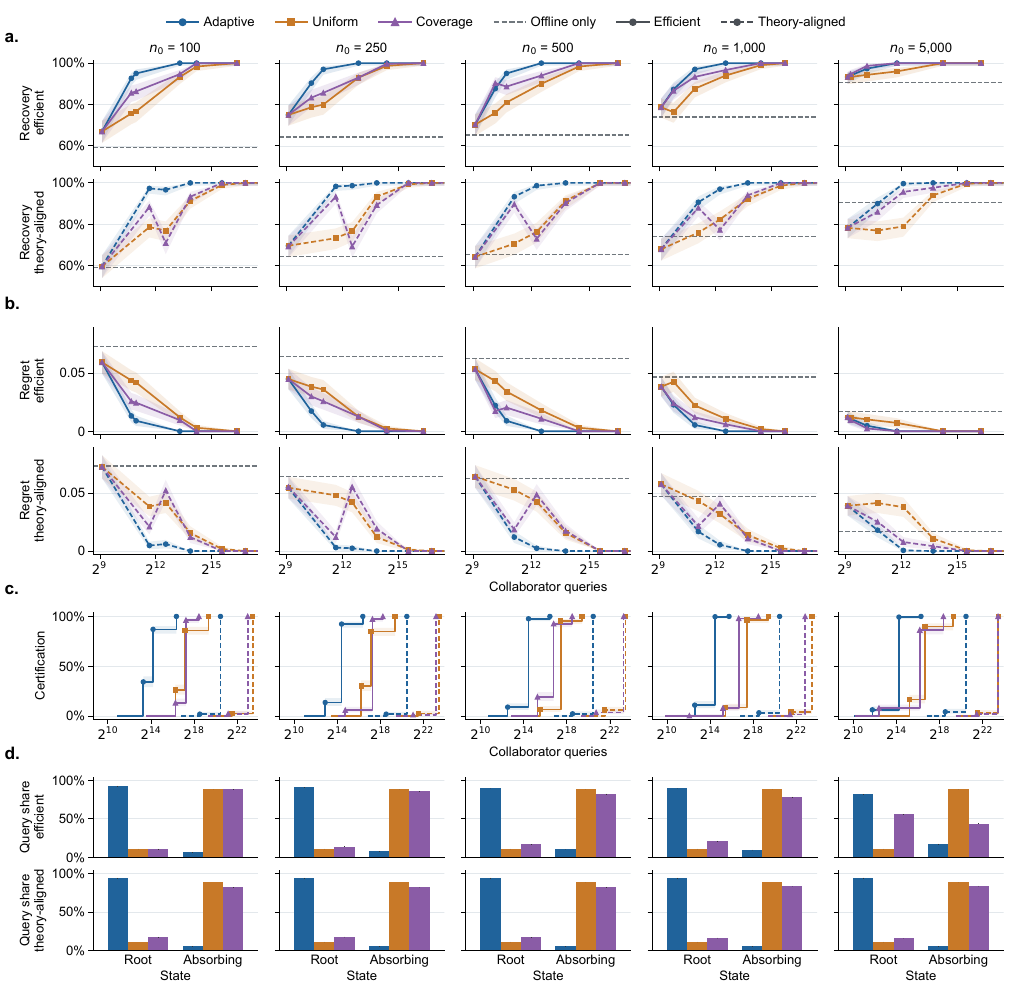}
\caption{Simulation~3 with $H=5$ and $C_0=1$ at every offline sample size; layout and conventions as in Figure~\ref{fig:app-sim1}. Query shares (\textbf{d.}) after epochs~4 (efficient) and~7 (theory-aligned).}
\label{fig:app-c01-sim3-h5}
\end{figure}

\begin{figure}[p]
\centering
\includegraphics[width=0.95\linewidth]{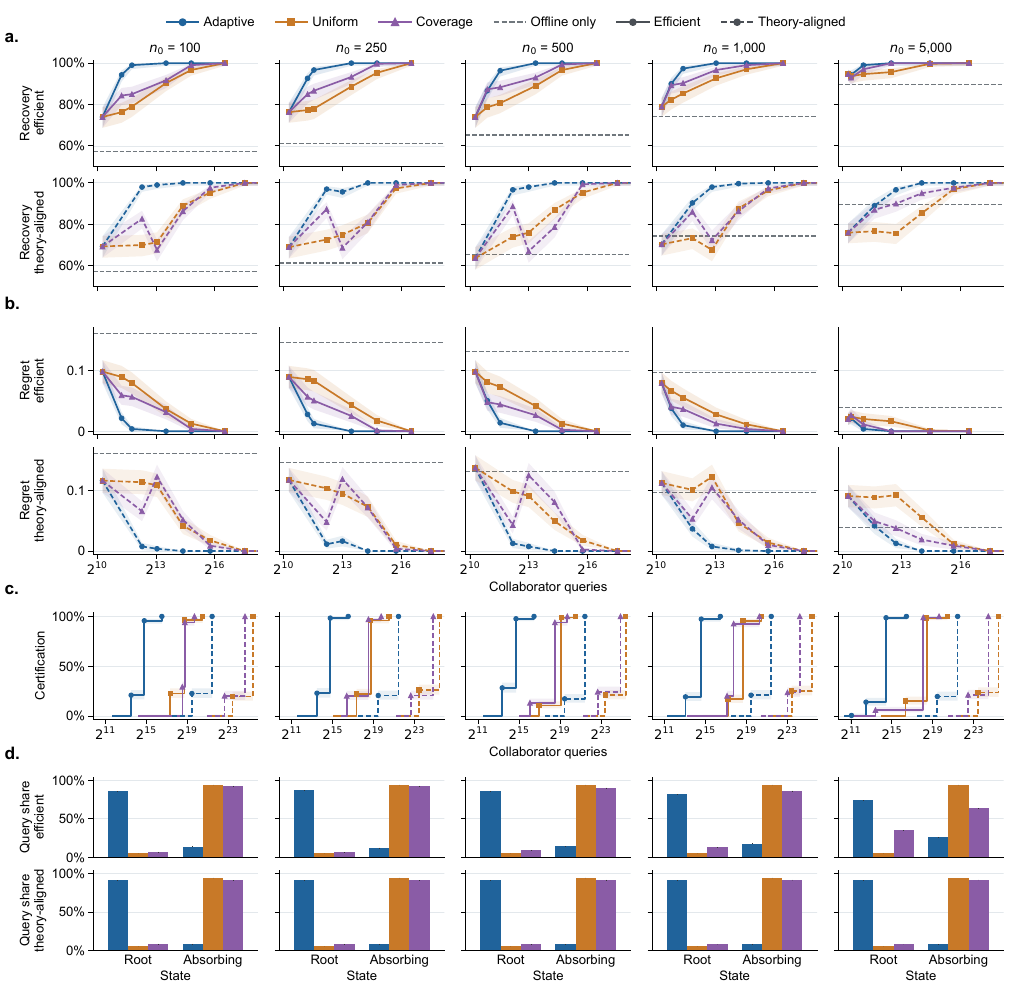}
\caption{Simulation~3 with $H=10$ and $C_0=1$ at every offline sample size; layout and conventions as in Figure~\ref{fig:app-sim1}. Query shares (\textbf{d.}) after epochs~4 (efficient) and~7 (theory-aligned).}
\label{fig:app-c01-sim3-h10}
\end{figure}

\begin{figure}[p]
\centering
\includegraphics[width=0.95\linewidth]{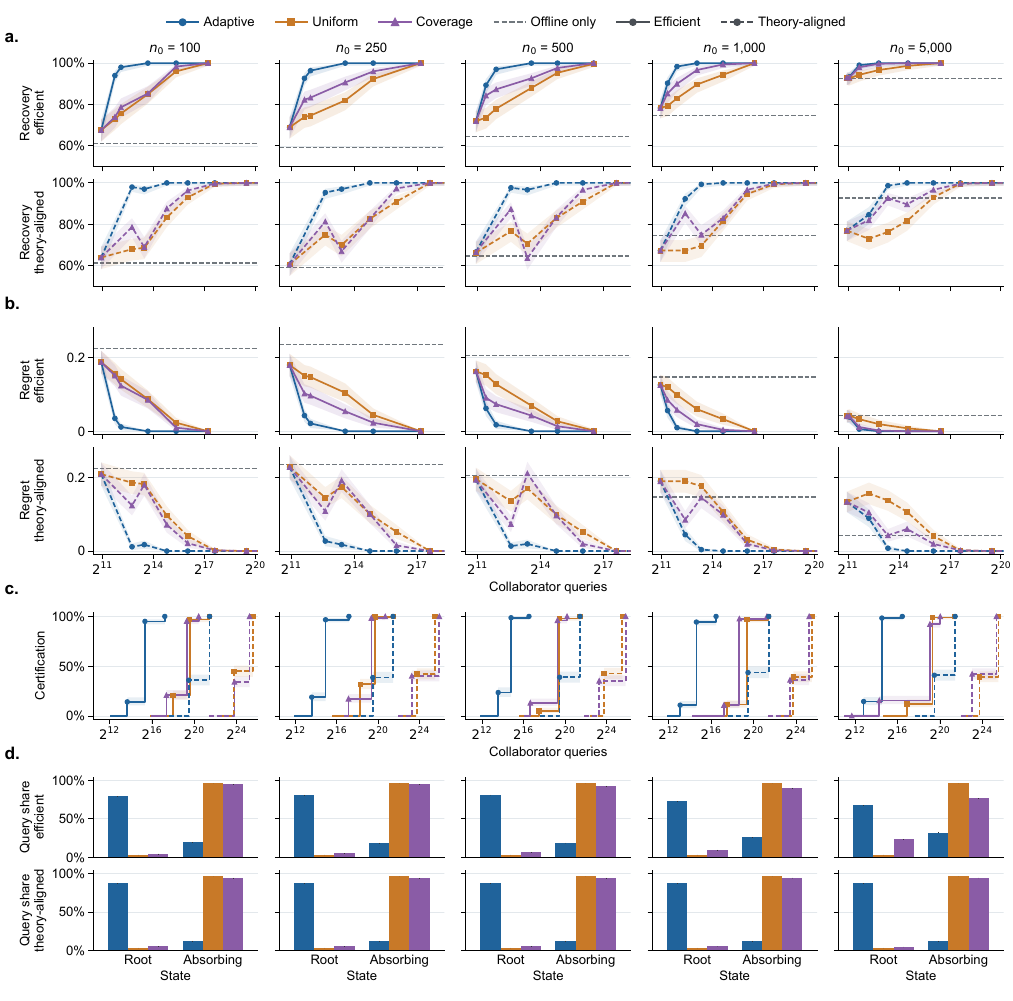}
\caption{Simulation~3 with $H=15$ and $C_0=1$ at every offline sample size; layout and conventions as in Figure~\ref{fig:app-sim1}. Query shares (\textbf{d.}) after epochs~4 (efficient) and~7 (theory-aligned).}
\label{fig:app-c01-sim3-h15}
\end{figure}

\clearpage

\section{PRISM Application Details}\label{app:prism}

\paragraph{Data.} We use the primary screen of the PRISM drug-repurposing resource \citep{corsello2020discovering} (19Q4 release), which reports for each cell line and compound the $\log_2$ fold change in viability relative to the DMSO control, ComBat-corrected and median-collapsed across replicates. Of the 578 cell lines in the screen, 10 fail STR identity profiling and 27 lack a response to at least one compound of a six-drug panel at $2.5\,\mu$M; of the remaining 541, we keep every lineage with at least 30 lines, which gives the six lineages of \Cref{tab:prism-benchmark} (284 lines). Within each lineage, lines are ordered by a hash of their DepMap identifier; the first $60\%$ (rounded down) form a development set of 168 lines, used only to choose the drug pair, and the remaining 116 lines form the benchmark used in every experiment. The reward of a line under a drug is 1 if its $\log_2$ fold change is at most $-1$ (viability at most half of control) and 0 otherwise, and the response rate $r_1(s,a)$ of drug $a$ in lineage $s$ is the fraction of benchmark lines in $s$ with reward 1.

\paragraph{Decision problem.} Contexts are the six lineages ($H=1$) and actions are dasatinib and selumetinib, so there are $2^6=64$ deterministic policies. A policy's value is $V^\pi_1=\sum_s\rho_s\,r_1(s,\pi_1(s))$ for lineage weights $\rho$, its regret is $V^\star_1-V^\pi_1$, and $\varepsilon=0.01$ is one percentage point of response rate. We use incidence weights, each lineage's share of new cases of the six cancers worldwide in 2022 \citep{bray2024global} (melanoma for skin), and equal weights $\rho_s=1/6$. A wrong choice in lineage $s$ costs $\rho_s\,|r_1(s,\mathrm{D})-r_1(s,\mathrm{S})|$. Dasatinib is better in lung, CNS, ovary and pancreas, and selumetinib in skin and colorectal by narrow margins. Which close calls cost more than $\varepsilon$ depends on the weights (\Cref{tab:prism-benchmark}): colorectal under incidence weights and both under equal weights, with pancreas above $\varepsilon$ under both.

\begin{table}[ht]
\centering
\caption{The PRISM benchmark. Responders are benchmark lines with $\log_2$ fold change at most $-1$ at $2.5\,\mu$M; D: dasatinib, S: selumetinib. Weights are in percent. The cost of a wrong choice is the lineage's weight times its gap,
reported in percentage points of response rate
($\epsilon=0.01$ corresponds to one percentage point).}
\label{tab:prism-benchmark}
\small
\setlength{\tabcolsep}{4pt}
\begin{tabular}{lrrrrcrrrrr}
\toprule
 & \multicolumn{2}{c}{Lines} & \multicolumn{2}{c}{Responders} & & & \multicolumn{2}{c}{Weight (\%)} & \multicolumn{2}{c}{Cost of a wrong choice} \\
\cmidrule(lr){2-3}\cmidrule(lr){4-5}\cmidrule(lr){8-9}\cmidrule(lr){10-11}
Lineage & Eligible & Benchmark & D & S & Better & Gap & Incidence & Equal & Incidence & Equal \\
\midrule
Lung       & 98 & 40 & 29 & 7  & D & 0.550 & 42.1 & 16.7 & 23.1 & 9.2  \\
Skin       & 43 & 18 & 10 & 12 & S & 0.111 & 5.6  & 16.7 & 0.6  & 1.9  \\
CNS        & 40 & 16 & 8  & 0  & D & 0.500 & 5.5  & 16.7 & 2.7  & 8.3  \\
Ovary      & 35 & 14 & 14 & 2  & D & 0.857 & 5.5  & 16.7 & 4.7  & 14.3 \\
Colorectal & 34 & 14 & 7  & 8  & S & 0.071 & 32.7 & 16.7 & 2.3  & 1.2  \\
Pancreas   & 34 & 14 & 8  & 5  & D & 0.214 & 8.7  & 16.7 & 1.9  & 3.6  \\
\midrule
Total      & 284 & 116 & & & & & 100 & 100 & & \\
\bottomrule
\end{tabular}
\end{table}

\paragraph{Collaborator, queries and initial log.} A query names a lineage, then the collaborator picks each drug with probability $1/2$ and returns the stored reward of a benchmark line drawn uniformly at random, with replacement, from that lineage. Each observation is therefore Bernoulli$(r_1(s,a))$, and repeated queries resample measured responses rather than run new assays. The initial log has as many records in each lineage as the lineage has benchmark lines (116 in total, drawn with replacement) and reveals a single drug: dasatinib in half of the trials and selumetinib in the other half. It cannot compare the drugs, so the offline-only policy, the empirical maximizer on the log, plays the logged drug in every lineage and is never $\varepsilon$-optimal (regret about $17$ percentage points).

\paragraph{Methods and constants.} Adaptive, uniform and coverage allocation are as in the simulations. All methods share the initial log and the stabilization queries and differ only in their design batches. For the recovery comparisons every method receives adaptive's batch size at each epoch; for certification each method uses its own batch sizes and stopping rule. Both versions use the constants of the simulations (Appendix~\ref{app:simulations}): $C_0=0.5$ (and $C_0=1$ in \Cref{fig:prism-appendix}) and we directly set the stabilization target to $\lceil4\beta_k\rceil$ queries per lineage. Each configuration has 400 trials, 200 per logged drug; within a weighting, both versions and both values of $C_0$ use the same seeds. A configuration takes 7-40 seconds for the efficient version and 1-2 minutes for the theory-aligned version with eight workers.

\begin{figure}[ht!]
    \centering
    \includegraphics[width=\linewidth]{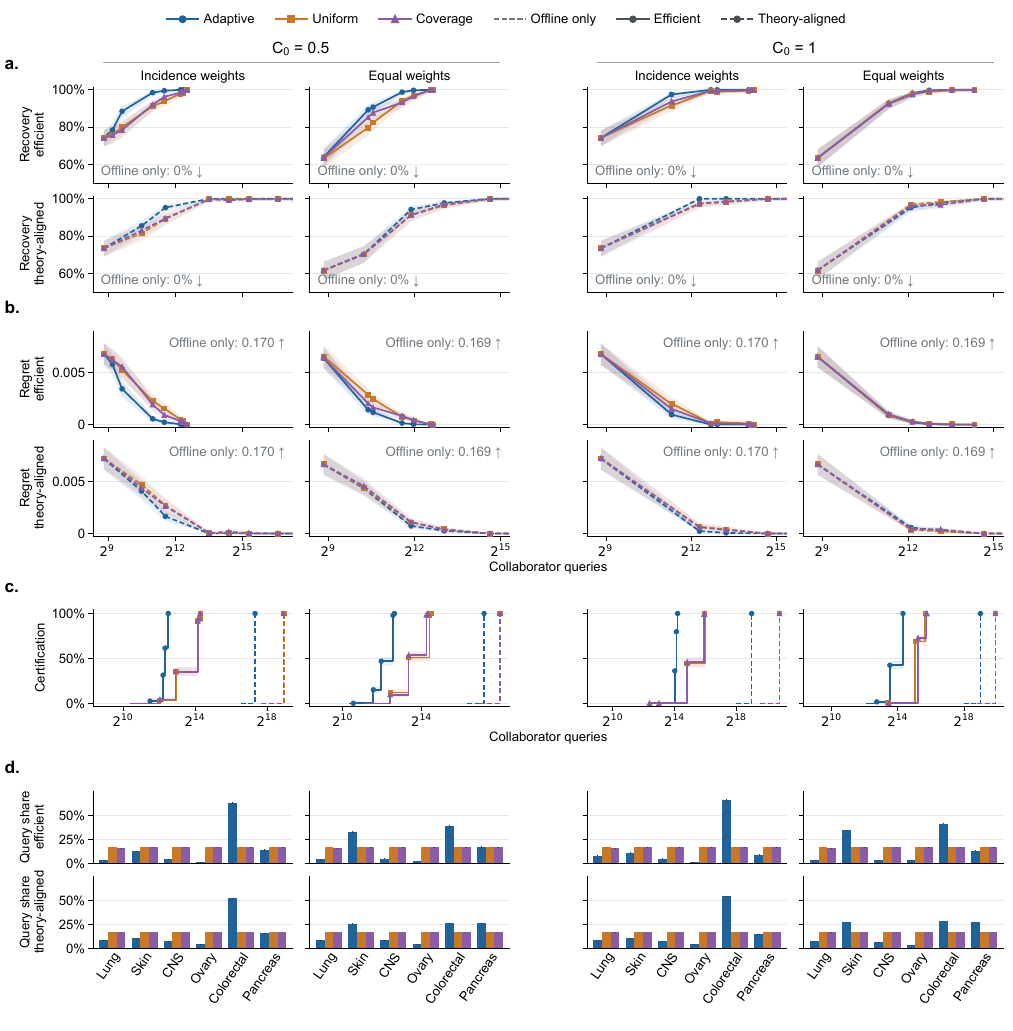}
    \caption{PRISM application at $C_0=0.5$ (left pair of columns) and $C_0=1$ (right pair), each under incidence and equal lineage weights (400 paired trials per column), for the efficient implementation (solid) and the theory-aligned algorithm of \Cref{sec:algorithm} (dashed). \textbf{a.}~Fraction of trials returning an $\varepsilon$-optimal policy and \textbf{b.}~mean regret of the returned policy, in response-rate units, at each epoch end, against mean collaborator queries (top: efficient; bottom: theory-aligned; the four rows of a column share the $x$-axis); offline-only levels outside the axis are printed in the panel. \textbf{c.}~Fraction of trials whose stopping rule has returned, each method under its own allocation. \textbf{d.}~Each method's share of the efficient (top) and theory-aligned (bottom) version's queries after that version's last epoch before adaptive typically stops. Bands: pointwise 95\% Wilson intervals (recovery, certification) and mean $\pm1.96$ standard errors (regret).}
    \label{fig:prism-appendix}
\end{figure}

\paragraph{Results.} No stopping run of any method, version, weighting or $C_0$ returned a policy that is not $\varepsilon$-optimal (9{,}600 runs). With $C_0=0.5$, the efficient version certifies after $5.8$ and $6.3\times10^{3}$ queries under incidence and equal weights, $3.2$-$3.6$ times fewer than uniform and coverage allocation. It also leads every baseline in recovery well before stopping: after $789$ queries under incidence weights it returns an $\varepsilon$-optimal policy in $88.5\%$ of trials, against $80.0\%$ and $78.5\%$ for uniform and coverage allocation, and after $3{,}010$ queries under equal weights in $98.8\%$, against $94.0\%$ and $93.2\%$, every difference being significant (paired 95\% intervals). It concentrates its queries on the difficult contexts: $63\%$ go to colorectal under incidence weights and $72\%$ to skin and colorectal under equal weights, while lung, an easy decision despite its large incidence weight, receives about $4\%$ (\Cref{fig:prism-appendix}d). The theory-aligned version stops at epoch~8 in every trial, the first epoch with $\varepsilon_k\le\varepsilon$, after $1.5$-$1.7\times10^{5}$ queries, and adaptive allocation needs $1.7$-$3.1$ times fewer queries than uniform and coverage allocation. Its recovery gain is significant only under incidence weights ($95.2\%$ against $89.2\%$ and $89.5\%$ after $3{,}007$ queries); under equal weights the methods are within noise. Its first epoch consists only of the shared stabilization queries, and these still make up about half or more of its second, so all methods see largely the same data until the design batches grow. With $C_0=1$, every cost rises about threefold, to about $2\times10^{4}$ queries for the efficient version and $5.2\times10^{5}$ for the theory-aligned version. Adaptive allocation keeps its certification advantage ($2.5$-$3.4$ times fewer queries than uniform and coverage for the efficient version and $1.9$-$3.5$ times for the theory-aligned version), while its recovery lead appears only under incidence weights.
\section{Code Availability \& Computing}

    All simulations and experiments were run on CPU only (no GPU), on a single MacBook Pro laptop with an Apple M5 chip (10 CPU cores) and 24~GB of memory. The code is written in Python~3.13.5 with NumPy~2.5.3 and SciPy~1.18.1, whose SLSQP solver solves the allocation design problems, and the independent trials of each setting run in parallel on eight or nine worker processes. For each value of $C_0$, Simulation~1 took under 7 minutes for the efficient version and under 1.5 hours for the theory-aligned version, Simulations~2 and~3 under 3 minutes each, and each PRISM configuration at most about two minutes; in total, all runs behind the reported results took under 3 hours of wall-clock time. The link to the repository containing full reproducible experiments and manuscript figures can be found at: 
    \begin{itemize}
        \item \url{https://github.com/hong-niu/ask-to-learn/blob/main/README.md}
        \end{itemize}


\newpage

\end{document}